\documentclass[twoside,11pt]{article} 
\usepackage{jmlr2e} 

\usepackage{enumitem}% http://ctan.org/pkg/enumitem
\usepackage{color}

\usepackage{algorithm}
\usepackage{algorithmic}
\usepackage{amsmath}
\usepackage{multirow}
\usepackage{bm}
\usepackage{mdframed}
\usepackage{hyperref}
\usepackage{url}

\newcommand{\reg}{\psi}

\newcommand{\mtd}{\mathcal{M}}
\newcommand{\Otilde}{\tilde{O}}
 \DeclareMathOperator*{\argmin}{arg\,min}

 \newcommand{\Acal}{\mathcal{A}}

\newcommand\peps{p^{\:\! \varepsilon}}
\newcommand\pepsk{p^{\:\! \varepsilon_k}}

\newcommand\gdelta{g^{\:\! \delta}}
\newcommand\gdeltak{g^{\:\! \delta_k}}
\newcommand\gdeltakp{g^{\:\! \delta_{k+1}}}
\newcommand\kmone{k - 1}

\newcommand{\R}{\mathbb{R}}

\newcommand{\E}{\mathbb{E}}
\newcommand{\p}{\mathbb{P}}

\renewcommand\leqslant\leq
\renewcommand\geqslant\geq
\renewcommand\epsilon\varepsilon
\renewcommand\ln\log
\newcommand\st{~~~{\text{s.t.}}~~~}
\renewcommand\star{*}
\def\defin{\triangleq}
\def\Real{{\mathbb{R}}}

\newcommand{\prox}{\mathrm{prox}}

\newcommand\custombox[1]{%
   \vspace*{0cm}
   \begin{mdframed}[leftmargin=0pt,innerleftmargin=6pt,innerrightmargin=6pt]
      \parindent=15pt#1
   \end{mdframed}
   \vspace*{0cm}
}

\newcommand\Ctrois{($\textsf{C3}$)~}
\newcommand\Cunstar{($\textsf{C1}^\star$)}

\usepackage{lastpage}
\jmlrheading{18}{2018}{1-54}{12/17; Revised 12/17}{4/18}{17-748}{Hongzhou Lin, Julien Mairal and Zaid Harchaoui}
\ShortHeadings{Catalyst for First-order Convex Optimization}{Lin, Mairal and Harchaoui}
\firstpageno{1}

\begin{document}
\title{Catalyst Acceleration for First-order Convex Optimization: \\ from Theory to Practice}

\author{\name Hongzhou Lin  \email hongzhou@mit.edu \\
   \addr  Massachusetts Institute of Technology \\ 
   Computer Science and Artificial Intelligence Laboratory \\
   Cambridge, MA 02139, USA
   \AND
   \name Julien Mairal \email  julien.mairal@inria.fr \\
   \addr Univ. Grenoble Alpes, Inria, CNRS, Grenoble INP$\hspace*{0.03cm}^*$\hspace*{-0.10cm}, LJK, \\
   Grenoble, 38000, France
   \AND
   \name Zaid Harchaoui  \email zaid@uw.edu \\ %\texttt{zaid@uw.edu}
   \addr University of Washington \\
   Department of Statistics \\
   Seattle, WA 98195, USA
}
\editor{L\'eon Bottou}
\maketitle

% REQUIRED
\begin{abstract}%
We introduce a generic scheme for accelerating gradient-based optimization
methods in the sense of Nesterov. The approach, called Catalyst, builds upon the inexact
accelerated proximal point algorithm for minimizing a convex objective function, and
consists of approximately solving a sequence of well-chosen auxiliary problems,
leading to faster convergence.  
One of the keys to achieve acceleration in theory and in practice is to
solve these sub-problems with appropriate accuracy by using the right stopping
criterion and the right warm-start strategy. We give practical guidelines to use Catalyst and present a comprehensive analysis of its global complexity. 
We show that Catalyst applies to a large class of algorithms,
including gradient descent, block coordinate descent, incremental algorithms
such as SAG, SAGA, SDCA, SVRG, MISO/Finito, and their proximal variants. For
all of these methods, we establish faster rates using the Catalyst acceleration, for
strongly convex and 
non-strongly convex objectives.  We conclude with extensive experiments showing
that acceleration is useful in practice, especially for ill-conditioned
problems.
\end{abstract}


\def\thefootnote{\fnsymbol{footnote}}
\footnotetext[1]{Institute of Engineering Univ. Grenoble Alpes}
\def\thefootnote{\arabic{footnote}}

% REQUIRED

\section{Introduction}
A large number of machine learning and signal processing problems are
formulated as the minimization of a convex objective
function:
%~$f: \Real^p \to \Real$:
\begin{equation}
   \min_{x \in \Real^p} \left\{ f(x) \defin f_0(x) + \reg(x) \right\}, \label{eq:obj}
\end{equation}
where $f_0$ is convex and $L$-smooth,
and~$\reg$ is convex but may not be differentiable. We call a function $L$-smooth when it is differentiable and its gradient is $L$-Lipschitz continuous.

In statistics or machine learning, the variable $x$ may represent model parameters,
and the role of~$f_0$ is to ensure that the estimated parameters fit some observed data.
Specifically, $f_0$ is often a large sum of functions and~(\ref{eq:obj}) is
a regularized empirical risk which writes as
\begin{equation}\label{eq: intro problem}
\min_{x \in \R^p} \left \{ f(x) \defin \frac{1}{n} \sum_{i=1}^n f_i(x) + \psi(x) \right \}.
\end{equation}
Each term $f_i(x)$ measures the fit between~$x$ and a data point indexed
by~$i$, whereas the function~$\reg$ acts as a regularizer; it is typically chosen to be
the squared~$\ell_2$-norm, which is smooth, or to be a non-differentiable
penalty such as the~$\ell_1$-norm or another sparsity-inducing
norm~\citep{bach2012optimization}. 

We present a unified framework allowing one to 
accelerate gradient-based or first-order methods, with a particular focus on 
problems involving large sums of functions.
By ``accelerating'', we mean generalizing a
mechanism invented by~\citet{nesterov1983} that improves the
convergence rate of the gradient descent algorithm.
When
$\reg=0$, gradient descent steps produce iterates~$(x_k)_{k \geq 0}$ such that
$f(x_k) - f^\star \leq \varepsilon$ in $O(1/\varepsilon)$ iterations, where~$f^\star$ denotes the minimum value of~$f$.
Furthermore, when the objective~$f$ is $\mu$-strongly convex, the previous iteration-complexity becomes $O(\left(L/\mu)\log(1/\varepsilon)\right)$,
which is proportional to the condition number $L/\mu$.
However, these rates were shown to be suboptimal
for the class of first-order methods, and a simple strategy of taking the gradient step at a
well-chosen point different from $x_k$ yields the optimal complexity---$O(1/\sqrt{\varepsilon})$
for the convex case and $O(\sqrt{L/\mu}\log(1/\varepsilon))$ for the
$\mu$-strongly convex one \citep{nesterov1983}.  Later, this acceleration technique was extended to deal
with non-differentiable penalties~$\psi$ for which the proximal operator defined below is easy to compute~\citep{fista,nesterov2013gradient}.
\begin{equation}
   \prox_{\psi}(x) \defin \argmin_{z \in \R^p} \left\{ \psi(z) + \frac{1}{2}\|x-z\|^2 \right\}, \label{eq:prox_psi}
\end{equation}
where $\|.\|$ denotes the Euclidean norm.

For machine learning problems involving a large sum of~$n$ functions, a
recent effort has been devoted to developing fast incremental
algorithms such as SAG \citep{sag}, SAGA \citep{saga}, SDCA \citep{sdca}, SVRG \citep{svrg,proxsvrg}, 
or MISO/Finito \citep{miso,finito}, which can exploit the
particular structure~(\ref{eq: intro problem}). Unlike full gradient
approaches, which require computing and averaging 
$n$ gradients $(1/n) \sum_{i=1}^n\nabla f_i(x)$ at every iteration, incremental
techniques have a cost per-iteration that is independent of~$n$. 
The price to pay is the need to store a moderate amount of information
regarding past iterates, but the benefits may be significant in terms of
computational complexity. 
In order to achieve an $\varepsilon$-accurate solution for a
$\mu$-strongly convex objective, the number of gradient evaluations required by
 the methods mentioned above is bounded by $O\left (  \left( n+\frac{\bar{L}}{\mu} \right)
 \log(\frac{1}{\varepsilon})  \right )$, where $\bar{L}$ is either the maximum
 Lipschitz constant across the gradients $\nabla f_i$, or the average value, depending on the algorithm variant considered.
 Unless there is a big mismatch between $\bar{L}$ and $L$ (global Lipschitz constant for the sum of gradients), 
 incremental approaches significantly outperform the full gradient method, whose complexity in terms of gradient evaluations is bounded by $O\left ( n\frac{L}{\mu} \log(\frac{1}{\varepsilon})  \right )$.
 
Yet, these incremental approaches do not use Nesterov's extrapolation steps and
whether or not they could be accelerated was an important open question when
these methods were introduced. It was indeed only known to be the case for SDCA~\citep{accsdca} for strongly convex objectives. 
Later, other accelerated incremental algorithms were proposed such as Katyusha~\citep{accsvrg}, or the method of~\citet{conjugategradient}. 

We give here a positive answer to this open question. 
By analogy with substances that increase chemical reaction rates, we call our
approach ``Catalyst''.
Given an optimization method $\mtd$ as input, Catalyst outputs an accelerated
version of it, eventually the same algorithm if the method~$\mtd$ is already
optimal. The sole requirement on the method in order to achieve acceleration
is that it should have linear convergence rate for strongly convex problems.
This is the case for full gradient methods \citep{fista,nesterov2013gradient} and
block coordinate descent methods~\citep{nesterov2012,richtarik2014}, which already
have well-known accelerated variants. More importantly, it also applies
to the previou incremental methods, whose complexity 
is then bounded by $\tilde{O}\left (  \left( n+ \sqrt{n{\bar{L}}/{\mu}} \right)
\log(\frac{1}{\varepsilon})  \right )$ after Catalyst acceleration, where
$\tilde{O}$ hides some logarithmic dependencies on the condition number
$\bar{L}/\mu$. This improves upon the non-accelerated variants, when the condition number is larger than $n$.
Besides, acceleration occurs
regardless of the strong convexity of the objective---that is, even if $\mu=0$---which brings us to our second achievement.

Some approaches such as MISO, SDCA, or SVRG are only defined
for strongly convex objectives. A classical trick to apply them to general
convex functions is to add a small regularization term
$\varepsilon \|x\|^2$ in the objective~\citep{sdca}. 
The drawback of this strategy is that it requires choosing in advance the parameter~$\varepsilon$, 
which is related to the target accuracy.  The approach we present here provides 
a \emph{direct support for non-strongly convex objectives}, thus
removing the need of selecting~$\varepsilon$ beforehand. Moreover, we can immediately 
establish a faster rate for the resulting algorithm. 
Finally, some methods such as MISO are numerically unstable when they are applied
to strongly convex objective functions with small strong convexity constant. By defining
better con\-di\-tioned auxiliary subproblems, Catalyst also provides better numerical
stability to these methods.

A short version of this paper has been published at the NIPS
conference in 2015 \citep{catalyst}; in addition to simpler convergence proofs
and more extensive numerical evaluation, we
extend the conference paper with a new Moreau-Yosida smoothing interpretation
with significant theoretical and practical consequences as well 
as new practical stopping criteria and warm-start strategies. 

%\paragraph{Moreau-Yosida smoothing interpretation.} 
%Catalyst can be seen as an accelerated gradient descent method with
%inexact gradients, applied to the Moreau envelope of the objective
%function. While the link between the proximal point algorithm and the
%Moreau-Yosida smoothing was known~\citep[see][]{themelis2016forward}, the fact that smoothing may be useful
%to accelerate existing algorithms, as in Catalyst, is less intuitive and opens
%new perspectives. In particular, it was a key observation to our subsequent
%work to Catalyst, called QNing~\citep{lin2016quickening}, which relies on 
%limited memory Quasi-Newton principles instead of Nesterov's acceleration.
%Besides, the interpretation relying on Moreau-Yosida smoothing also leads to new
%practical strategies.
%
%\paragraph{New stopping criteria and warm start strategies.}
%Catalyst is a two-loop algorithm, which requires solving sub-problems
%in an inner-loop with enough accuracy. In this paper, we introduce new
%parameter-free stopping criteria for solving these sub-problems and new warm
%start strategies.

The paper is structured as follows.
We complete this introductory section with some related work in Section~\ref{subsec:related},
and give a short description of the two-loop Catalyst algorithm in Section~\ref{subsec:catalyst}.
Then, Section~\ref{sec:related} introduces the Moreau-Yosida smoothing
and its inexact variant.  In Section~\ref{sec:algorithm}, we introduce formally
the main algorithm, and its convergence analysis is presented in
Section~\ref{sec:theory}. Section~\ref{sec:exp} is devoted to numerical
experiments and Section~\ref{sec:ccl} concludes the paper.

\subsection{Related Work}\label{subsec:related}
Catalyst can be interpreted as a variant of the proximal point
algorithm~\citep{rockafellarppa,guler:1991}, which is a central concept in convex optimization,
underlying augmented Lagrangian approaches, and composite minimization schemes~\citep{bertsekas:2015,Parikh13}.
The proximal point algorithm consists of solving~(\ref{eq:obj}) by minimizing a
sequence of auxiliary problems involving a quadratic regularization term. In general,
these auxiliary problems cannot be solved
with perfect accuracy, and several notions of inexactness were proposed by~\citet{newppa,he2012accelerated} and~\citet{salzo2012inexact}.
The Catalyst approach hinges upon (i) an acceleration technique for the proximal
point algorithm originally introduced in the pioneer work of~\citet{newppa}; (ii)
a more practical inexactness criterion than those proposed in the
past.\footnote{Note that our inexact criterion was also studied, among others,
by~\citet{salzo2012inexact}, but their analysis led to
the conjecture that this criterion was too weak to warrant acceleration. Our
analysis refutes this conjecture.}
As a result, we are able to control the rate of convergence for approximately solving the
auxiliary problems with an optimization method~$\mtd$. In turn, we are also able to obtain
the computational complexity of the global procedure, which was not possible with previous
analysis~\citep{newppa,he2012accelerated,salzo2012inexact}.
When instantiated in different first-order optimization settings, our 
analysis yields systematic acceleration.

Beyond~\citet{newppa}, several works have inspired this work. In particular,
accelerated SDCA~\citep{accsdca} is an instance of an inexact accelerated
proximal point algorithm, even though this was not explicitly stated
in the original paper. Catalyst can be seen as a generalization of their 
algorithm, originally designed for stochastic dual coordinate ascent approaches.
Yet their proof of convergence relies on different tools than
ours. Specifically, we introduce an approximate sufficient descent condition,
which, when satisfied, grants acceleration to any optimization method,
whereas the direct proof of~\citet{accsdca}, in the
context of SDCA, does not extend to non-strongly convex objectives.
 Another useful methodological contribution was the convergence
analysis of inexact proximal gradient methods of~\citet{proxinexact} and~\citet{inexactnesterov}.
Finally, similar ideas appeared in the independent work~\citep{frostig}. Their results partially overlap 
with ours, but the two papers adopt rather different directions. Our analysis is more general, 
covering both strongly-convex and non-strongly convex objectives, 
and comprises several variants including an almost parameter-free variant. 

Then, beyond accelerated SDCA \citep{accsdca}, other accelerated incremental methods have been proposed, such as
APCG~\citep{apbcd}, SDPC~\citep{zhangxiao}, RPDG~\citep{conjugategradient},
Point-SAGA \citep{pointsaga} and Katyusha \citep{accsvrg}. Their
techniques are algorithm-specific and cannot be directly generalized into a
unified scheme. However, we should mention that the complexity obtained by
applying Catalyst acceleration to incremental methods matches the optimal bound up to a 
logarithmic factor, which may be the price to pay for a generic acceleration scheme. 

A related recent line of work has also combined smoothing techniques
with outer-loop algorithms such as Quasi-Newton
methods~\citep{themelis2016forward,giselsson2016nonsmooth}. Their
purpose was not to accelerate existing techniques, but rather to derive new
algorithms for nonsmooth optimization.

To conclude this survey, we mention the broad family of extrapolation methods~\citep{sidi}, 
which allow one to extrapolate to the limit sequences generated by iterative algorithms
for various numerical analysis problems. 
\citet{scieur} proposed such an approach for convex optimization problems 
with smooth and strongly convex objectives. The approach we present here allows us
to obtain global complexity bounds for strongly convex and non strongly convex objectives,
which can be decomposed into a smooth part and a non-smooth proximal-friendly part.

\subsection{Overview of Catalyst}\label{subsec:catalyst}
Before introducing Catalyst precisely in Section~\ref{sec:algorithm}, we give a quick overview of the algorithm and its main ideas.
Catalyst is a generic approach that wraps an algorithm $\mtd$ into an
accelerated one~$\Acal$, in order to achieve the same accuracy as $\mtd$ with
reduced computational complexity. The resulting method $\Acal$ is an inner-outer loop construct, presented in Algorithm~\ref{alg:catalyst0}, 
where in the \emph{inner loop} the method~$\mtd$  is called to solve an auxiliary strongly-convex optimization problem, 
and where in the \emph{outer loop} the sequence of iterates produced by $\mtd$ are \emph{extrapolated} for faster convergence. 
\begin{algorithm}[t!]
   \caption{Catalyst - Overview} \label{alg:catalyst0}
        \begin{algorithmic}[1]
           \INPUT initial estimate $x_0$ in $\R^p$, smoothing parameter $\kappa$, optimization method $\mtd$.
           \STATE Initialize $y_0 = x_0$.
           \WHILE {the desired accuracy is not achieved}
               \STATE Find $x_{k}$ using~$\mtd$
    \begin{equation}
       x_{k}  \approx \argmin_{x \in \Real^p}  \left \{ h_k(x) \defin f(x)+ \frac{\kappa}{2} \| x - y_{k-1}\|^2  \right \}. \label{eq:approx}
    \end{equation}
           \STATE Compute $y_k$ using an extrapolation step, with $\beta_k$ in $(0,1)$
    \begin{equation*} 
       y_{k} = x_{k}+ \beta_k (x_{k} - x_{k-1}) . \label{eq:yk}
    \end{equation*}
    \ENDWHILE
    \OUTPUT $x_k$ (final estimate).
	\end{algorithmic}
\end{algorithm}
There are therefore 
three main ingredients in Catalyst: 
a) a smoothing technique that produces strongly-convex sub-problems; 
b) an extrapolation technique to accelerate the convergence; 
c) a balancing principle to optimally tune the inner and outer computations.

\paragraph{Smoothing by infimal convolution}
Catalyst can be used on any algorithm $\mtd$ that enjoys a linear-convergence 
guarantee when minimizing strongly-convex objectives. However the objective at hand 
may be poorly conditioned or even might not be strongly convex. 
In Catalyst, we use $\mtd$ to \emph{approximately minimize} an auxiliary objective~$h_k$ at iteration~$k$, defined in~(\ref{eq:approx}), 
which is strongly convex and better conditioned than $f$. Smoothing by infimal convolution
allows one to build a well-conditioned convex function $F$
from a poorly-conditioned convex function $f$ (see Section~\ref{sec:algorithm} for a refresher on Moreau envelopes). 
We shall show in Section~\ref{sec:algorithm} that a notion of \emph{approximate Moreau envelope}
allows us to define precisely the information collected when
approximately minimizing the auxiliary objective.

\paragraph{Extrapolation by Nesterov acceleration}
Catalyst uses an extrapolation scheme `` \`{a} la Nesterov '' to build a sequence $(y_k)_{k \geq 0}$ updated as 
\begin{equation*} 
y_{k} = x_{k}+ \beta_k (x_{k} - x_{k-1}) \; ,
\end{equation*}
where $(\beta_k)_{k \geq 0}$ is a positive decreasing sequence, which we shall define in Section~\ref{sec:algorithm}. 

We shall show in Section~\ref{sec:theory} that we can get faster rates of convergence 
thanks to this extrapolation step when the smoothing parameter $\kappa$, the inner-loop stopping criterion, 
and the sequence $(\beta_k)_{k \geq 0}$ are carefully built. 

\paragraph{Balancing inner and outer complexities}
The optimal balance between inner loop and outer loop complexity 
derives from the complexity bounds established in Section~\ref{sec:theory}. 
Given an estimate about the condition number of $f$,
our bounds dictate a choice of $\kappa$ that gives the optimal setting for the inner-loop stopping criterion 
and all technical quantities involved in the algorithm. 
We shall demonstrate in particular the power of an appropriate warm-start strategy to achieve near-optimal complexity. 

\paragraph{Overview of the complexity results}
Finally, we provide in Table~\ref{tab:accel-summary} a brief overview of the
complexity results obtained from the Catalyst acceleration, when applied to various
optimization methods~$\mtd$ for minimizing a large finite sum of $n$ functions.
Note that the complexity results obtained with Catalyst are optimal, up to some
logarithmic factors~\citep[see][]{agarwal,tightbound_yossi,tightbound_blake}.
\begin{table}
   \centering
   \renewcommand{\arraystretch}{1.5}
   \begin{tabular}{|l|c|c||c|c|}
      \hline
      & \multicolumn{2}{c||}{Without Catalyst} &  \multicolumn{2}{c|}{With Catalyst} \\
      \hline
      & $\mu > 0$ &  $\mu = 0$ & $\mu > 0$ & $\mu = 0$ \\
      \hline
      \hline
      FG  & $O\!\left(n\frac{L}{\mu}\log\left(\frac{1}{\varepsilon}\right)\right)$ & $O\!\left(n\frac{L}{\varepsilon}\right)$ 
      &  $\Otilde\!\left(n\sqrt{\frac{L}{\mu}}\log\left(\frac{1}{\varepsilon}\right)\right)$ &  $\Otilde\!\left(n\sqrt{\frac{L}{\varepsilon}}\right)$\\
      \hline
      SAG/SAGA  & \multirow{4}{*}{$O\!\left(\left(n+\frac{\bar{L}}{\mu}\right) \log\left(\frac{1}{\varepsilon}\right)\right)$} & $O\!\left(n \frac{\bar{L}}{\varepsilon} \right)$ & \multirow{4}{*}{$\Otilde\!\left(\!\left(n+\sqrt{\frac{n \bar{L}}{\mu}} \right)\log\left(\frac{1}{\varepsilon}\right)\!\right)$}&  \multirow{4}{*}{$\Otilde\!\left(\!\sqrt{\frac{n \bar{L}}{\varepsilon}}\right)$}\\
      \cline{1-1}\cline{3-3}
      MISO &  & \multirow{3}{*}{not avail.} &  &  \\
      \cline{1-1}
      SDCA  &  &  & &  \\
      \cline{1-1}
      SVRG  &  &  &   &  \\
      \hline
      Acc-FG & $O\!\left(n\sqrt{\frac{L}{\mu}}\log\left(\frac{1}{\varepsilon}\right)\right)$ 
      & $O\!\left(n\frac{L}{{\sqrt{\varepsilon}}}\right)$ & \multicolumn{2}{c|}{\multirow{2}{*}{no acceleration}} \\
      \cline{1-3}
      Acc-SDCA & $\Otilde\!\left(\!\left(n+\sqrt{\frac{n\bar{L}}{\mu}}\right)\log\left(\frac{1}{\varepsilon}\right)\!\right)$ &  not avail.
      &  \multicolumn{2}{c|}{} \\
      \hline
   \end{tabular}
   \caption{Comparison of rates of convergence, before and after the Catalyst
   acceleration, in the strongly-convex and non strongly-convex cases, respectively. The notation $\Otilde$ hides logarithmic factors. The constant $L$ is the global Lipschitz constant of the gradient's objective, while $\bar{L}$ is the average Lipschitz constants of the gradients $\nabla f_i$, or the maximum value, depending on the algorithm's variants considered. }
   \label{tab:accel-summary}
\end{table}

\section{The Moreau Envelope and its Approximate Variant}\label{sec:related}
In this section, we recall a classical tool from convex analysis called
the Moreau envelope or Moreau-Yosida smoothing~\citep{moreau1962fonctions,yosida}, which plays a
key role for understanding the Catalyst acceleration.  This tool can be
seen as a smoothing technique, which can turn any 
convex lower semicontinuous function $f$ into a smooth function, 
and an ill-conditioned smooth convex function into a well-conditioned smooth convex function. 

The Moreau envelope results from the infimal
convolution of $f$ with a quadratic penalty:
\begin{equation} \label{eq:MY}
   F(x)  \defin \min_{z \in \R^p} \left \{  f(z) + \frac{\kappa}{2} \Vert z-x \Vert^2  \right  \}, 
\end{equation}
where $\kappa$ is a positive regularization parameter. The proximal operator  is then the unique minimizer of the problem---that is,
\begin{equation*} \label{eq:prox}
   p(x)  \defin \prox_{f/\kappa}(x) = \argmin_{z \in \R^p} \left \{  f(z) + \frac{\kappa}{2} \Vert z-x \Vert^2  \right  \}. 
\end{equation*}
Note that $p(x)$ does not admit a closed form in general. Therefore,
computing it requires to solve the sub-problem to high accuracy with some iterative algorithm.

\subsection{Basic Properties of the Moreau Envelope}

The smoothing effect of the Moreau regularization can be characterized by the next proposition \citep[see][for elementary proofs]{lemarechal1997practical}.
\begin{proposition}[\bfseries Regularization properties of the Moreau Envelope]\label{propMY}
   Given a convex continuous function $f$ and a regularization parameter $\kappa>0$, consider the Moreau envelope $F$ defined in~(\ref{eq:MY}). Then, 
\begin{enumerate}
\item $F$ is convex and minimizing $f$ and $F$ are equivalent in the sense that 
\begin{equation*}\label{eq:minF=minf}
\min_{x \in \R^p} F(x) = \min_{x \in \R^p} f(x) \; .
\end{equation*}
Moreover the solution set of the two above problems coincide with each other.
\item $F$ is continuously differentiable even when $f$ is not and 
   \begin{equation}\label{eq:gradF}
\nabla F(x) = \kappa(x- p(x)) \; .
\end{equation}
      Moreover the gradient $\nabla F$ is Lipschitz continuous with constant $L_{F}=\kappa$.
\item If $f$ is $\mu$-strongly convex, then $F$ is $\mu_{F}$-strongly convex with $\mu_{F}=\frac{\mu \kappa}{\mu+\kappa}.$  
\end{enumerate}
\end{proposition}
Interestingly, $F$ is friendly from an optimization point of view as it is convex and differentiable. Besides, $F$ is $\kappa$-smooth with condition number $\frac{\mu+\kappa}{\mu}$ when $f$ is $\mu$-strongly convex. Thus $F$ can be made arbitrarily well conditioned by choosing a small~$\kappa$. Since both functions $f$ and $F$ admit the same
solutions, a naive approach to minimize a non-smooth function~$f$ is to first construct its Moreau envelope $F$ and then apply a smooth optimization method on it. As we will see next, Catalyst can be seen as an accelerated 
gradient descent technique applied to $F$ with inexact gradients.

\subsection{A Fresh Look at Catalyst}
First-order methods applied to~$F$ provide us several well-known algorithms.

\paragraph{The proximal point algorithm.} Consider gradient descent steps on~$F$:
\begin{equation*}\label{eq:gd on MY}
x_{k+1} = x_k - \frac{1}{L_{F}} \nabla F(x_k).
\end{equation*}
By noticing that $\nabla F(x_k) = \kappa (x_k - p(x_k))$ and $L_f =\kappa$, we obtain in fact
\begin{equation*}\label{eq:ppa} 
   x_{k+1} = p(x_k) = \argmin_{z \in \R^p} \left \{ f(z) + \frac{\kappa}{2} \Vert z- x_k \Vert^2 \right \},
\end{equation*}
which is exactly the proximal point algorithm~\citep{martinet1970,rockafellarppa}.

\paragraph{Accelerated proximal point algorithm.}
If gradient descent steps on~$F$ yields the proximal point algorithm, it is then natural to consider the following sequence
\begin{equation*}\label{eq:nesterov on F}
  x_{k+1} = y_k - \frac{1}{L_{F}} \nabla F(y_k)  \quad \text{ and } \quad  y_{k+1} = x_{k+1} + \beta_{k+1} (x_{k+1} -x_k),
\end{equation*}
where $\beta_{k+1}$ is Nesterov's extrapolation parameter~\citep{nesterov}.
Again, by using the closed form of the gradient, this is equivalent to the update
 \begin{equation*}\label{eq:newppa}
  x_{k+1} = p(y_k)  \quad \text{ and } \quad  y_{k+1} = x_{k+1} + \beta_{k+1} (x_{k+1} -x_k),
 \end{equation*}
which is known as the accelerated proximal point algorithm of~\citet{newppa}.

While these algorithms are conceptually elegant, they suffer from a major
drawback in practice: each update requires to evaluate the proximal
operator $p(x)$.  Unless a closed form is available, which is almost
never the case, we are not able to evaluate $p(x)$ exactly. Hence an iterative algorithm is required for each evaluation of the proximal
operator which leads to the inner-outer construction (see Algorithm~\ref{alg:catalyst0}).
Catalyst can then be interpreted as an accelerated proximal point algorithm 
that calls an optimization method~$\mtd$ to compute inexact solutions to the sub-problems. 
The fact that such a strategy could be used to solve non-smooth optimization 
problems was well-known, but the fact that it could be used for acceleration
is more surprising.
The main challenge that will be addressed in Section~\ref{sec:algorithm}
is how to control the complexity of the inner-loop
minimization.

\subsection{The Approximate Moreau Envelope}\label{subsec:ME}
Since Catalyst uses inexact gradients of the Moreau envelope, we start with specifying the inexactness criteria.

\paragraph{Inexactness through absolute accuracy.}
Given a proximal center $x$, a smoothing parameter $\kappa$, and an accuracy $\varepsilon \!>\! 0$, we denote the set of $\varepsilon$-approximations of the proximal operator $p(x)$ by
\begin{equation}
   \peps(x) \defin \left\{z \in \mathbb{R}^p \st \,  h(z) - h^\star \leq \varepsilon \right\}~~~\text{where}~~~ h(z) = f(z)+\frac{\kappa}{2}\|x-z\|^2,\tag{\textsf{C1}}\label{C1}
\end{equation}
and $h^\star$ is the minimum function value of $h$.

Checking whether $h(z) - h^\star \leq \varepsilon$ may be impactical since $h^\star$ is unknown in many situations. We may then replace $h^*$ by a lower bound that can be computed more easily. We may use the Fenchel conjugate for instance. Then, given a point $z$ and a lower-bound $d(z) \leq h^\star$, we can guarantee $z \in \peps(x)$ if~$h(z)-d(z) \leq \varepsilon$. 
There are other choices for the lower bounding function $d$ which result from the specific construction of the optimization algorithm. For instance, dual type algorithms such as SDCA~\citep{sdca} or MISO~\citep{miso} maintain a lower bound along the iterations, allowing one to compute $h(z)-d(z) \leq \varepsilon$. 

When none of the options mentioned above are available, we can use the following fact, based on the notion of gradient mapping; see Section 2.3.2 of \citep{nesterov}. The intuition comes from the smooth case: when $h$ is smooth, the strong convexity yields 
\begin{displaymath}
	h(z) - \frac{1}{2\kappa} \Vert \nabla h(z) \Vert^2 \leq h^\star.
\end{displaymath}
In other words, the norm of the gradient provides enough information to assess how far we are from the optimum. From this perspective, the gradient mapping can be seen as an extension of the gradient for the composite case where the objective decomposes as a sum of a smooth part and a non-smooth part~\citep{nesterov}. 
\begin{lemma}[\bfseries Checking the absolute accuracy criterion]\label{prop:abs}
   Consider a proximal center~$x$, a smoothing parameter $\kappa$ and an accuracy $\varepsilon > 0$. Consider an objective with the composite form~(\ref{eq:obj}) and we set function~$h$ as
   \begin{displaymath}
   	h(z) =f(z) + \frac{\kappa}{2}\|x-z\|^2 = \underbrace{f_0(z) + \frac{\kappa}{2}\|x-z\|^2}_{\defin \,\, h_0} + \psi(x).
   \end{displaymath}
   For any $z \in \R^p$, we define
      \begin{equation}
         [z]_\eta = \prox_{\eta \psi }\left( z -  \eta \nabla h_0\left (z \right ) \right), \quad \text{with} \quad  \eta=\frac{1}{\kappa+L}. \label{eq:gradient_mapping} 
      \end{equation}
           Then, the gradient mapping of $h$ at $z$ is defined by $\frac{1}{\eta} (z - [z]_\eta)$ and  
       \begin{equation*}
          \frac{1}{\eta}\left\|z - [z]_\eta \right\| \leq \sqrt{2\kappa \varepsilon}  ~~~\text{implies}~~~ [z]_\eta \in \peps(x). %\tag{\textsf{C2}}\label{C2}
       \end{equation*}
\end{lemma}
The proof is given in Appendix~\ref{appendix:proofs}. The lemma shows that it is sufficient to check the norm of the gradient mapping to ensure condition (\ref{C1}). However, this requires an additional full gradient step and proximal step at each iteration.

As soon as we have an approximate proximal operator $z$ in $\peps(x)$ in hand, we can define an approximate gradient of the Moreau envelope, 
\begin{equation}
   g(z) \defin \kappa (x-z), \label{eq:approx_gradient}
\end{equation}
by mimicking the exact gradient formula $\nabla F(x) = \kappa (x - p(x))$. 
As a consequence, we may immediately draw a link 
\begin{equation}
   z \in \peps(x) \quad \Longrightarrow \quad \|z-p(x)\| \leq \sqrt{\frac{2\varepsilon}{\kappa}} \quad \Longleftrightarrow \quad \Vert g(z) - \nabla F(x) \Vert \leq \sqrt{2\kappa \varepsilon}, \label{eq:equiv}
\end{equation}
where the first implication is a consequence of the strong convexity of $h$ at its minimum~$p(x)$. We will then apply the approximate gradient $g$ instead of $\nabla F$ to build the inexact proximal point algorithm. Since the inexactness of the approximate gradient can be bounded by an absolute value $\sqrt{2\kappa \varepsilon}$, we call (\ref{C1}) the absolute accuracy criterion.

\paragraph{Relative error criterion.}
Another natural way to bound the gradient approximation is by using a relative error, namely in the form $\Vert g(z) - \nabla F(x) \Vert \leq \delta' \|\nabla F(x)\|$ for some~$\delta' > 0$. This leads us to the following inexactness criterion.

Given a proximal center~$x$, a smoothing parameter $\kappa$ and a relative accuracy $\delta$ in $[0, 1)$, we denote the set of $\delta$-relative approximations by
\begin{equation}
   \gdelta(x) \defin \left\{z \in \mathbb{R}^p \st  h(z) - h^\star \leq \frac{\delta\kappa}{2}\|x-z\|^2 \right\},\label{C2}\tag{\textsf{C2}}
\end{equation}
At a first glance, we may interpret the criterion~(\ref{C2}) as~(\ref{C1}) by setting $\varepsilon = \frac{\delta\kappa}{2}\|x-z\|^2$. But we should then notice that the accuracy depends on the point~$z$, which is is no longer an absolute constant. In other words, the accuracy varies from point to point, which is proportional to the squared distance between $z$ and $x$. First one may wonder whether $\gdelta(x)$ is an empty set. Indeed, it is easy to see that $p(x) \in \gdelta(x)$ since $h(p(x))-h^* = 0 \leq \frac{\delta\kappa}{2}\|x-p(x)\|^2$. Moreover, by continuity, $\gdelta(x)$ is closed set around $p(x)$. 
Then, by following similar steps as in~(\ref{eq:equiv}), we have 
\begin{displaymath}
  z \in \gdelta(x) \quad \implies \quad  \|z-p(x)\| \leq \sqrt{\delta}\|x-z\| \leq \sqrt{\delta}(\|x-p(x)\| + \|p(x)-z\|).
\end{displaymath}
By defining the approximate gradient in the same way $g(z)=\kappa(x-z)$ yields, 
\begin{displaymath}
   z \in \gdelta(x) \quad \Longrightarrow \quad \Vert g(z) - \nabla F(x) \Vert \leq \delta' \|\nabla F(x)\| ~~~\text{with}~~~\delta' = \frac{\sqrt{\delta}}{1-\sqrt{\delta}},
\end{displaymath}
which is the desired relative gradient approximation.

Finally, the discussion about bounding $h(z)-h^\star$ still holds here. In particular, Lemma~\ref{prop:abs} may be used by setting the value $\varepsilon = \frac{\delta \kappa}{2}\|x-z\|^2$. The price to pay is as an additional gradient step and an additional proximal step per iteration.

\paragraph{A few remarks on related works.}
Inexactness criteria with respect to subgradient norms have been investigated
in the past, starting from the pioneer work of \citet{rockafellarppa} in the
context of the inexact proximal point algorithm. Later, different
works have been dedicated to more practical inexactness criteria \citep{auslender1987numerical,correa1993convergence,unifiedppa,fuentes2012}. 
These criteria include duality gap, $\varepsilon$-subdifferential, or
decrease in terms of function value.
Here, we present a more intuitive point of view using the Moreau
envelope.

While the proximal point algorithm has caught a lot of attention, very few
works have focused on its accelerated variant. The first accelerated 
proximal point algorithm with inexact gradients was proposed by \citet{newppa}.
Then, \citet{salzo2012inexact} proposed a more rigorous convergence analysis,
and more inexactness criteria, which are typically stronger than ours. In the
same way, a more general inexact oracle framework has been proposed later
by \citet{inexactnesterov}. To achieve the Catalyst acceleration, our main
effort was to propose and analyze criteria that allow us to control the
complexity for finding approximate solutions of the sub-problems.

\section{Catalyst Acceleration}\label{sec:algorithm}

Catalyst is presented in Algorithm~\ref{alg:catalyst1}. As discussed in Section~\ref{sec:related},
this scheme can be interpreted as an inexact accelerated proximal point
algorithm, or equivalently as an accelerated gradient descent method applied to
the Moreau envelope of the objective with inexact gradients. Since an
overview has already been presented in Section~\ref{subsec:catalyst}, we now
present important details to obtain acceleration in theory and in practice.

 \begin{algorithm}[hbtp!]
    \caption{Catalyst} \label{alg:catalyst1}
         \begin{algorithmic}[1]
            \INPUT Initial estimate $x_0$ in $\R^p$, smoothing parameter $\kappa$, strong convexity parameter~$\mu$, optimization method $\mtd$ and a stopping criterion based on a sequence of accuracies $(\varepsilon_k)_{k \geq 0}$, or $(\delta_k)_{k \geq 0}$, or a fixed budget $T$.
            \STATE Initialize $y_0 = x_0$, $q=\frac{\mu}{\mu+\kappa}$. If $\mu>0$, set $\alpha_0 =\sqrt{q}$, otherwise $\alpha_0=1$.
            \WHILE {the desired accuracy is not achieved}
                \STATE Compute an approximate solution of the following problem with~$\mtd$
     \begin{equation*}
        x_{k}  \approx \argmin_{x \in \Real^p}  \left \{ h_k(x) \defin f(x)+ \frac{\kappa}{2} \| x - y_{k-1}\|^2  \right \},
     \end{equation*}
            using the warm-start strategy of Section~\ref{sec:algorithm} and one of the following stopping criteria:
            \begin{itemize}[itemsep=0.1cm]
               \item[(a)] {\it absolute accuracy:} find $x_k$ in $\pepsk(y_{k-1})$ by using criterion~(\ref{C1});
               \item[(b)] {\it relative accuracy:} find $x_k$ in $\gdeltak(y_{k-1})$ by using criterion~(\ref{C2});
               \item[(c)] {\it fixed budget:} run~$\mtd$ for $T$ iterations and output $x_k$.
            \end{itemize}
            \STATE Update $\alpha_k$ in $(0,1)$ by solving the equation 
            \begin{equation}\label{eq:alpha}
           	 \alpha_k^2 = (1-\alpha_k)\alpha_{k-1}^2 + q \alpha_k.
            \end{equation}
            \STATE Compute $y_k$ with Nesterov's extrapolation step
 	\begin{equation}\label{def:yk}
 	 y_k = x_k + \beta_k(x_k-x_{k-1}) \quad \text{with} \quad \beta_k =  \frac{\alpha_{k-1}(1-\alpha_{k-1})}{\alpha_{k-1}^2+\alpha_k}.
 	\end{equation}
     \ENDWHILE
     \OUTPUT $x_k$ (final estimate).
 	\end{algorithmic}
\end{algorithm}

\paragraph{Requirement: linear convergence of the method $\mtd$.}
One of the main characteristic of Catalyst is to apply the method $\mtd$ to
strongly-convex sub-problems, without requiring strong convexity of the
objective~$f$. As a consequence, Catalyst provides direct support for convex
but non-strongly convex objectives to $\mtd$, which may be useful to extend the scope
of application of techniques that need strong convexity to operate. Yet, Catalyst requires solving these
sub-problems efficiently enough in order to control the complexity of the
inner-loop computations. When applying~$\mtd$ to minimize
a strongly-convex function~$h$, we assume that~$\mtd$ is able to
produce a sequence of iterates $(z_t)_{t \geq 0}$  such that
\begin{equation}\label{eq:assumption}
   h(z_t) - h^\star \leq C_\mtd (1- \tau_{\mtd})^t (h(z_0)- h^\star),
\end{equation}
where~$z_0$ is the initial point given to~$\mtd$, and $\tau_{\mtd}$ in $(0,1)$,
$C_{\mtd} > 0$ are two constants. In such a case, we say that $\mtd$ admits
a linear convergence rate. The quantity~$\tau_{\mtd}$ controls the speed of convergence
for solving the sub-problems: the larger is $\tau_{\mtd}$, the
faster is the convergence. For a given algorithm $\mtd$, the quantity
$\tau_{\mtd}$ depends usually on the condition number of~$h$.
For instance, for the proximal gradient method and many first-order algorithms, we simply have $\tau_{\mtd}=O((\mu+\kappa)/(L+\kappa))$, as $h$ is $(\mu+\kappa)$-strongly convex and $(L+\kappa)$-smooth.
Catalyst can also be applied to randomized methods~$\mtd$ that
satisfy~(\ref{eq:assumption}) in expectation:
\begin{equation}\label{eq:assumption_random}
   \E[h(z_t) - h^\star] \leq C_\mtd (1- \tau_{\mtd})^t (h(z_0)- h^\star),
\end{equation}
Then, the complexity
results of Section~\ref{sec:theory} also hold in expectation. This allows us to apply Catalyst to randomized block
coordinate descent algorithms~\citep[see][and references therein]{richtarik2014}, and some
incremental algorithms such as SAG, SAGA, or SVRG. For other methods that admit a linear convergence
rates in terms of duality gap, such as SDCA, MISO/Finito, Catalyst can also be applied  as explained in Appendix~\ref{appendix:miso}.

\paragraph{Stopping criteria.} \label{para:stop}
Catalyst may be used with three types of stopping criteria for solving the inner-loop problems.
We now detail them below.

\custombox{
         \begin{itemize}
            \item[(a)] \emph{absolute accuracy}: we predefine a sequence
               $(\varepsilon_k)_{k \geq 0}$ of accuracies, and stop the
               method~$\mtd$ by using the absolute stopping criterion (\ref{C1}).  
               Our analysis suggests
               \begin{itemize}
                  \item if $f$ is $\mu$-strongly convex, 
                     \begin{displaymath}
                        \varepsilon_k = \frac{1}{2}(1-\rho)^{k}(f(x_0)-f^\star)~~~\text{with}~~~\rho < \sqrt{q} \; .
                     \end{displaymath}
                  \item if $f$ is convex but not strongly convex, 
                     \begin{displaymath}
                        \varepsilon_k = \frac{f(x_0) -f^*}{2(k+2)^{4+\gamma}} ~~~\text{with}~~~\gamma > 0 \; .
                     \end{displaymath}
               \end{itemize}
               Typically, $\gamma=0.1$ and $\rho=0.9\sqrt{q}$ are reasonable choices, both in theory and in
               practice. Of course, the quantity $f(x_0)-f^\star$ is unknown and we need to upper bound it by a duality gap or by Lemma~\ref{prop:abs} as discussed in Section~\ref{subsec:ME}.
            \item[(b)] \emph{relative accuracy}: To use the relative stopping criterion (\ref{C2}), our analysis suggests the following choice for the sequence $(\delta_k)_{k \geq 0}$:
               \begin{itemize}
                  \item if $f$ is $\mu$-strongly convex, 
                     \begin{displaymath}
                        \delta_k = \frac{\sqrt{q}}{2-\sqrt{q}} \; .
                     \end{displaymath}
                  \item if $f$ is convex but not strongly convex, 
                     \begin{displaymath}
                        \delta_k = \frac{1}{(k+1)^2} \; .
                     \end{displaymath}
               \end{itemize}
            \item[(c)] \hypertarget{fixbudget}{\emph{fixed budget}}: Finally, the simplest way of using
               Catalyst is to fix in advance the number~$T$ of iterations of
               the method~$\mtd$ for solving the sub-problems
               without checking any optimality criterion. Whereas our analysis provides 
               theoretical budgets that are compatible with this strategy, we found
               them to be pessimistic and impractical. Instead, we propose an
               aggressive strategy for incremental methods that simply consists
               of setting $T=n$. This setting was called the ``one-pass''
               strategy in the original Catalyst paper~\citep{catalyst}.
         \end{itemize}
      }

\paragraph{Warm-starts in inner loops.}
Besides linear convergence rate, an adequate
warm-start strategy needs to be used to guarantee that the sub-problems will be solved in reasonable computational time. The intuition is that the previous solution may still be a good approximation of the current subproblem.  Specifically,
the following choices arise from the convergence analysis that will be detailed in Section~\ref{sec:theory}. 

\custombox{
Consider the minimization of the $(k+1)$-th subproblem $h_{k+1}(z) = f(z)+ \frac{\kappa}{2} \| z - y_{k}\|^2$, we warm start the optimization method $\mtd$ at $z_0$ as following: 
         \begin{itemize}
            \item[(a)] when using criterion~(\ref{C1}) to find $x_{k+1}$ in $\pepsk(y_{k})$, 
               \begin{itemize}
                  \item if $f$ is smooth ($\psi=0$), then choose $z_0 = x_{k} + \frac{\kappa}{\kappa+\mu}(y_{k} - y_{k-1})$.
                  \item if $f$ is composite as in~(\ref{eq:obj}), then define $w_0 = x_{k} + \frac{\kappa}{\kappa+\mu}(y_{k} - y_{k-1})$ and 
                     \begin{equation*}
                        z_0 = [w_0]_\eta = \text{prox}_{\eta \psi} \left( w_0 - \eta g \right ) \text{ with }  \eta = \frac{1}{L+\kappa} \text{ and } g= \nabla f_0(w_0) + \kappa (w_0-y_k)  .
                     \end{equation*}
               \end{itemize}
            \item[(b)] when using criteria~(\ref{C2}) to find $x_{k+1}$ in $\gdeltak(y_{k})$, 
               \begin{itemize}
                  \item if $f$ is smooth ($\psi=0$), then choose $z_0 = y_{k}$.
                  \item if $f$ is composite as in~(\ref{eq:obj}), then choose
                     \begin{equation*}
                        z_0 = [y_k]_\eta = \text{prox}_{\eta \psi} ( y_{k}- \eta \nabla f_0(y_{k})) \quad \text{with} \quad \eta = \frac{1}{L+\kappa}.
                     \end{equation*}
               \end{itemize}
            \item[(c)] when using a fixed budget $T$, choose the same warm start strategy as in (b).
         \end{itemize}
      }
Note that the earlier conference paper~\citep{catalyst} considered the
 the warm start rule~$z_0=x_{k-1}$. That variant is also theoretically validated but it does not perform as well as the ones proposed here in practice. 

\paragraph{Optimal balance: choice of parameter~$\kappa$.} 
Finally, the last ingredient is to find an optimal
balance between the inner-loop (for solving each sub-problem) and outer-loop
computations. To do so, we minimize our global complexity bounds with
respect to the value of~$\kappa$. As we shall see in Section~\ref{sec:exp},
this strategy turns out to be reasonable in practice. Then, as shown in the
theoretical section, the resulting rule of thumb is 
\custombox{
   \centering
   We select $\kappa$ by
      maximizing the ratio $\tau_{\mtd}/\sqrt{\mu+\kappa}$. 
   }
   We recall that $\tau_{\mtd}$ characterizes how fast $\mtd$ solves the
   sub-problems, according to~(\ref{eq:assumption}); typically, $\tau_{\mtd}$
   depends on the condition number $\frac{L+\kappa}{\mu+\kappa}$ and is a
   function of~$\kappa$.\footnote{ Note that the rule for the non strongly
      convex case, denoted here by $\mu=0$, slightly differs
   from~\citet{catalyst} and results from a tighter complexity analysis.}
In Table~\ref{table:choice_kappa}, we illustrate the choice of~$\kappa$ for different methods. Note that the resulting rule for incremental methods is very simple for the pracitioner: select~$\kappa$ such that the condition number $\frac{\bar{L}+\kappa}{\mu+\kappa}$ is of the order of $n$; then, the inner-complexity becomes $O(n \log(1/\varepsilon))$.
\begin{table}[hbtp!]
   \centering
   \renewcommand{\arraystretch}{1.5}
   \begin{tabular}{|c|c|c|c|}
      \hline
      Method $\mtd$ & Inner-complexity & $\tau_{\mtd}$ & Choice for $\kappa$   \\
      \hline
      FG &  $O\left(n\frac{L+\kappa}{\mu+\kappa}\log\left(\frac{1}{\varepsilon}\right)\right)$ & $\propto \frac{\mu+\kappa}{L+\kappa}$ &  $L-2\mu$ \\
      \hline
      SAG/SAGA/SVRG &  $O\left(\left(n+\frac{\bar{L}+\kappa}{\mu+\kappa}\right)\log\left(\frac{1}{\varepsilon}\right)\right)$ & $\propto \frac{\mu+\kappa}{n(\mu+\kappa)+\bar{L}+\kappa} $ 
      & $\frac{\bar{L}-\mu}{n+1} - \mu$ \\
      \hline
   \end{tabular}
   \caption{Example of choices of the parameter~$\kappa$ for the full gradient (FG) and incremental methods SAG/SAGA/SVRG. See Table~\ref{tab:accel-summary} for details about the complexity.}\label{table:choice_kappa}
\end{table}

\section{Convergence and Complexity Analysis}\label{sec:theory}
We now present the complexity analysis of Catalyst.
In Section~\ref{subsec:outer_loop}, we analyze the convergence rate of the
outer loop, regardless of the complexity for solving the sub-problems.
Then, we analyze the complexity of the inner-loop computations for our various
stopping criteria and warm-start strategies in Section~\ref{subsec:inner_loop}.
Section~\ref{subsec:global_complexity} combines the outer- and inner-loop
analysis to provide the global complexity of Catalyst applied to a given
optimization method~$\mtd$. 

\subsection{Complexity Analysis for the Outer-Loop}\label{subsec:outer_loop}
The complexity analysis of the first variant of Catalyst we presented
in~\citep{catalyst} used a tool called ``estimate sequence'', which was
introduced by~\citet{nesterov}.  Here, we provide a simpler proof.
We start with criterion~(\ref{C1}), before extending the result to~(\ref{C2}).

\subsubsection{Analysis for Criterion~(\ref{C1})}
The next theorem describes how the errors $(\varepsilon_k)_{k \geq 0}$ accumulate in Catalyst.
\begin{theorem}[\bfseries Convergence of outer-loop for criterion (\ref{C1})]\label{thm:outer-loop-12}
   Consider the sequences $(x_k)_{k \geq 0}$ and $(y_k)_{k \geq 0}$ produced by Algorithm~\ref{alg:catalyst1}, assuming that $x_k$ is in $\pepsk(y_{k-1})$ for all $k \geq 1$,
   Then, 
   \begin{displaymath}
      f(x_k)-f^\star \leq  A_{k-1} \left( \sqrt{(1-\alpha_0)(f(x_0) -f^*) +   \frac{\gamma_0}{2} \Vert x^*-x_0 \Vert^2}  + 3 \sum_{j=1}^k \sqrt{\frac{{\varepsilon_{j}}}{A_{j-1}}}\right)^2,
   \end{displaymath}
   where
   \begin{equation}
      \gamma_0 = (\kappa+\mu)\alpha_0(\alpha_0-q)~~~\text{and}~~~ A_k = \prod_{j=1}^{k}(1-\alpha_{j}) ~~\text{with } A_0 =1 \; .\label{eq:Ak}
   \end{equation}
\end{theorem}
Before we prove this theorem, we note that by setting $\varepsilon_k = 0$ for
all $k$, the speed of convergence of $f(x_k)-f^\star$ is driven by the
sequence~$(A_k)_{k \geq 0}$. Thus we first show the speed of $A_k$ by recalling the Lemma 2.2.4 of \citet{nesterov}.
\begin{lemma}[\bfseries Lemma 2.2.4 of~\citealt{nesterov}]\label{lem:nesterov Ak}
   Consider the quantities~$\gamma_0, A_k$ defined in~(\ref{eq:Ak}) and the~$\alpha_k$'s defined in Algorithm~\ref{alg:catalyst1}. Then, if $\gamma_0 \geq \mu$,
   \begin{displaymath}
      A_k \leq \min \left\{ \left( 1- \sqrt{q} \right)^k, \frac{4}{\left(2+k\sqrt{\frac{\gamma_0}{\kappa}}\right)^2} \right\}.
   \end{displaymath}
\end{lemma}
For non-strongly convex objectives, $A_k$ follows the classical accelerated
$O(1/k^2)$ rate of convergence, whereas it achieves a linear convergence rate for
the strongly convex case. Intuitively, we are applying an inexact Nesterov method on the Moreau envelope~$F$, thus the convergence rate naturally depends on the inverse of its condition number, which is $q= \frac{\mu}{\mu+\kappa}$. We now provide the proof of the theorem below.\\

\begin{proof}
   We start by defining an approximate sufficient descent condition inspired by
   a remark of~\citet{chambolle2015remark} regarding accelerated
   gradient descent methods. A related condition was also used
   by~\citet{ncvxcatalyst} in the context of non-convex optimization.

   \paragraph{Approximate sufficient descent condition.}
   Let us define the function 
   $$h_{k}(x) = f(x) + \frac{\kappa}{2} \Vert x-y_{k-1}\Vert^2.$$
   Since $p(y_{k-1})$ is the unique minimizer of
   $h_{k}$, the strong convexity of $h_k$ yields: for any $k \geq 1$, for all $x$ in $\R^p$ and any $\theta_{k} > 0$,
   \begin{equation*}
      \begin{split}
         h_{k}(x) & \geq  h_{k}^\star + \frac{\kappa+\mu}{2} \Vert x-p(y_{k-1}) \Vert^2  \\
                                 & \geq  h_{k}^\star + \frac{\kappa+\mu}{2} \left (1-\theta_k \right ) \Vert x- x_{k} \Vert^2 + \frac{\kappa+\mu}{2} \left (1- \frac{1}{\theta_k}\right ) \Vert x_k - p(y_{k-1}) \Vert^2  \\
                                 & \geq  h_k(x_k) -\epsilon_k + \frac{\kappa+\mu}{2} \left (1-\theta_k \right ) \Vert x- x_k \Vert^2  + \frac{\kappa+\mu}{2} \left (1-\frac{1}{\theta_k}\right ) \Vert x_k - p(y_{k-1}) \Vert^2 ,  \\
      \end{split} 
   \end{equation*}   
   where the $(\mu+\kappa)$-strong convexity of $h_k$ is used in the first  inequality; Lemma~\ref{lemma:quadratic} is used in the second inequality, and the last one 
   uses the relation $h_k(x_k)-h_k^\star \leq \varepsilon_k$. 
   Moreover, when~$\theta_k \geq 1$, the last term is positive and we have 
   $$h_k(x)  \geq  h_k(x_k) -\epsilon_k + \frac{\kappa+\mu}{2} \left (1-\theta_k \right ) \Vert x- x_k \Vert^2 . $$
   If instead $\theta_k \leq 1$, the coefficient $\frac{1}{\theta_k} -1$ is non-negative and we have 
   $$ - \frac{\kappa+\mu}{2} \left (\frac{1}{\theta_k}-1\right ) \Vert x_k - p(y_{k-1}) \Vert^2  \geq - \left ( \frac{1}{\theta_k}-1 \right ) (h_k(x_k)-h_k^\star) \geq - \left ( \frac{1}{\theta_k}-1 \right )  \epsilon_k. $$
   In this case, we have 
      $$h_k(x)  \geq  h_k(x_k) -\frac{\epsilon_k}{\theta_k} + \frac{\kappa+\mu}{2} \left (1-\theta_k \right ) \Vert x- x_k \Vert^2 . $$
   As a result, we have for all value of~$\theta_k > 0$, 
         $$h_k(x)  \geq  h_k(x_k) + \frac{\kappa+\mu}{2} \left (1-\theta_k \right ) \Vert x- x_k \Vert^2 -\frac{\epsilon_k}{\min\{ 1,\theta_k\}}. $$
   After expanding the expression of~$h_k$, we then obtain the approximate descent condition
   \begin{equation}
      f(x_k) + \frac{\kappa}{2}\|x_k-y_{k-1}\|^2 + \frac{\kappa+\mu}{2} \left (1-\theta_k \right ) \Vert x- x_k \Vert^2 \leq f(x) + \frac{\kappa}{2}\|x-y_{k-1}\|^2 + \frac{\varepsilon_k}{\min\{1,\theta_k\}}. \label{eq:inexact}
   \end{equation}
   
   \paragraph{Definition of the Lyapunov function.} 
   We introduce a sequence~$(S_k)_{k\geq 0}$ that will act as a
   Lyapunov function, with 
   \begin{equation}
   	S_k = (1-\alpha_k)(f(x_k) -f^*) + \alpha_k  \frac{\kappa \eta_k}{2} \Vert x^*-v_k \Vert^2. \label{eq:Sk} 
   \end{equation}
   where $x^\star$ is a minimizer of~$f$,  
   $(v_k)_{k\geq 0}$ is a sequence defined by $v_0=x_0$ and
   \begin{equation*}
      v_k=x_k+\frac{1-\alpha_{k-1}}{\alpha_{k-1}} (x_k-x_{k-1})~~~\text{for $k \geq 1$}, 
   \end{equation*} 
   and $(\eta_k)_{k\geq 0}$ is an auxiliary quantity defined by 
   $$ \eta_k = \frac{\alpha_k -q}{1-q} . $$
   The way we introduce these variables allow us to write the following relationship, 
   $$
y_k = \eta_k v_k + (1-\eta_k)x_k,~~~\text{for all $k\geq 0$,}
$$
which follows from a simple calculation.   
Then by setting $z_k = \alpha_{k-1} x^\star+ (1-\alpha_{k-1})x_{k-1}$ the following relations hold for all $k \geq 1$.
   \begin{displaymath}
      \begin{split}
         f(z_k) & \leq \alpha_{k-1} f^*+ (1- \alpha_{k-1}) f(x_{k-1}) -\frac{\mu \alpha_{k-1}(1-\alpha_{k-1})}{2} \Vert x^* - x_{k-1} \Vert^2, \\
         z_k - x_{k} & =\alpha_{k-1}(x^*-v_{k}), 
      \end{split}
\end{displaymath}
and also the following one
\begin{equation*}
   \begin{split}
      \Vert z_k  - y_{k-1} \Vert^2 & = \Vert (\alpha_{k-1}-\eta_{k-1})(x^*-x_{k-1})+\eta_{k-1} (x^*-v_{k-1}) \Vert^2 \\
      & = \alpha_{k-1}^2 \left\| \left(1-\frac{\eta_{k-1}}{\alpha_{k-1}}\right)(x^*-x_{k-1})+\frac{\eta_{k-1}}{\alpha_{k-1}} (x^*-v_{k-1}) \right\|^2 \\
      & \leq \alpha_{k-1}^2 \left(1-\frac{\eta_{k-1}}{\alpha_{k-1}}\right)\Vert x^*-x_{k-1} \Vert^2 +\alpha_{k-1}^2 \frac{\eta_{k-1}}{\alpha_{k-1}} \Vert x^*-v_{k-1} \Vert^2\\
      & =  \alpha_{k-1} (\alpha_{k-1}-\eta_{k-1})  \Vert x^*-x_{k-1} \Vert^2 +  \alpha_{k-1} \eta_{k-1} \Vert x^*-v_{k-1} \Vert^2,
   \end{split}
\end{equation*}
where we used the convexity of the norm and the fact that $\eta_k \leq \alpha_k$.
   Using the previous relations in (\ref{eq:inexact}) with $x=z_k = \alpha_{k-1} x^\star+ (1-\alpha_{k-1})x_{k-1}$, gives for all $k \geq 1$,
\begin{multline*}
   f(x_{k}) + \frac{\kappa}{2}\|x_{k}-y_{\kmone}\|^2 + \frac{\kappa+\mu}{2} \left (1-\theta_{k} \right )\alpha_{\kmone}^2 \Vert x^\star- v_{k} \Vert^2  \\
    \leq \alpha_{\kmone} f^*+ (1-\alpha_{\kmone}) f(x_{\kmone}) -\frac{\mu}{2} \alpha_{\kmone}(1-\alpha_{\kmone}) \Vert x^* - x_{\kmone} \Vert^2  \\
                        + \frac{\kappa \alpha_{\kmone}(\alpha_{\kmone}-\eta_{\kmone})}{2} \Vert x^* - x_{\kmone} \Vert^2 + \frac{\kappa \alpha_{\kmone}\eta_{\kmone}}{2} \Vert x^* -v_{\kmone} \Vert^2  + \frac{\epsilon_{k}}{\min\{1, \theta_{k}\}}. 
\end{multline*}
Remark that for all $k \geq 1$,
$$  \alpha_{k-1} -\eta_{k-1} = \alpha_{k-1} - \frac{\alpha_{k-1} -q}{1-q} = \frac{q(1-\alpha_{k-1})}{1-q} = \frac{\mu}{\kappa} (1-\alpha_{k-1}), $$
and the quadratic terms involving $x^* -x_{k-1}$ cancel each other. Then, after noticing that for all $k \geq 1$,
   $$\eta_{k} \alpha_{k} = \frac{\alpha_{k}^2 - q \alpha_{k} }{1-q}= \frac{(\kappa+\mu)(1-\alpha_{k})\alpha_{k-1}^2}{\kappa},$$
   which allows us to write 
   \begin{equation}\label{eq:Sk/1-alphak}
   	f(x_k)-f^* + \frac{\kappa+\mu}{2} \alpha_{k-1}^2\Vert x^*-v_k \Vert^2 = \frac{S_k}{1-\alpha_k} .
   \end{equation}
   We are left, for all $k \geq 1$, with
\begin{equation}
   \frac{1}{1-\alpha_k} S_{k} \leq  S_{\kmone}  +\frac{\varepsilon_{k}}{\min \{ 1,\theta_{k}\}} -\frac{\kappa}{2}\|x_{k}-y_{\kmone}\|^2 + \frac{(\kappa+\mu)\alpha_{\kmone}^2\theta_{k}}{2}\|x^\star- v_{k}\|^2. \label{eq:recursion}
\end{equation}

   \paragraph{Control of the approximation errors for criterion~(\ref{C1}).}
   Using the fact that 
   $$ \frac{1}{\min\{ 1,\theta_k\}} \leq 1+ \frac{1}{\theta_k}, $$
   we immediately derive from equation (\ref{eq:recursion}) that 
   \begin{equation}\label{eq:recursion C1}
   \frac{1}{1-\alpha_k}	S_{k} \leq  S_{\kmone}  + \varepsilon_k+ \frac{\varepsilon_{k}}{\theta_{k}} -\frac{\kappa}{2}\|x_{k}-y_{\kmone}\|^2 + \frac{(\kappa+\mu)\alpha_{\kmone}^2\theta_{k}}{2}\|x^\star- v_{k}\|^2. 
   \end{equation}      
   By minimizing the right-hand side of~(\ref{eq:recursion C1}) with respect to $\theta_{k}$, we obtain the following inequality
\begin{equation*}
   \frac{1}{1-\alpha_k}S_{k} \leq S_{\kmone}  + \varepsilon_k +\sqrt{2\varepsilon_{k}(\mu+\kappa)}\alpha_{\kmone}\|x^\star-v_{k}\|,
\end{equation*}
and after unrolling the recursion,
\begin{equation*}
   \frac{S_k}{A_k} \leq  S_0 + \sum_{j=1}^k \frac{\varepsilon_j}{A_{j-1}}  + \sum_{j=1}^k \frac{\sqrt{2\varepsilon_{j}(\mu+\kappa)}\alpha_{j-1}\|x^\star-v_{j}\|}{A_{j-1}}.
\end{equation*}
From Equation (\ref{eq:Sk/1-alphak}), the lefthand side is larger than $\frac{(\mu+\kappa) \alpha_{k-1}^2 \Vert x^* - v_k \Vert^2}{2A_{k-1}}$. We may now define $u_j = \frac{\sqrt{(\mu+\kappa)}\alpha_{j-1}\|x^\star-v_{j}\|}{\sqrt{2 A_{j-1}}}$ and $a_j = 2\frac{\sqrt{\varepsilon_{j}}}{\sqrt{A_{j-1}}}$, 
and we have 
$$u_k^2 \leq S_0 + \sum_{j=1}^k \frac{\varepsilon_j}{A_{j-1}} + \sum_{j=1}^k a_j u_j ~~~\text{for all~$k \geq 1$.}$$ This allows us to apply Lemma~\ref{lemma:sequences}, which yields
\begin{equation*}
\begin{split}
	   \frac{S_{k}}{A_k} & \leq   \left( \sqrt{S_0 + \sum_{j=1}^k \frac{\varepsilon_j}{A_{j-1}}}  + 2 \sum_{j=1}^k \sqrt{\frac{{\varepsilon_{j}}}{{A_{j-1}}}}\right)^2, \\
	   & \leq \left( \sqrt{S_0 }  + 3 \sum_{j=1}^k \sqrt{\frac{{\varepsilon_{j}}}{{A_{j-1}}}}\right)^2
\end{split}
\end{equation*}
which provides us the desired result given that $f(x_k)-f^\star \leq \frac{S_k}{1-\alpha_k}$ and that $v_0 = x_0$.
\end{proof}

We are now in shape to state the convergence rate of the Catalyst algorithm with criterion~(\ref{C1}),
without taking into account yet the cost of solving the sub-problems. The next
two propositions specialize Theorem~\ref{thm:outer-loop-12} to the strongly convex case and non strongly convex cases, respectively.
Their proofs are provided in Appendix~\ref{appendix:proofs}.
\begin{proposition}[\bfseries $\mu$-strongly convex case, criterion~(\ref{C1})]~\label{prop:convergence}\newline
In Algorithm~\ref{alg:catalyst1}, choose  $\alpha_0 = \sqrt{q}$ and
$$  \epsilon_k = \frac{2}{9} (f(x_0) -f^*) (1- \rho)^{k} \quad  \text{ with } \quad  \rho < \sqrt{q}.$$
Then, the sequence of iterates $(x_k)_{k \geq 0}$
satisfies
\begin{equation*} 
f(x_{k}) -f^* \leqslant \frac{8}{(\sqrt{q} -\rho)^2}(1-\rho)^{k+1} (f(x_0) -f^*).
\end{equation*}
\end{proposition}
\begin{proposition}[\bfseries Convex case, criterion~(\ref{C1})]~\label{prop:convergence cvx}\newline
   When~$\mu=0$, choose $\alpha_0= 1$ and 
$$  \epsilon_k =  \frac{2(f(x_0) -f^*)}{9(k+1)^{4+\gamma}}  \quad  \text{ with } \quad \gamma >0 .$$
Then, Algorithm~\ref{alg:catalyst1} generates iterates $(x_k)_{k \geq 0}$ such that
\begin{equation*} 
f(x_{k}) -f^* \leqslant \frac{8}{(k+1)^2} \left ( \frac{\kappa}{2} \Vert x_0 - x^*\Vert^2 + \frac{4}{\gamma^2} (f(x_0) -f^*) \right ).
\end{equation*}
\end{proposition}

\subsubsection{Analysis for Criterion~(\ref{C2})}
Then, we may now analyze the convergence of Catalyst under criterion~(\ref{C2}), which offers similar guarantees as~(\ref{C1}), as far as the outer loop is concerned.
\begin{theorem}[\bfseries Convergence of outer-loop for criterion (\ref{C2})]\label{thm:outer-loop-c2}
   Consider the sequences $(x_k)_{k \geq 0}$ and $(y_k)_{k \geq 0}$ produced by Algorithm~\ref{alg:catalyst1}, assuming that $x_k$ is in $\gdeltak(y_{k-1})$ for all $k \geq 1$ and $\delta_k$ in $(0,1)$.
   Then, 
   \begin{displaymath}
      f(x_k)-f^\star \leq \frac{A_{k-1}}{\prod_{j=1}^k \left(1-\delta_j \right)  } \left( (1-\alpha_0)(f(x_0)-f^\star) + \frac{\gamma_0}{2}\|x_0-x^\star\|^2\right),
   \end{displaymath}
   where $\gamma_0$ and $(A_k)_{k \geq 0}$ are defined in (\ref{eq:Ak}) in Theorem~\ref{thm:outer-loop-12}.
\end{theorem}
\begin{proof}
   Remark that $x_k$ in~$\gdeltak(y_{k-1})$ is equivalent to $x_k$ in~$\pepsk(y_{k-1})$ with an adaptive error $\varepsilon_k = \frac{\delta_k \kappa}{2}\|x_k-y_{k-1}\|^2$. All steps of the proof of Theorem~\ref{thm:outer-loop-12} hold for such values of $\varepsilon_k$ and from~(\ref{eq:recursion}), we may deduce
   \begin{equation*}
   \frac{S_{k}}{1-\alpha_k} - \frac{(\kappa+\mu)\alpha_{\kmone}^2\theta_{k}}{2}\|x^\star- v_{k}\|^2 \leq  S_{\kmone}  +\left(\frac{\delta_{k} \kappa}{2\min\{1,\theta_{k}\}} -\frac{\kappa}{2}\right)\|x_{k}-y_{\kmone}\|^2 . 
\end{equation*}
   Then, by choosing $\theta_k=\delta_k < 1$, the quadratic term on the right disappears and the left-hand side is greater than $\frac{1-\delta_k}{1-\alpha_k}S_k$. Thus,
   \begin{equation*}
      S_{k} \leq  \frac{1-\alpha_{k}}{1-\delta_k} S_{\kmone}  \leq \frac{A_k}{\prod_{j=1}^{k}\left(1-\delta_j\right)} S_0,
\end{equation*}
which is sufficient to conclude since $(1-\alpha_k)(f(x_k)-f^*) \leq S_k$.
\end{proof}

The next propositions specialize Theorem~\ref{thm:outer-loop-c2} for specific choices of sequence $(\delta_k)_{k \geq 0}$ in the strongly and non strongly convex cases.
\begin{proposition}[\bfseries $\mu$-strongly convex case, criterion~(\ref{C2})]~\label{prop:convergence2}\newline
In Algorithm~\ref{alg:catalyst1}, choose  $\alpha_0 = \sqrt{q}$ and
$$  \delta_k = \frac{\sqrt{q}}{2-\sqrt{q}}.$$
Then, the sequence of iterates $(x_k)_{k \geq 0}$
satisfies
\begin{equation*} 
f(x_{k}) -f^* \leqslant 2\left (1 - \frac{\sqrt{q}}{2} \right )^k \left (f(x_0)-f^*\right ). % + \frac{\mu}{2} \Vert x_0 -x^* \Vert^2  \right ).
\end{equation*}
\end{proposition}
\begin{proof}
	This is a direct application of Theorem~\ref{thm:outer-loop-c2} by remarking  that $\gamma_0 = (1-\sqrt{q})\mu$ and
	$$S_0 = (1-\sqrt{q})\left (f(x_0)-f^*+ \frac{\mu}{2} \Vert x^*-x_0 \Vert^2 \right ) \leq 2(1-\sqrt{q})(f(x_0)-f^*). $$
	And $\alpha_k = \sqrt{q}$ for all $k\geq 0$ leading to
	$$\frac{1-\alpha_k}{1- {\delta_k}} = 1- \frac{\sqrt{q}}{2}$$
\end{proof}
\begin{proposition}[\bfseries Convex case, criterion~(\ref{C2})]~\label{prop:convergence cvx2}\newline
   When~$\mu=0$, choose $\alpha_0= 1$ and 
$$  \delta_k =  \frac{1}{(k+1)^2}.$$
Then, Algorithm~\ref{alg:catalyst1} generates iterates $(x_k)_{k \geq 0}$ such that
\begin{equation} 
   f(x_{k}) -f^* \leqslant \frac{4\kappa \Vert x_0 -x^* \Vert^2 }{(k+1)^2}. \label{eq:convc2}
\end{equation}
\end{proposition}
\begin{proof}
 This is a direct application of Theorem~\ref{thm:outer-loop-c2} by remarking that $\gamma_0 = \kappa$, $A_k \leq \frac{4}{(k+2)^2}$ (Lemma~\ref{lem:nesterov Ak}) and 
 $$\prod_{i=1}^k \left (1 - \frac{1}{(i+1)^2} \right ) = \prod_{i=1}^k \frac{i(i+2)}{(i+1)^2}  = \frac{k+2}{2(k+1)} \geq \frac{1}{2}.   $$ 
\end{proof}
\begin{remark}
   In fact, the choice of $\delta_k$ can be improved by taking $\delta_k = \frac{1}{(k+1)^{1+\gamma}}$ for any $\gamma>0$, which comes at the price of a larger constant in~(\ref{eq:convc2}). 
\end{remark}

\subsection{Analysis of Warm-start Strategies for the Inner Loop}\label{subsec:inner_loop}
In this section, we study the complexity of solving the subproblems with the
proposed warm start strategies. The only assumption we make on the
optimization method $\mtd$ is that it enjoys linear convergence when
solving a strongly convex problem---meaning, it satisfies either~(\ref{eq:assumption}) or
its randomized variant~(\ref{eq:assumption_random}). 
Then, the following lemma gives us a relation between the accuracy required to solve the sub-problems
and the corresponding complexity.

\begin{lemma}[\bfseries Accuracy vs. complexity]\label{lemma:accuracy}
   Let us consider a strongly convex objective $h$ and  a linearly convergent method $\mtd$ generating a sequence of iterates $(z_t)_{t\geq 0}$ for minimizing~$h$. Consider the complexity $T(\varepsilon) = \inf \{t \geq 0, h(z_t)-h^* \leq \varepsilon \}$, where $\varepsilon>0$ is the target accuracy and $h^\star$ is the minimum value of~$h$. Then,
\begin{enumerate}
   \item If $\mtd$ is deterministic and satisfies~(\ref{eq:assumption}), we have 
$$ T(\varepsilon) \leq \frac{1}{\tau_{\mtd}} \log \left ( \frac{C_{\mtd} (h(z_0) -h^*)}{\varepsilon}  \right ). $$ 
\item If $\mtd$ is randomized and satisfies~(\ref{eq:assumption_random}), we have
$$  \mathbb{E}[T(\varepsilon)] \leq \frac{1}{\tau_{\mtd}} \log \left ( \frac{2C_{\mtd} (h(z_0) -h^*)}{\tau_{\mtd}  \varepsilon}  \right ) +1$$
\end{enumerate}
\end{lemma}
The proof of the deterministic case is straightforward and the proof of the randomized case is provided in Appendix~\ref{appendix:accuracy}.
From the previous result, a good initialization is essential for fast convergence. More precisely, it suffices to control the initialization
$  \frac{h(z_0)-h^*}{\varepsilon} $ in order to bound the number of iterations
$T(\varepsilon)$. For that purpose, we analyze the quality of various warm-start
strategies.

\subsubsection{Warm Start Strategies for Criterion~(\ref{C1})}
The next proposition characterizes the quality of initialization for~(\ref{C1}).
\begin{proposition}[\bfseries Warm start for criterion~(\ref{C1})] \label{prop:inner_loop:c1}
   Assume that $\mtd$ is linearly convergent for strongly convex problems with parameter $\tau_\mtd$ according to~(\ref{eq:assumption}), or according to~(\ref{eq:assumption_random}) in the randomized case.
   At iteration $k+1$ of Algorithm~\ref{alg:catalyst1}, given the previous iterate $x_k$ in~$\pepsk(y_{k-1})$, we consider the following function
   $$h_{k+1}(z)  = f(z) + \frac{\kappa}{2} \Vert z-y_k \Vert^2,$$
   which we minimize with~$\mtd$, producing a sequence $(z_t)_{t \geq 0}$. 
 Then, 
\begin{itemize}
	\item when $f$ is smooth, choose $z_0 = x_k + \frac{\kappa}{\kappa+\mu}(y_k-y_{k-1})$;
	\item when $f = f_0 +\psi$ is composite, choose $z_0 = [w_0]_\eta = \prox_{\eta \psi }(w_0 - \eta \nabla h_0(w_0))$ with $w_0 = x_k + \frac{\kappa}{\kappa+\mu}(y_k-y_{k-1})$, $\eta = \frac{1}{L+\kappa}$ and $h_0 =f_0+\frac{\kappa}{2} \Vert \cdot - y_k \Vert^2$.
\end{itemize}
   We also assume that we choose $\alpha_0$ and $(\varepsilon_k)_{k \geq 0}$ according to Proposition~\ref{prop:convergence} for $\mu > 0$, or Proposition~\ref{prop:convergence cvx} for $\mu=0$. Then,
   \begin{enumerate}
      \item if $f$ is $\mu$-strongly convex, $h_{k+1}(z_0)-h_{k+1}^\star \leq C \varepsilon_{k+1}$ where, 
         \begin{equation}
            C =  \frac{L+\kappa}{\kappa+\mu} \left(\frac{2}{1-\rho} + \frac{2592(\kappa+\mu)}{(1-\rho)^2(\sqrt{q}-\rho)^2 \mu}\right)~~~\text{if $f$ is smooth},  \label{eq:warmstart}
         \end{equation}
         or 
         \begin{equation}
            C = \frac{L+\kappa}{\kappa+\mu}\left(\frac{2}{1-\rho} + \frac{23328(L+\kappa)}{(1-\rho)^2(\sqrt{q}-\rho)^2 \mu}\right)~~~\text{if $f$ is composite}.  \label{eq:warmstart3}
         \end{equation}
      \item if $f$ is convex with bounded level sets, there exists a constant $B > 0$ that only depends on $f, x_0$ and~$\kappa$ such that
         \begin{equation}
            h_{k+1}(z_0)-h_{k+1}^\star \leq B.  \label{eq:warmstart2}
         \end{equation}
   \end{enumerate}
\end{proposition}
\begin{proof}
   We treat the smooth and composite cases separately.
 \paragraph{Smooth and strongly-convex case.}
 When $f$ is smooth, by the gradient Lipschitz assumption, 
$$ h_{k+1}(z_0) -h_{k+1}^* \leq \frac{(L+\kappa)}{2} \Vert z_0 - p(y_k)\Vert^2.$$
Moreover, 
\begin{align*}
\Vert z_0 -  p(y_k) \Vert^2 & = \left \Vert x_k+ \frac{\kappa}{\kappa+\mu}(y_k-y_{k-1}) -p(y_k) \right \Vert^2  \\
				    & = \left \Vert x_k -p(y_{k-1})+ \frac{\kappa}{\kappa+\mu}(y_k-y_{k-1}) - (p(y_k)-p(y_{k-1})) \right \Vert^2 \\
				    & \leq  2  \Vert x_k -p(y_{k-1}) \Vert^2 + 2\left \Vert \frac{\kappa}{\kappa+\mu}(y_k-y_{k-1}) - (p(y_k)-p(y_{k-1})) \right \Vert^2.
\end{align*}
Since $x_k$ is in $p^{\varepsilon_k} (y_{k-1})$, we may control the first quadratic term on the right by noting that
   $$\Vert x_k - p(y_{k-1}) \Vert^2 \leq \frac{2}{\kappa+\mu}(h_k(x_k)-h_k^\star) \leq  \frac{2\epsilon_k}{\kappa+\mu}. $$
Moreover, by the coerciveness property of the proximal operator, 
$$ \left \Vert \frac{\kappa}{\kappa+\mu} (y_k-y_{k-1}) - (p(y_k)-p(y_{k-1})) \right \Vert^2 \leq \Vert y_k -y_{k-1} \Vert^2, $$
see Appendix~\ref{appendix:nonexpansive} for the proof. As a consequence,
\begin{equation}\label{eq:hk}
\begin{split}
	 h_{k+1}(z_0) -h_{k+1}^* & \leq \frac{(L+\kappa)}{2} \Vert z_0 - p(y_k)\Vert^2  \\ 
	 %& \leq (L+\kappa) (\Vert x_k -p(y_{k-1}) \Vert^2 + \Vert y_k-y_{k-1} - (p(y_k)-p(y_{k-1})) \Vert^2 ),  \\
         & \leq 2 \frac{L+\kappa}{\mu+\kappa} \epsilon_k + (L+\kappa)  \Vert y_k-y_{k-1}\Vert^2, 
\end{split}
\end{equation}
   Then, we need to control the term $\|y_k-y_{k-1}\|^2$. Inspired by the proof of accelerated SDCA of~\citet{accsdca}, 
   \begin{equation*}
   \begin{split}
\Vert y_k - y_{k-1} \Vert  &= \Vert x_k + \beta_k (x_{k} - x_{k-1}) - x_{k-1} - \beta_{k-1} (x_{k-1} - x_{k-2}) \Vert \\
                               &\leqslant (1+\beta_{k}) \Vert x_{k} - x_{k-1} \Vert + \beta_{k-1} \Vert x_{k-1} - x_{k-2} \Vert \\
				     &\leqslant 3 \max \left \{ \Vert x_{k} - x_{k-1} \Vert, \Vert x_{k-1} - x_{k-2} \Vert  \right \},
   \end{split}
\end{equation*}
The last inequality was due to the fact that~$\beta_k \leq 1$.  
In fact,
\begin{displaymath}
   \beta_k^2 = \frac{\left(\alpha_{k-1} - \alpha_{k-1}^2\right)^2}{\left(\alpha_{k-1}^2 + \alpha_k\right)^2} = 
\frac{\alpha_{k-1}^2 + \alpha_{k-1}^4 - 2\alpha_{k-1}^3}{\alpha_{k}^2  + 2\alpha_k \alpha_{k-1}^2 + \alpha_{k-1}^4} = 
\frac{\alpha_{k-1}^2 + \alpha_{k-1}^4 - 2\alpha_{k-1}^3}{\alpha_{k-1}^2 + \alpha_{k-1}^4 + q\alpha_k + \alpha_k \alpha_{k-1}^2} \leq 1,
\end{displaymath}
where the last equality uses the relation $\alpha_{k}^2 + \alpha_{k} \alpha_{k-1}^2= \alpha_{k-1}^2 + q \alpha_{k}$ from (\ref{eq:alpha}).
Then,
$$ \Vert x_{k} - x_{k-1} \Vert \leqslant  \Vert x_{k}-x^* \Vert + \Vert x_{k-1}-x^* \Vert,$$
and by strong convexity of~$f$
$$  \frac{\mu}{2}\Vert x_k - x^* \Vert^2 \leqslant f(x_k) -f^* \leqslant \frac{36}{(\sqrt{q}-\rho)^2} \epsilon_{k+1},$$
where the last inequality is obtained from Proposition~\ref{prop:convergence}. As a result, 
\begin{align*}
\Vert y_{k} - y_{k-1} \Vert^2 & \leqslant 9 \max \left \{ \Vert x_{k} - x_{k-1} \Vert^2, \Vert x_{k-1} - x_{k-2} \Vert^2  \right \} \\
					& \leqslant 36 \max \left \{ \Vert x_{k} - x^* \Vert^2, \Vert x_{k-1} - x^* \Vert^2, \Vert x_{k-2} - x^* \Vert^2   \right \} \\
                                        & \leqslant \frac{2592\,  \epsilon_{k-1} }{(\sqrt{q}-\rho)^2 \mu}  .
\end{align*}
   Since $\epsilon_{k+1}=(1-\rho)^2\epsilon_{k-1}$, we may now obtain~(\ref{eq:warmstart}) from~(\ref{eq:hk}) and the previous bound.

   \paragraph{Smooth and convex case.} When $\mu=0$, Eq.~(\ref{eq:hk}) is still valid but we need to control $\|y_k-y_{k-1}\|^2$ in a different way.
   From Proposition~\ref{prop:convergence cvx}, the sequence $(f(x_k))_{k \geq 0}$ is bounded by a constant that only depends on $f$ and $x_0$; therefore, by the bounded level set assumption, there exists $R > 0$ such that 
 $$\Vert x_k -x^* \Vert \leq R, \quad \text{for all } k \geq 0. $$
 Thus, following the same argument as the strongly convex case, we have 
 $$\Vert y_k -y_{k-1} \Vert \leq 36 R^2 \quad \text{for all } k \geq 1, $$
   and we obtain~(\ref{eq:warmstart2}) by combining the previous inequality with~(\ref{eq:hk}).

 \paragraph{Composite case.}
   By using the notation of gradient mapping introduced in~(\ref{eq:gradient_mapping}), we have~$z_0 = [w_0]_\eta$. By following similar steps as in the proof of Lemma~\ref{prop:abs}, the gradient mapping satisfies the following relation
   $$ h_{k+1}(z_0) - h_{k+1}^* \leq \frac{1}{2(\kappa+\mu)} \left \Vert \frac{1}{\eta} (w_0 - z_0) \right \Vert^2,$$
and it is sufficient to bound $\Vert w_0 -z_0 \Vert = \Vert w_0-[w_0]_\eta \Vert$. For that, we introduce 
$$[x_k]_\eta = \prox_{\eta \psi} ( x_k - \eta(\nabla f_0(x_k)+ \kappa(x_k - y_{k-1}) ) ). $$
Then,  
   \begin{equation}
      \Vert w_0 - [w_0]_\eta \Vert \leq  \Vert w_0 -x_k\Vert + \Vert x_k- [x_k]_\eta \Vert + \Vert [x_k]_\eta - [w_0]_\eta \Vert , \label{eq:hk1}
   \end{equation}
and we will bound each term on the right. By construction 
$$ \Vert w_0 -x_k \Vert = \frac{\kappa}{\kappa+\mu}\Vert y_k -y_{k-1} \Vert \leq \Vert y_k -y_{k-1} \Vert.$$    
Next, it is possible to show that the gradient mapping satisfies the following relation~\citep[see][]{nesterov2013gradient}, 
   $$ \frac{1}{2\eta}\Vert x_k - [x_k]_\eta \Vert^2 \leq h_k(x_k) - h_k^* \leq \epsilon_k.$$
And then since $[x_k]_\eta = \prox_{\eta \psi} ( x_k - \eta(\nabla f_0(x_k)+ \kappa(x_k - y_{k-1}) ) )$ and 
   $ [w_0]_\eta = \prox_{\eta \psi}(w_0 - \eta(\nabla f_0(w_0) + \kappa(w_0 - y_k)) ) $. 
   From the non expansiveness of the proximal operator, we have
   \begin{align*}
   \Vert [x_k]_\eta - [w_0]_\eta \Vert & \leq \Vert  x_k - \eta(\nabla f_0(x_k)+ \kappa(x_k - y_{k-1}) ) - \left ( w_0 - \eta(\nabla f_0(w_0) + \kappa(w_0 - y_k))  \right ) \Vert  \\
   			& \leq  \Vert  x_k - \eta(\nabla f_0(x_k)+ \kappa(x_k - y_{k-1}) ) - \left ( w_0 - \eta(\nabla f_0(w_0) + \kappa(w_0 - y_{k-1}))  \right ) \Vert \\ 
                        & ~~~~~~~~~+ \eta \kappa \Vert y_k-y_{k-1} \Vert \\
			& \leq \Vert x_k -w_0 \Vert + \eta \kappa \Vert y_k-y_{k-1} \Vert \\
			& \leq 2 \Vert y_k-y_{k-1} \Vert.
   \end{align*}
 We have used the fact that $ \Vert x - \eta \nabla h(x) - (y - \eta \nabla h(y)) \Vert \leq \Vert x-y \Vert$. By combining the previous inequalities with~(\ref{eq:hk1}), we finally have 
 $$ \Vert w_0 - [w_0]_\eta \Vert \leq \sqrt{ 2\eta \epsilon_k} +3\Vert y_k-y_{k-1}\Vert.   $$ 
Thus, by using the fact that $(a+b)^2 \leq 2a^2 + 2 b^2$ for all $a,b$,
$$ h_{k+1}(z_0) - h_{k+1}^* \leq \frac{L+\kappa}{\kappa+\mu} \left ( 2\epsilon_k + 9(L+\kappa) \Vert y_k -y_{k-1} \Vert^2 \right ), $$
and we can obtain~(\ref{eq:warmstart3}) and~(\ref{eq:warmstart2}) by upper-bounding $\|y_k-y_{k-1}\|^2$ in a similar way as in the smooth case, both when $\mu > 0$ and $\mu=0$.

\end{proof}

Finally, the complexity of the inner loop can be obtained directly by combining the previous proposition with Lemma~\ref{lemma:accuracy}.
\begin{corollary}[\bfseries Inner-loop Complexity for Criterion~(\ref{C1})]\label{cor:inner complexity1}
   Consider the setting of Proposition~\ref{prop:inner_loop:c1}; then, the sequence~$(z_t)_{t \geq 0}$ minimizing $h_{k+1}$ is such that 
   the complexity $T_{k+1} = \inf \{t \geq 0, h_{k+1}(z_t)-h_{k+1}^* \leq \varepsilon_{k+1} \}$ satisfies
   \begin{displaymath}
      T_{k+1} \leq \frac{1}{\tau_\mtd} \log \left ( C_\mtd C \right ) ~~~\text{if $\mu>0$} ~~~\Longrightarrow~~~ T_{k+1} = \Otilde\left(\frac{1}{\tau_{\mtd}}\right),
   \end{displaymath}
where $C$ is the constant defined in (\ref{eq:warmstart}) or in (\ref{eq:warmstart3}) for the composite case; and
   \begin{displaymath}
      T_{k+1} \leq  \frac{1}{\tau_\mtd} \log \left ( \frac{9C_\mtd (k+2)^{4+\eta} B}{2(f(x_0)-f^*)} \right )~~~\text{if $\mu=0$}~~~\Longrightarrow~~~ T_{k+1} = \Otilde\left(\frac{\log(k+2)}{\tau_{\mtd}}\right),
   \end{displaymath}
   where $B$ is the uniform upper bound in (\ref{eq:warmstart2}). Furthermore, when $\mtd$ is randomized, the expected complexity $\E[T_{k+1}]$ is similar, up to a factor $2/\tau_{\mtd}$ in the logarithm---see Lemma~\ref{lemma:accuracy}, and we have $\E[T_{k+1}]=\Otilde(1/\tau_{\mtd})$ when $\mu > 0$ and $\E[T_{k+1}]=\Otilde(\log(k+2)/\tau_{\mtd})$.
     Here, $\Otilde(.)$ hides logarithmic dependencies in parameters $\mu, L, \kappa, C_{\mtd}$, $\tau_{\mtd}$ and $f(x_0)-f^*$.
\end{corollary}

\subsubsection{Warm Start Strategies for Criterion~(\ref{C2})}\label{subsection:warm start c2}
We may now analyze the inner-loop complexity for criterion~(\ref{C2}) leading to upper bounds with
smaller constants and simpler proofs.  Note also that in the convex case, the
bounded level set condition will not be needed, unlike for criterion~(\ref{C1}).
To proceed, we start with a simple lemma that gives us a sufficient condition for~(\ref{C2}) to 
be satisfied.
\begin{lemma}[\bfseries Sufficient condition for criterion (\ref{C2})] \label{lemma:inner technique}
If a point $z$ satisfies 
   $$h_{k+1}(z)-h_{k+1}^* \leq \frac{\delta_{k+1}\kappa}{8} \Vert p(y_k)- y_{k} \Vert^2,$$ 
   then $z$ is in $\gdeltakp(y_k)$.
\end{lemma}
\begin{proof}
\begin{align*}
   h_{k+1}(z) -h_{k+1}^* & \leq \frac{\delta_{k+1}\kappa}{8} \Vert p(y_k) - y_{k} \Vert^2 \\
                 & \leq \frac{\delta_{k+1}\kappa}{4} \left(\Vert p(y_k) -z \Vert^2+ \Vert z - y_{k} \Vert^2 \right) \\
                 & \leq \frac{\delta_{k+1}\kappa}{4} \left ( \frac{2}{\mu+\kappa} (h_{k+1}(z)-h_{k+1}^*)+ \Vert z - y_{k} \Vert^2 \right )\\
                 & \leq \frac{1}{2} \left (h_{k+1}(z)-h_{k+1}^*\right)+  \frac{\delta_{k+1}\kappa}{4} \Vert z - y_{k} \Vert^2.\\
\end{align*}
Rearranging the terms gives the desired result.
\end{proof}
With the previous result, we can control the complexity of the inner-loop minimization with Lemma~\ref{lemma:accuracy} 
by choosing $\varepsilon = \frac{\delta_{k+1}\kappa}{8} \Vert p(y_k)- y_{k} \Vert^2$. However, to obtain a meaningful upper bound, 
we need to control the ratio 
\begin{displaymath}
   \frac{h_{k+1}(z_0)-h_{k+1}^\star}{\epsilon} =
   \frac{8(h_{k+1}(z_0)-h_{k+1}^\star)}{\delta_{k+1}\kappa \|p(y_k)-y_k\|^2}.
\end{displaymath}

\begin{proposition}[\bfseries Warm start for criterion~(\ref{C2})] \label{prop:inner_loop:c2}
   Assume that $\mtd$ is linearly convergent for strongly convex problems with parameter $\tau_\mtd$ according to~(\ref{eq:assumption}), or according to~(\ref{eq:assumption_random}) in the randomized case.
   At iteration $k+1$ of Algorithm~\ref{alg:catalyst1}, given the previous iterate $x_k$ in~$\gdeltak(y_{k-1})$, we consider the following function
   $$h_{k+1}(z)  = f(z) + \frac{\kappa}{2} \Vert z-y_k \Vert^2,$$
   which we minimize with~$\mtd$, producing a sequence $(z_t)_{t \geq 0}$. 
 Then, 
\begin{itemize}
	\item when $f$ is smooth, set $z_0 = y_k$;
	\item when $f = f_0 +\psi$ is composite, set $z_0 = [y_k]_\eta = \prox_{\eta \psi }(y_k - \eta \nabla f_0(y_k))$ with $\eta = \frac{1}{L+\kappa}$.
\end{itemize}
   Then,
   \begin{equation}
      h_{k+1}(z_0)-h_{k+1}^\star  \leq \frac{L+\kappa}{2}\|p(y_k)-y_k\|^2. \label{eq:warmstart4}
   \end{equation}
\end{proposition}
\begin{proof}
   When $f$ is smooth, the optimality conditions of $p(y_k)$ yield $\nabla h_{k+1}(p(y_k)) = \nabla f(p(y_k)) + \kappa (p(y_k) - y_k) = 0$. As a result,  
\begin{equation*}
   \begin{split}
      h_{k+1}(z_0) -h_{k+1}^* & =  f(y_k) - \left(f(p(y_k)) + \frac{\kappa}{2} \left\| p(y_k) - y_k \right\|^2 \right) \\
   & \leq  f(p(y_k)) + \langle \nabla f(p(y_k)), y_k - p(y_k) \rangle + \frac{L}{2} \| y_k-p(y_k) \|^2 \\
      & ~~~~~~~~~~~~- \left(f(p(y_k)) + \frac{\kappa}{2} \| p(y_k) - y_k \|^2 \right) \\
  & =  \frac{L+\kappa}{2} \| p(y_k) - y_k \|^2. 
   \end{split}
\end{equation*} 
When $f$ is composite, we use the inequality in Lemma 2.3 of \cite{fista}: for any~$z$,
 \begin{equation*}
    h_{k+1}(z) - h_{k+1}(z_0) \geq \frac{L+\kappa}{2} \Vert z_0 - y_k \Vert^2 + (L+\kappa) \langle z_0 -y_k, y_k-z \rangle,
 \end{equation*}
Then, we apply this inequality with $z = p(y_k)$, and thus,
 \begin{displaymath}
 \begin{split}
 	 h_{k+1}(z_0) -h_{k+1}^*  & \leq - \frac{L+\kappa}{2} \Vert z_0 - y_k \Vert^2 - (L+\kappa) \langle z_0 -y_k,y_k-p(y_k) \rangle \\
& \leq  \frac{L+\kappa}{2} \Vert p(y_k) - y_k \Vert^2 .
 \end{split}
 \end{displaymath}
\end{proof}
 
We are now in shape to derive a complexity bound for criterion~(\ref{C2}), which is obtained by
combining directly Lemma~\ref{lemma:accuracy} with the value $\varepsilon =
\frac{\delta_{k+1}\kappa}{8} \Vert p(y_k)- y_{k} \Vert^2$,
Lemma~\ref{lemma:inner technique}, and the previous proposition.
 \begin{corollary}[\bfseries Inner-loop Complexity for Criterion~(\ref{C2})] \label{cor:inner complexity2}
    Consider the setting of Proposition~\ref{prop:inner_loop:c2} when $\mtd$ is deterministic; assume further that $\alpha_0$ and $(\delta_k)_{k \geq 0}$ are chosen
    according to Proposition~\ref{prop:convergence2} for $\mu > 0$, or Proposition~\ref{prop:convergence cvx2} for $\mu=0$.
    
   Then, the sequence~$(z_t)_{t \geq 0}$  is such that 
   the complexity $T_{k+1} = \inf \{t \geq 0, z_t \in \gdeltakp(y_k) \}$ satisfies
    \begin{displaymath}
       T_{k+1}  \leq  \frac{1}{\tau_{\mtd}}\log \left( 4 C_{\mtd} \frac{(L+\kappa)}{\kappa} \frac{2-\sqrt{q}}{\sqrt{q}} \right)~~~\text{when $\mu > 0$},
    \end{displaymath}
    and
    \begin{displaymath}
       T_{k+1}  \leq  \frac{1}{\tau_{\mtd}}\log \left( 4 C_{\mtd} \frac{(L+\kappa)}{\kappa} (k+2)^2\right)~~~\text{when $\mu = 0$}.
    \end{displaymath}
    When $\mtd$ is randomized, the expected complexity is similar, up to a factor $2/\tau_{\mtd}$ in the logarithm---see Lemma~\ref{lemma:accuracy}, and we have $\E[T_{k+1}]=\Otilde(1/\tau_{\mtd})$ when $\mu > 0$ and $\E[T_{k+1}]=\Otilde(\log(k+2)/\tau_{\mtd})$.
\end{corollary}
The inner-loop complexity is asymptotically similar with criterion~(\ref{C2}) as with criterion~(\ref{C1}), but the constants are significantly better.

\subsection{Global Complexity Analysis}\label{subsec:global_complexity}
In this section, we combine the previous outer-loop and inner-loop convergence results to derive a global complexity bound. 
We treat here the strongly convex $(\mu > 0)$ and convex~$(\mu=0)$ cases separately.

\subsubsection{Strongly Convex Case}
When the problem is strongly convex, we remark that the subproblems are solved in a constant number of iterations $T_k = T=\Otilde \left (\frac{1}{\tau_\mtd} \right )$ for both criteria (\ref{C1}) and~(\ref{C2}). This means that the iterate $x_k$ in Algorithm~\ref{alg:catalyst1} is obtained after $s = kT$ iterations of the method~$\mtd$. Thus, the true convergence rate of Catalyst applied to~$\mtd$ is of the form
\begin{equation}
   f_s -f^* = f\left (x_{\frac{s}{T}} \right ) -f^* \leq C'(1-\rho)^{\frac{s}{T}}(f(x_0)-f^*) \leq C' \left ( 1 - \frac{\rho}{T} \right )^s(f(x_0)-f^*), \label{eq:rate_mup}
\end{equation}
where $f_s = f(x_k)$ is the function value after $s$ iterations of $\mtd$. 
Then, choosing~$\kappa$ consists of maximizing the rate of
convergence~(\ref{eq:rate_mup}). In other words, we want to maximize
$\sqrt{q}/T = \Otilde(\sqrt{q}\tau_{\mtd})$. Since $q =
\frac{\mu}{\mu+\kappa}$, this naturally lead to the maximization of
$\tau_{\mtd}/\sqrt{\mu+\kappa}$.  
We now state more formally the global convergence result in terms of complexity.

\begin{proposition}[\bfseries Global Complexity for strongly convex objectives]\label{prop:global mu>0}
   When $f$ is $\mu$-strongly convex and all parameters are chosen according to Propositions \ref{prop:convergence} and \ref{prop:inner_loop:c1} when using criterion~(\ref{C1}), 
   or Propositions~\ref{prop:convergence2} and~\ref{prop:inner_loop:c2} for~(\ref{C2}), then
   Algorithm~\ref{alg:catalyst1} finds a solution~$\hat{x}$ such that $f(\hat{x})-f^\star \leq \epsilon$ in at most $N_\mtd$ iterations of a deterministic method $\mtd$ with 
\begin{enumerate}
	\item when criterion (\ref{C1}) is used, 
	$$N_\mtd \leq \frac{1}{\tau_\mtd \rho} \log \left ( C_\mtd C \right ) \cdot \log \left ( \frac{8(f(x_0)-f^*)}{(\sqrt{q}-\rho)^2 \varepsilon} \right ) = \Otilde \left ( \frac{1}{\tau_\mtd \sqrt{q}} \log \left ( \frac{1}{\epsilon} \right ) \right ), $$
	where $\rho =0.9 \sqrt{q}$ and $C$ is the constant defined in (\ref{eq:warmstart}) or (\ref{eq:warmstart3}) for the composite case;
	\item when criterion (\ref{C2}) is used, 
	$$N_\mtd \leq \frac{2}{\tau_\mtd \sqrt{q}} \log \left ( 4C_\mtd \frac{L+\kappa}{\kappa} \frac{2-\sqrt{q}}{\sqrt{q}} \right ) \cdot \log \left ( \frac{2(f(x_0)-f^*)}{ \varepsilon} \right ) = \Otilde \left ( \frac{1}{\tau_\mtd \sqrt{q}} \log \left ( \frac{1}{\epsilon} \right ) \right ). $$
\end{enumerate}
   Note that similar results hold in terms of expected number of iterations
   when the method~$\mtd$ is randomized (see the end of
   Proposition~\ref{prop:inner_loop:c1}).
\end{proposition}
\begin{proof}
	Let $K$ be the number of iterations of the outer-loop algorithm required to obtain an $\epsilon$-accurate solution. From Proposition~\ref{prop:convergence}, using (\ref{C1}) criterion yields
		$$ K \leq \frac{1}{\rho} \log \left( \frac{8(f(x_0)-f^*)}{(\sqrt{q} -\rho)^2 \varepsilon}\right ). $$
	From Proposition~\ref{prop:convergence2}, using (\ref{C2}) criterion yields
	$$ K \leq \frac{2}{\sqrt{q}} \log \left ( \frac{2(f(x_0)-f^*)}{\varepsilon} \right ). $$
	Then since the number of runs of $\mtd$ is constant for any inner loop, the total number $N_\mtd$ is given by $KT$ where $T$ is respectively given by Corollaries~\ref{cor:inner complexity1} and~\ref{cor:inner complexity2}.
\end{proof}

\subsubsection{Convex, but not Strongly Convex Case}

When~$\mu=0$, the number of iterations for solving each subproblems grows logarithmically, which means that the iterate $x_k$ in Algorithm~\ref{alg:catalyst1} is obtained after $s = \leq kT\log(k+2)$ iterations of the method $\mtd$, where $T$ is a constant. By using the global iteration counter $s = kT \log(k+2)$, we finally have
\begin{equation}
f_s -f^* \leq C' \frac{\log^2(s)}{s^2} \left (f(x_0)-f^* + \frac{\kappa}{2} \Vert x_0-x^* \Vert^2  \right ).
\end{equation}
This rate is \textit{near-optimal}, up to a logarithmic factor, when compared to the optimal rate $O(1/s^2)$. This may be the price to pay for using a generic acceleration scheme. 
As before, we detail the global complexity bound for convex objectives in the next proposition.

\begin{proposition}[\bfseries Global complexity for convex objectives]\label{prop:global mu=0}
When $f$ is convex and all parameters are chosen according to Propositions \ref{prop:convergence cvx} and \ref{prop:inner_loop:c1} when using criterion~(\ref{C1}), 
   or Propositions~\ref{prop:convergence cvx2} and~\ref{prop:inner_loop:c2} for criterion~(\ref{C2}), then
   Algorithm~\ref{alg:catalyst1} finds a solution~$\hat{x}$ such that $f(\hat{x})-f^\star \leq \epsilon$ in at most $N_\mtd$ iterations of a deterministic method $\mtd$ with 

\begin{enumerate}
	\item when criterion (\ref{C1}) is applied
	$$N_\mtd \leq \frac{1}{\tau_\mtd} K \log \left (  \frac{9C_\mtd B K^{4+\gamma}}{ 2(f(x_0)-f^*)} \right ) = \Otilde \left (\frac{1}{\tau_\mtd} \sqrt{\frac{\kappa}{\varepsilon}} \log \left ( \frac{1}{\varepsilon}\right ) \right ),$$
	where,
	$$ K_\varepsilon = \sqrt{\frac{8\left ( \frac{\kappa}{2} \Vert x_0 -x^* \Vert^2 + \frac{4}{\gamma^2}(f(x_0)-f^*) \right )}{\varepsilon}};$$
	\item when criterion (\ref{C2}) is applied, 
	\begin{equation*}
		\begin{split}
			N_\mtd & \leq \frac{1}{\tau_\mtd} \sqrt{\frac{4 \kappa \Vert x_0 -x^* \Vert^2 }{\varepsilon}} \log \left (  \frac{16 C_\mtd (L+\kappa) \Vert x_0 -x^* \Vert^2 }{\varepsilon}\right )\\
			& = \Otilde \left ( \frac{1}{\tau_\mtd} \sqrt{\frac{\kappa}{\epsilon}} \log\left( \frac{1}{\epsilon} \right ) \right ).
		\end{split}
	\end{equation*} 
\end{enumerate}
   Note that similar results hold in terms of expected number of iterations
   when the method~$\mtd$ is randomized (see the end of
   Proposition~\ref{prop:inner_loop:c2}).
\end{proposition}

\begin{proof}
Let $K$ denote the number of outer-loop iterations required to achieve an $\epsilon$-accurate solution. From Proposition~\ref{prop:convergence cvx}, when (\ref{C1}) is applied, we have 
$$K \leq \sqrt{\frac{8\left ( \frac{\kappa}{2} \Vert x_0 -x^* \Vert^2 + \frac{4}{\gamma^2}(f(x_0)-f^*) \right )}{\varepsilon}}. $$
From Proposition~\ref{prop:convergence cvx2}, when (\ref{C2}) is applied, we have 
$$K \leq \sqrt{\frac{4\kappa \Vert x_0-x^* \Vert^2}{\varepsilon}}. $$
Since the number of runs in the inner loop is increasing, we have 
$$ N_\mtd = \sum_{i=1}^K T_i \leq KT_K.$$
Respectively apply $T_K$ obtained from Corollary~\ref{cor:inner complexity1} and Corollary~\ref{cor:inner complexity2} gives the result.
\end{proof}

\paragraph{Theoretical foundations of the choice of $\bm \kappa$.} 
The parameter $\kappa$ plays an important rule in the global complexity result. The linear convergence parameter $\tau_\mtd$ depends typically on $\kappa$ since it controls the strong convexity parameter of the subproblems. The natural way to choose $\kappa$ is to minimize the global complexity given by Proposition~\ref{prop:global mu>0} and Proposition~\ref{prop:global mu=0}, which leads to the following rule 
\custombox{
\centering
\textbf{Choose $\bm \kappa$ to maximize $ \bm \displaystyle \frac{\tau_\mtd}{\sqrt{\mu+\kappa}}$,}
      }
where $\mu=0$ when the problem is convex but not strongly convex. We now illustrate two examples when applying Catalyst to the classical gradient descent method and to the incremental approach SVRG. 

   \paragraph{\bf Gradient descent.} When $\mtd$ is the gradient descent method, we have 
	$$\tau_\mtd = \frac{\mu+\kappa}{L+\kappa}. $$
	Maximizing the ratio $\displaystyle \frac{\tau_\mtd}{\sqrt{\mu+\kappa}}$ gives $$\kappa = L-2\mu, ~~~~\text{when $L>2\mu$}.$$
	Consequently, the complexity in terms of gradient evaluations for minimizing the finite sum~(\ref{eq: intro problem}), where each iteration of~$\mtd$ cost $n$ gradients,
        is given by
\[ N_\mtd = \left\{
\begin{array}{cc}
      \Otilde \left ( n \sqrt{\frac{L}{\mu}} \log \left( \frac{1}{\varepsilon}\right )\right ) & \text{when $\mu>0$;} \vspace*{0.2cm} \\
      \Otilde \left( n \sqrt{\frac{L}{\varepsilon}} \log \left ( \frac{1}{\varepsilon}\right ) \right ) & \text{when $\mu=0$.} \\
\end{array} 
\right. \]
These rates are near-optimal up to logarithmic constants according to the first-order lower bound \citep{nemirovskii1983problem, nesterov}.

\paragraph{\bf SVRG.} For SVRG \citep{proxsvrg} applied to the same finite-sum objective, 
$$ \tau_\mtd  = \frac{1}{n+\frac{\bar{L}+\kappa}{\mu+\kappa}}. $$  
Thus, maximizing the corresponding ratio gives 
$$ \kappa = \frac{\bar{L}-\mu}{n+1} - \mu, ~~~~\text{when $\bar{L}> (n+2)\mu$}.$$
Consequently, the resulting global complexity, here in terms of expected number of gradient evaluations, is given by
\[ \E[N_\mtd] = \left\{
\begin{array}{cc}
      \Otilde \left ( \sqrt{n\frac{\bar{L}}{\mu}} \log \left( \frac{1}{\varepsilon}\right )\right ) & \text{when $\mu>0$;} \vspace*{0.2cm} \\
      \Otilde \left(  \sqrt{\frac{n\bar{L}}{\varepsilon}} \log \left ( \frac{1}{\varepsilon}\right ) \right ) & \text{when $\mu=0$.} \\
\end{array} 
\right. \]
Note that we treat here only the case $\bar{L} > (n+2)\mu$ to simplify, see Table~\ref{tab:accel-summary} for a general results.
We also remark that Catalyst can be applied to similar incremental algorithms such as SAG/SAGA \citep{sag,saga} or dual-type algorithm MISO/Finito \citep{miso,finito} or SDCA \cite{sdca}. 
Moreover, the resulting convergence rates are near-optimal up to logarithmic constants according to the first-order lower bound \citep{tightbound_blake,tightbound_yossi}. 

\subsubsection{Practical Aspects of the Theoretical Analysis}

So far, we have not discussed the fixed budget criterion mentioned in Section~\ref{sec:algorithm}. The idea is quite natural and simple to implement: we predefine the number of iterations to run for solving each subproblems and stop worrying about the stopping condition. For example, when $\mu>0$ and $\mtd$ is deterministic, we can simply run $T_\mtd$ iterations of $\mtd$ for each subproblem where $T_\mtd$ is greater than the value given by Corollaries~\ref{cor:inner complexity1} or~\ref{cor:inner complexity2}, then the criterions~(\ref{C1}) and~(\ref{C2}) are guaranteed to be satisfied.
Unfortunately, the theoretical bound of $T_\mtd$ is relatively poor and does not lead to a practical strategy. On the other hand, using a more aggressive strategy such as $T_\mtd =n$ for incremental algorithms, meaning one pass over the data, seems to provide outstanding results, as shown in the experimental part of this paper.

Finally, one could argue that choosing $\kappa$ according to a worst-case convergence analysis is not necessarily a good choice. In particular, the convergence rate of the method~$\mtd$, driven by the parameter~$\tau_{\mtd}$ is probably often under estimated in the first place. This suggests that using a smaller value for~$\kappa$ than the one we have advocated earlier is a good thing. In practice, we have observed that indeed Catalyst is often robust to smaller values of~$\kappa$ than the theoretical one, but we have also observed that the theoretical value performs reasonably well, as we shall see in the next section.

\section{Experimental Study}\label{sec:exp}
In this section, we conduct various experiments to study the effect of the
Catalyst acceleration and its different variants, showing in particular how to
accelerate SVRG, SAGA, and MISO.
In Section~\ref{dataset}, we describe the data sets and formulations  considered for our evaluation,
and in Section~\ref{subsec:variants}, we present the different variants of Catalyst.
Then, we study different questions: which variant of Catalyst should we use for incremental
approaches? (Section~\ref{subsec:comparison_catalyst}); how do various incremental methods compare
when accelerated with Catalyst? (Section~\ref{subsec:comparison_methods}); what is the effect
of Catalyst on the test error when Catalyst is used to minimize a regularized empirical risk? (Section~\ref{subsec:comparison_test});
is the theoretical value for~$\kappa$ appropriate? (Section~\ref{subsec:comparison_kappa}). The code used for
      all our experiments is available at
      \url{https://github.com/hongzhoulin89/Catalyst-QNing/}.

\subsection{Data sets, Formulations, and Metric}\label{dataset}
\paragraph{Data sets.}
We consider six machine learning data sets with different characteristics in
terms of size and dimension to cover a variety of situations.
\vspace*{0.1cm}
 \begin{center}
    \begin{tabular}{|l|c|c|c|c|c|c|c|}
       \hline
       name & \textsf{covtype} & \textsf{alpha} & \textsf{real-sim} & \textsf{rcv1} & \textsf{MNIST-CKN} & \textsf{CIFAR-CKN} \\
       \hline
       $n$ & $581\,012$ & $250\,000$ & $72\,309$ & $781\,265$ & $60\,000$ & $50\,000$ \\
       \hline
       $d$ & $54$ & $500$ & $20\,958$ & $47\, 152$ & $2\,304$ & $9\,216$ \\
       \hline
       \end{tabular}
 \end{center}
\vspace*{0.1cm}
While the first four data sets are standard ones that were used in previous work
about optimization methods for machine learning, the last two are coming from a
computer vision application. MNIST and CIFAR-10 are two image classification
data sets involving 10 classes. The feature representation of each image was computed 
using an unsupervised convolutional kernel network~\citet{mairal2016end}. 
We focus here on the the task of classifying class \#1 vs. the rest of the data set. 

\paragraph{Formulations.} We consider three common optimization problems in
machine learning and signal processing, which admit a particular structure
(large finite sum, composite, strong convexity). For each formulation, 
we also consider a training set $(b_i,a_i)_{i=1}^n$ of $n$ data points, where
the $b_i$'s are scalars in $\{-1,+1\}$ and the~$a_i$ are feature vectors in~$\Real^p$.
Then, the goal is to fit a linear model $x$ in~$\Real^p$ such that the scalar $b_i$ can be well
predicted by the inner-product $\approx a_i^\top x$, or by its sign.
Specifically, the three formulations we consider are listed below.
\begin{itemize}
   \item {\bfseries $\ell_2^2$-regularized Logistic Regression}:
    \begin{equation*}
    \min_{x \in \R^p} \quad \frac{1}{n} \sum_{i=1}^n \log\left(1+\exp(-b_i \,a_i^{T} x)\right) + \frac{\mu}{2} \Vert x \Vert^2,
    \end{equation*}
      which leads to a $\mu$-strongly convex smooth optimization problem.   \item {\bfseries $\ell_1$-regularized Linear Regression (LASSO)}: 
  \begin{equation*}
    \min_{x \in \R^p} \quad \frac{1}{2n} \sum_{i=1}^n   ( b_i -a_i^{T}x)^2 + \lambda \Vert x \Vert_1,
 \end{equation*}
      which is non smooth and convex but not strongly convex.
  \item {\bfseries $\ell_1-\ell_2^2$-regularized Linear Regression (Elastic-Net)}:
 \begin{equation*}
    \min_{x \in \R^p} \quad \frac{1}{2n} \sum_{i=1}^n ( b_i -a_i^{T}x )^2 + \lambda \Vert x \Vert_1+  \frac{\mu}{2} \Vert x \Vert^2,
 \end{equation*}
      which is based on the Elastic-Net regularization~\citep{zou2005regularization} and leading to a strongly-convex optimization problem.
\end{itemize}
Each feature vector~$a_i$ is normalized, and a natural upper-bound on the
Lipschitz constant~$L$ of the un-regularized objective can be easily obtained
with $L_{\text{logistic}} = 1/4$ and
$L_{\text{lasso}} =1$. The regularization parameter $\mu$ and $\lambda$ are choosing in the following way:
\begin{itemize}
        \item For \textbf{Logistic Regression}, we find an optimal
           regularization parameter $\mu^\star$ by 10-fold cross validation for each
           data set on a logarithmic grid $2^i/n$, with $i \in [-12,3]$.
           Then, we set $\mu = \mu^\star/2^3$ which corresponds to a
           small value of the regularization parameter and a relatively ill-conditioned problem. 
	\item For \textbf{Elastic-Net}, we set $\mu =0.01/n$ to simulate the ill-conditioned situation and add a small $l_1$-regularization penalty with $\lambda = 1/n$ that produces sparse solutions.
	\item For the \textbf{Lasso problem}, we consider a logarithmic grid $10^i/n$, with $i=-3,-2,\ldots,3$, and we select the parameter~$\lambda$ that provides a sparse optimal solution closest to $10\%$ non-zero coefficients, which leads to $\lambda = 10/n$ or $100/n$.
\end{itemize}
Note that for the strongly convex problems, the regularization parameter~$\mu$
yields a lower bound on the strong convexity parameter of the problem. 

\paragraph{Metric used.}
In this chapter, and following previous work about incremental
methods~\citep{sag}, we plot objective values as a function of the number of
gradients evaluated during optimization, which appears to be the computational
bottleneck of all previously mentioned algorithms. Since no metric is perfect for
comparing algorithms' speed, we shall make the two following remarks, such that the 
reader can interpret our results and the limitations of our study with no
difficulty. 
\begin{itemize}
   \item Ideally, CPU-time is the gold standard but CPU time is
      implementation-dependent and hardware-dependent.
   \item We have chosen to count only gradients computed with random data
      access. Thus, computing $n$ times a gradient $f_i$ by picking each time
      one function at random counts as ``$n$ gradients'', whereas we ignore the cost
      of computing a full gradient $(1/n)\sum_{i=1}^n \nabla f_i$ at once, where the $f_i$'s can 
      be accessed in sequential order.  Similarly, we ignore the cost of
      computing the function value $f(x)= (1/n)\sum_{i=1}^n f_i(x)$, which is
      typically performed every pass on the data when computing a duality gap.
      While this assumption may be inappropriate in some contexts, the cost of
      random gradient computations was significantly dominating the cost of sequential access in our
      experiments, where (i) data sets fit into memory; (ii) computing full
      gradients was done in C++ by calling BLAS2 functions exploiting multiple
      cores.
\end{itemize}

\subsection{Choice of Hyper-parameters and Variants}\label{subsec:variants}
Before presenting the numerical results, we discuss the choice of default
parameters used in the experiments as well as different variants.  

\paragraph{Choice of method~$\mtd$.} We consider the acceleration of incremental algorithms which are able to adapt to the problem structure we consider: large sum of functions and possibly non-smooth regularization penalty. 
\begin{itemize}
	\item The proximal SVRG algorithm of~\citet{proxsvrg} with stepsize $\eta = 1/L$.
	\item The SAGA algorithm~\citet{saga} with stepsize $\eta = 1/3L$.
	\item The proximal MISO algorithm of \citet{catalyst}.
\end{itemize}

\paragraph{Choice of regularization parameter $\kappa$.} As suggested by the theoretical analysis, we take $\kappa$ to minimize the global complexity, leading to the choice 

$$ \displaystyle \kappa = \frac{L-\mu}{n+1}-\mu.$$

\paragraph{Stopping criteria for the inner loop.} 
The choice of the accuracies are driven from the theoretical analysis described in paragraph~\ref{para:stop}. Here, we specify it again for the clarity of presentation:
\begin{itemize}
	\item \textbf{Stopping criterion (\ref{C1}).} Stop when $h_k(z_t)-h_k^*\leq \varepsilon_k$, where  
	\[ \varepsilon_k = \footnotemark[5] \left\{
\begin{array}{cc}
       \frac{1}{2} (1- \rho)^k f(x_0)  \,\,\, \text{with} \,\,\, \rho = 0.9 \sqrt{\frac{\mu}{\mu+\kappa}} & \text{when $\mu>0$;} \vspace*{0.2cm} \\
      \frac{f(x_0)}{2(k+1)^{4.1}} & \text{when $\mu=0$.} \\
\end{array} 
\right.  \]
\footnotetext[5]{Here we upper bound $f(x_0)-f^*$ by $f(x_0)$ since $f$ is always positive in our models.}\stepcounter{footnote}
The duality gap $h(w_t)-h^*$ can be estimated either by evaluating the Fenchel conjugate function or by computing the squared norm of the gradient.
\item \textbf{Stopping criterion (\ref{C2}).}  Stop when $h_k(z_t)-h_k^*\leq \delta_k \cdot \frac{\kappa}{2} \Vert z_t - y_{k-1} \Vert^2$, where 
	\[ \delta_k = \left\{
\begin{array}{cc}
       \frac{\sqrt{q}}{2-\sqrt{q}} \,\,\, \text{with} \,\,\, q = \frac{\mu}{\mu+\kappa} & \text{when $\mu>0$;} \vspace*{0.2cm} \\
      \frac{1}{(k+1)^2} & \text{when $\mu=0$.} \\
\end{array} 
\right. \]
\item \textbf{Stopping criterion \Ctrois.} Perform exactly
one pass over the data in the inner loop without checking any stopping
criteria.\footnote{This stopping criterion is heuristic since one pass may not be enough to achieve the required accuracy. What we have shown is that with a large enough $T_\mtd$, then the convergence will be guaranteed. Here we take heuristically $T_\mtd$ as one pass.}
\end{itemize}

\paragraph{Warm start for the inner loop.} This is an important point to achieve acceleration which was not highlighted in the conference paper~\citep{catalyst}. At iteration~$k+1$, we consider the minimization of
$$ h_{k+1}(z) = f_0(z) + \frac{\kappa}{2}\Vert z -y_{k} \Vert^2 +\psi(z).$$
We warm start according to the strategy defined in Section~\ref{sec:algorithm}. 
Let $x_{k}$ be the approximate minimizer of $h_{k}$, obtained from the last iteration.
\begin{itemize}
	\item \textbf{Initialization for (\ref{C1}).} Let us define $\eta = \frac{1}{L+\kappa}$, then initialize at 
\begin{equation*}
	z_0^{C1} =  \left\{
\begin{array}{cc}
   w_0 \defin x_{k} +  \frac{\kappa}{\kappa+\mu}(y_{k} -y_{k-1}) & \text{if $\psi=0$; } \vspace*{0.2cm} \\ \relax
       [w_0]_{\eta}  & \text{otherwise.} \\
\end{array} 
	\right. 
\end{equation*} 
where $[w_0]_{\eta} = \prox_{\eta\psi}(w_0 -\eta g) $  with $g= \nabla f_0(w_0)+\kappa(w_0-y_{k})$.
	\item \textbf{Initialization for (\ref{C2}).} Intialize at 
	\[ z_0^{C2} =  \left\{
\begin{array}{cc}
   y_{k} & \text{if $\psi=0$; } \vspace*{0.2cm} \\ \relax
     [y_k]_\eta = \prox_{\eta \psi} (y_{k} - \eta \nabla f_0(y_{k}))  & \text{otherwise.} \\
\end{array} 
	\right. \]
     \item \textbf{Initialization for (\textsf{C3}).} Take the best initial point among $x_{k}$ and $z_0^{C1}$
	 $$z_0^{C3} \,\, \text{ such that } \,\, h_{k}(z_0^{C3})=\min \{ h_{k}(x_{k-1}), h_{k}(z_0^{C1}) \}. $$
      \item \textbf{Initialization for (\textsf{C1}$^*$).} Use the strategy (\ref{C1}) with $z_0^{C3}$.
\end{itemize}
The warm start at $z_0^{C3}$ requires to choose the best point between the last iterate $x_{k}$ and the point $z_0^{C1}$. The motivation is that since the one-pass strategy is an aggressive heuristic, the solution of the subproblems may not be as accurate as the ones obtained with other criterions. Allowing using the iterate $x_{k}$ turned out to be significantly more stable in practice. Then, it is also natural to use a similar strategy for criterion~(\ref{C1}), which we call~\Cunstar. Using a similar strategy for (\ref{C2}) turned out not to provide any benefit in practice and is thus omitted from the list here.

\graphicspath{{./results/}}
\subsection{Comparison of Stopping Criteria and Warm-start Strategies}\label{subsec:comparison_catalyst}
First, we evaluate the performance of the previous strategies when applying Catalyst to SVRG, SAGA and MISO. The results are presented in Figures~\ref{catalyst:fig:svrg}, \ref{catalyst:fig:saga}, and \ref{catalyst:fig:miso}, respectively.
\begin{figure}[hbtp]
   \centering
   ~~\includegraphics[width=0.31\linewidth]{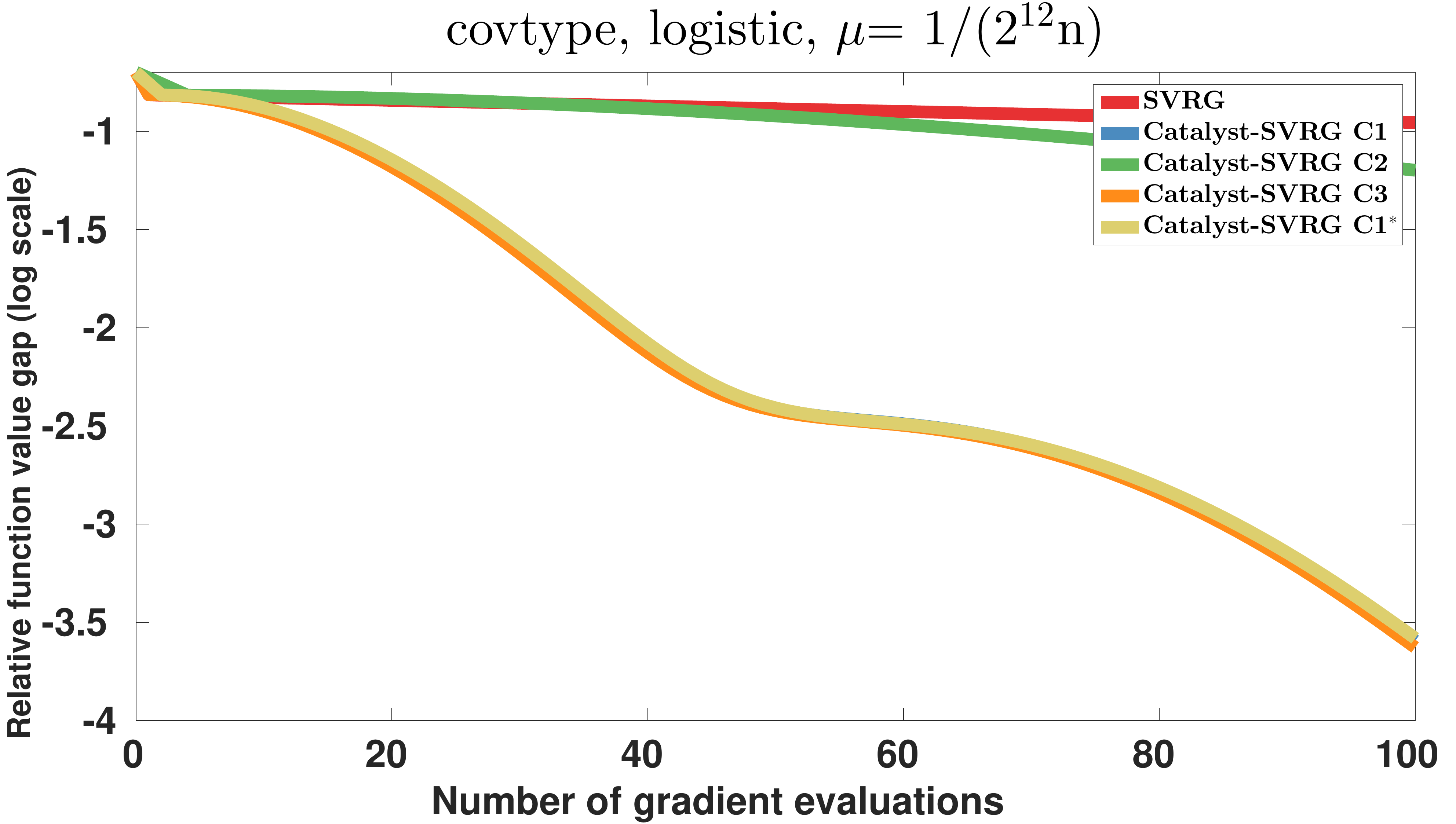}~ 
   ~~\includegraphics[width=0.31\linewidth]{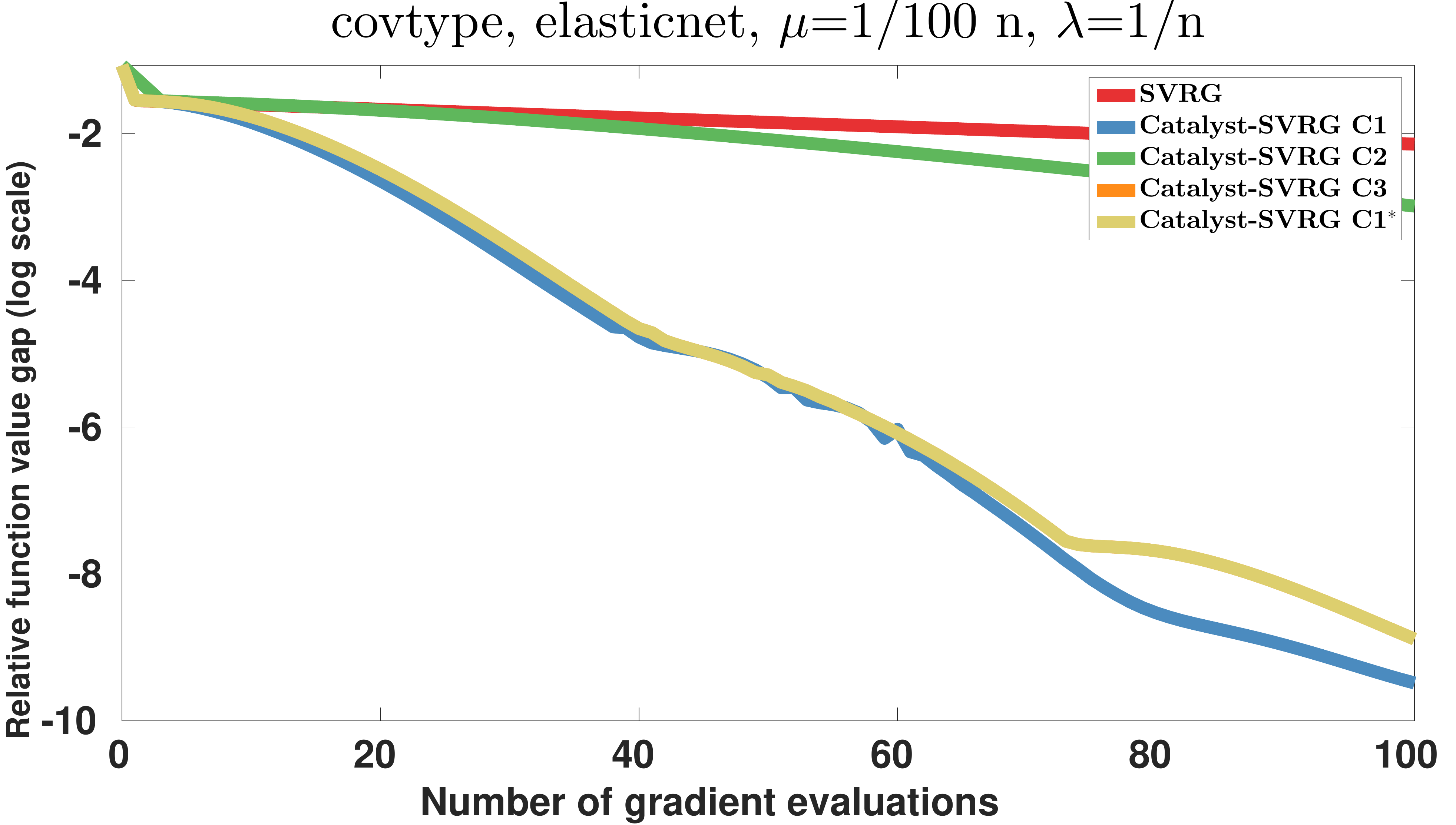}~ 
   ~~\includegraphics[width=0.31\linewidth]{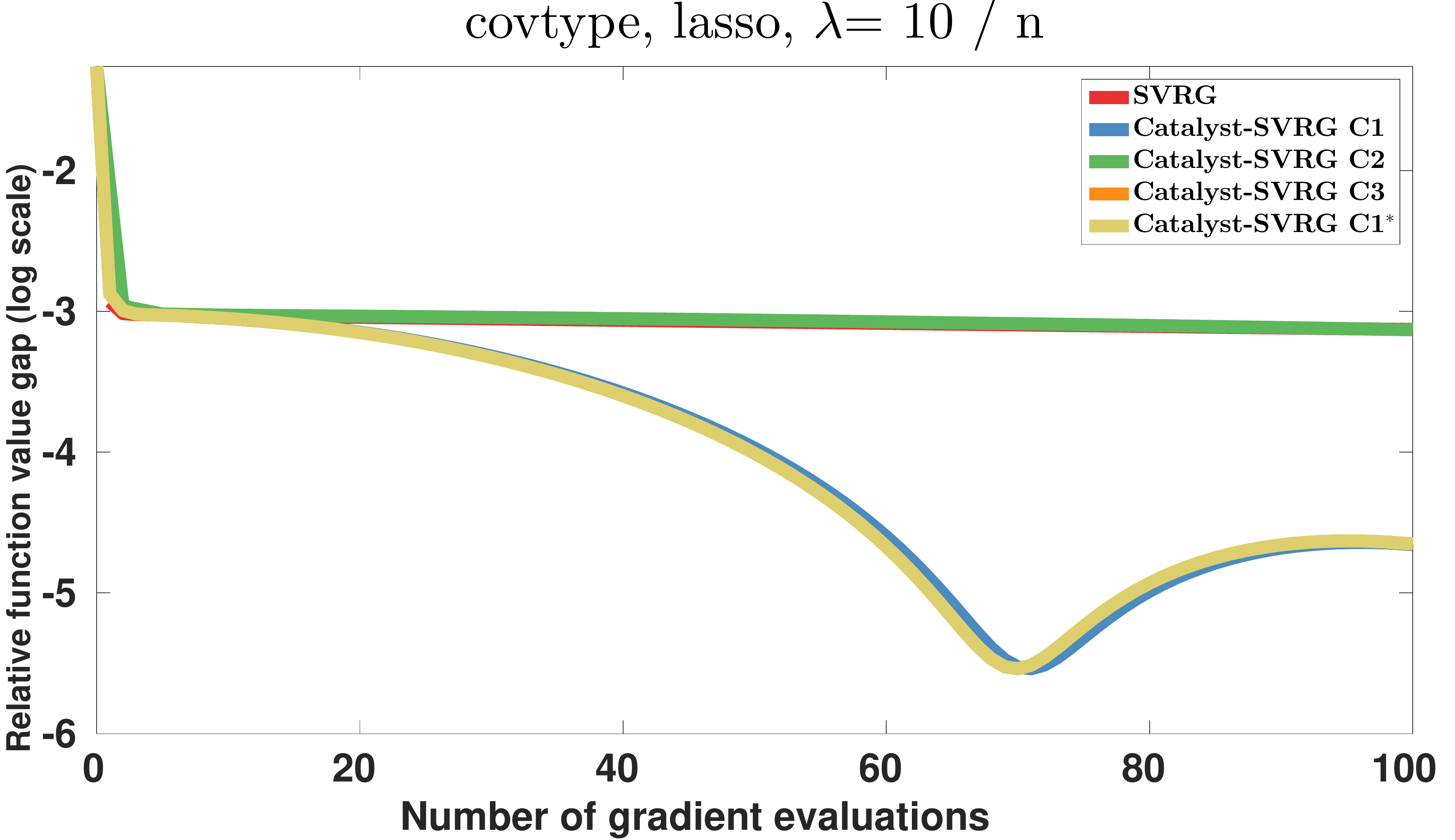}\\
   ~~\includegraphics[width=0.31\linewidth]{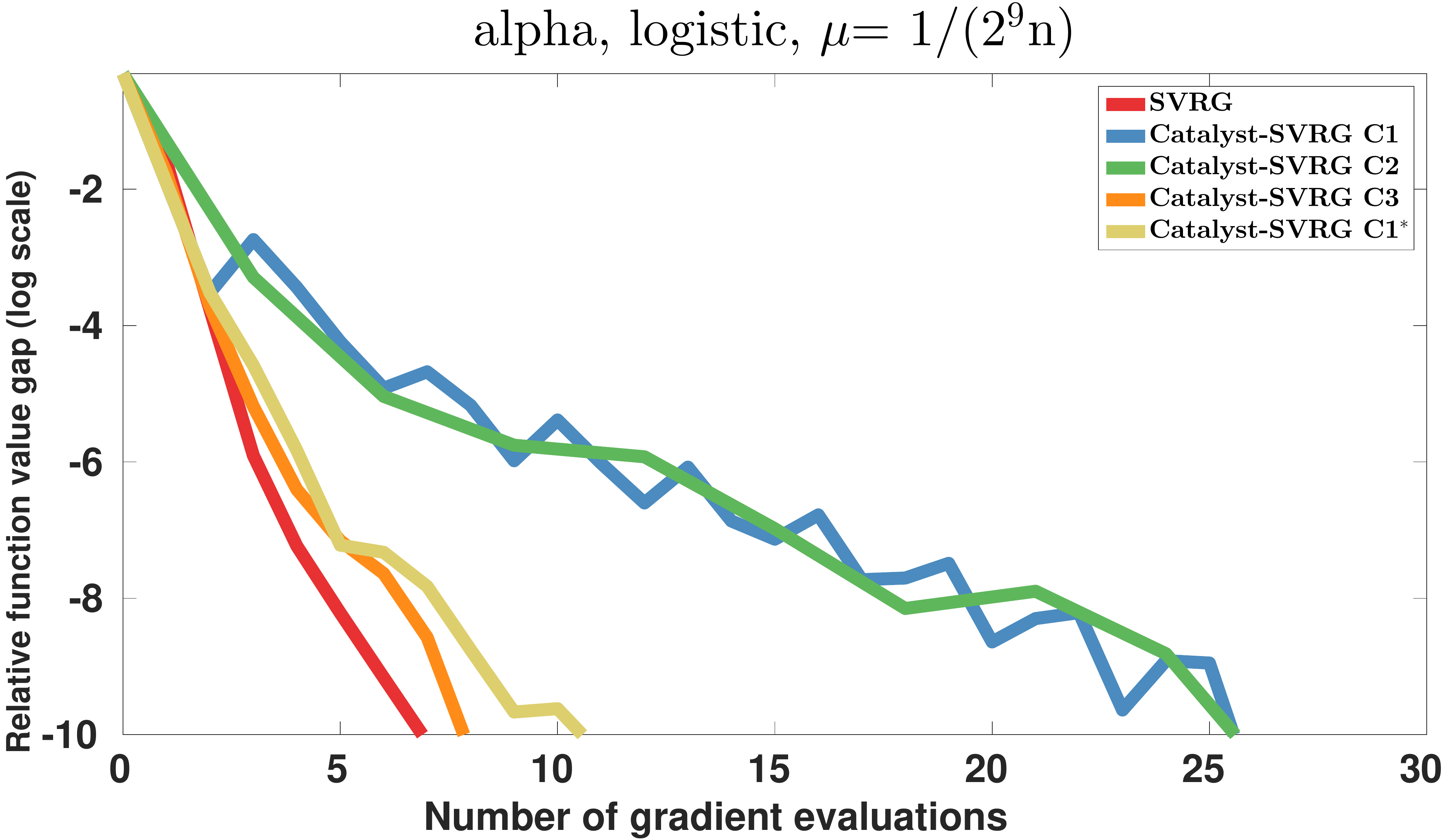}~ 
   ~~\includegraphics[width=0.31\linewidth]{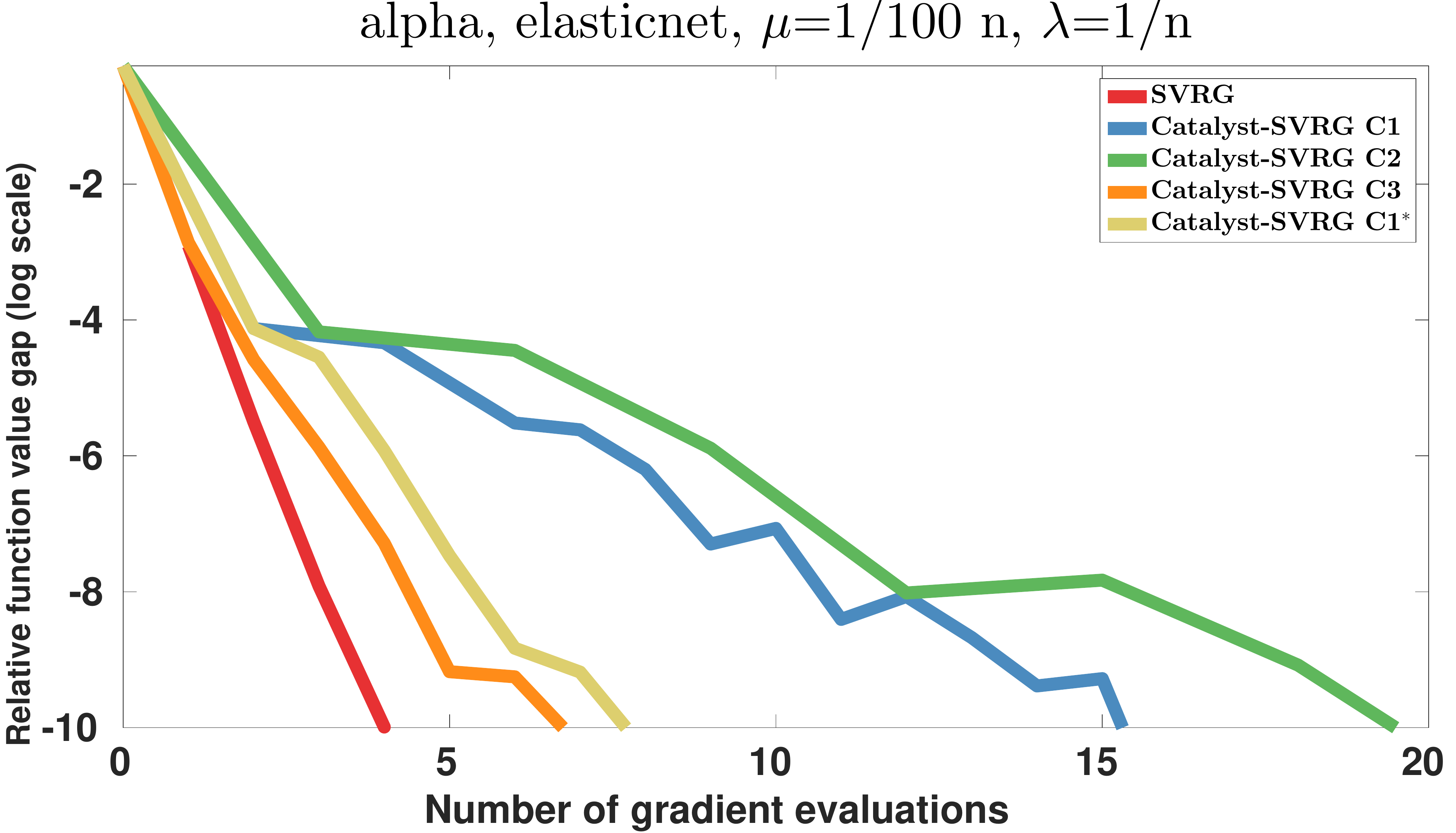}~ 
   ~~\includegraphics[width=0.31\linewidth]{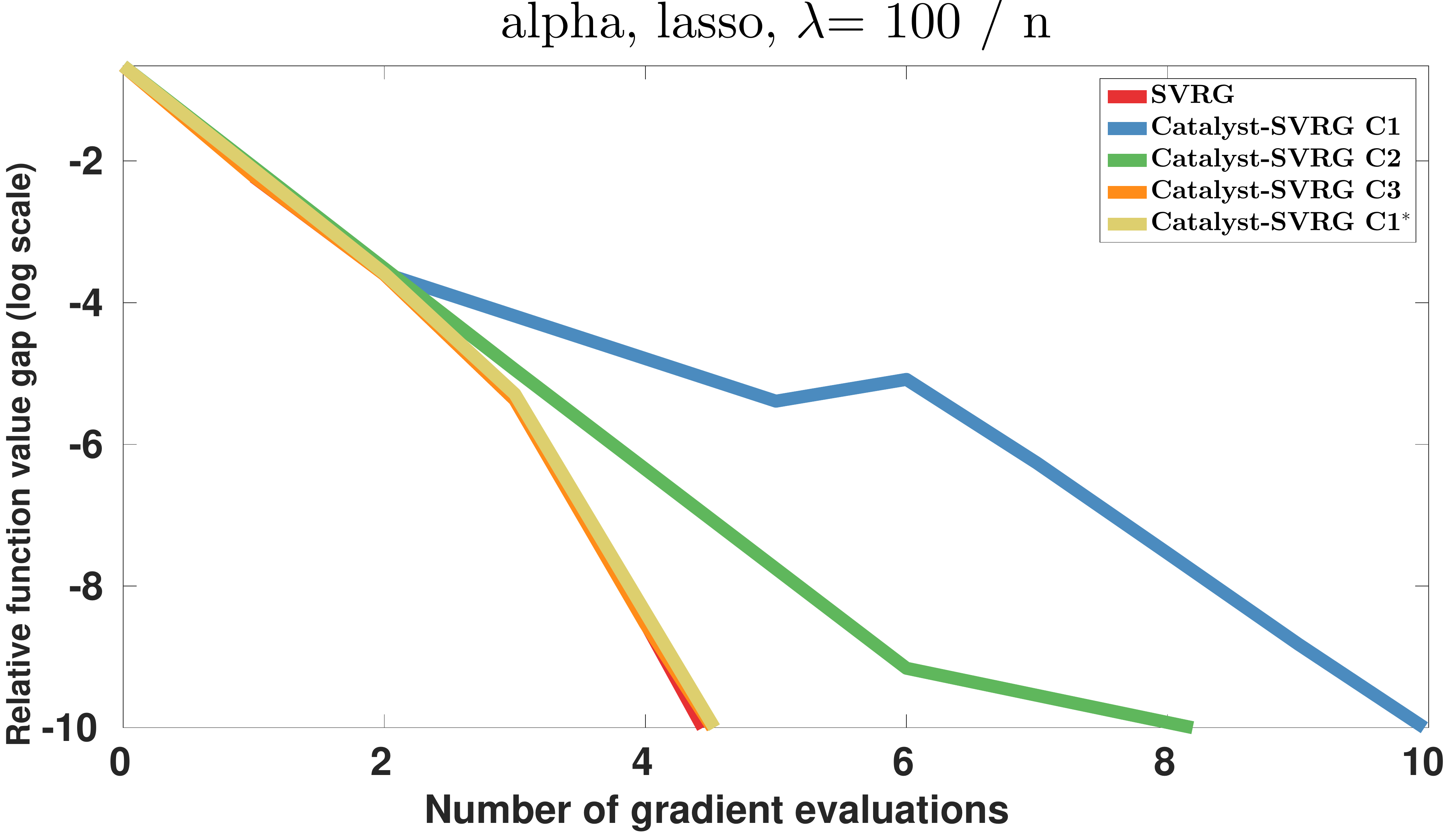}\\
   ~~\includegraphics[width=0.31\linewidth]{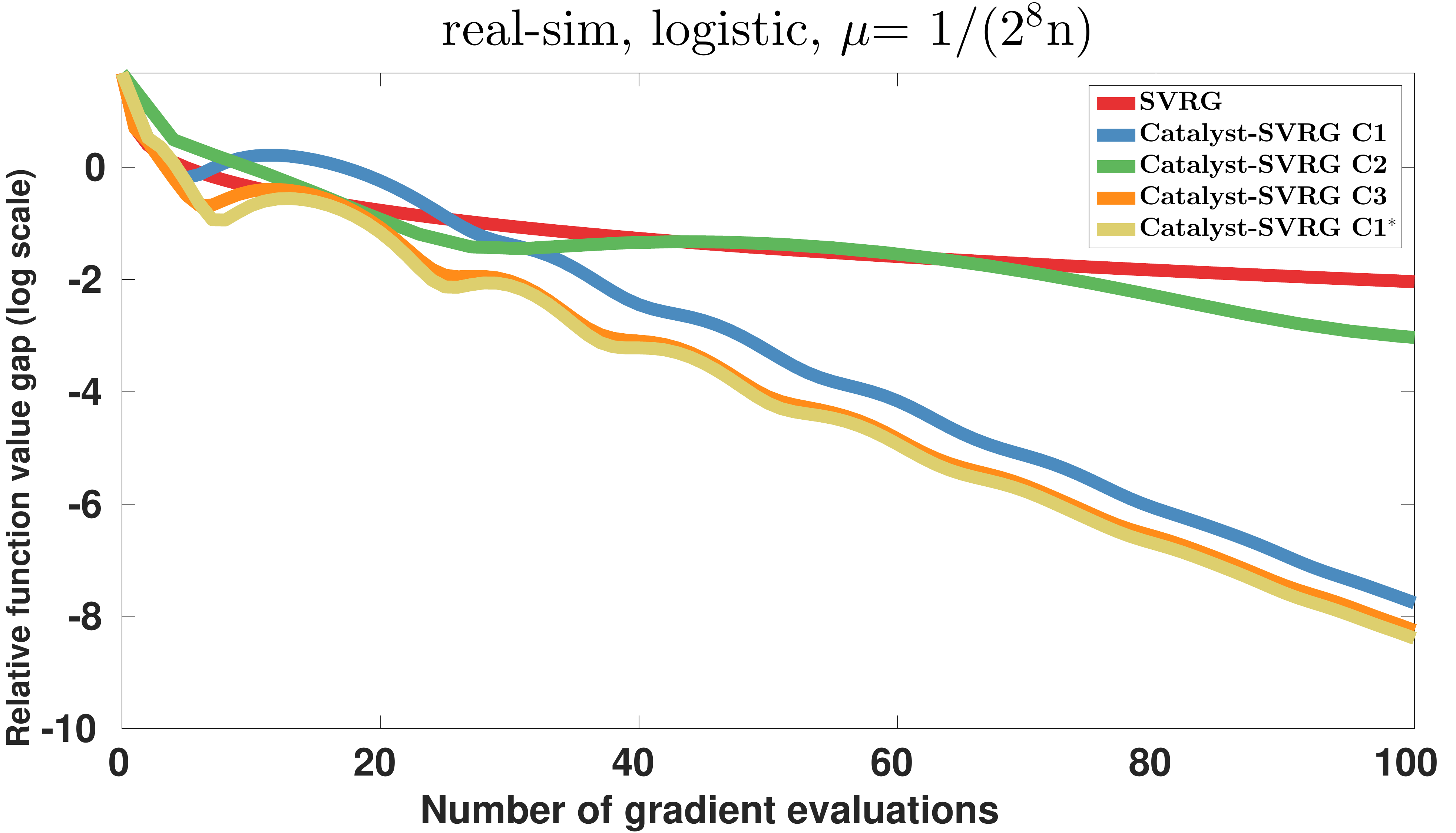}~ 
   ~~\includegraphics[width=0.31\linewidth]{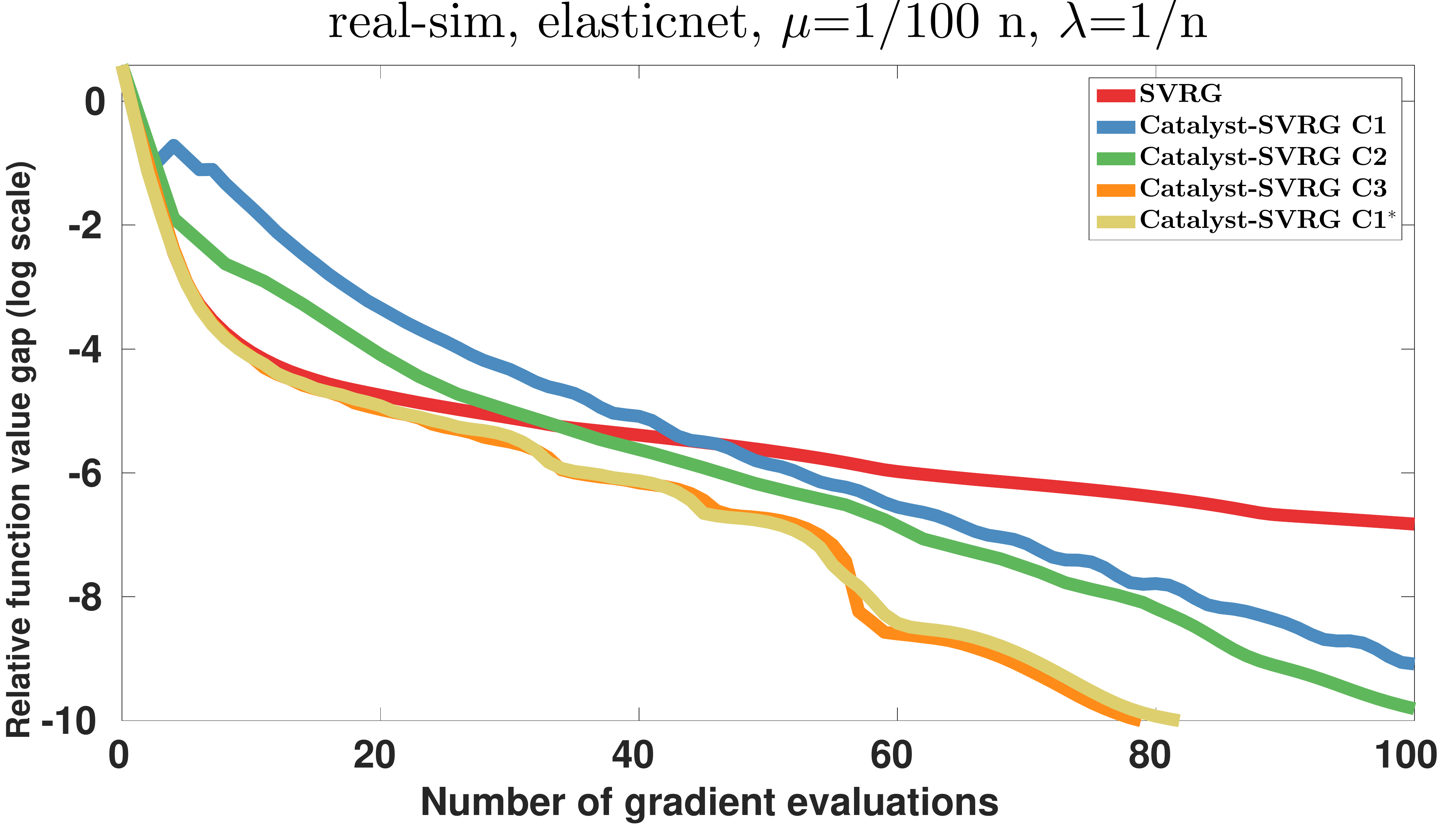}~ 
   ~~\includegraphics[width=0.31\linewidth]{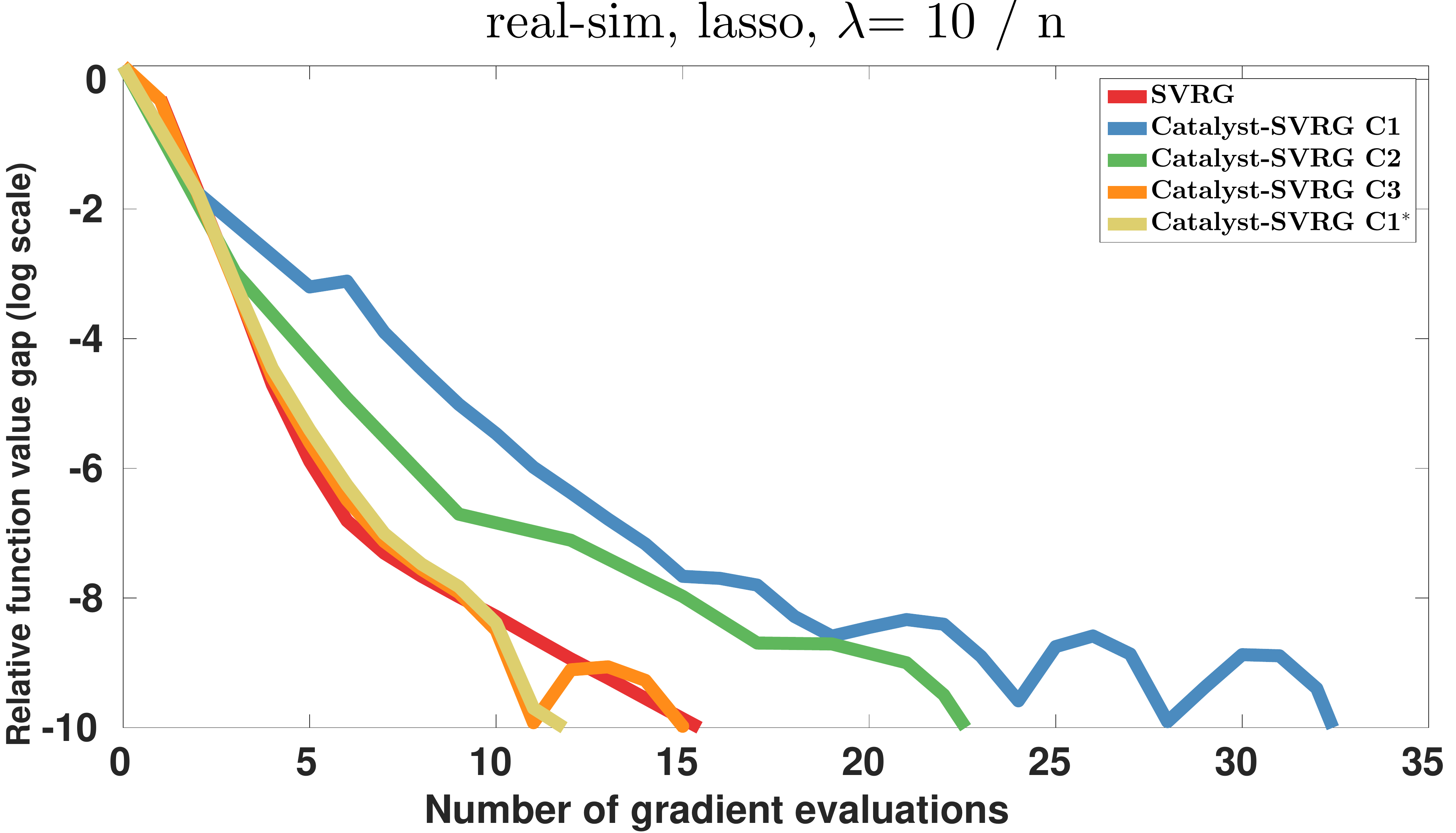}\\
   ~~\includegraphics[width=0.31\linewidth]{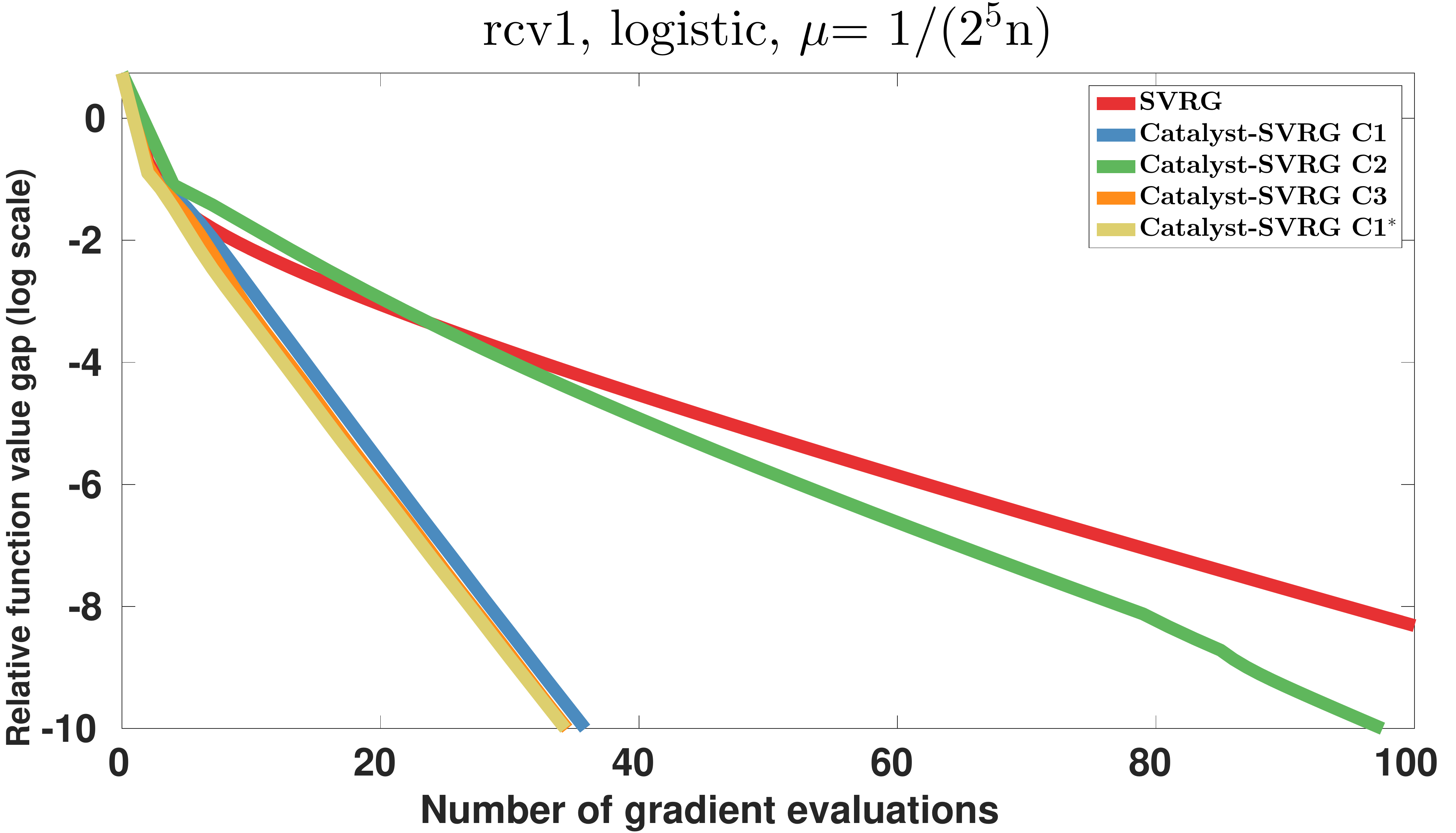}~ 
   ~~\includegraphics[width=0.31\linewidth]{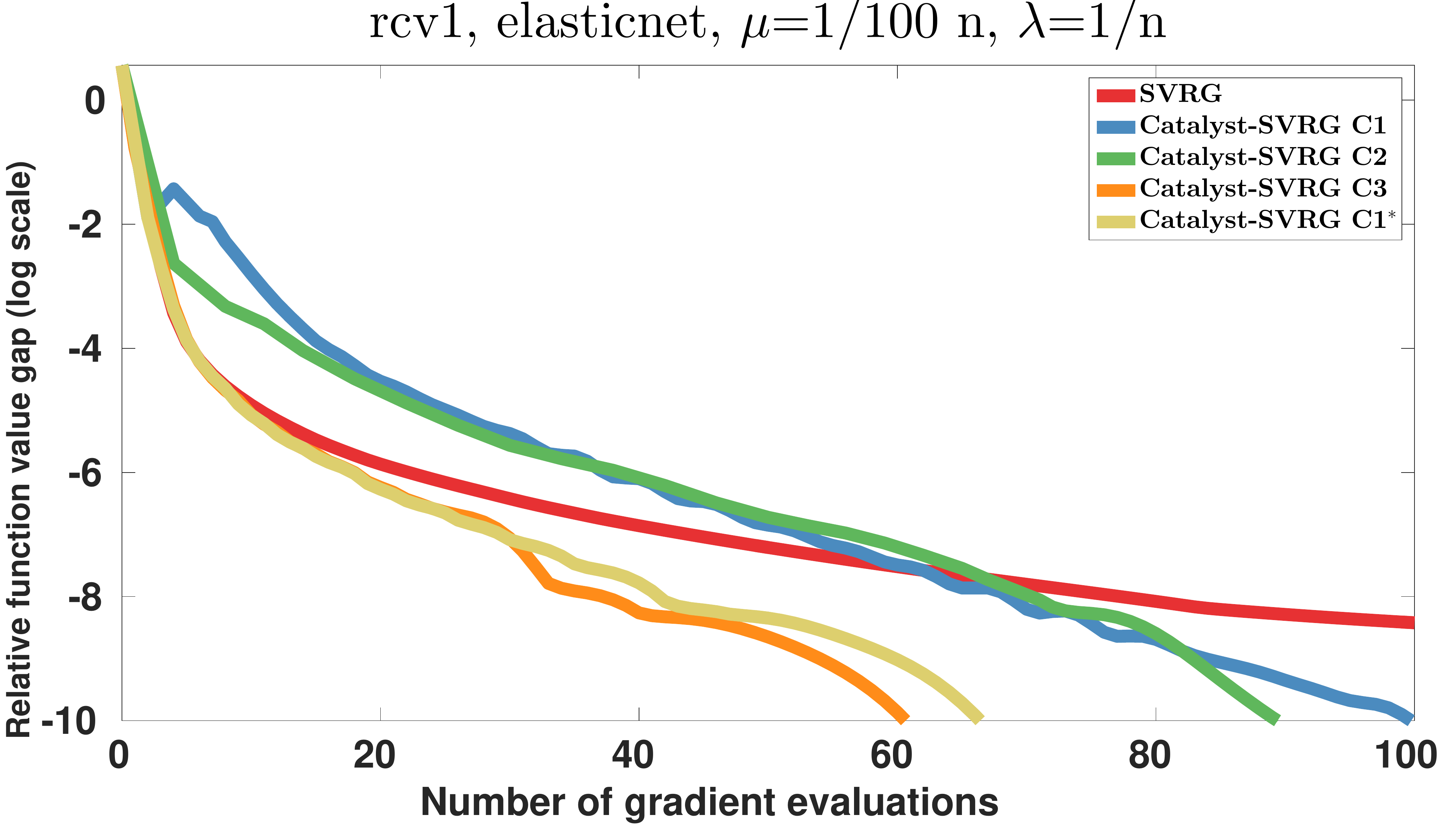}~ 
   ~~\includegraphics[width=0.31\linewidth]{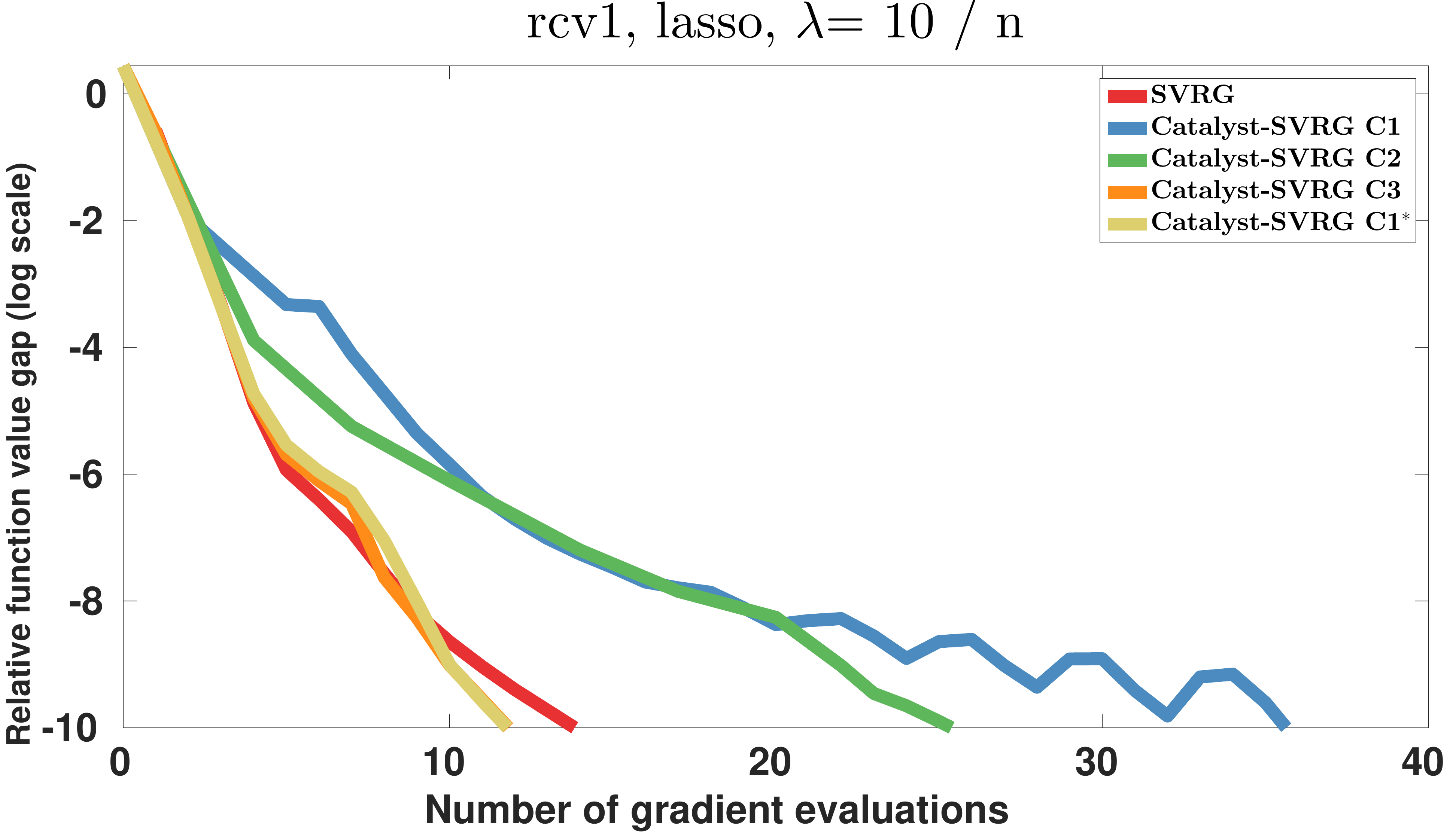}\\
   ~~\includegraphics[width=0.31\linewidth]{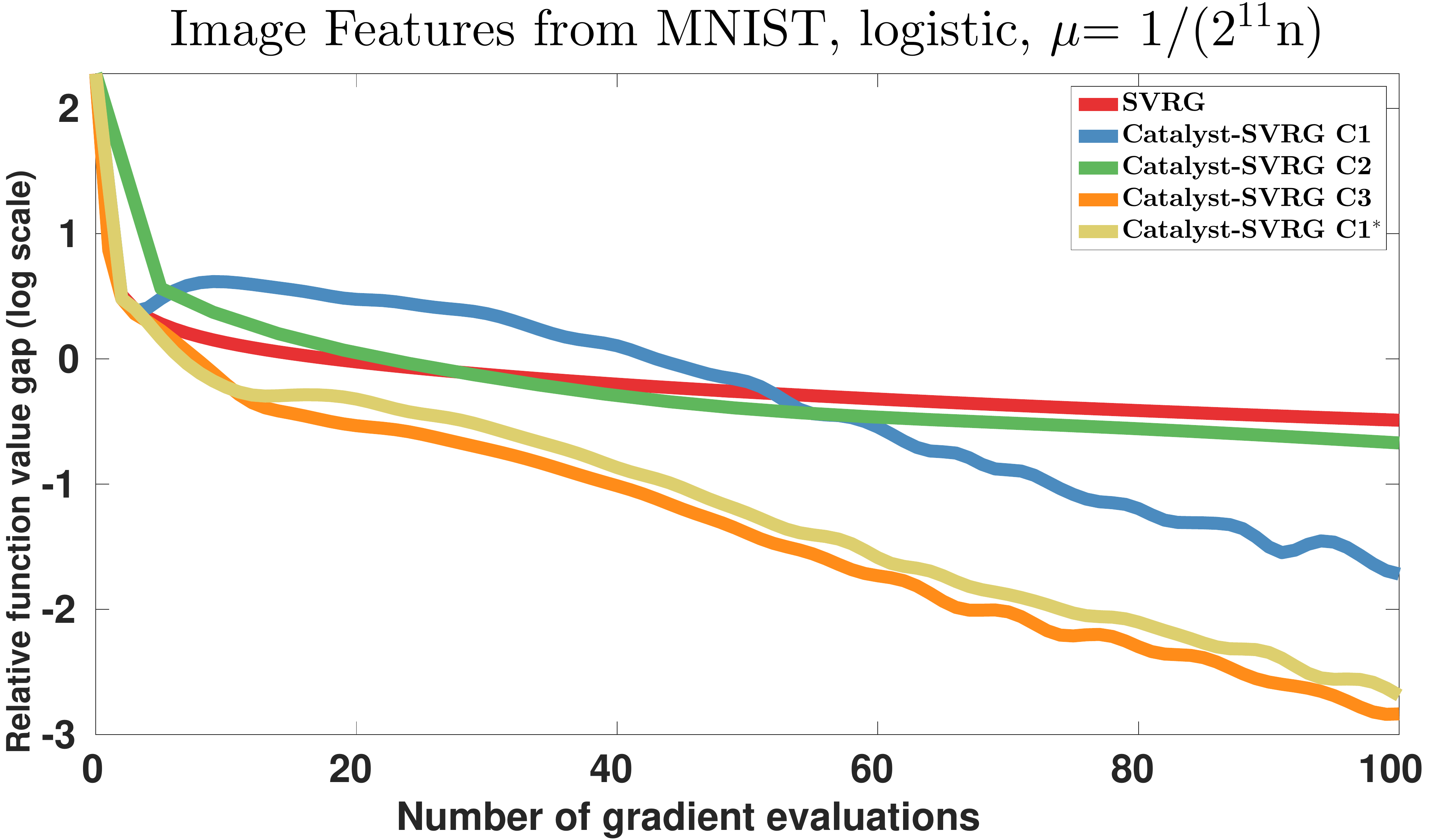}~ 
   ~~\includegraphics[width=0.31\linewidth]{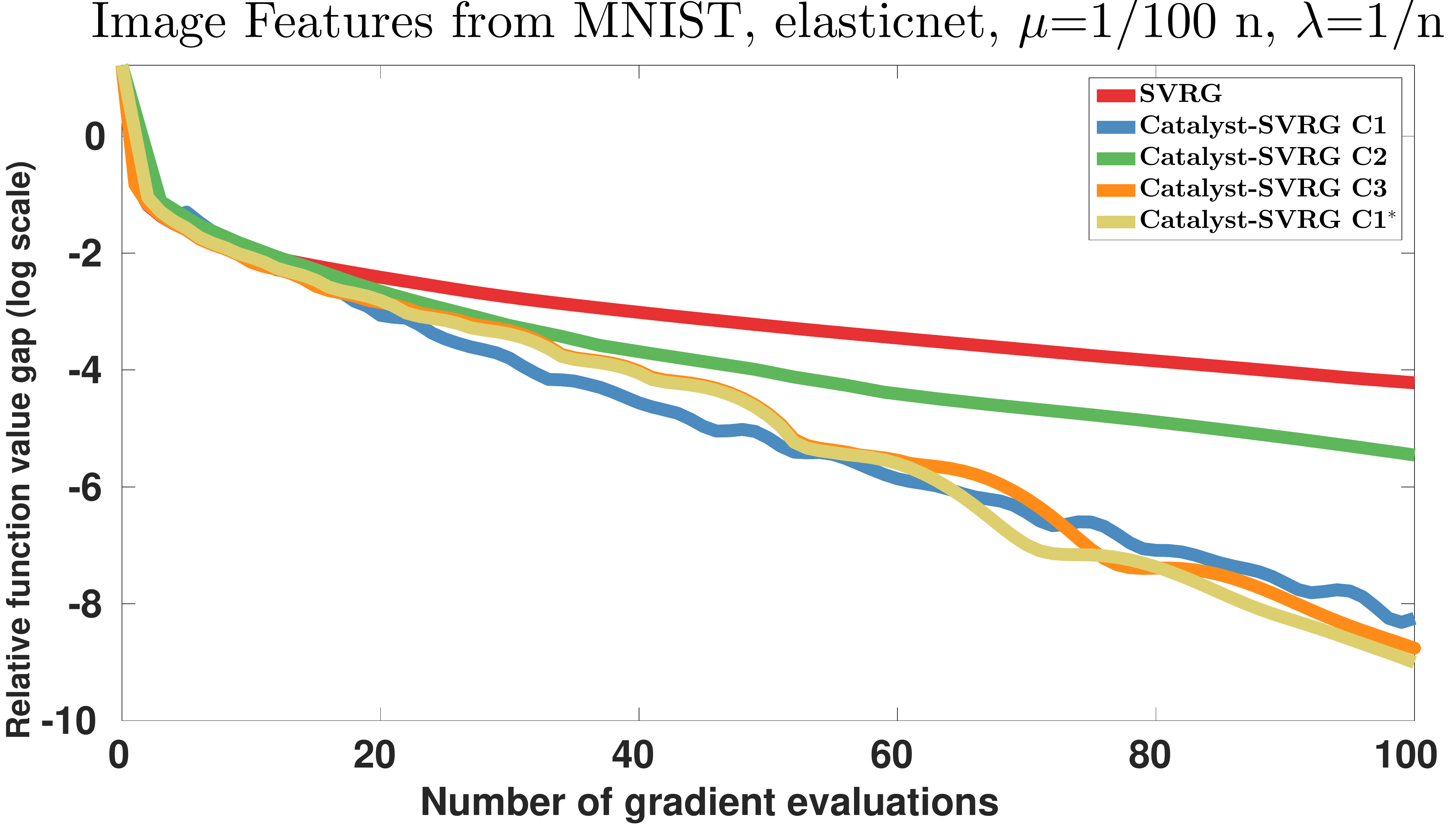}~ 
   ~~\includegraphics[width=0.31\linewidth]{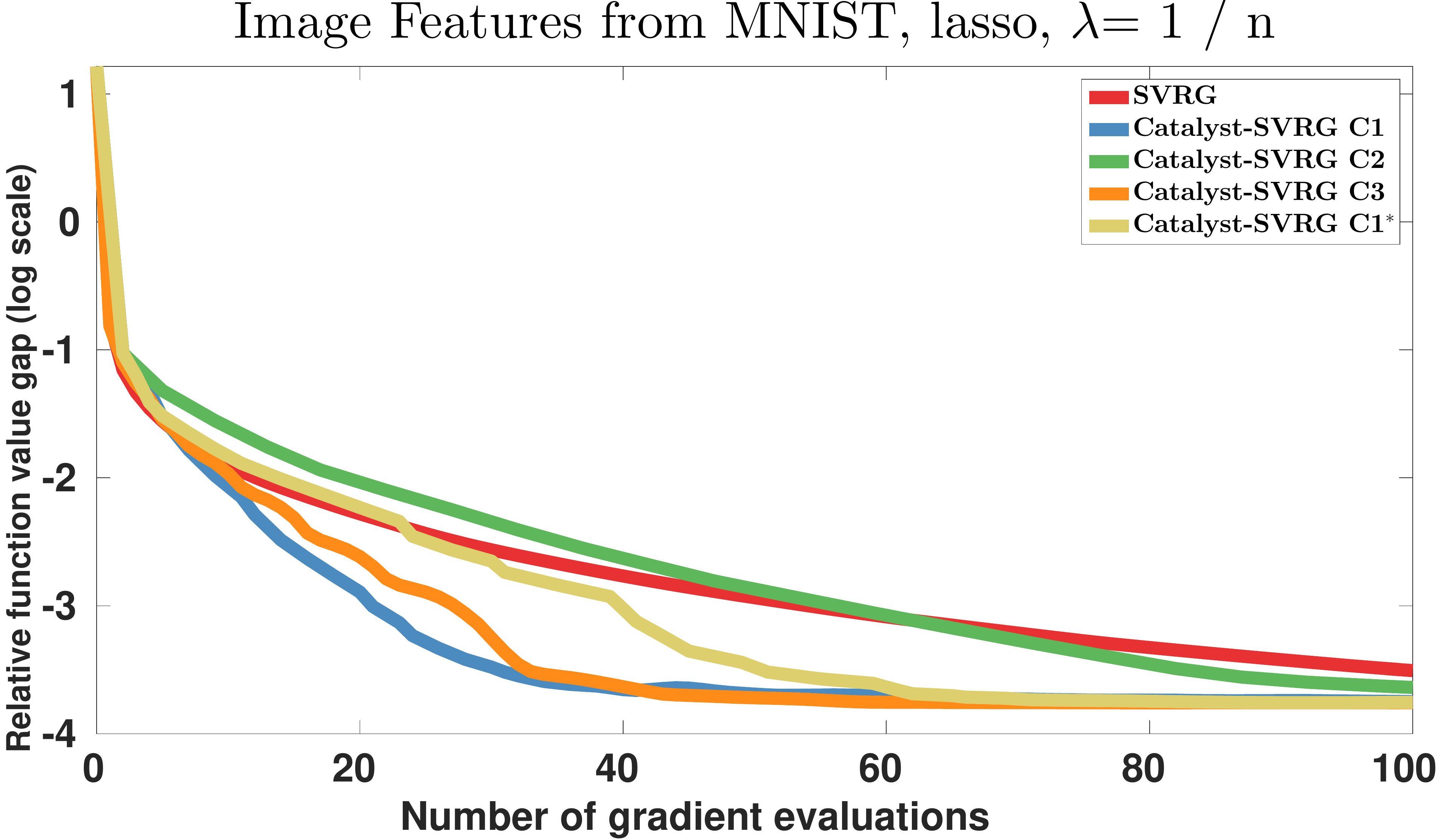}\\
   ~~\includegraphics[width=0.31\linewidth]{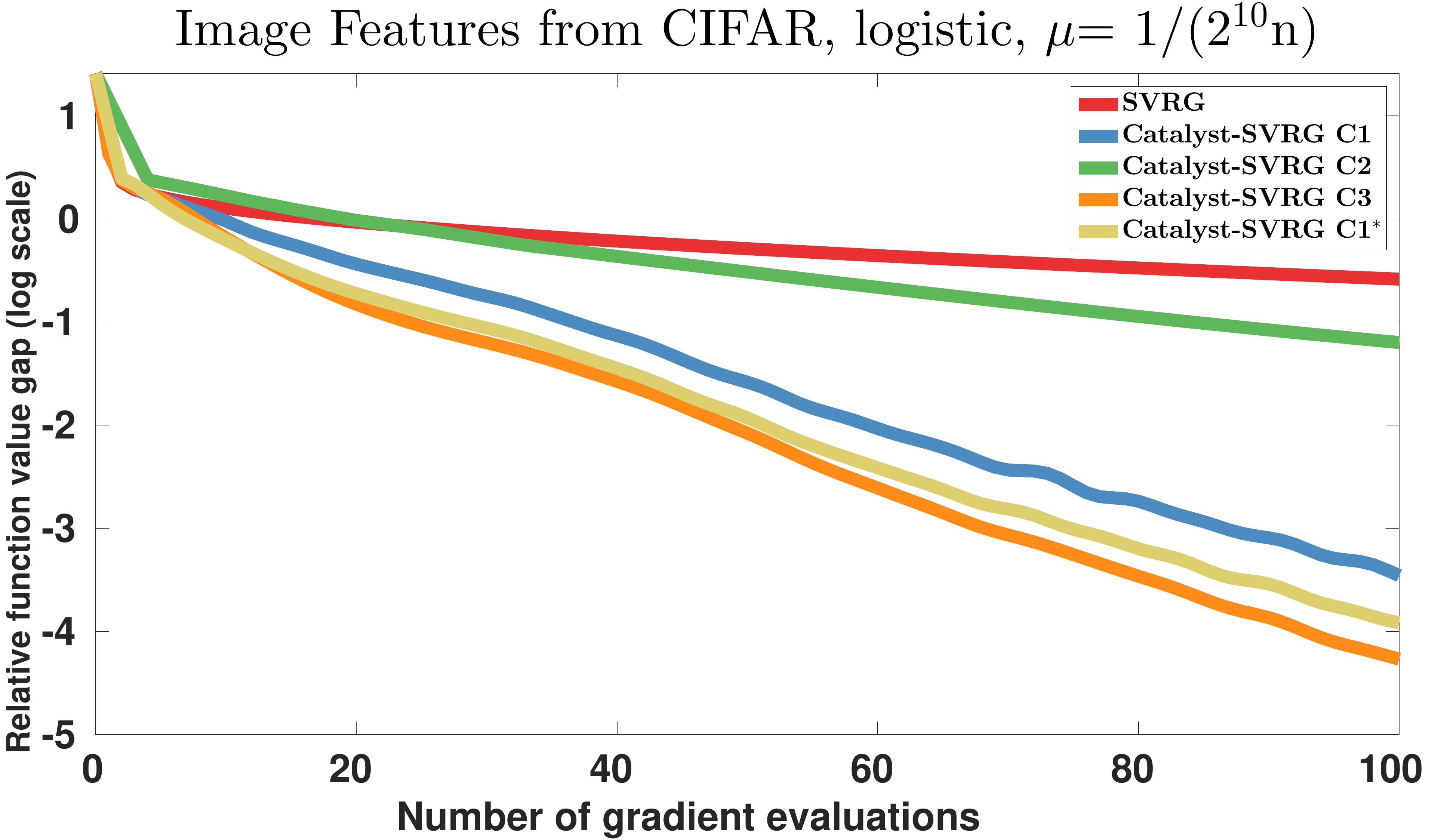}~ 
   ~~\includegraphics[width=0.31\linewidth]{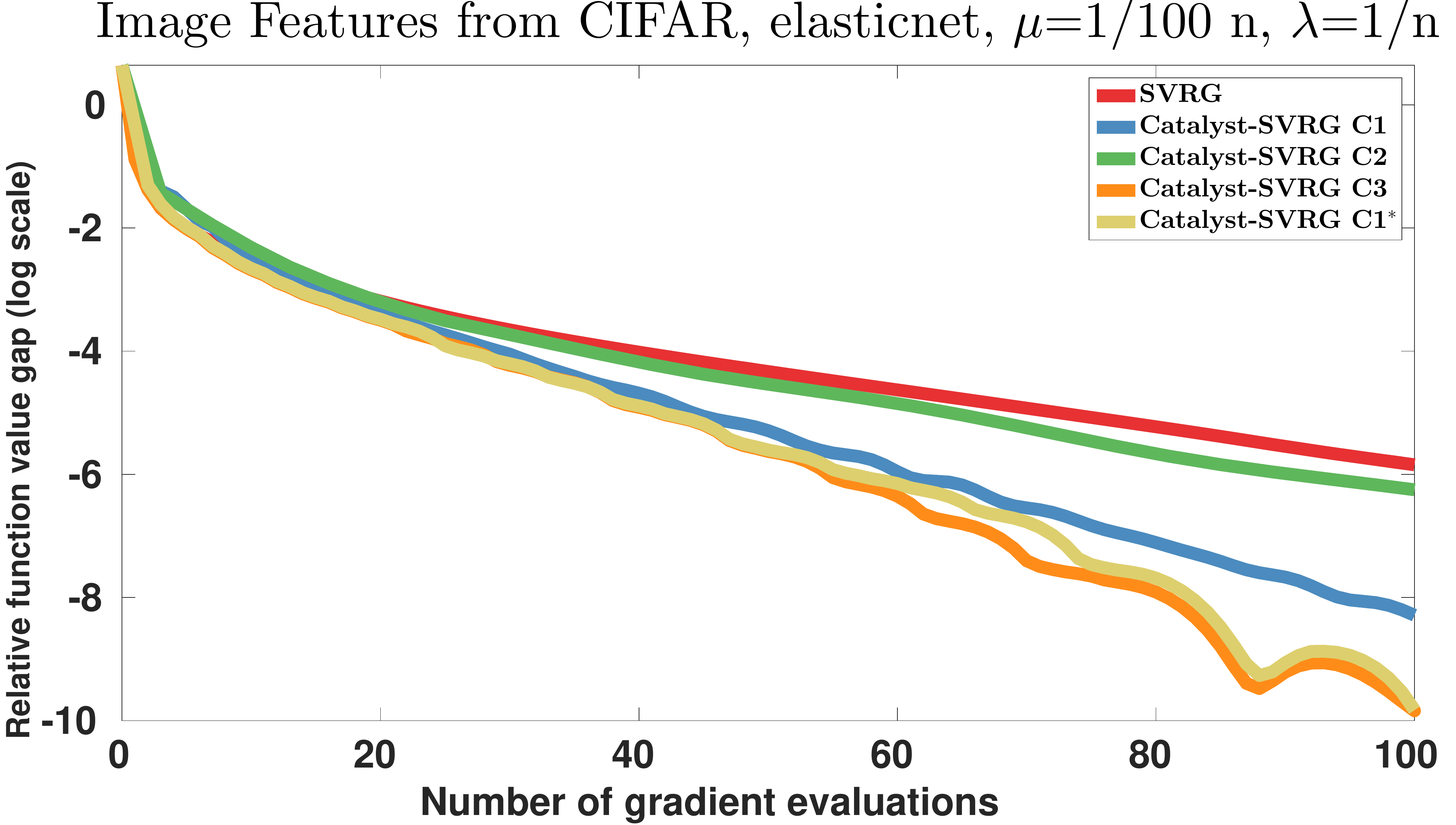}~ 
   ~~\includegraphics[width=0.31\linewidth]{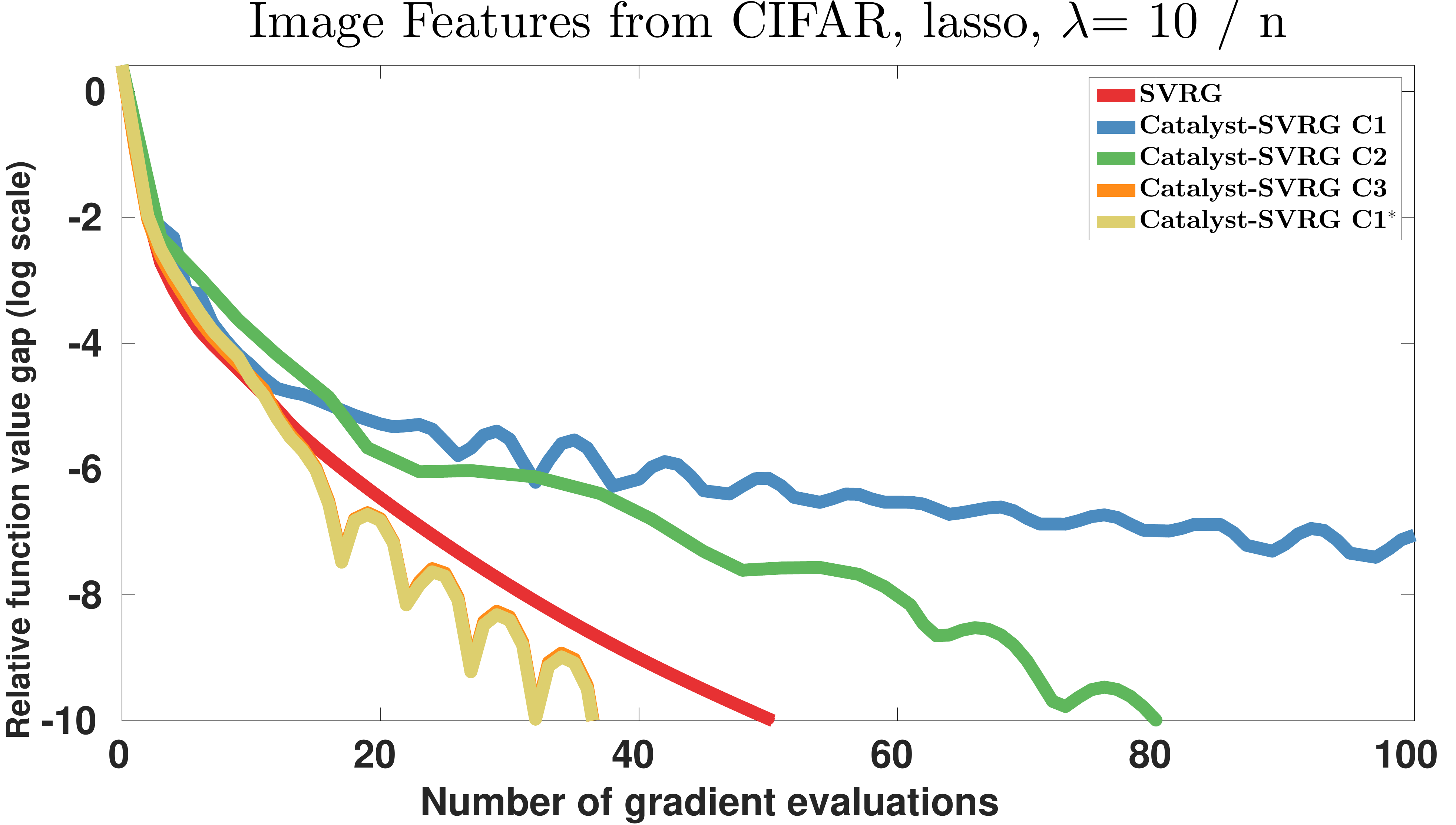}
   \caption{Experimental study of different stopping criterions for Catalyst-SVRG. We plot the value~$f(x_k)/f^\star-1$ as
   a function of the number of gradient evaluations, on a logarithmic scale;
   the optimal value $f^\star$ is estimated with a duality gap. }\label{catalyst:fig:svrg}
\end{figure}

\paragraph{Observations for Catalyst-SVRG.} We remark that in most of the
cases, the curve of \Ctrois and \Cunstar are superimposed, meaning that one
pass through the data is enough for solving the subproblem up to the required
accuracy. Moreover, they give the best performance among all criterions.
Regarding the logistic regression problem, the acceleration is significant
(even huge for the covtype data set) except for alpha, where only \Ctrois and~\Cunstar do not degrade significantly the performance.  For sparse problems,
the effect of acceleration is more mitigated, with 7 cases
out of 12 exhibiting important acceleration and 5 cases no acceleration. As
before, \Ctrois and \Cunstar are the only strategies that never degrade
performance.

One reason explaining why acceleration is not systematic may be the ability of
incremental methods to adapt to the unknown strong convexity parameter~$\mu'
\geq \mu$ hidden in the objective's loss, or local strong convexity near the solution.
When~$\mu'/L \geq 1/n$, we indeed obtain a well-conditioned regime where
acceleration should not occur theoretically. In fact the complexity
$O(n\log(1/\varepsilon))$ is already optimal in this regime, see \cite{tightbound_yossi,tightbound_blake}.
For sparse problems, conditioning of the problem with respect to the linear
subspace where the solution lies might also play a role, even though our analysis
does not study this aspect.
 Therefore, this experiment suggests that
adaptivity to unknown strong convexity is of high interest for incremental
optimization.

\begin{figure}[hbtp]
   \centering
   ~~\includegraphics[width=0.31\linewidth]{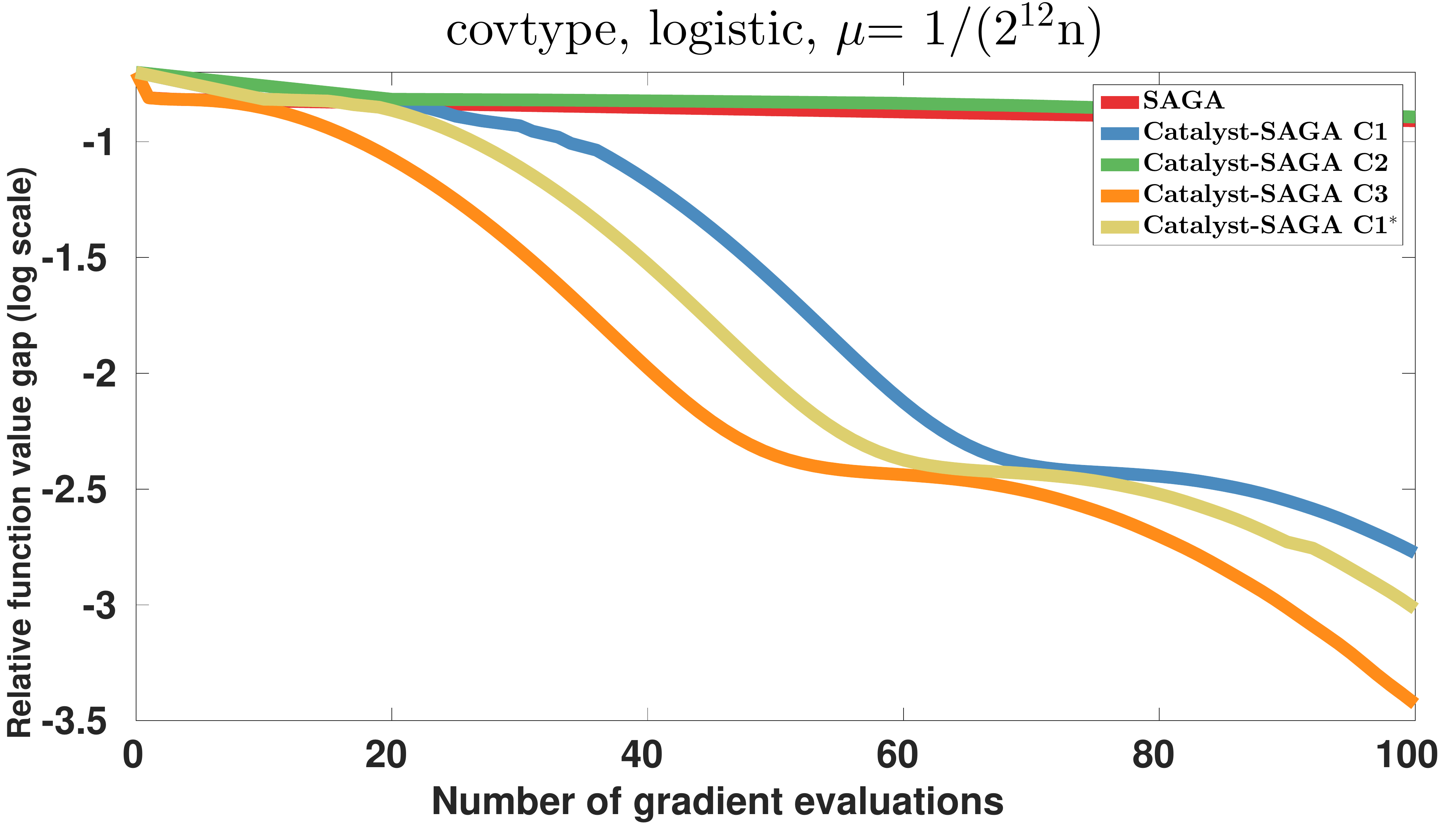}~ 
   ~~\includegraphics[width=0.31\linewidth]{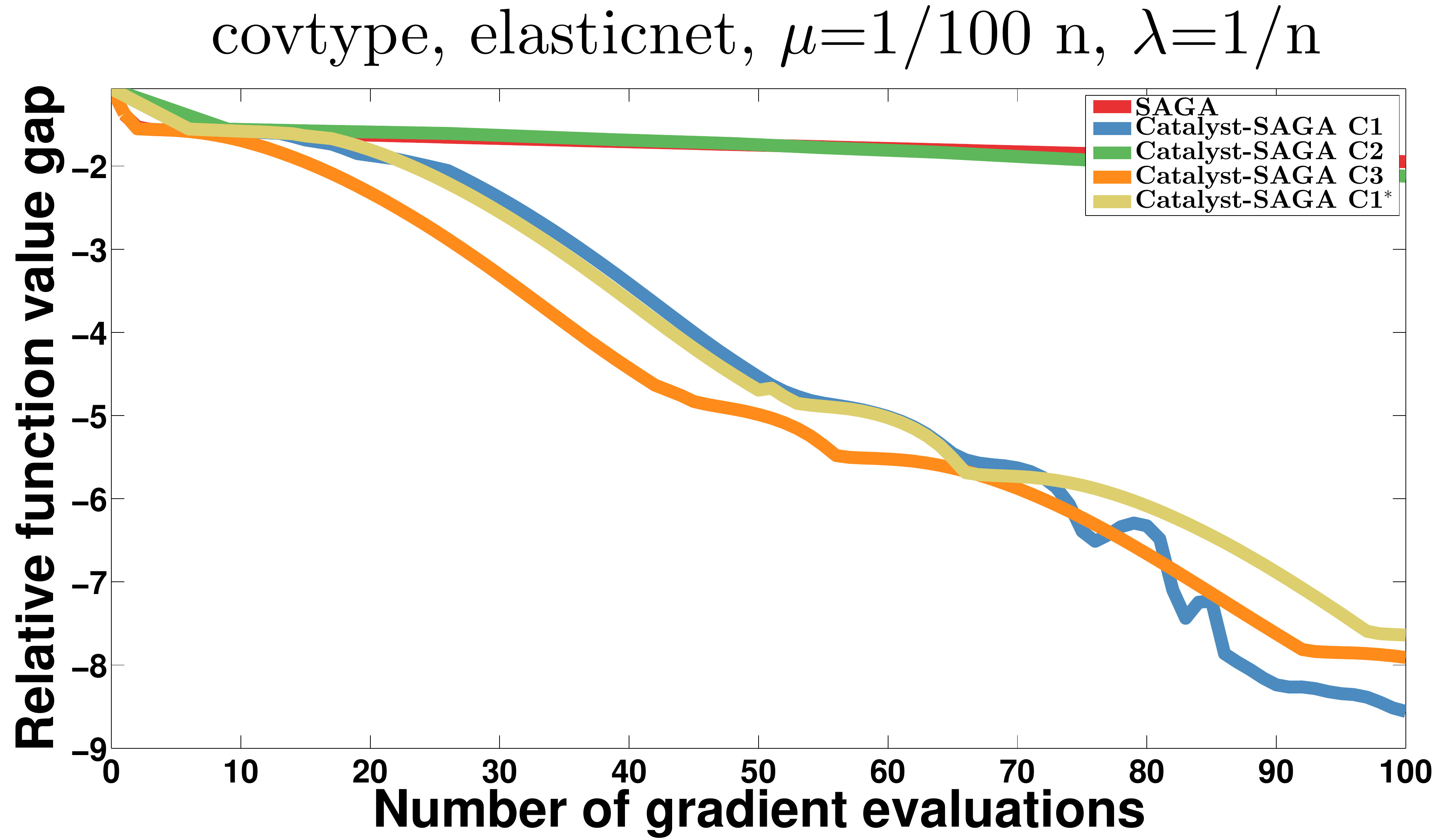}~ 
   ~~\includegraphics[width=0.31\linewidth]{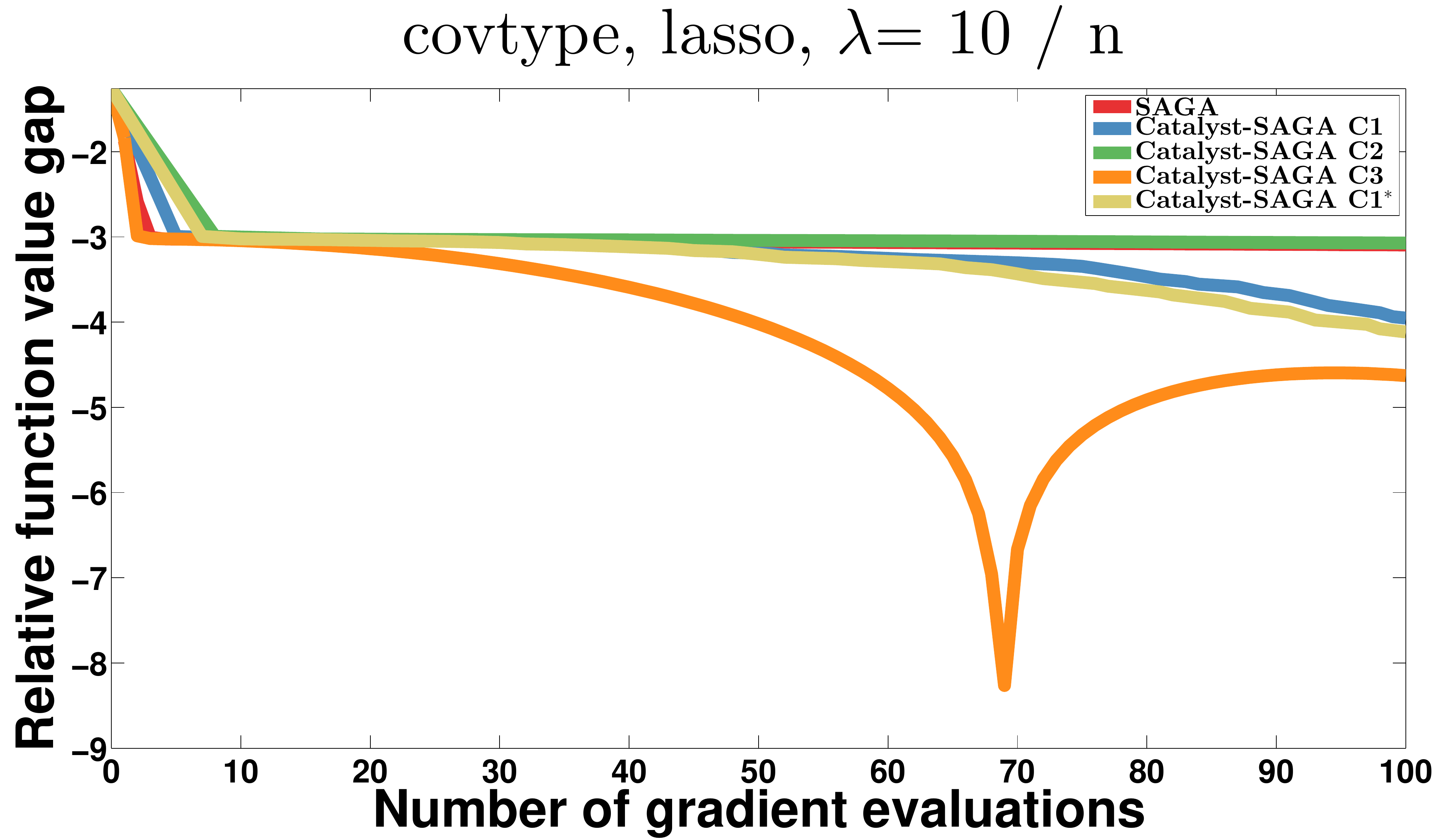}\\
   ~~\includegraphics[width=0.31\linewidth]{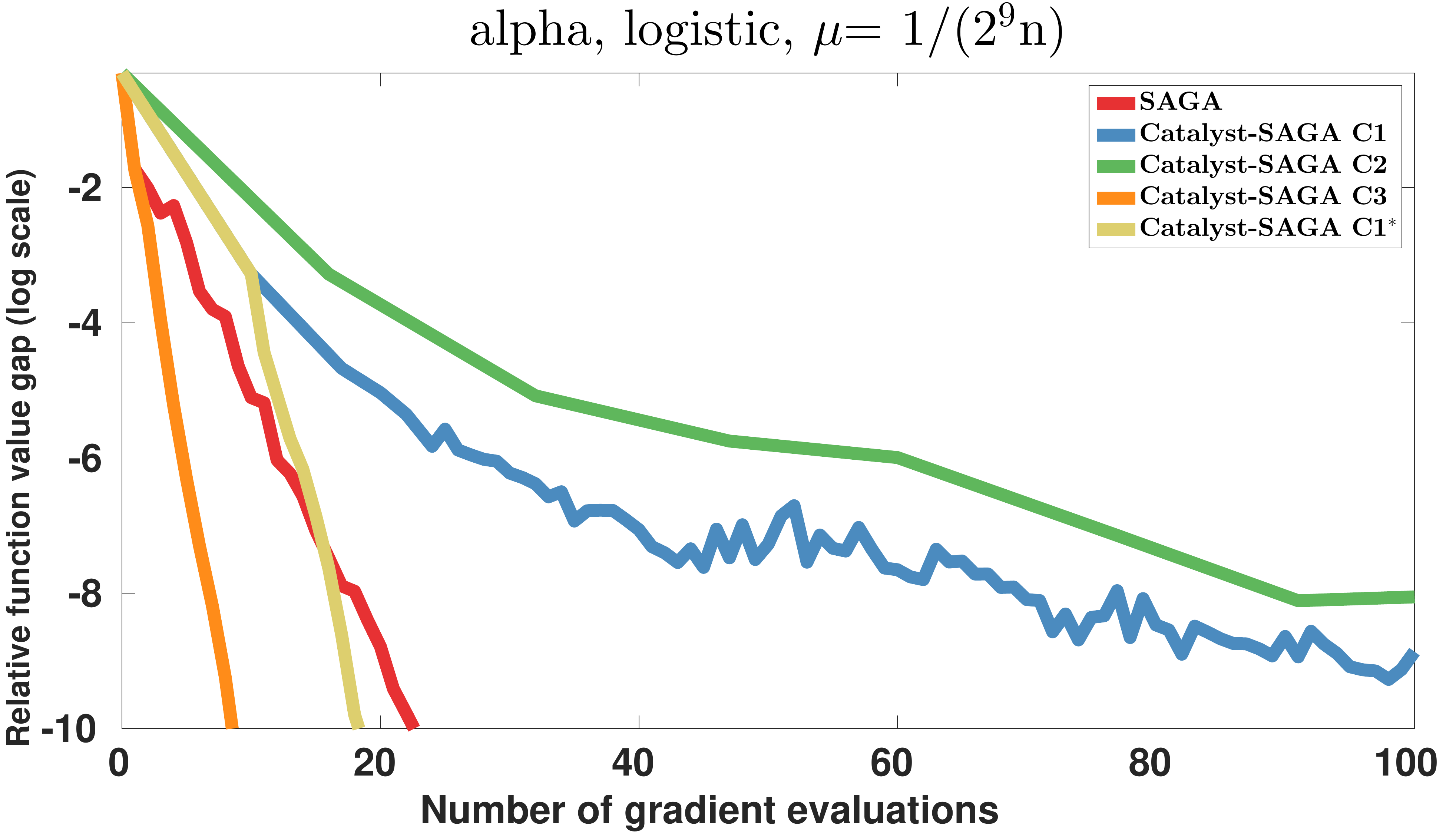}~ 
   ~~\includegraphics[width=0.31\linewidth]{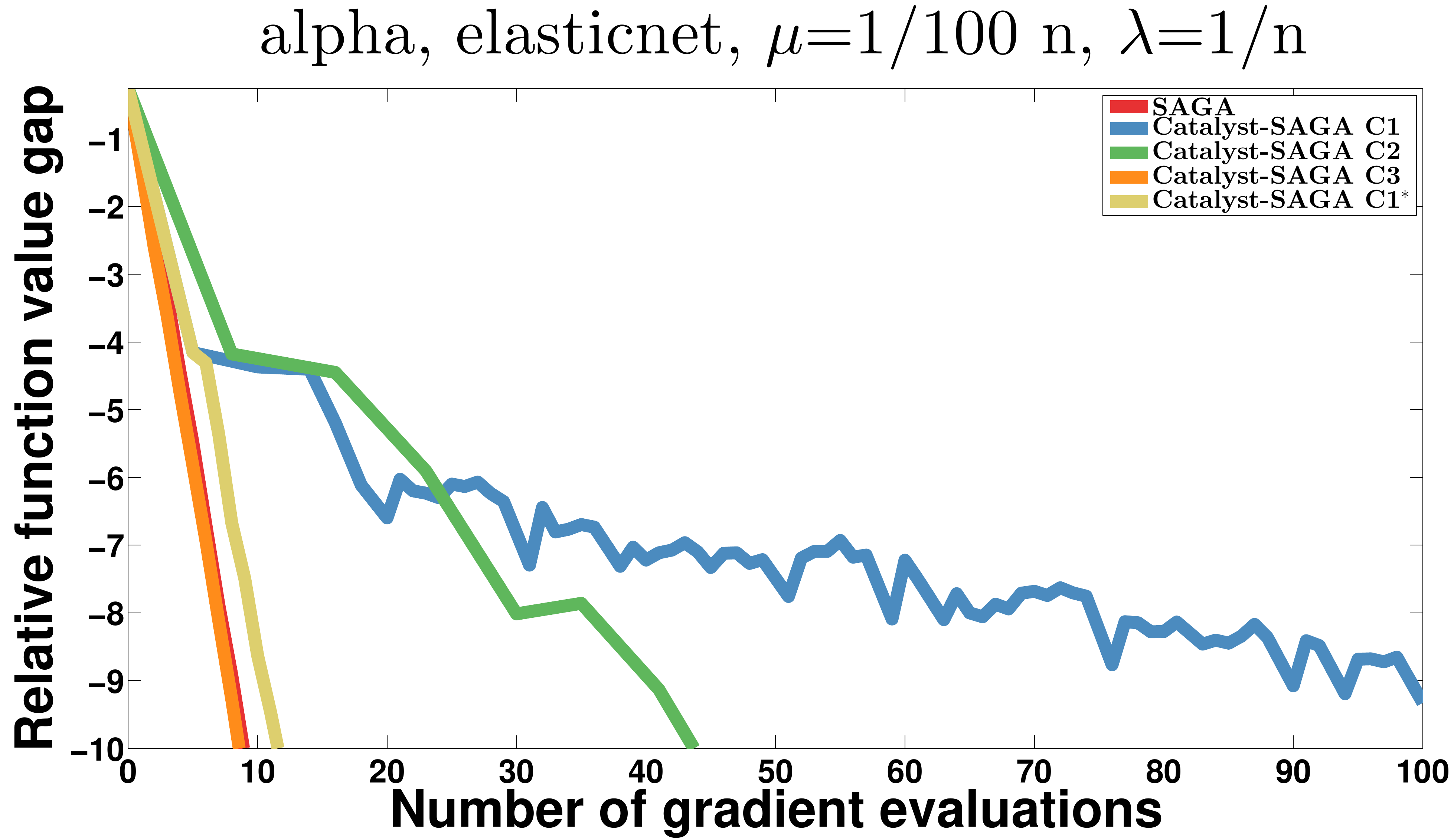}~ 
   ~~\includegraphics[width=0.31\linewidth]{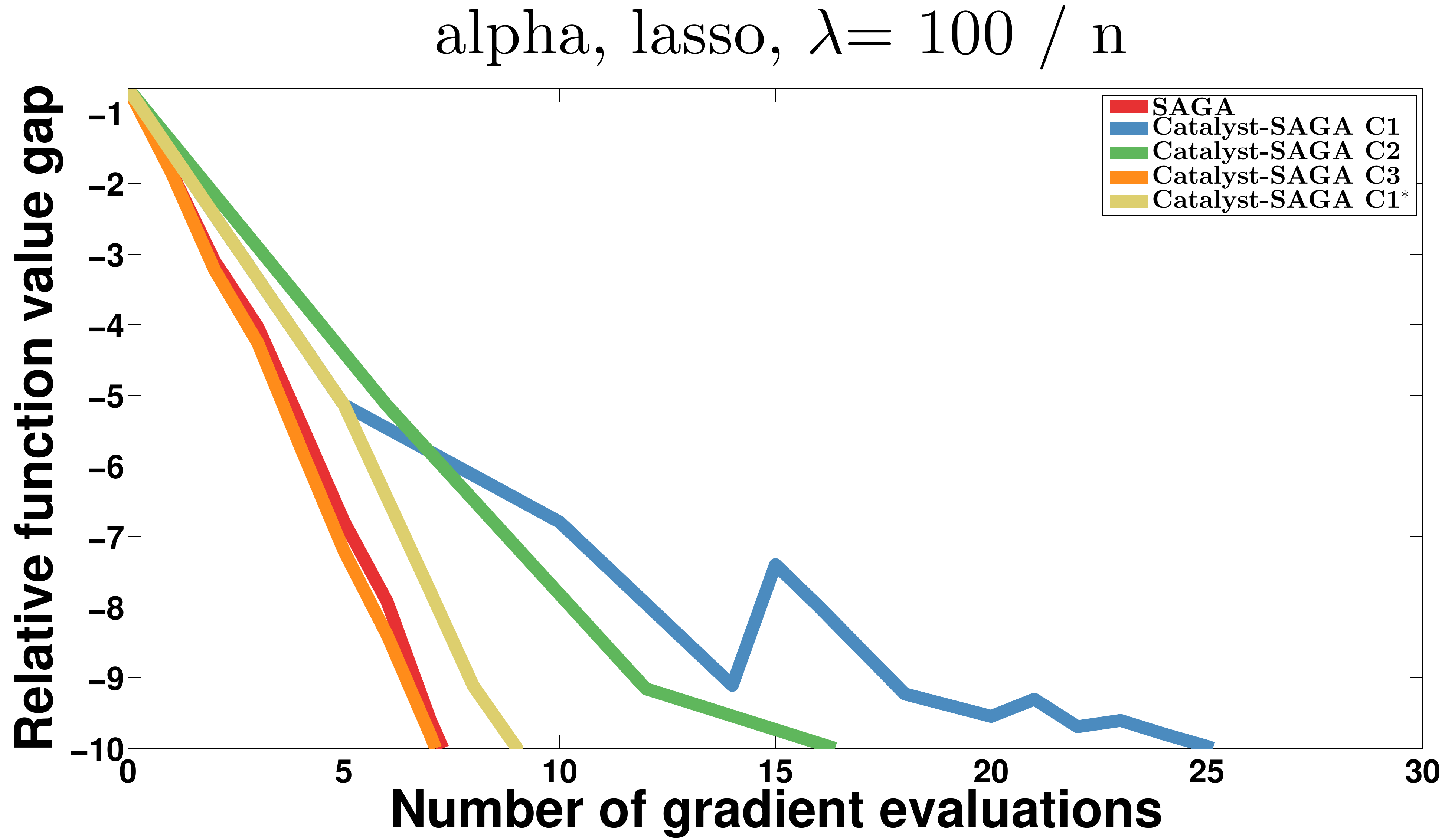}\\
   ~~\includegraphics[width=0.31\linewidth]{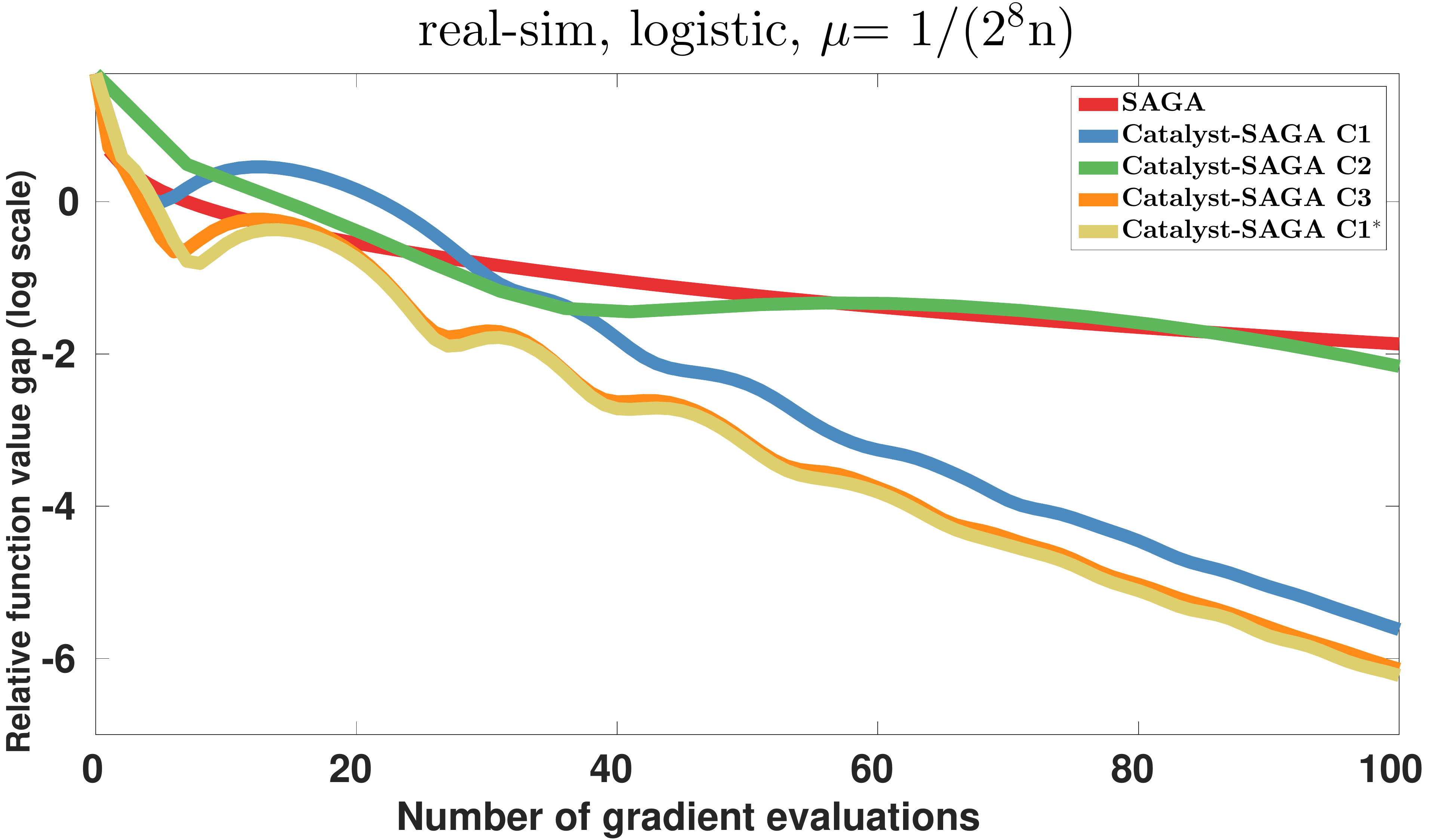}~ 
   ~~\includegraphics[width=0.31\linewidth]{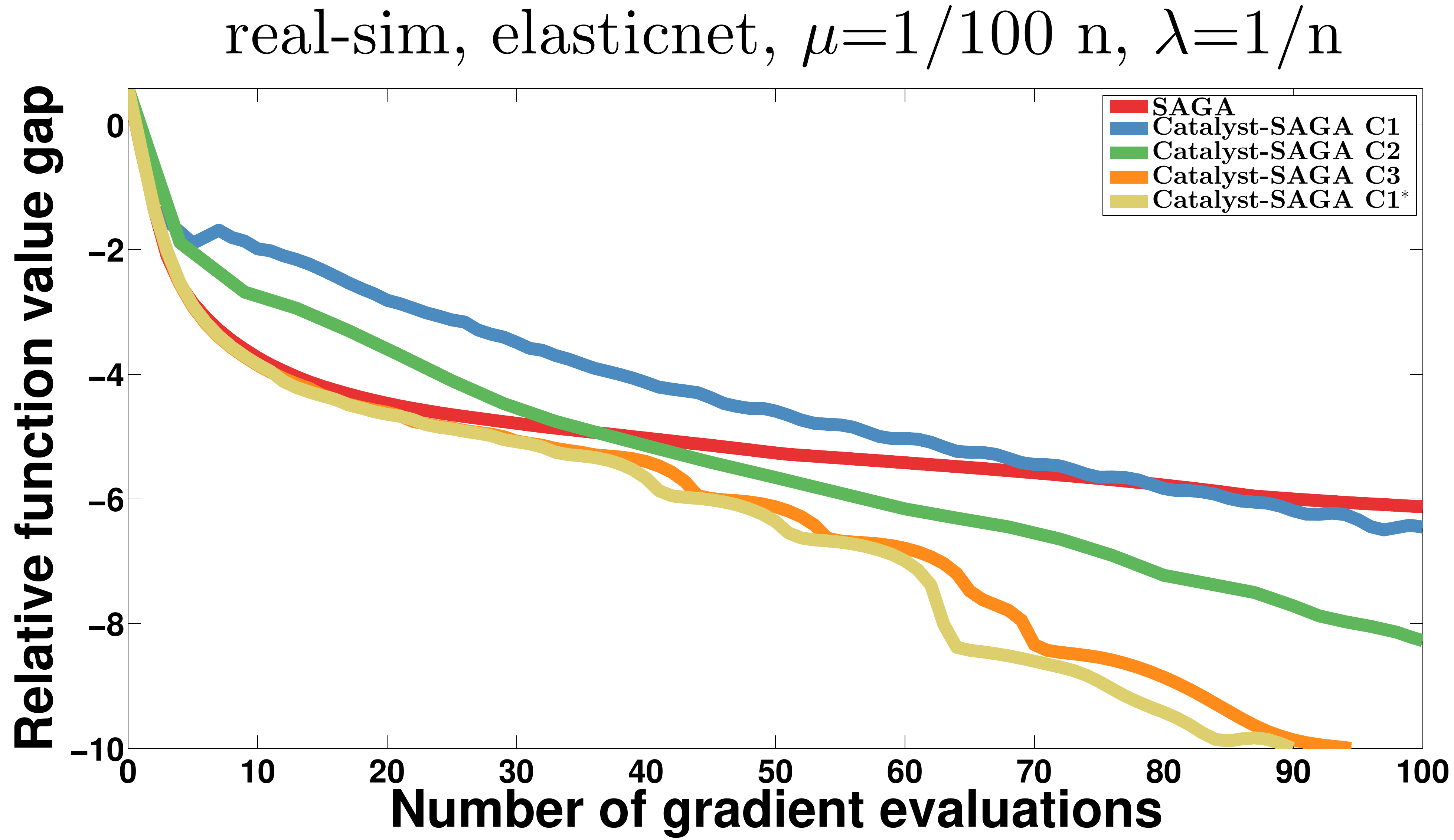}~ 
   ~~\includegraphics[width=0.31\linewidth]{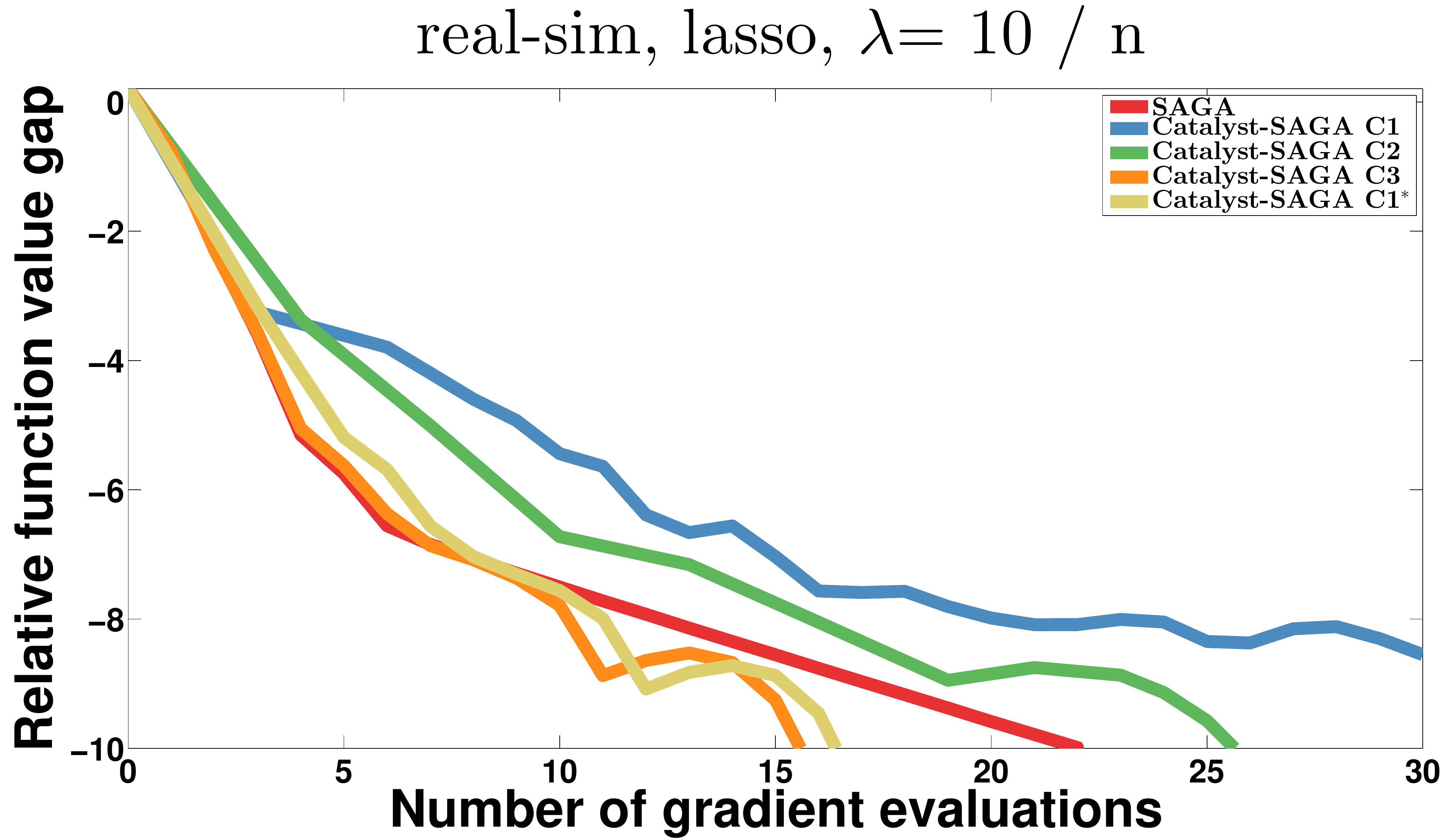}\\
   ~~\includegraphics[width=0.31\linewidth]{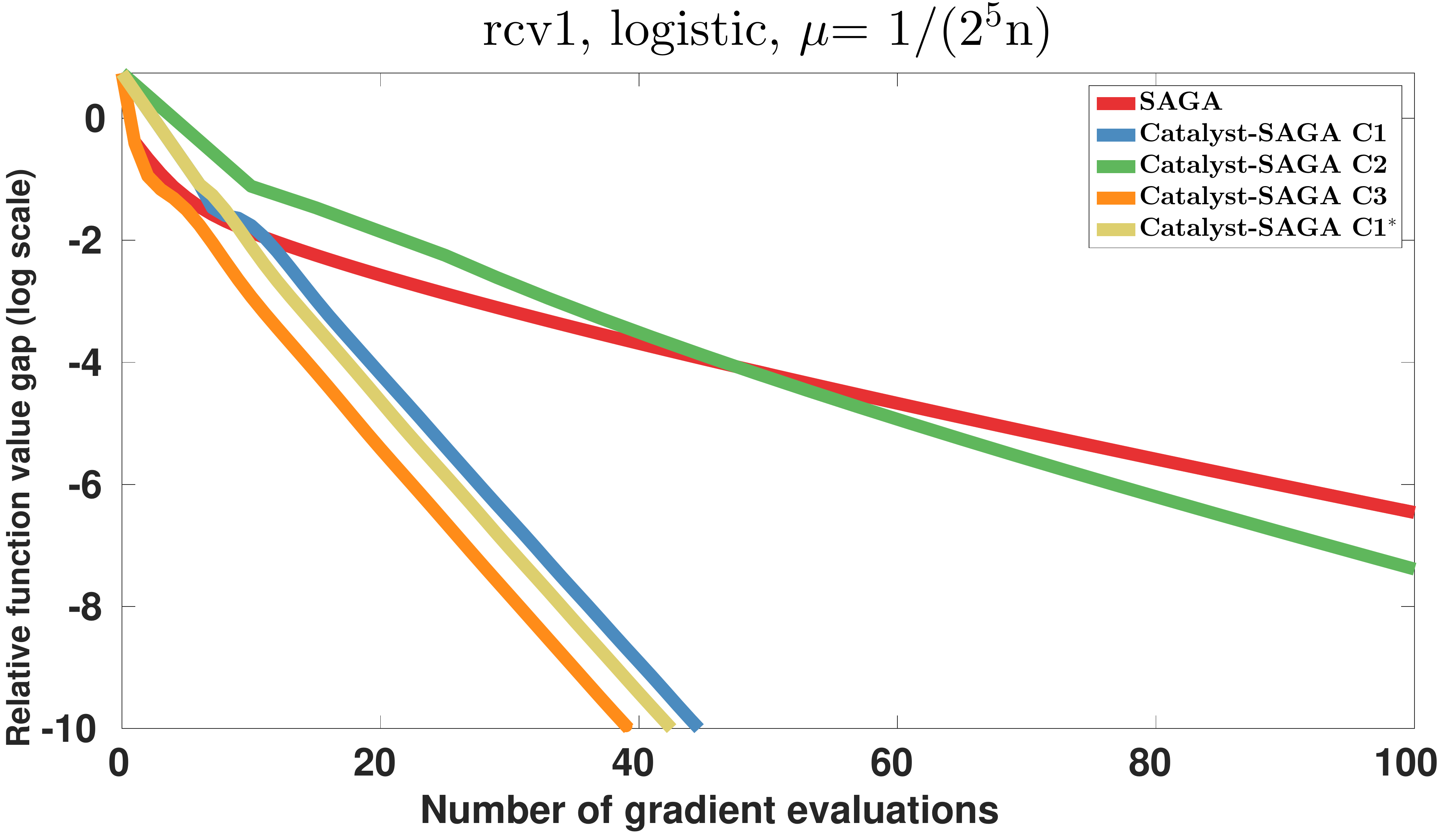}~ 
   ~~\includegraphics[width=0.31\linewidth]{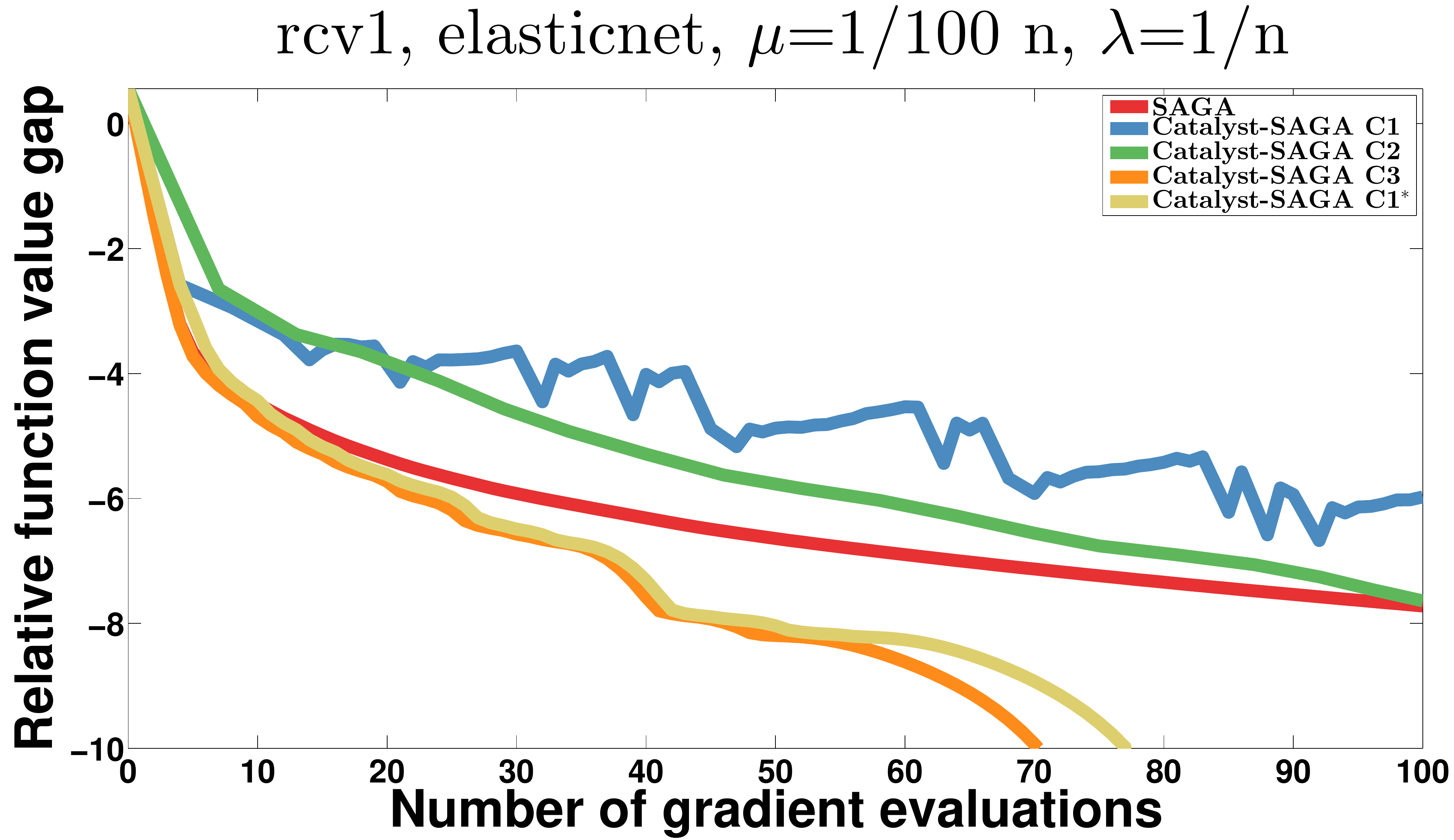}~ 
   ~~\includegraphics[width=0.31\linewidth]{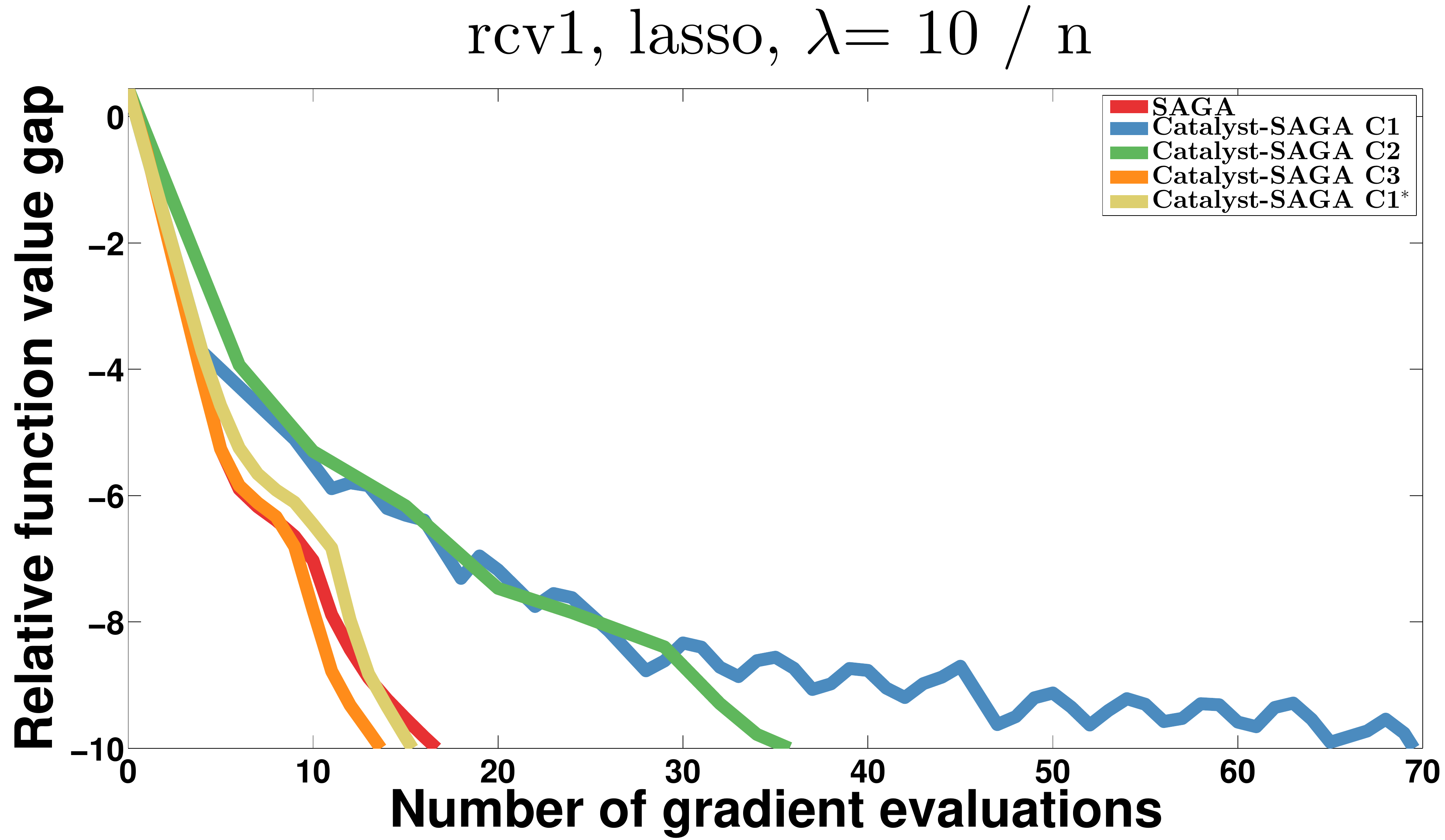}\\
   ~~\includegraphics[width=0.31\linewidth]{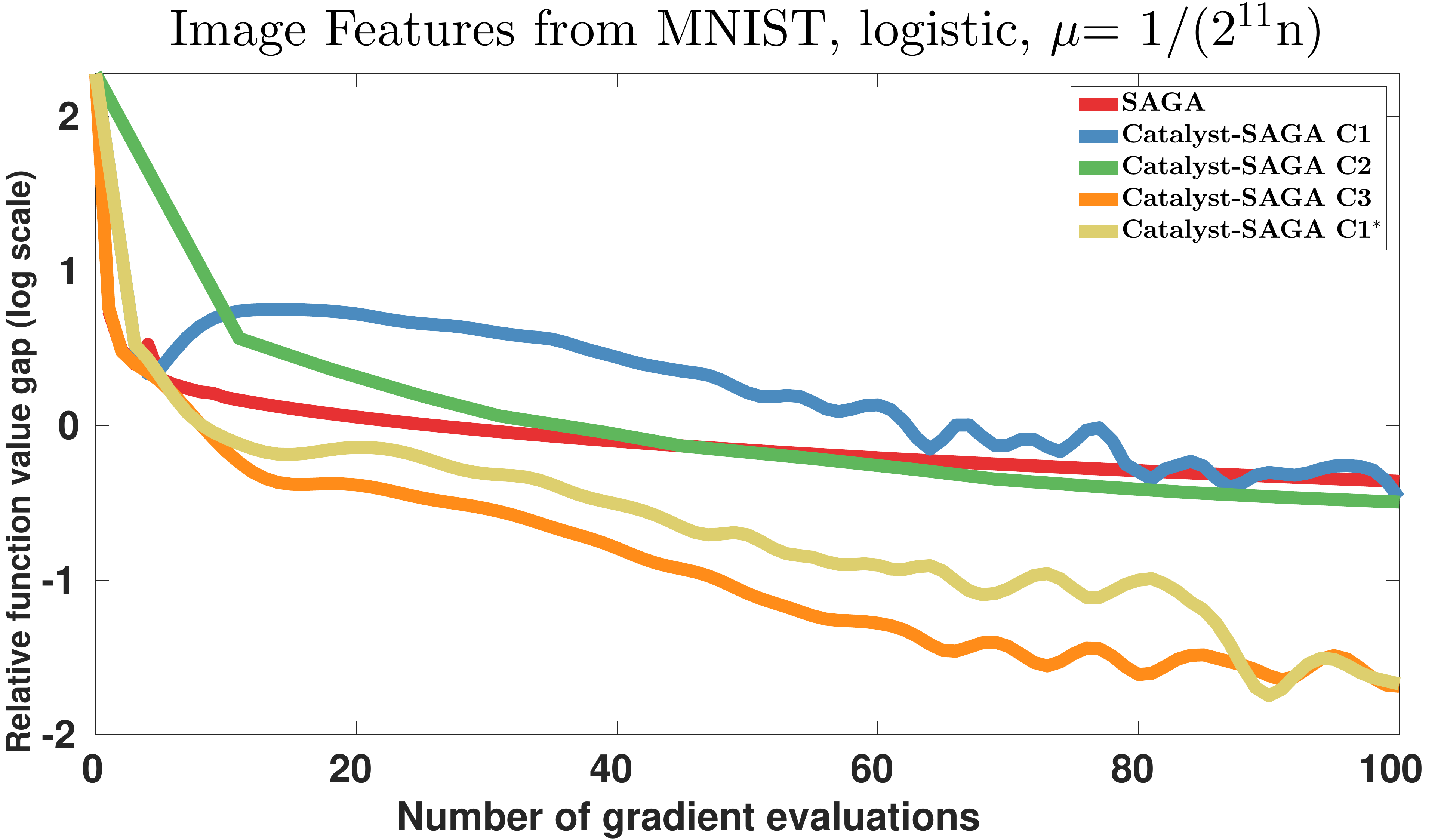}~ 
   ~~\includegraphics[width=0.31\linewidth]{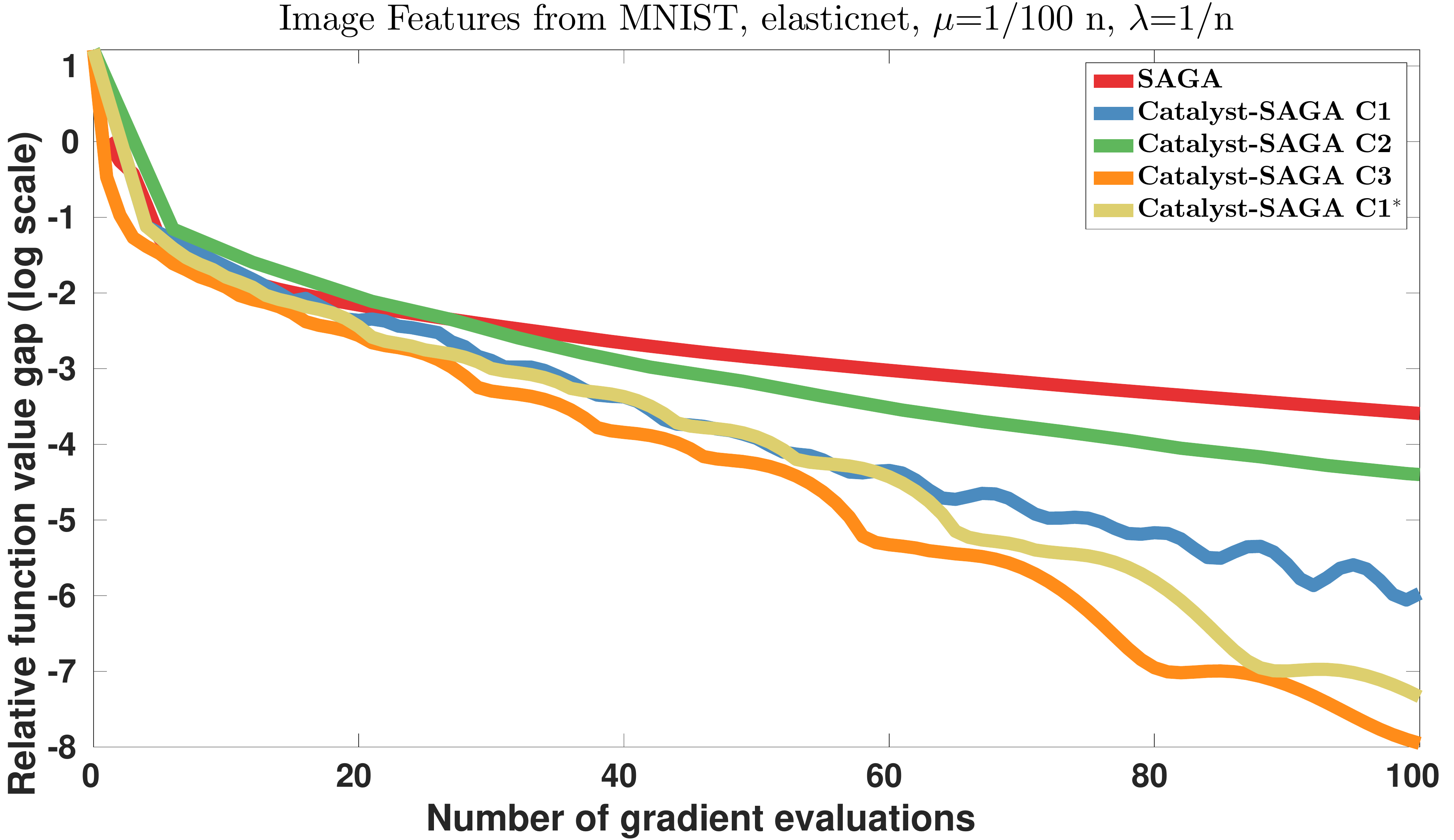}~ 
   ~~\includegraphics[width=0.31\linewidth]{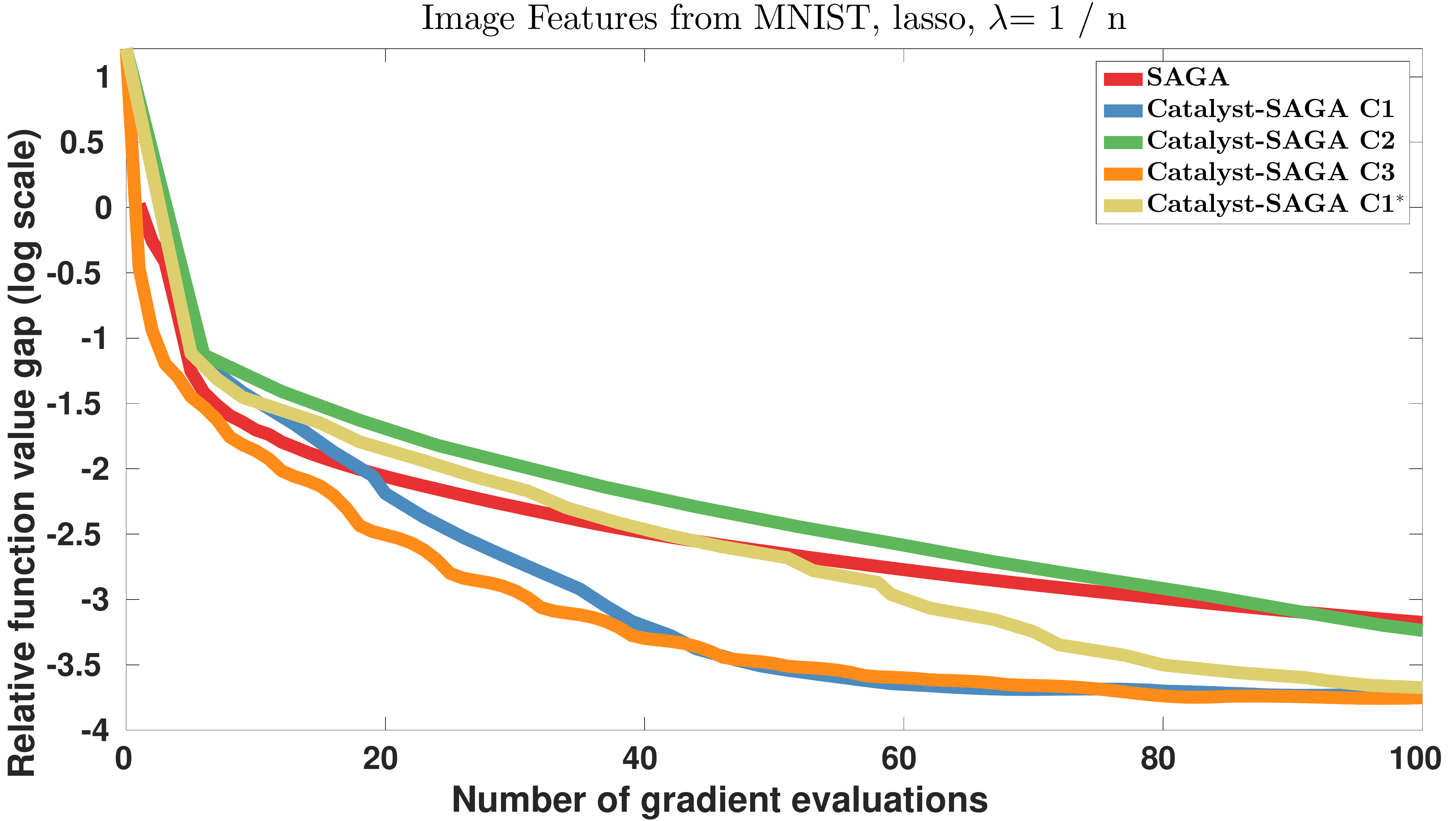}\\
   ~~\includegraphics[width=0.31\linewidth]{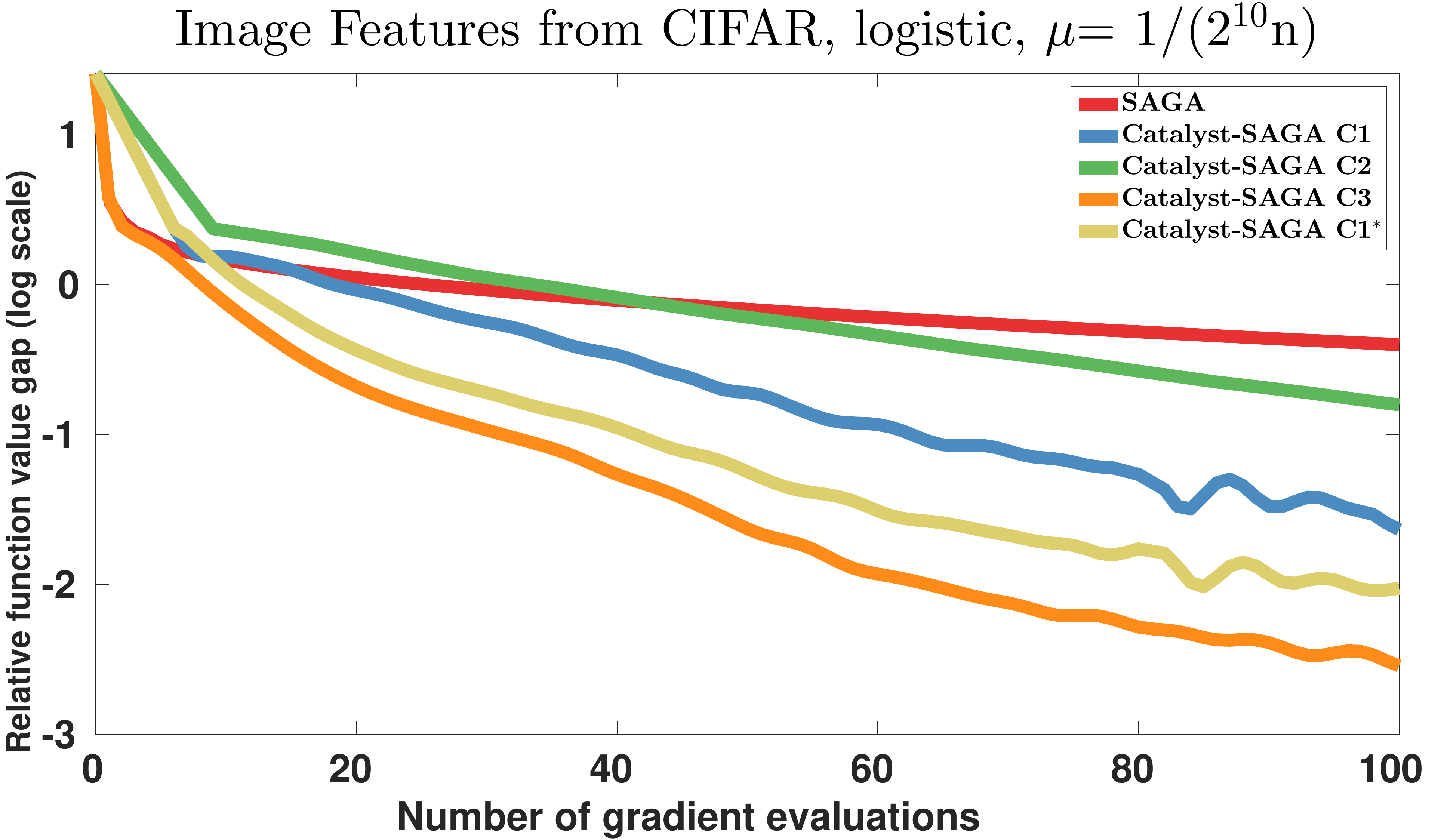}~ 
   ~~\includegraphics[width=0.31\linewidth]{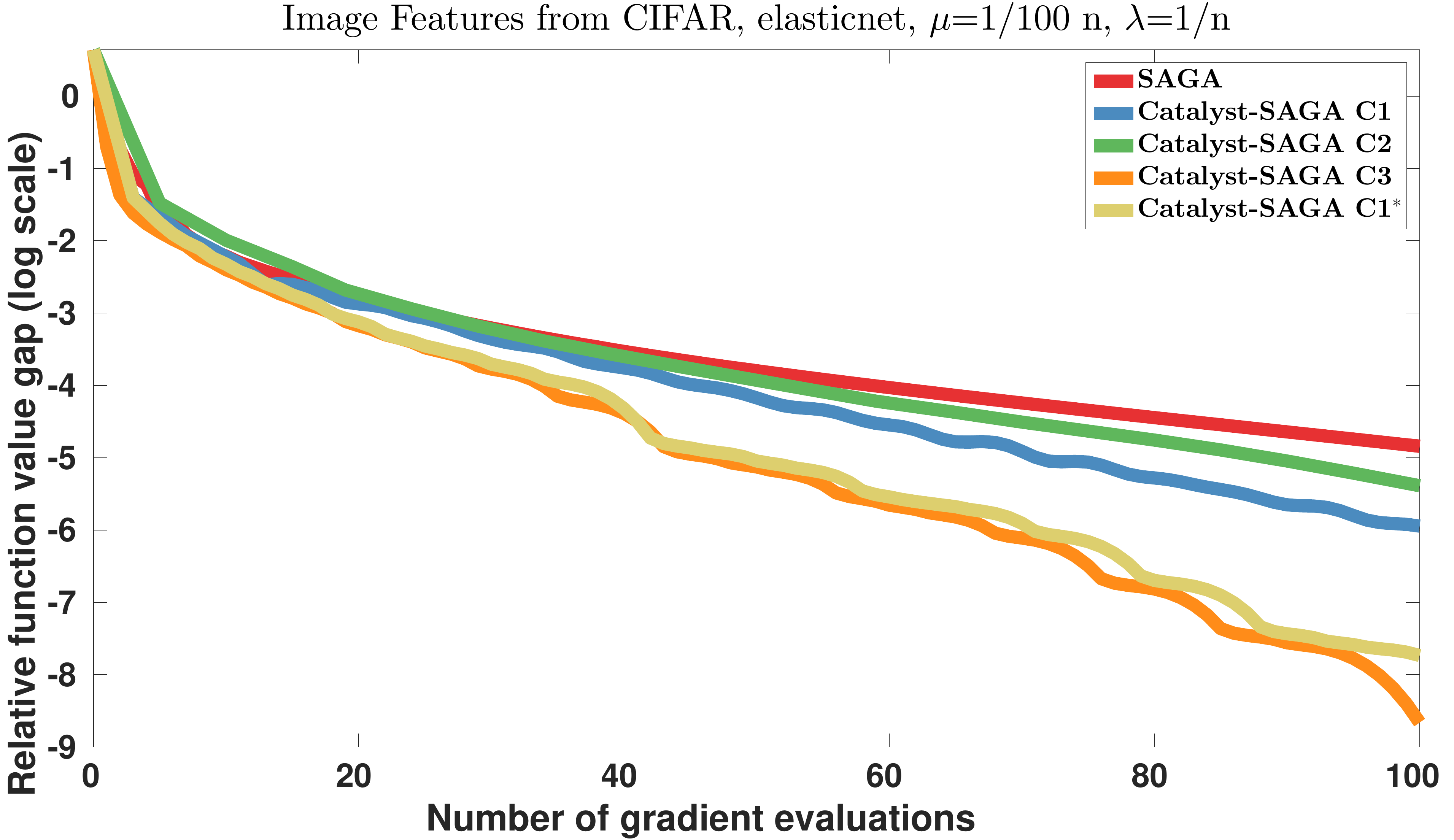}~ 
   ~~\includegraphics[width=0.31\linewidth]{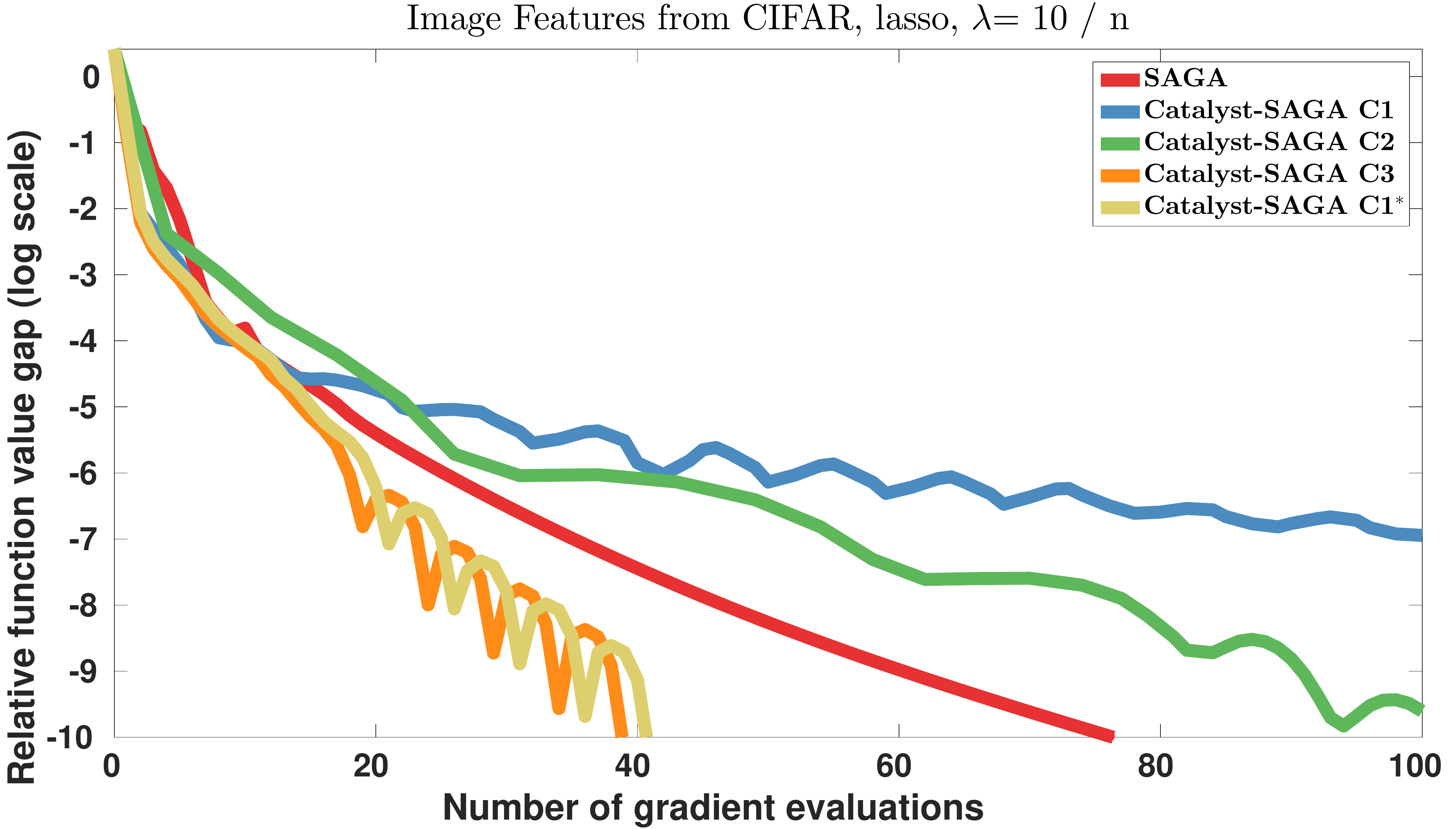}
    \caption{Experimental study of different stopping criterions for Catalyst-SAGA, with a similar setting as in Figure~\ref{catalyst:fig:svrg}.  }\label{catalyst:fig:saga}
\end{figure}

\paragraph{Observations for Catalyst-SAGA.} Our conclusions with SAGA are almost the same as with SVRG. However, in a few cases, we also notice that criterion \ref{C1} lacks stability, or at least exhibits some oscillations, which may suggest that SAGA has a larger variance compared to SVRG. The difference in the performance of (\ref{C1}) and \Cunstar can be huge, while they differ from each other only by the warm start strategy. Thus, \textit{choosing a good initial point for solving the sub-problems is a key for obtaining acceleration in practice.} 

\begin{figure}[hbtp]
   \centering
   ~~\includegraphics[width=0.31\linewidth]{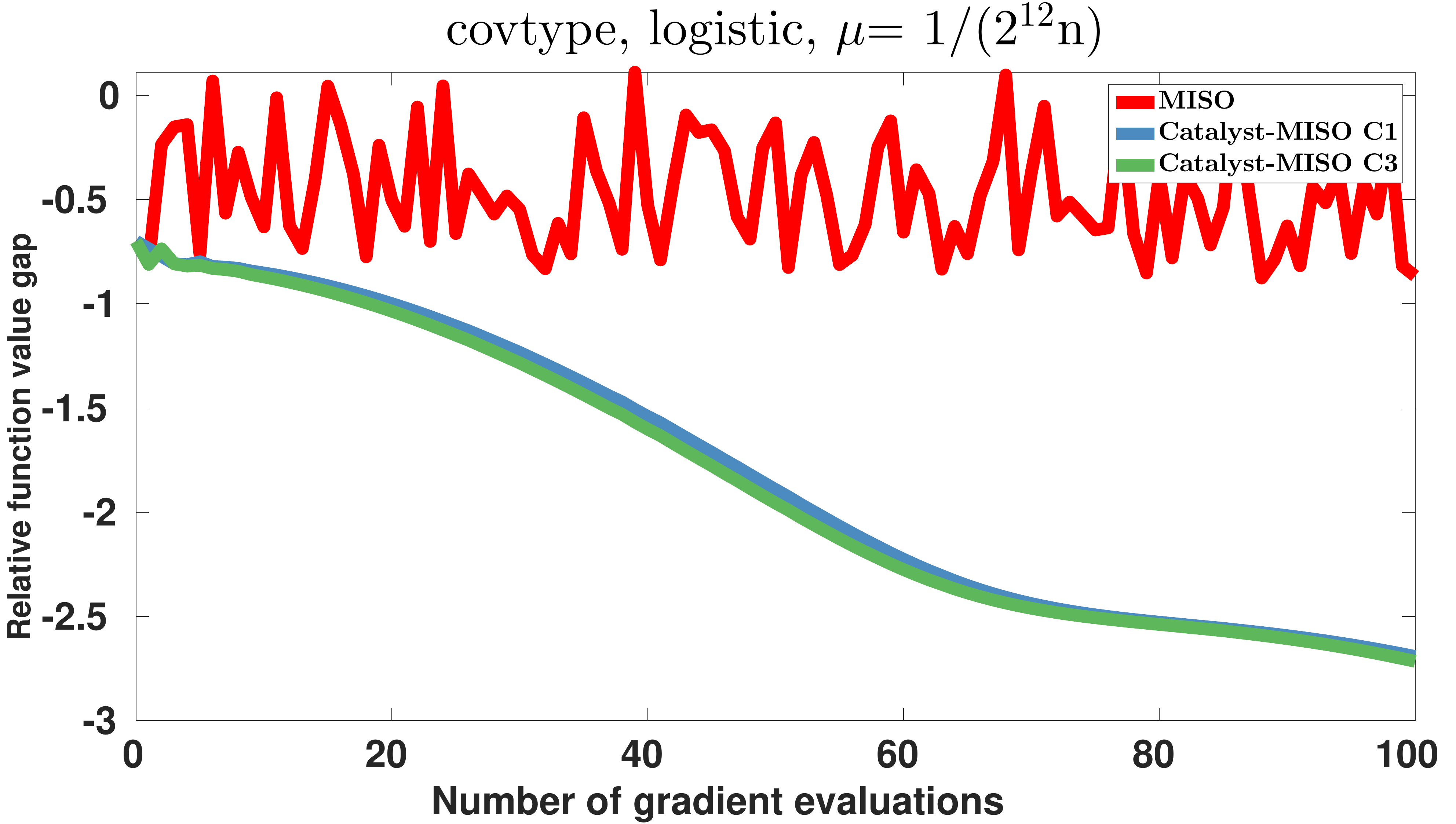}~ 
   ~~\includegraphics[width=0.31\linewidth]{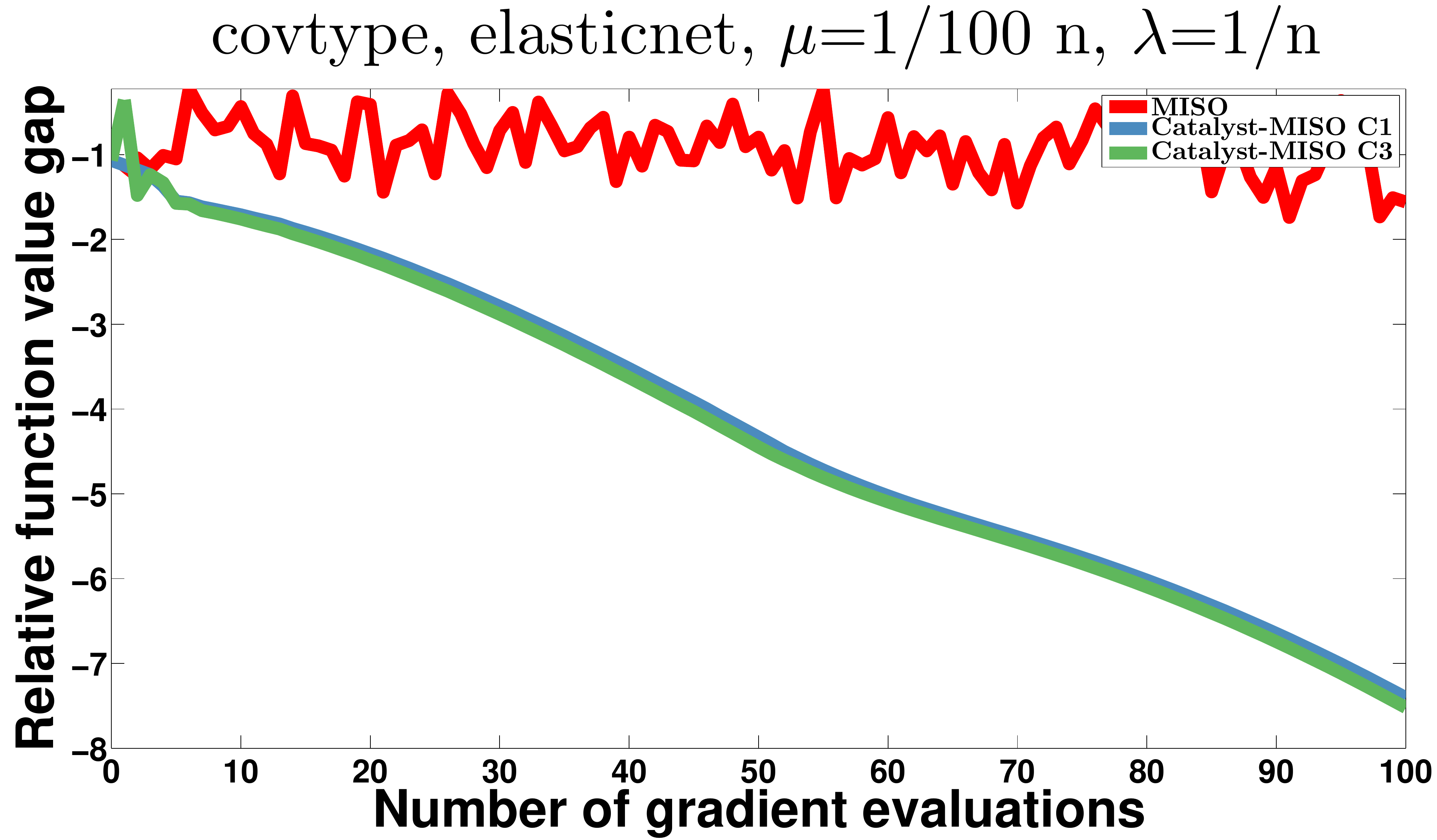}~ 
   ~~\includegraphics[width=0.31\linewidth]{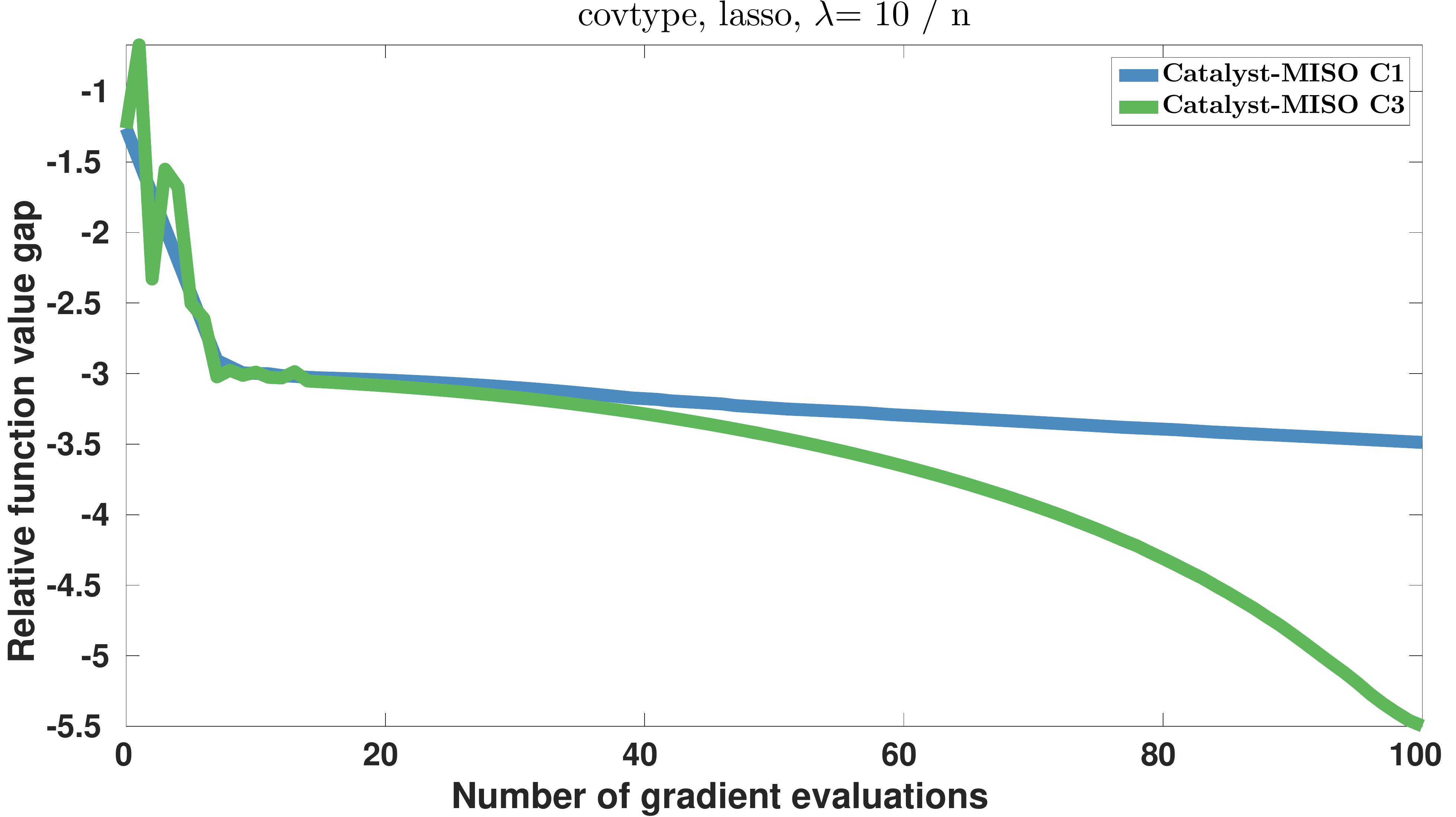}\\
   ~~\includegraphics[width=0.31\linewidth]{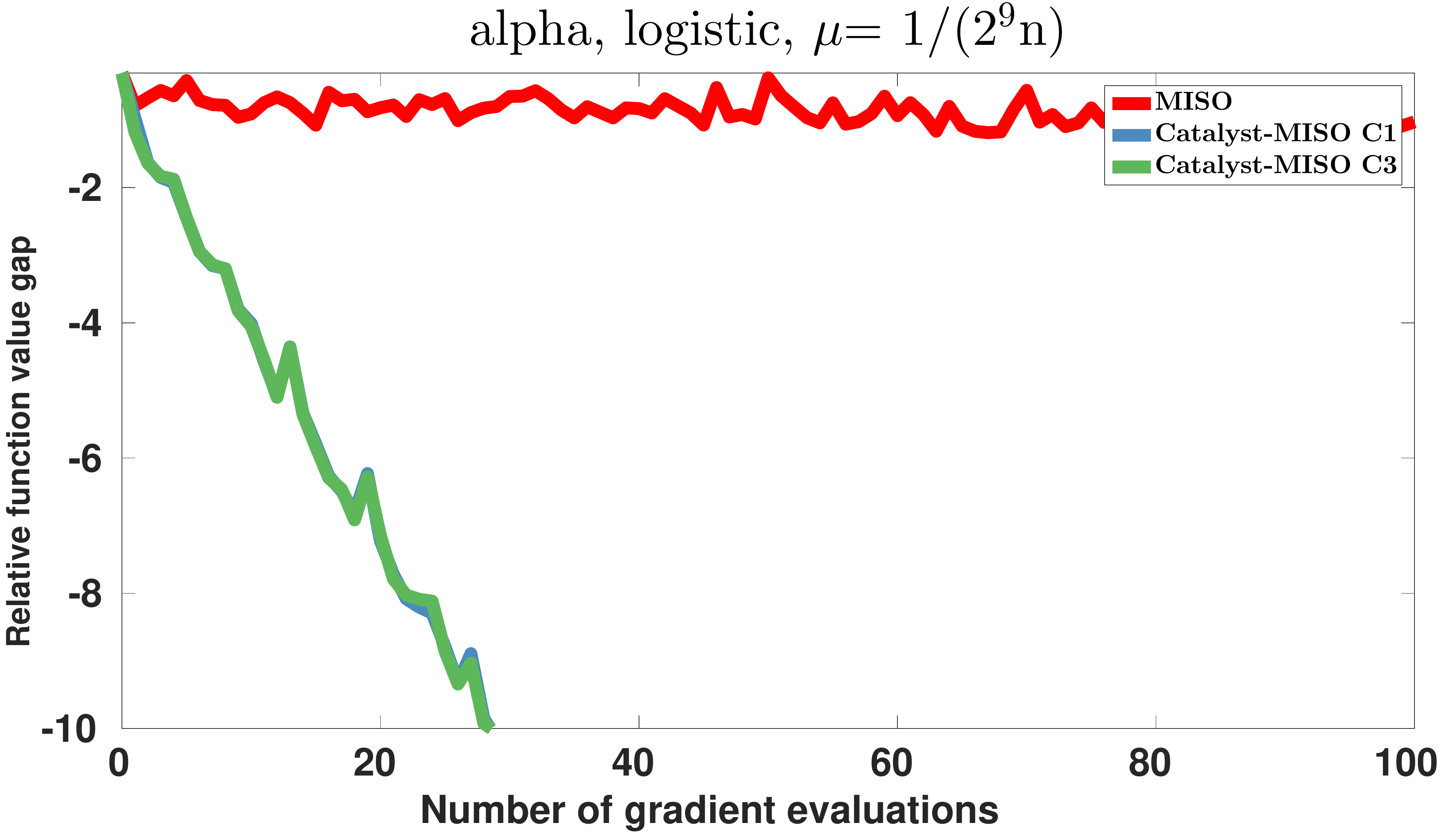}~ 
   ~~\includegraphics[width=0.31\linewidth]{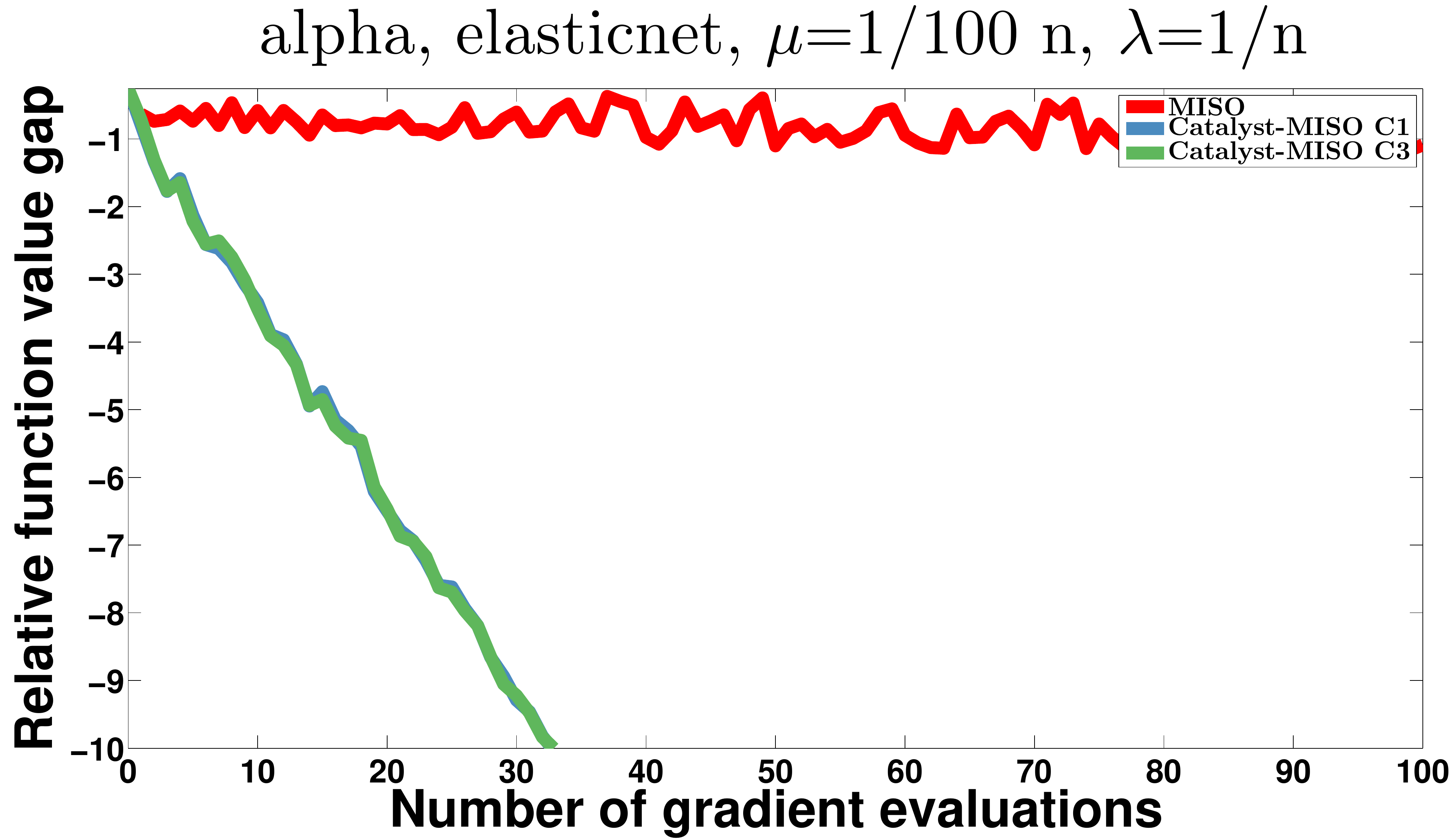}~ 
   ~~\includegraphics[width=0.31\linewidth]{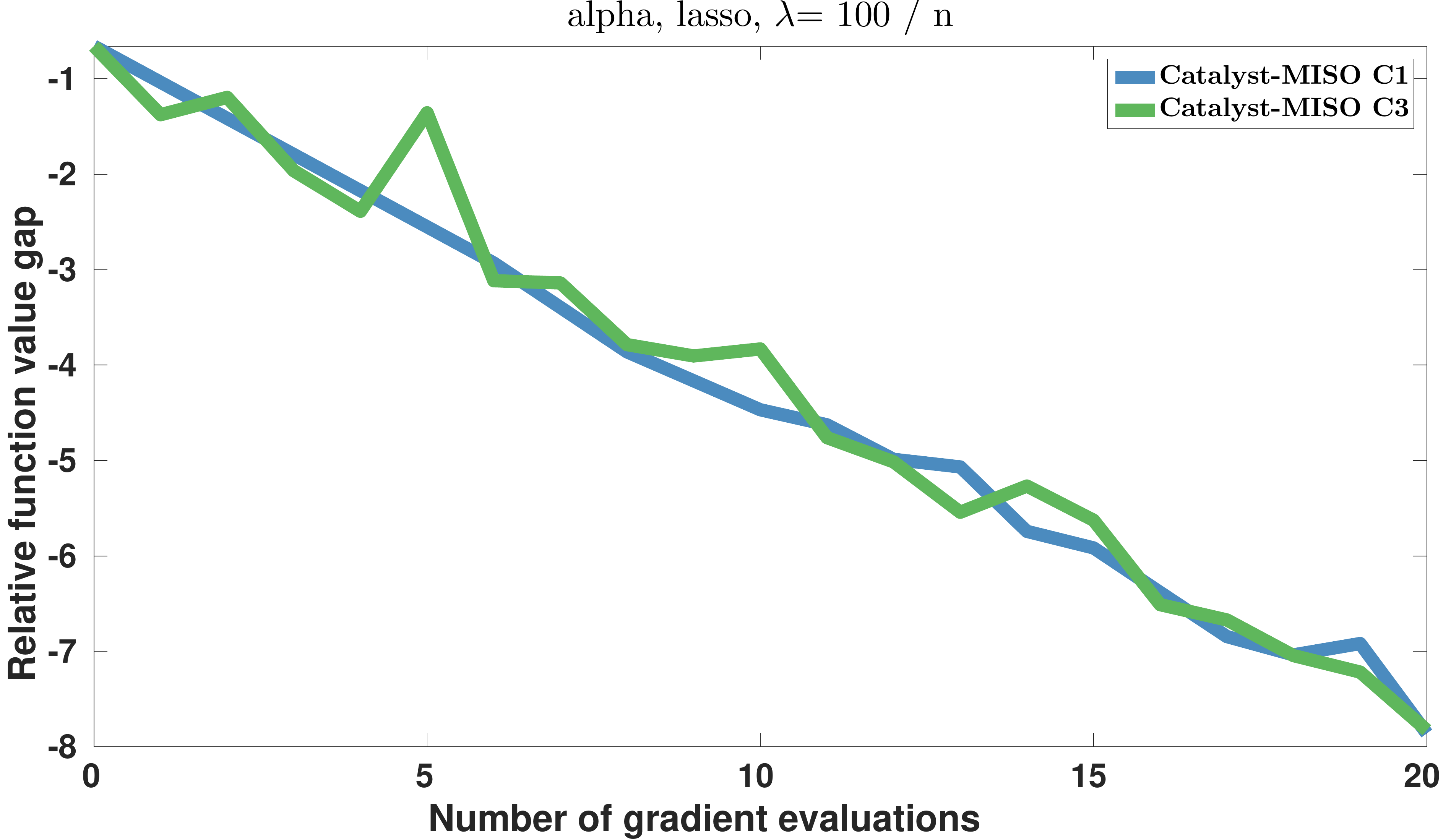}\\
   ~~\includegraphics[width=0.31\linewidth]{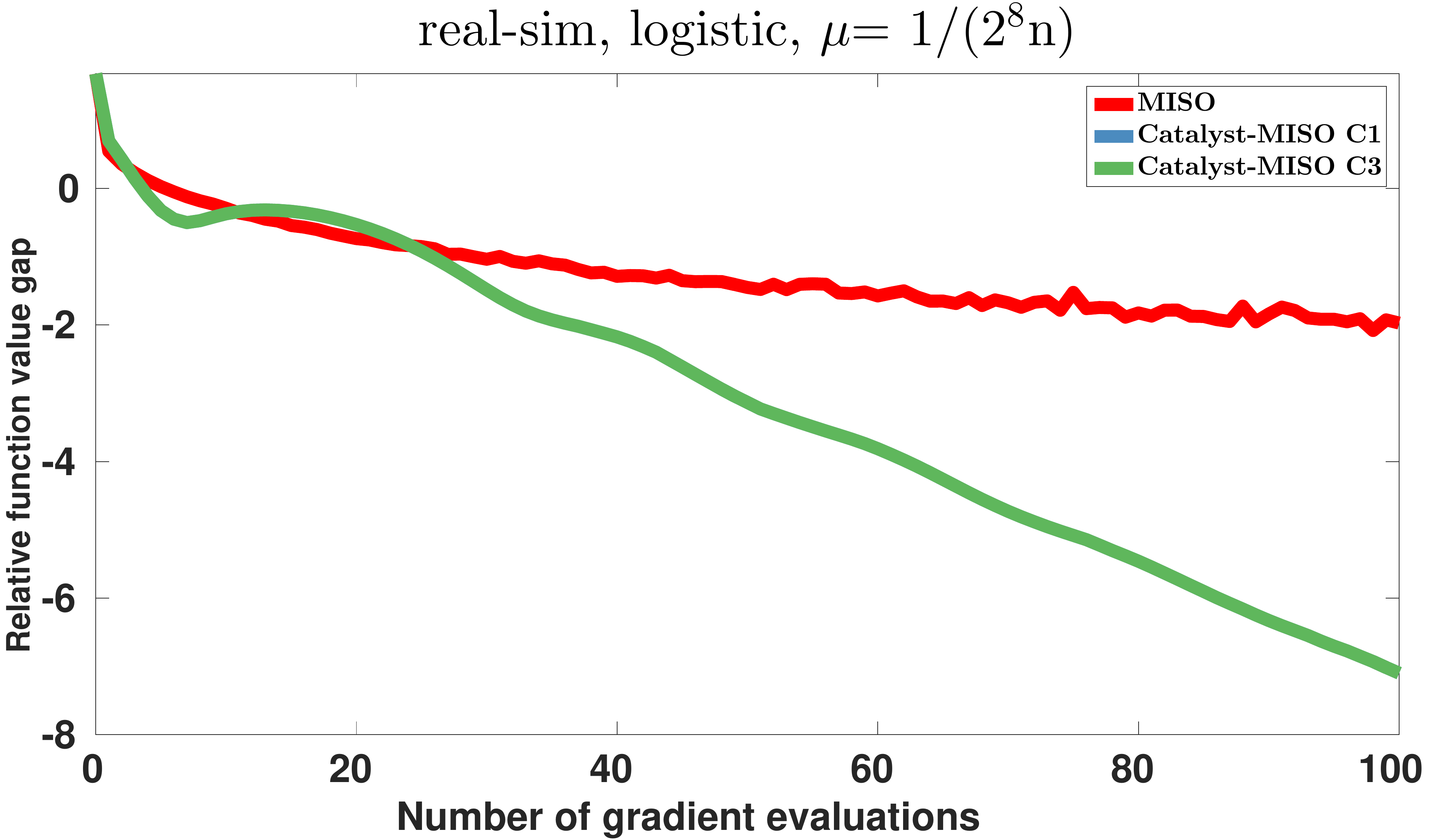}~ 
   ~~\includegraphics[width=0.31\linewidth]{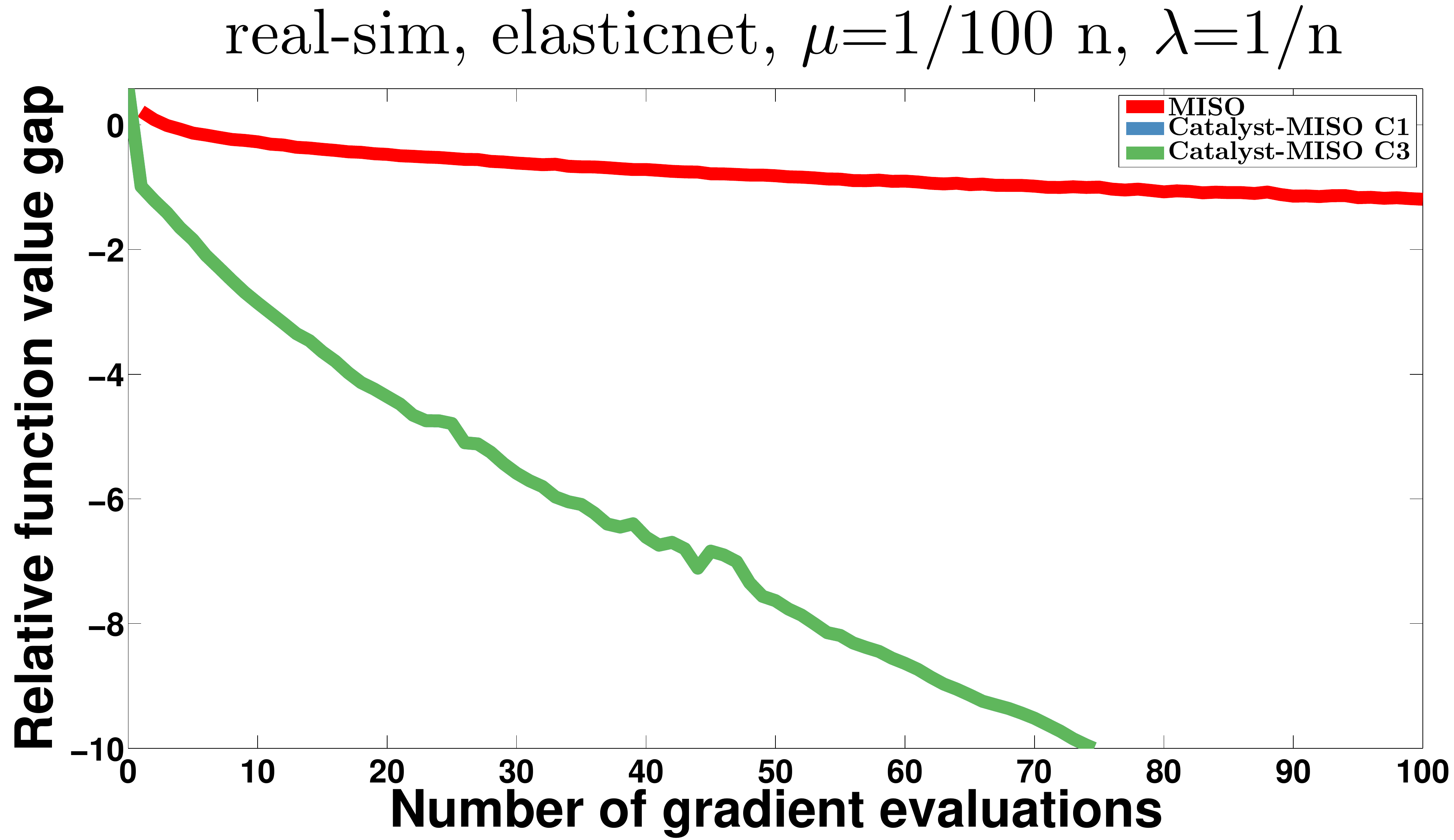}~ 
   ~~\includegraphics[width=0.31\linewidth]{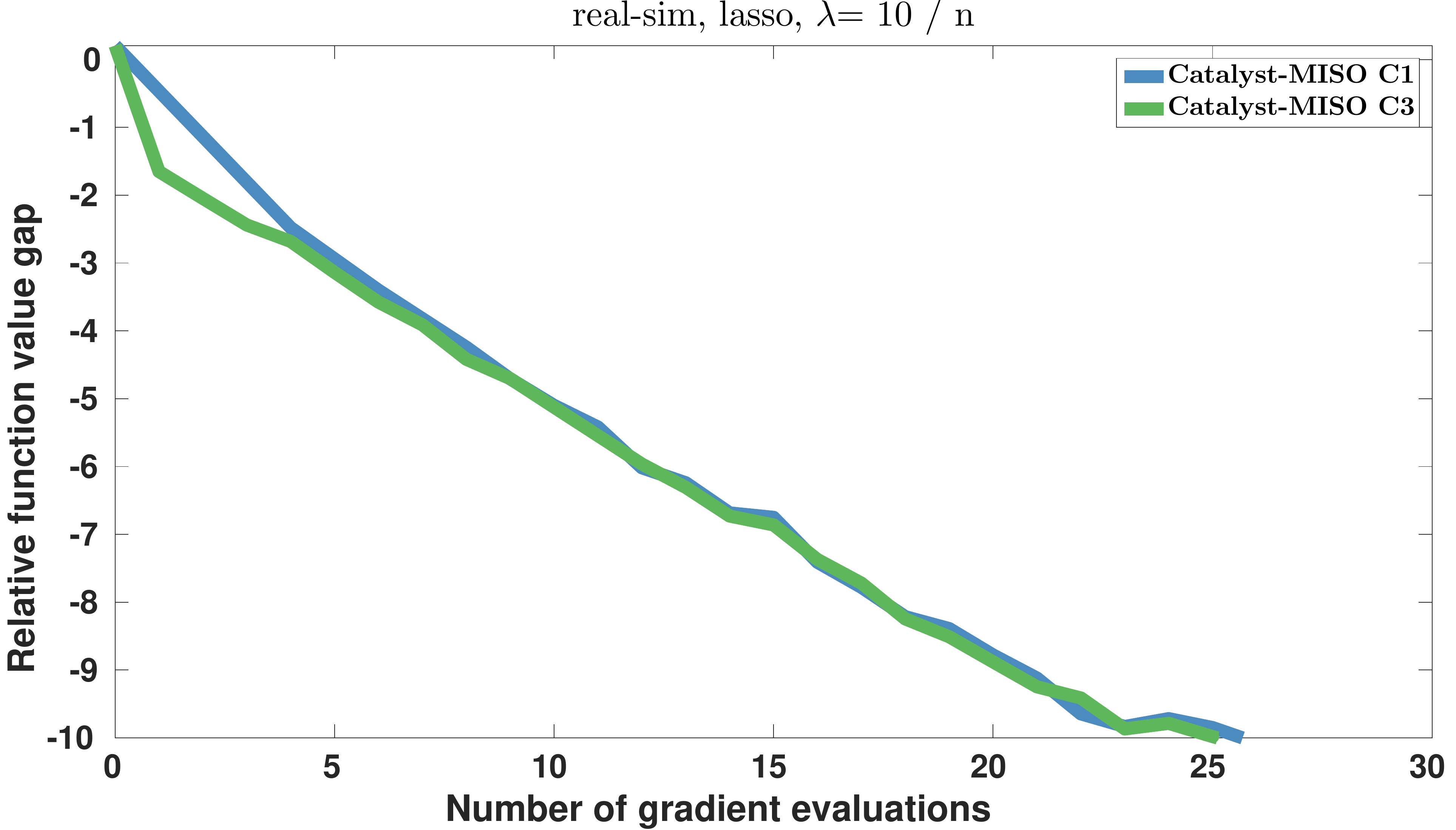}\\
   ~~\includegraphics[width=0.31\linewidth]{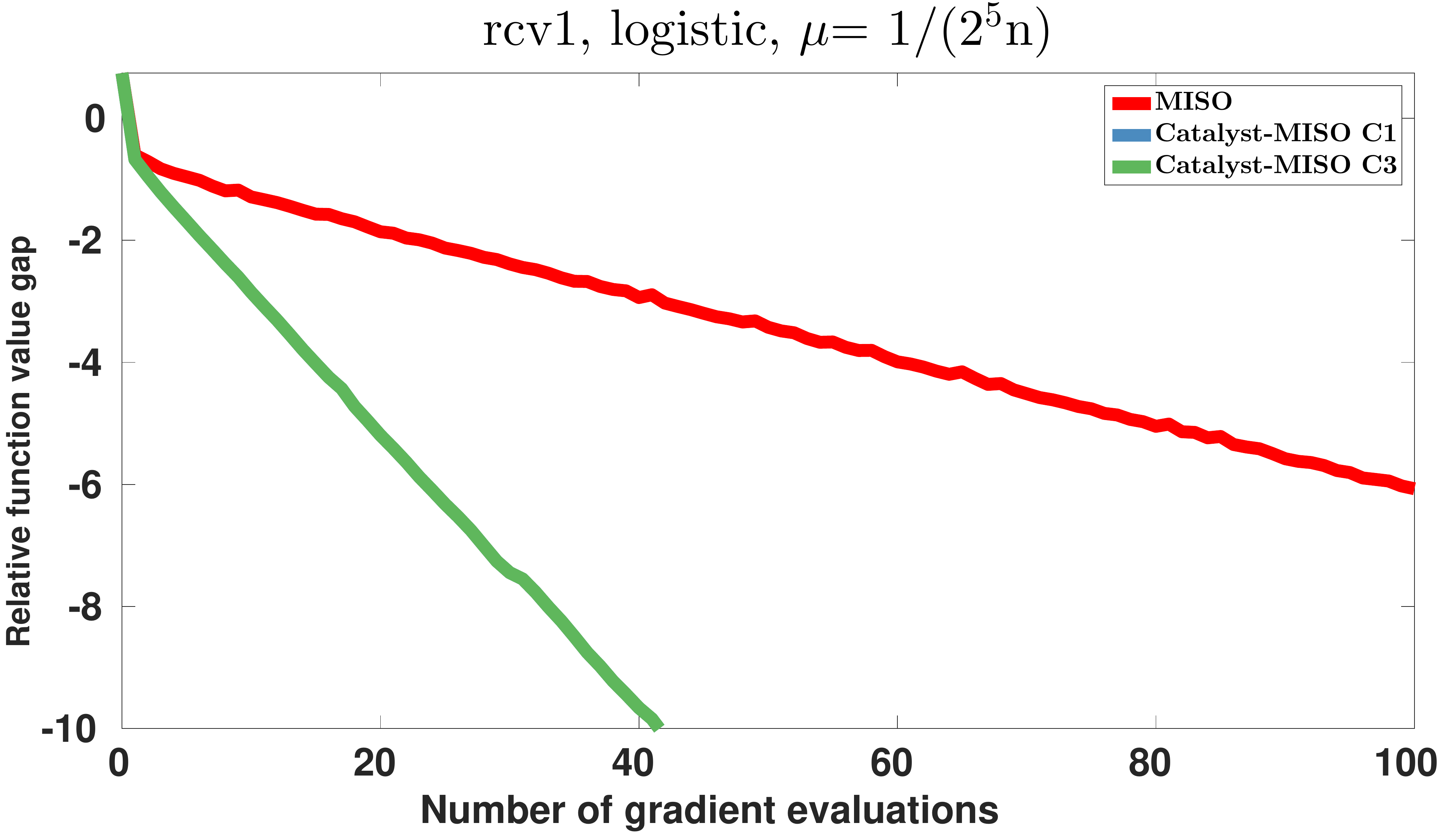}~ 
   ~~\includegraphics[width=0.31\linewidth]{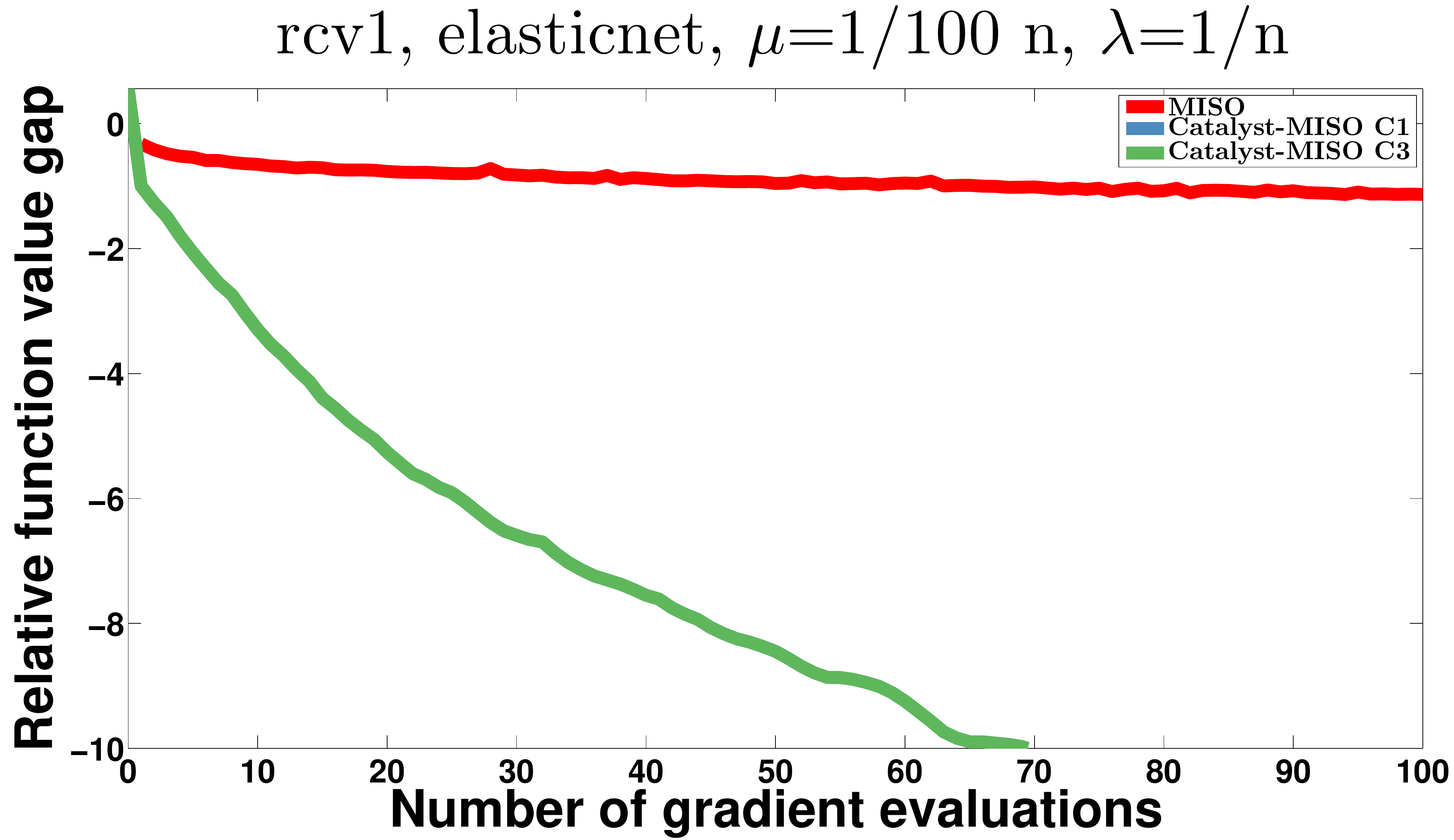}~ 
   ~~\includegraphics[width=0.31\linewidth]{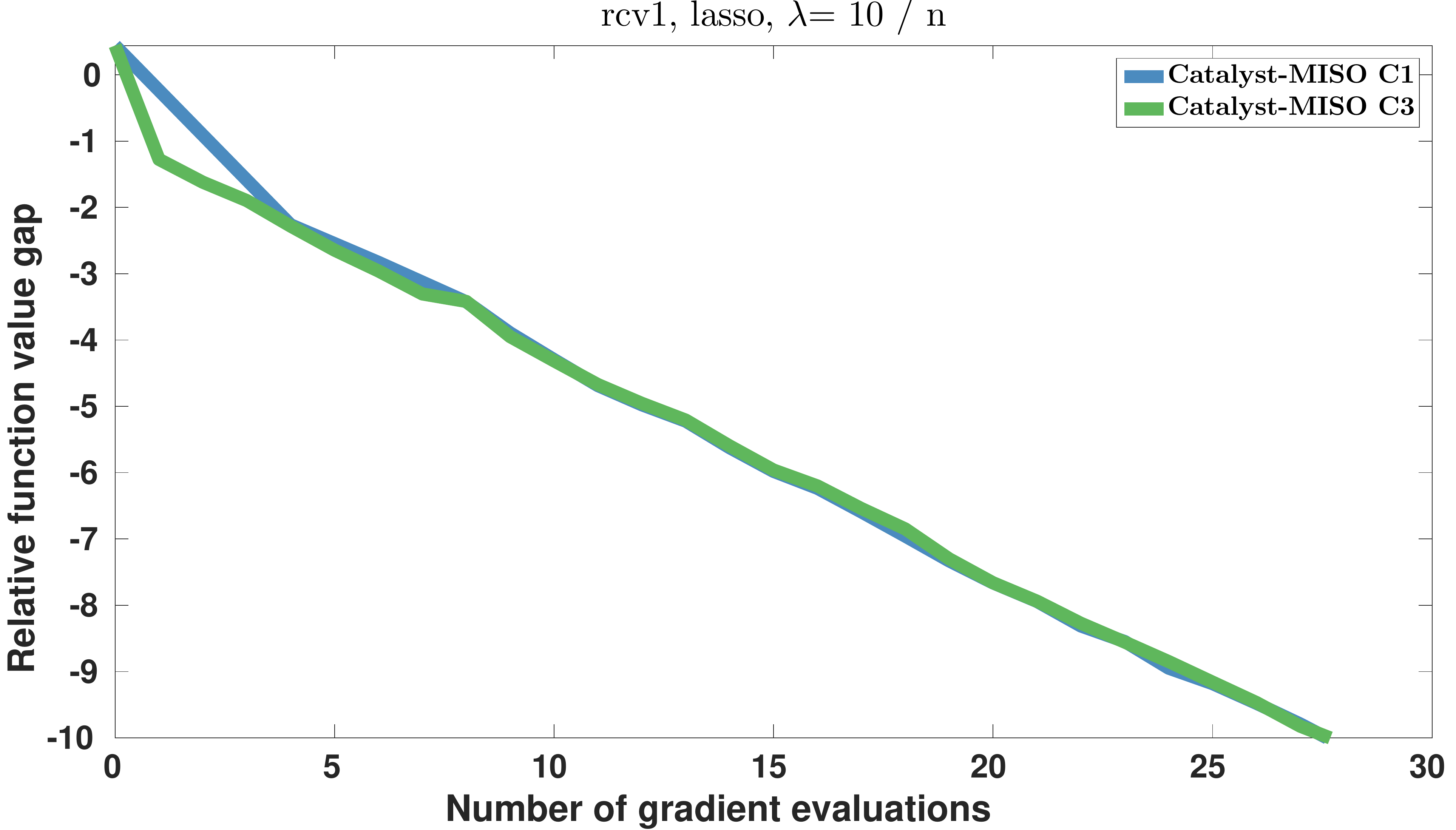}\\
   ~~\includegraphics[width=0.31\linewidth]{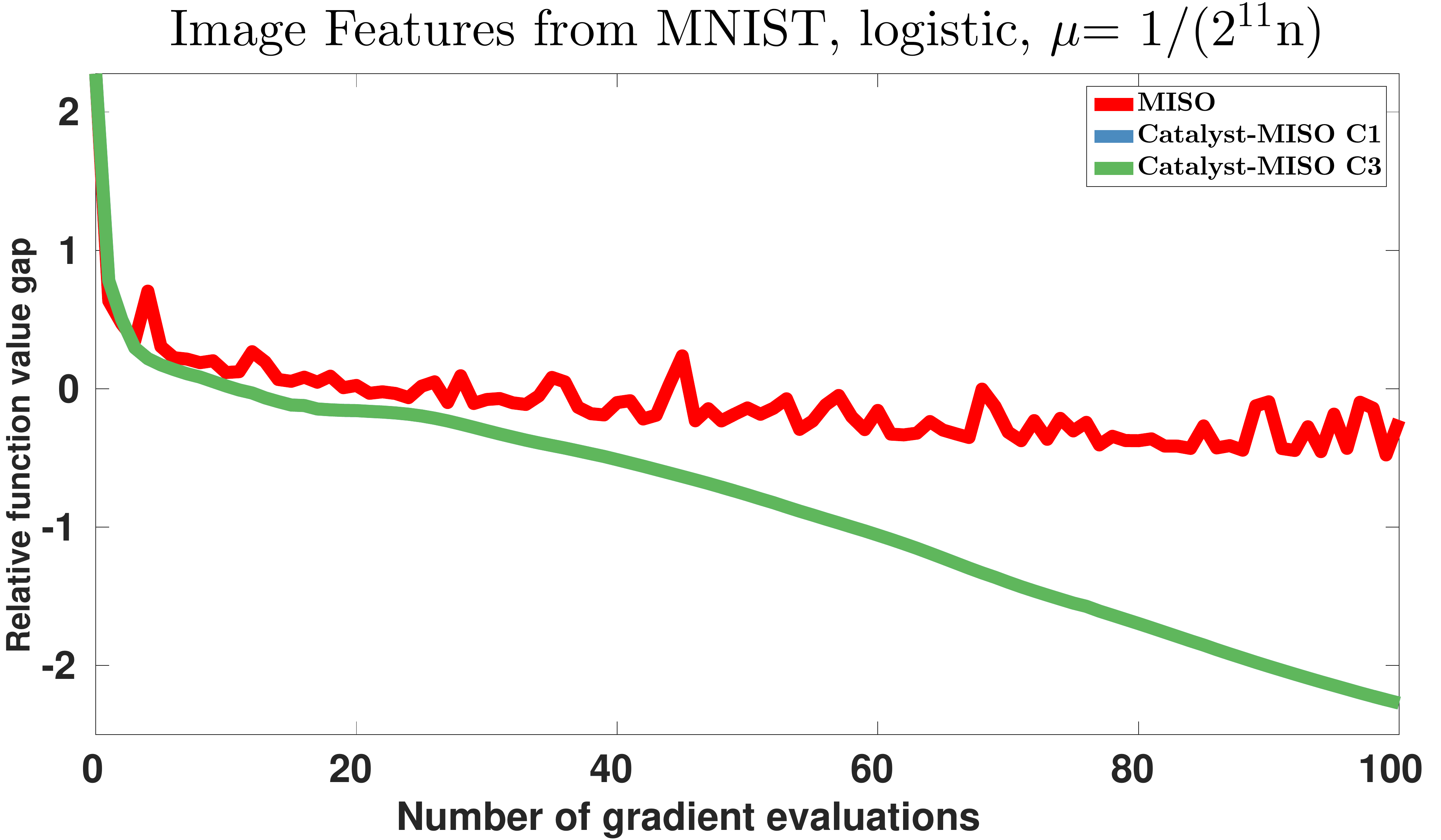}~ 
   ~~\includegraphics[width=0.31\linewidth]{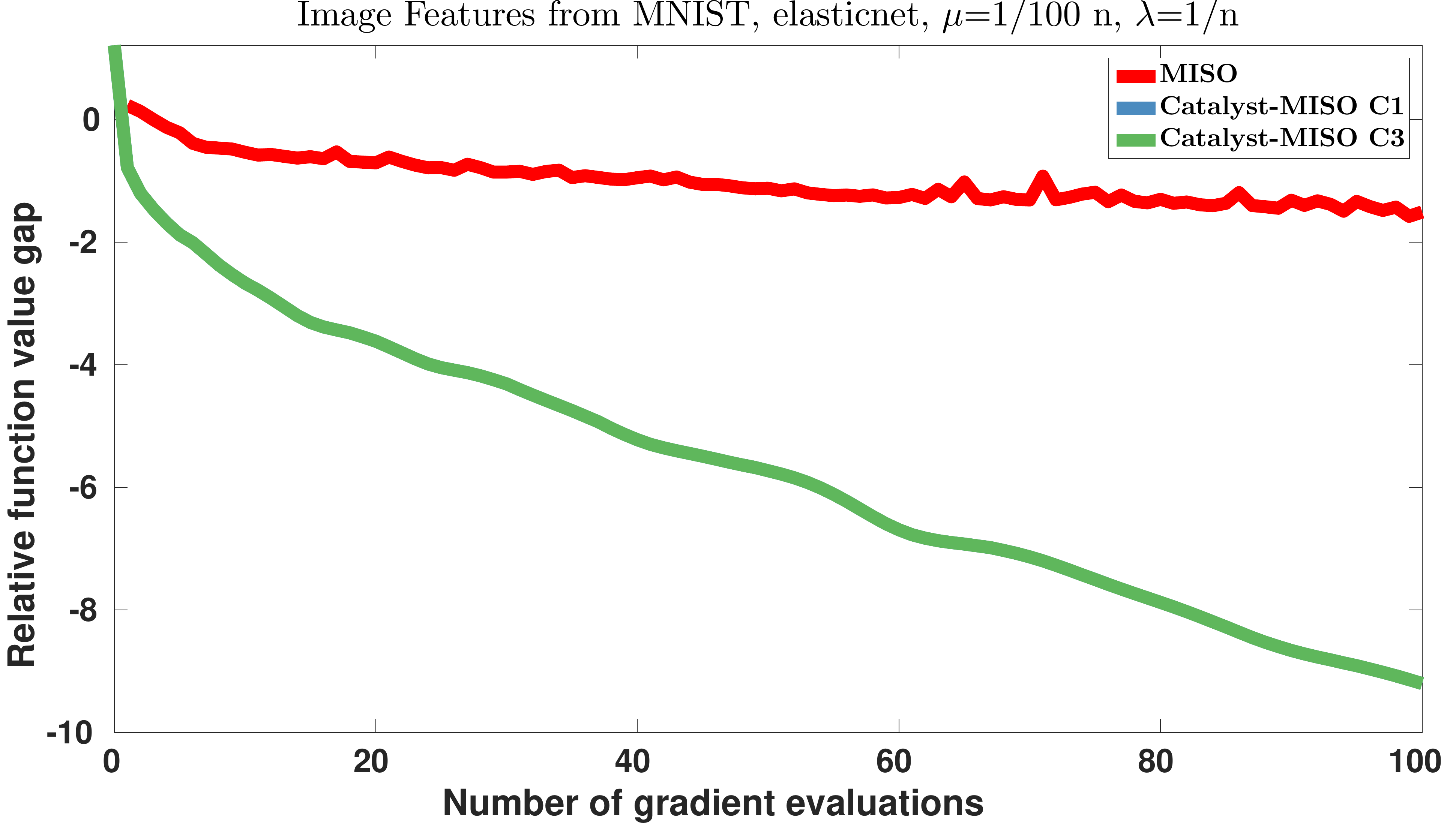}~ 
   ~~\includegraphics[width=0.31\linewidth]{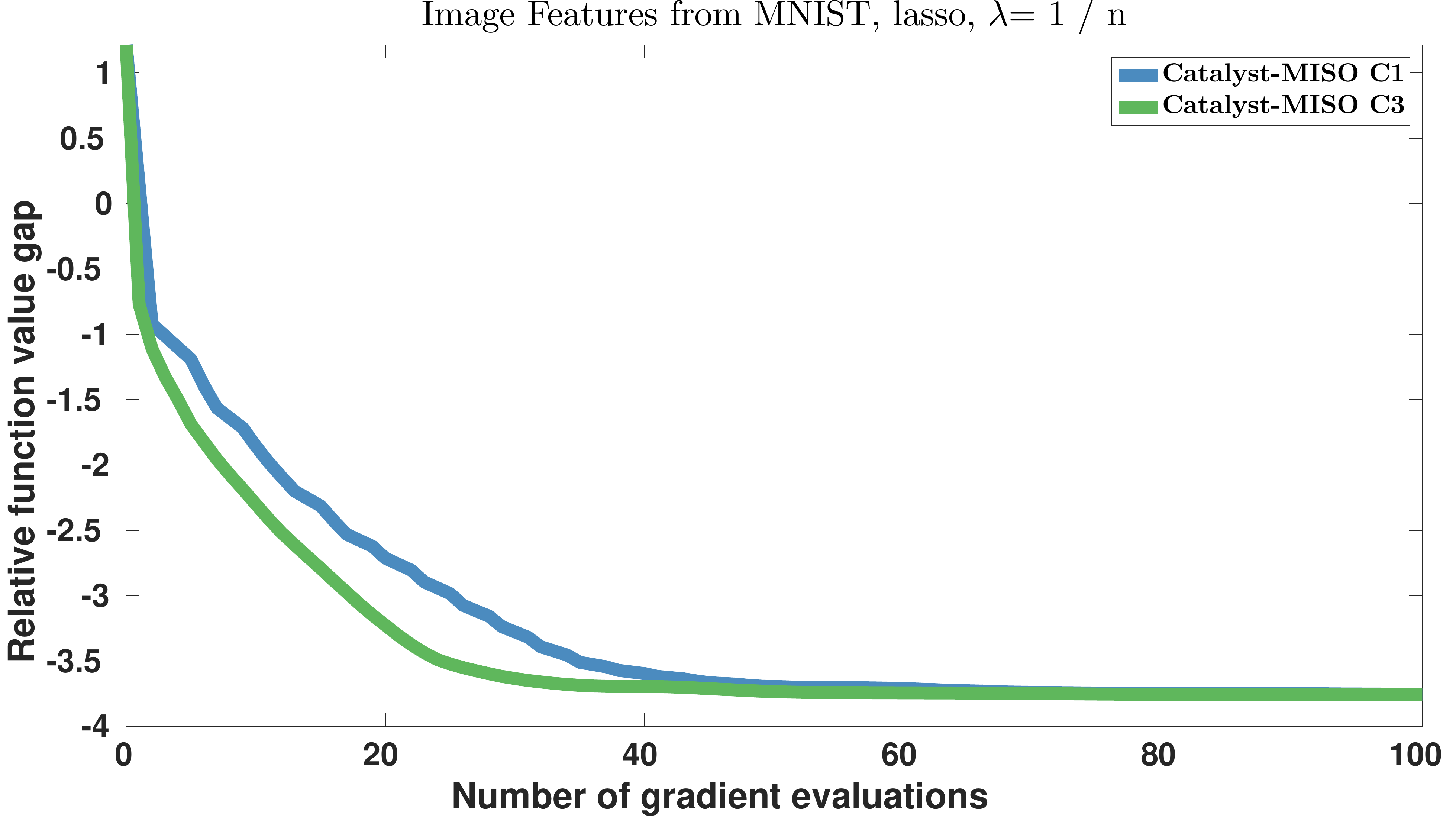}\\
  ~~\includegraphics[width=0.31\linewidth]{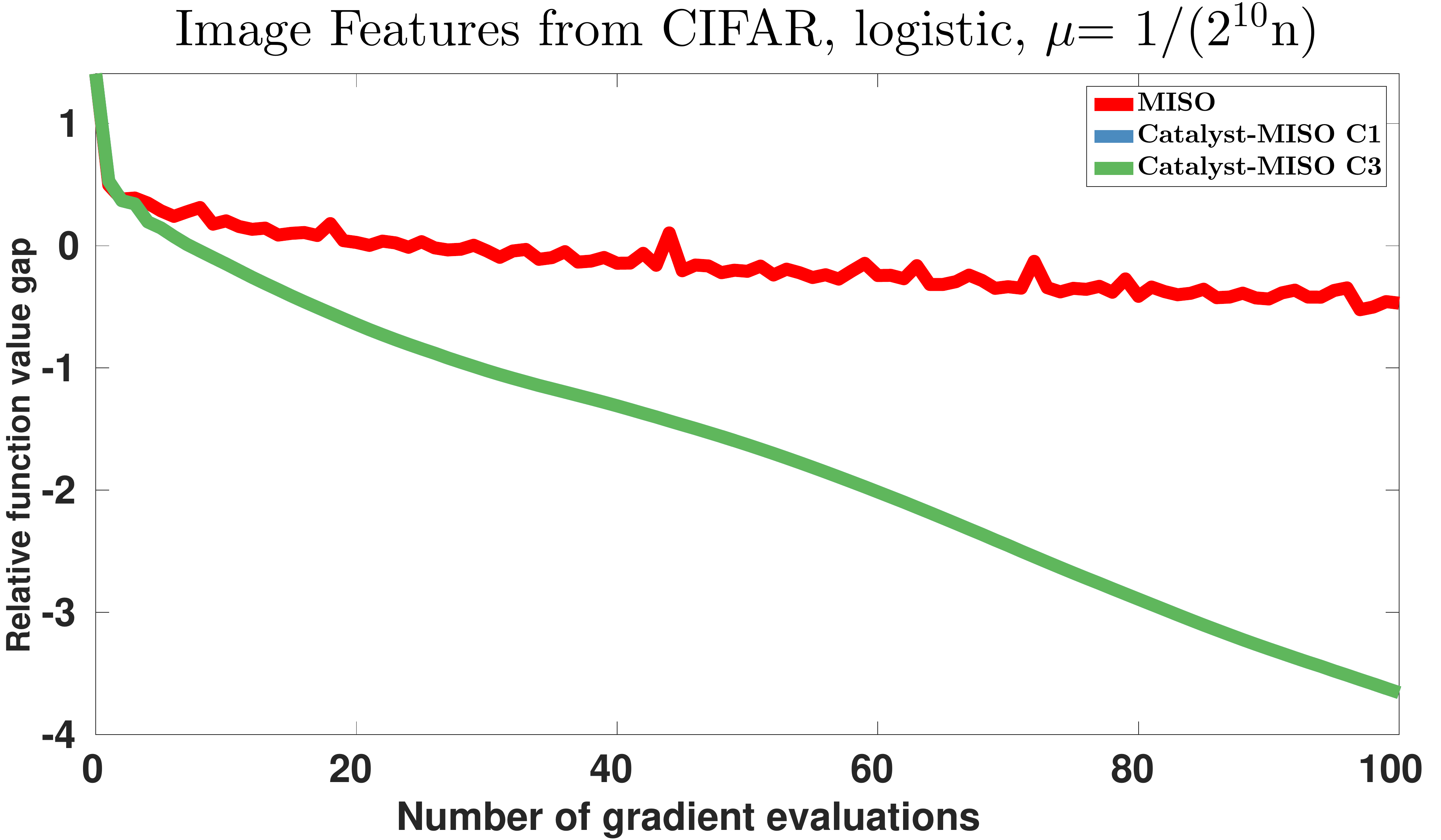}~ 
   ~~\includegraphics[width=0.31\linewidth]{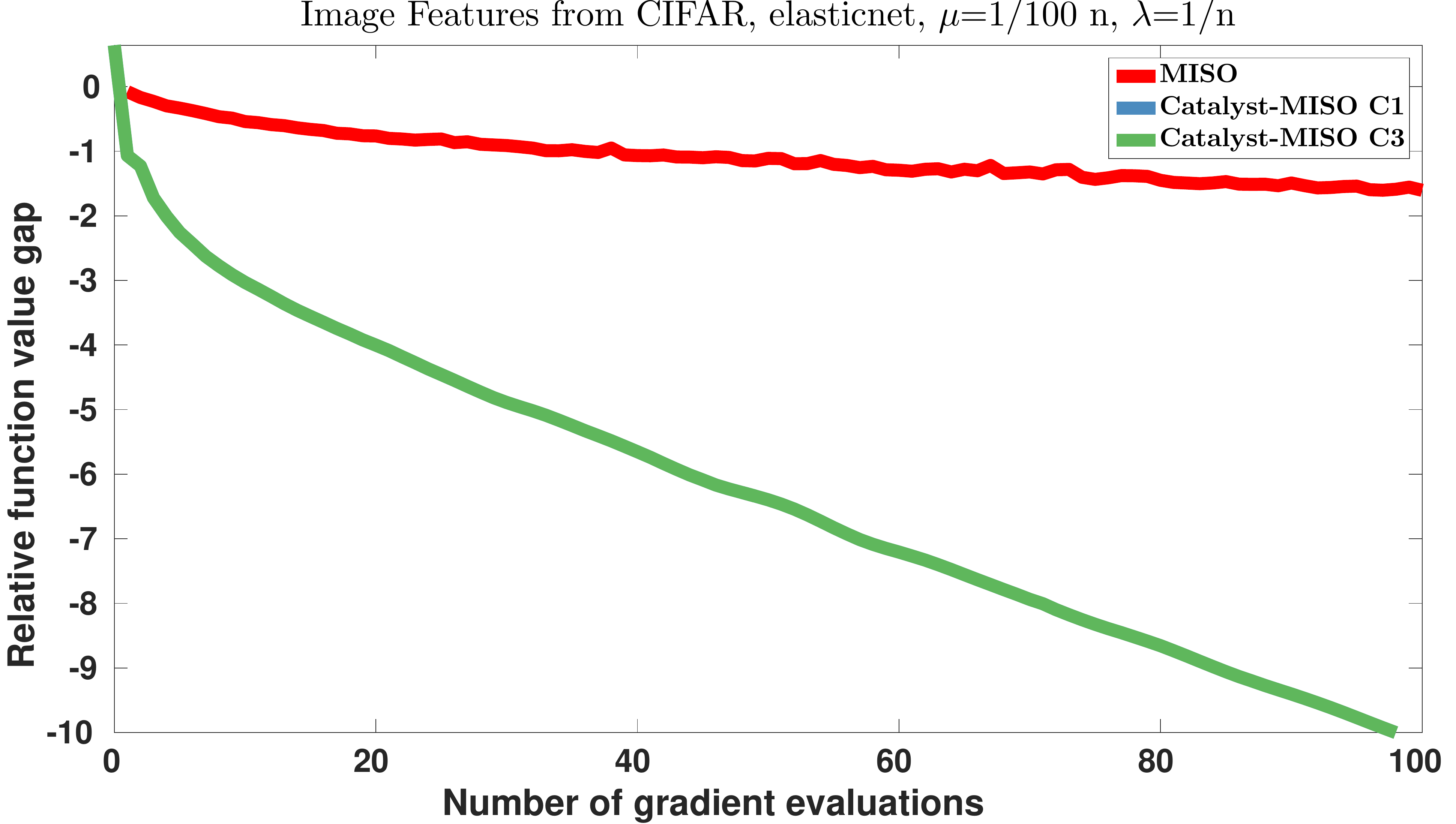}~ 
   ~~\includegraphics[width=0.31\linewidth]{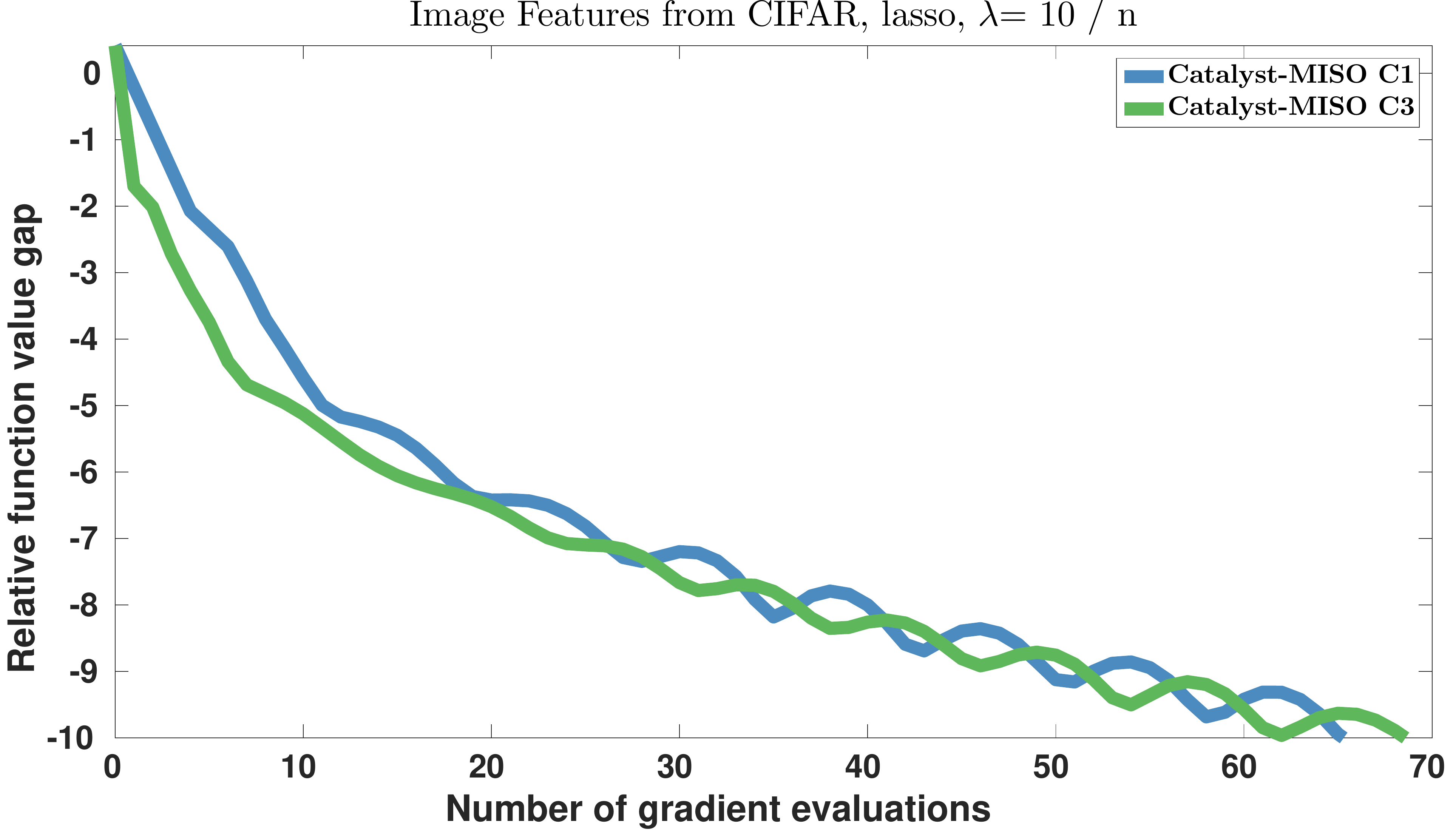}
   \caption{Experimental study of different stopping criterions for Catalyst-MISO, with a similar setting as in Figure~\ref{catalyst:fig:svrg} }\label{catalyst:fig:miso}
\end{figure}

\paragraph{Observations for Catalyst-MISO.} 
The warm-start strategy of MISO is different
from primal algorithms because parameters for the dual function need to be
specified. The most natural way for warm starting the dual functions is to set
$$ d_{k+1}(x) = d_k(x) + \frac{\kappa}{2} \Vert x -y_k\Vert^2 -
\frac{\kappa}{2} \Vert x-y_{k-1}\Vert^2 ,$$
where $d_k$ is the last dual function of the previous subproblem $h_k$. This gives the warm start 
$$ z_0 = \prox \left(x_k + \frac{\kappa}{\kappa+\mu}(y_k- y_{k-1})\right). $$
For other choices of $z_0$, the dual function needs to be recomputed from scratch, which is computationally expensive and unstable for ill-conditioned problems. Thus, we only present the experimental results with respect to criterion (\ref{C1}) and the one-pass heuristic~\Ctrois. As we observe, a huge acceleration is obtained in logistic regression and Elastic-net formulations. For Lasso problem, the original Prox-MISO is not defined since the problem is not strongly convex. 
Thus, in order to make a comparison, we compare with Catalyst-SVRG which shows that the acceleration achieves a similar performance. This aligns with the theoretical result stating that Catalyst applied to incremental algorithms yields a similar convergence rate.
Notice also that the original MISO algorithm suffers from numerical stability in this ill-conditioned regime chosen for our experiments. Catalyst not only accelerates MISO, but it also stabilizes it.

\subsection{Acceleration of Existing Methods}\label{subsec:comparison_methods}
Then, we put the previous curves into perspective and make a comparison of the performance before and after applying Catalyst across methods. We show the best performance among the three developed stopping criteria, which corresponds to be \Ctrois. 
\begin{figure}[hbtp]
   \centering
   ~~\includegraphics[width=0.31\linewidth]{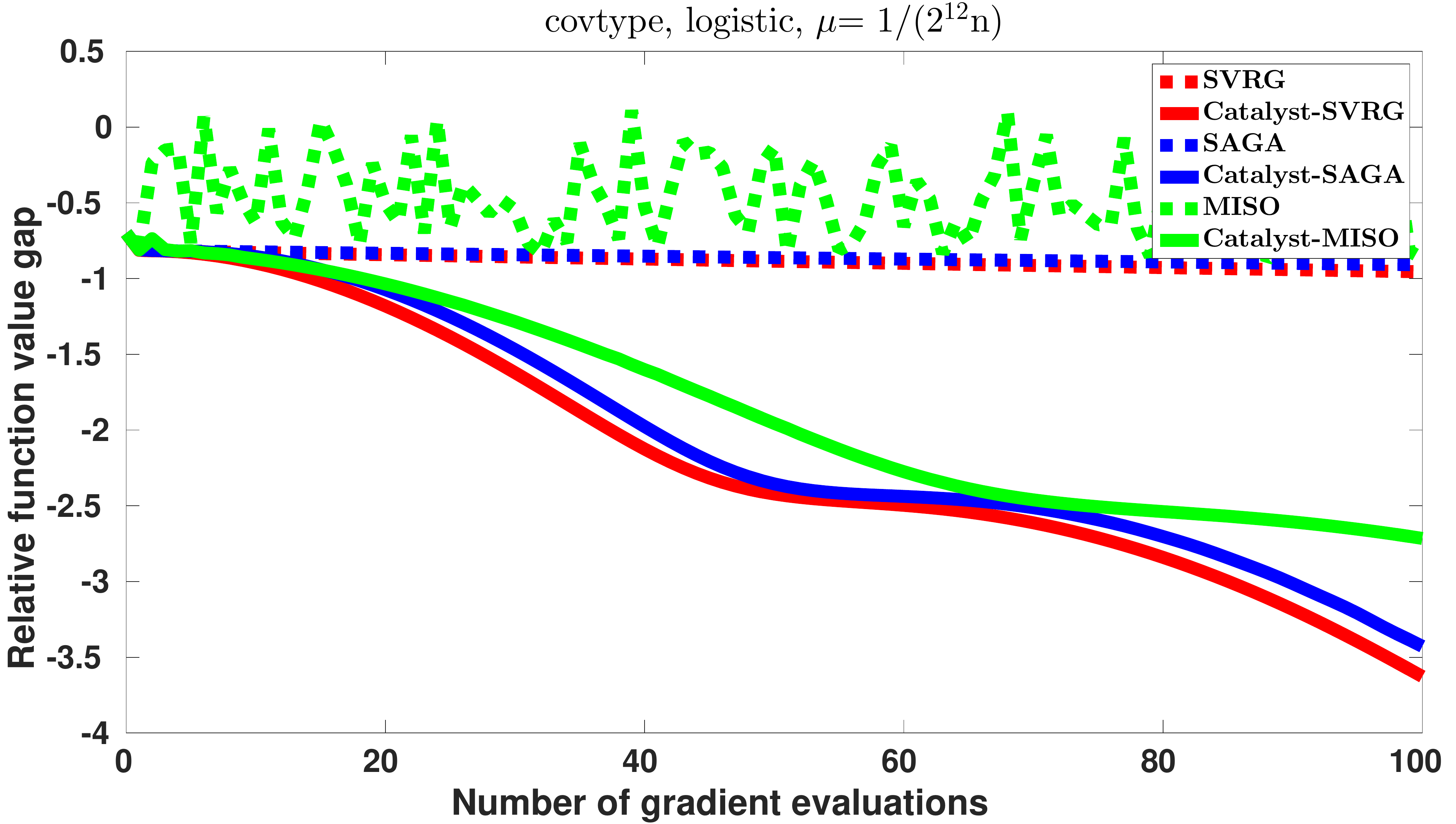}~ 
   ~~\includegraphics[width=0.31\linewidth]{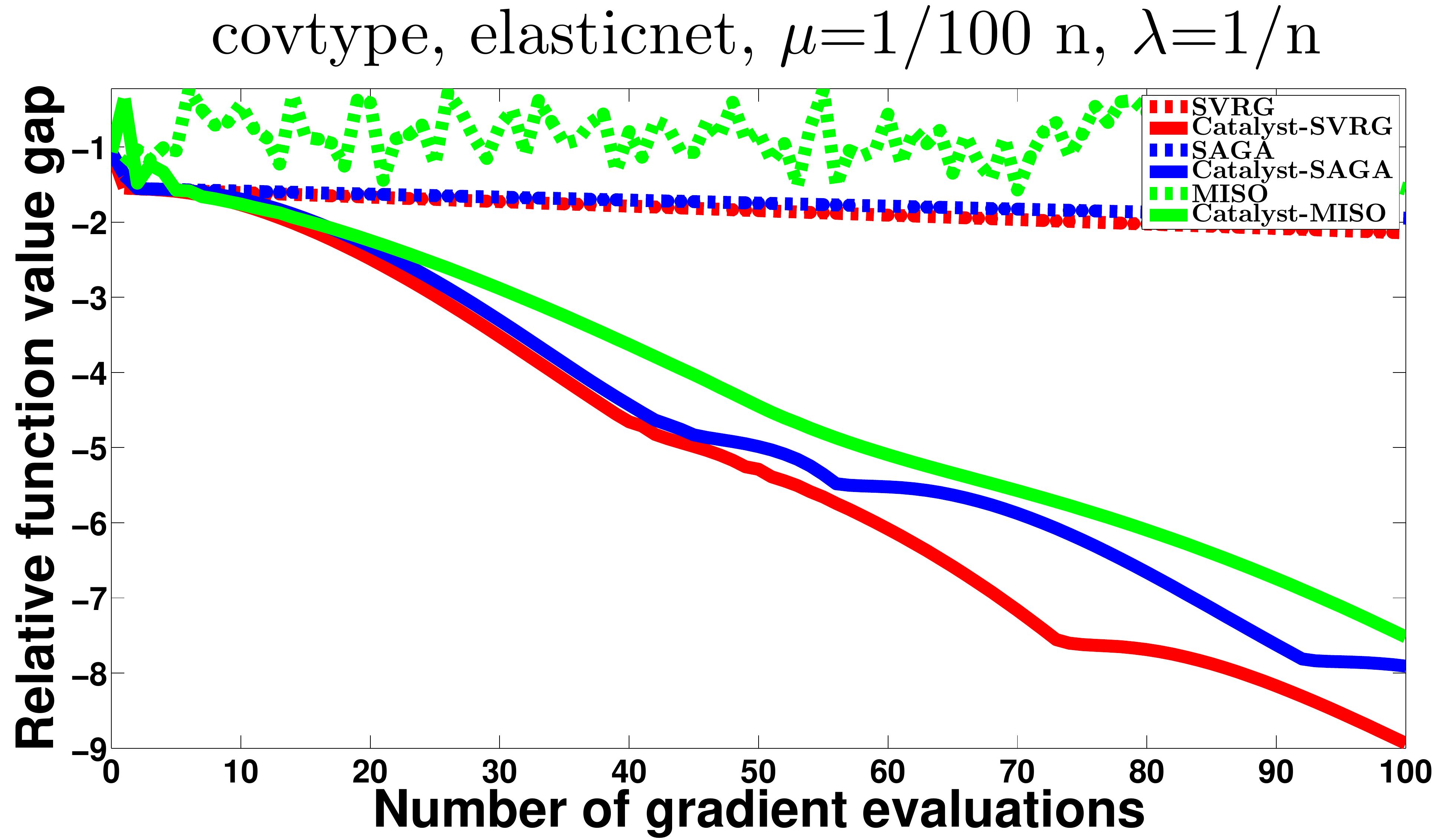}~ 
   ~~\includegraphics[width=0.31\linewidth]{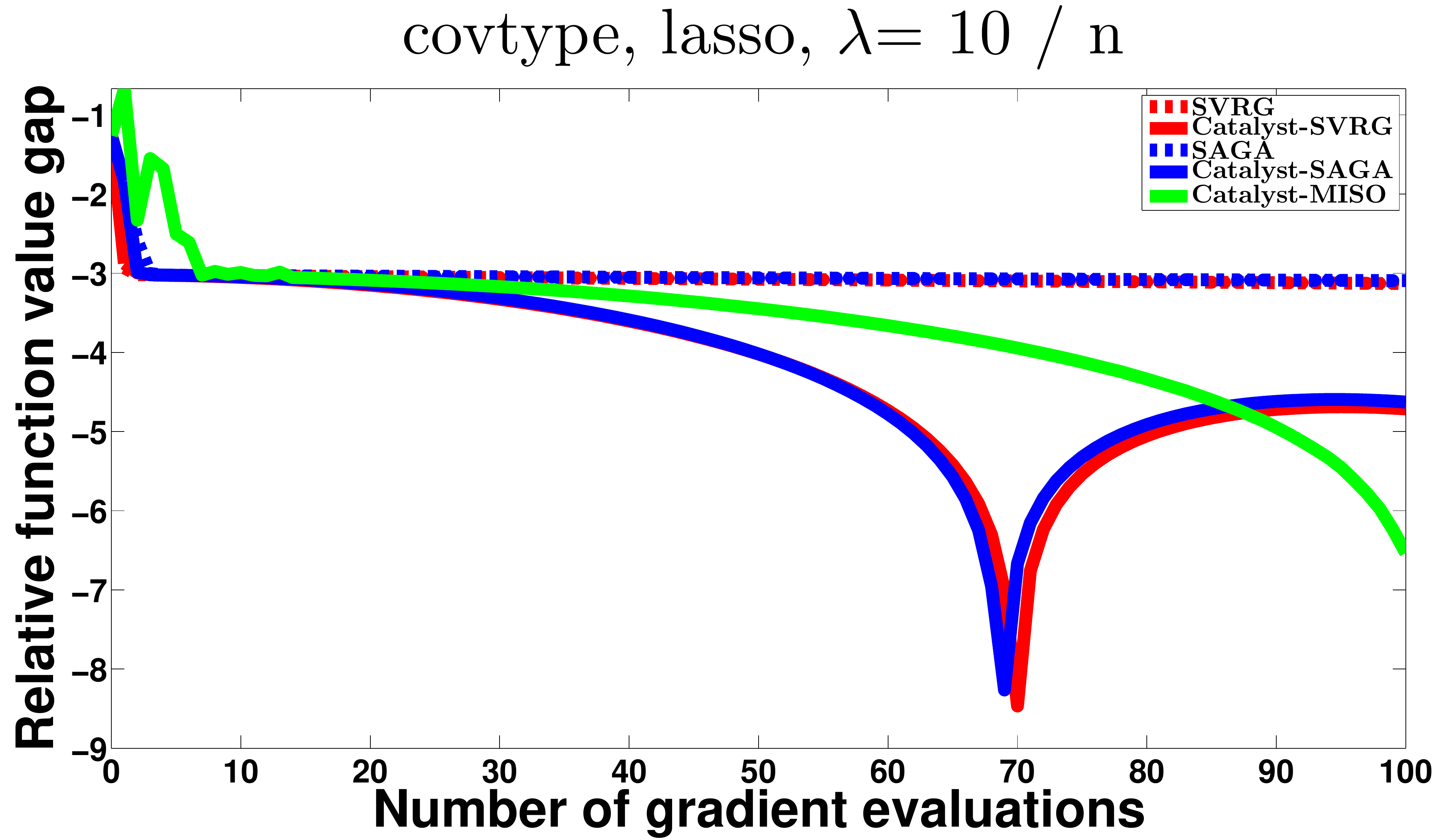}\\
   ~~\includegraphics[width=0.31\linewidth]{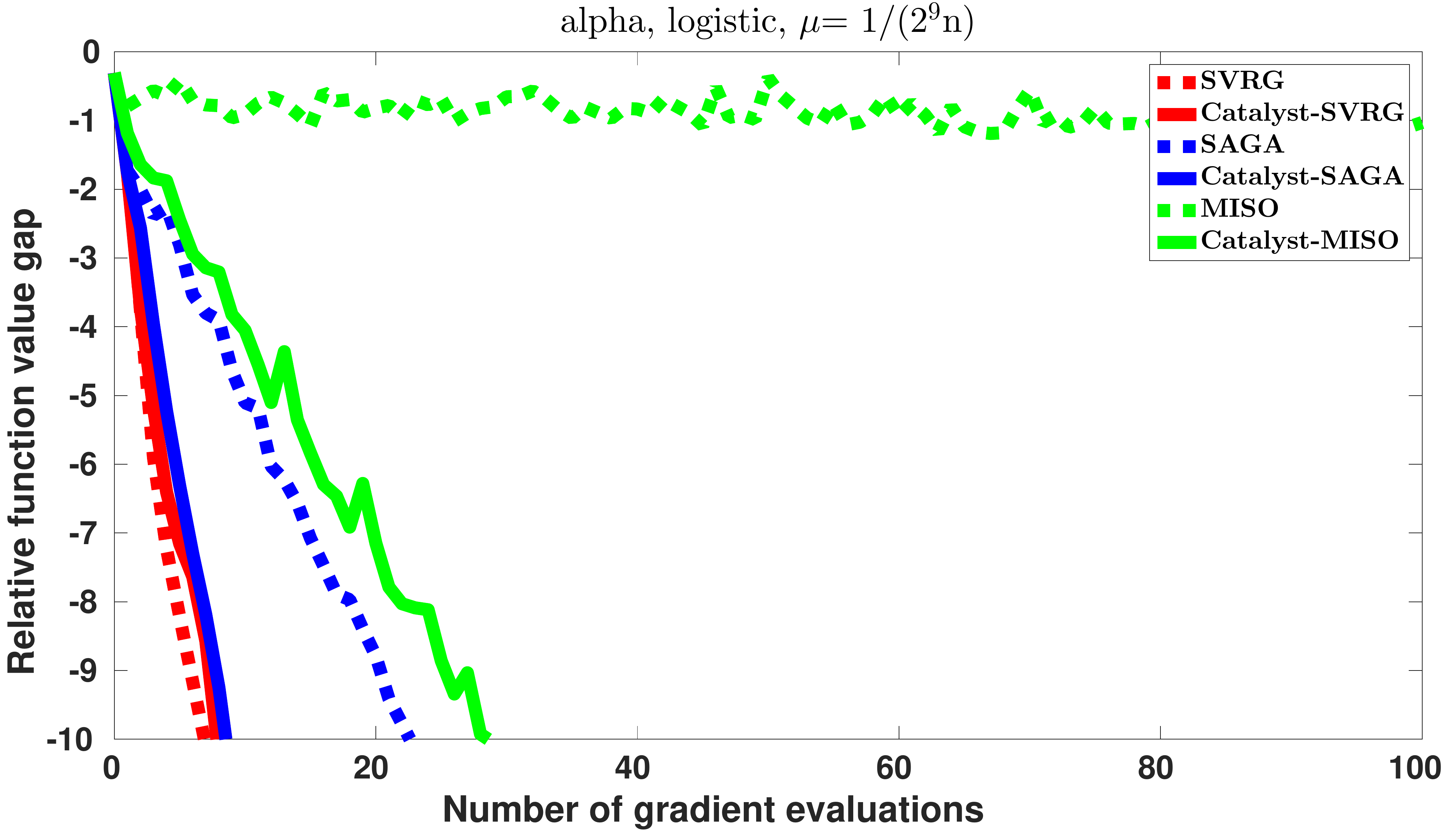}~ 
   ~~\includegraphics[width=0.31\linewidth]{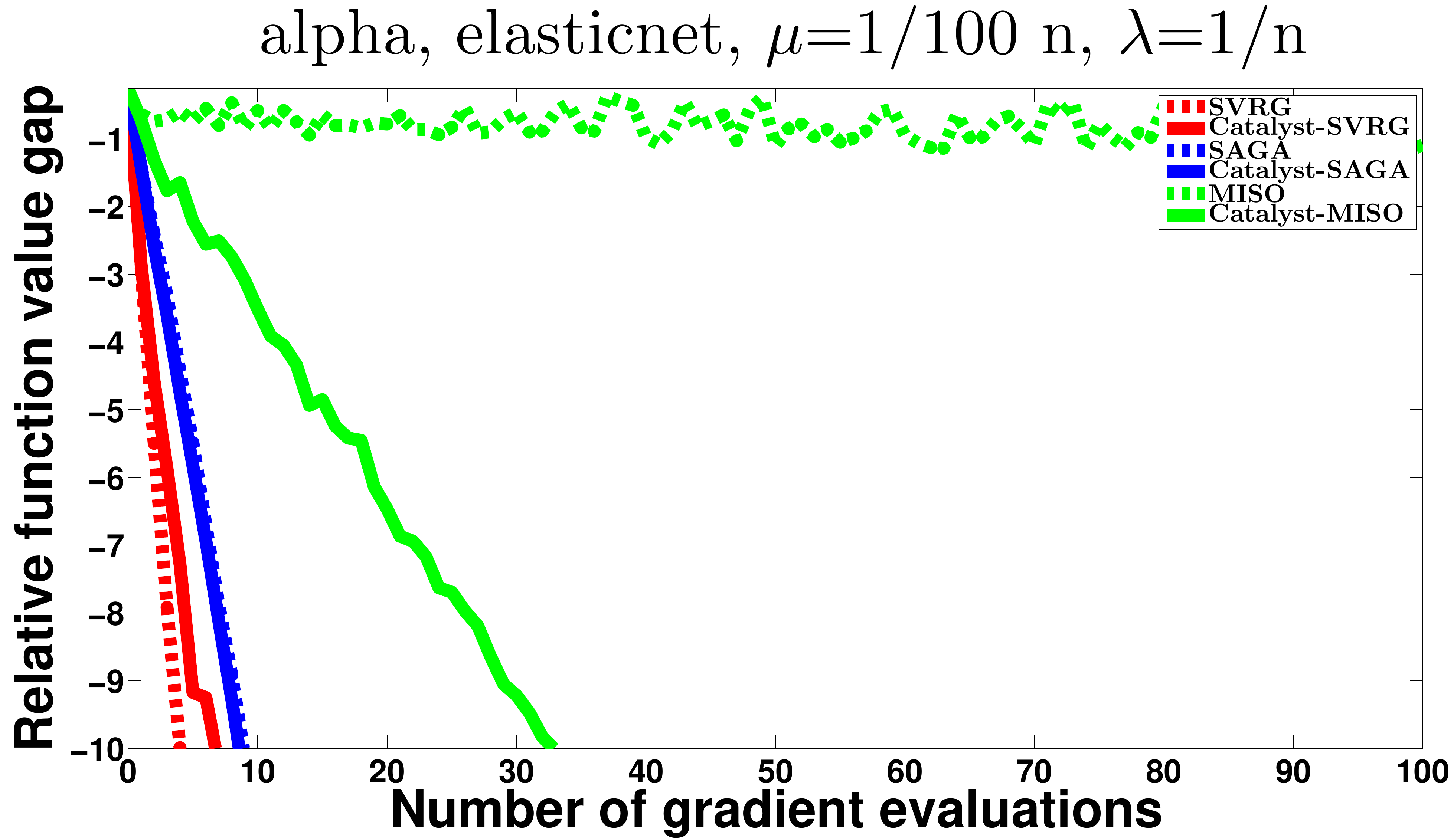}~ 
   ~~\includegraphics[width=0.31\linewidth]{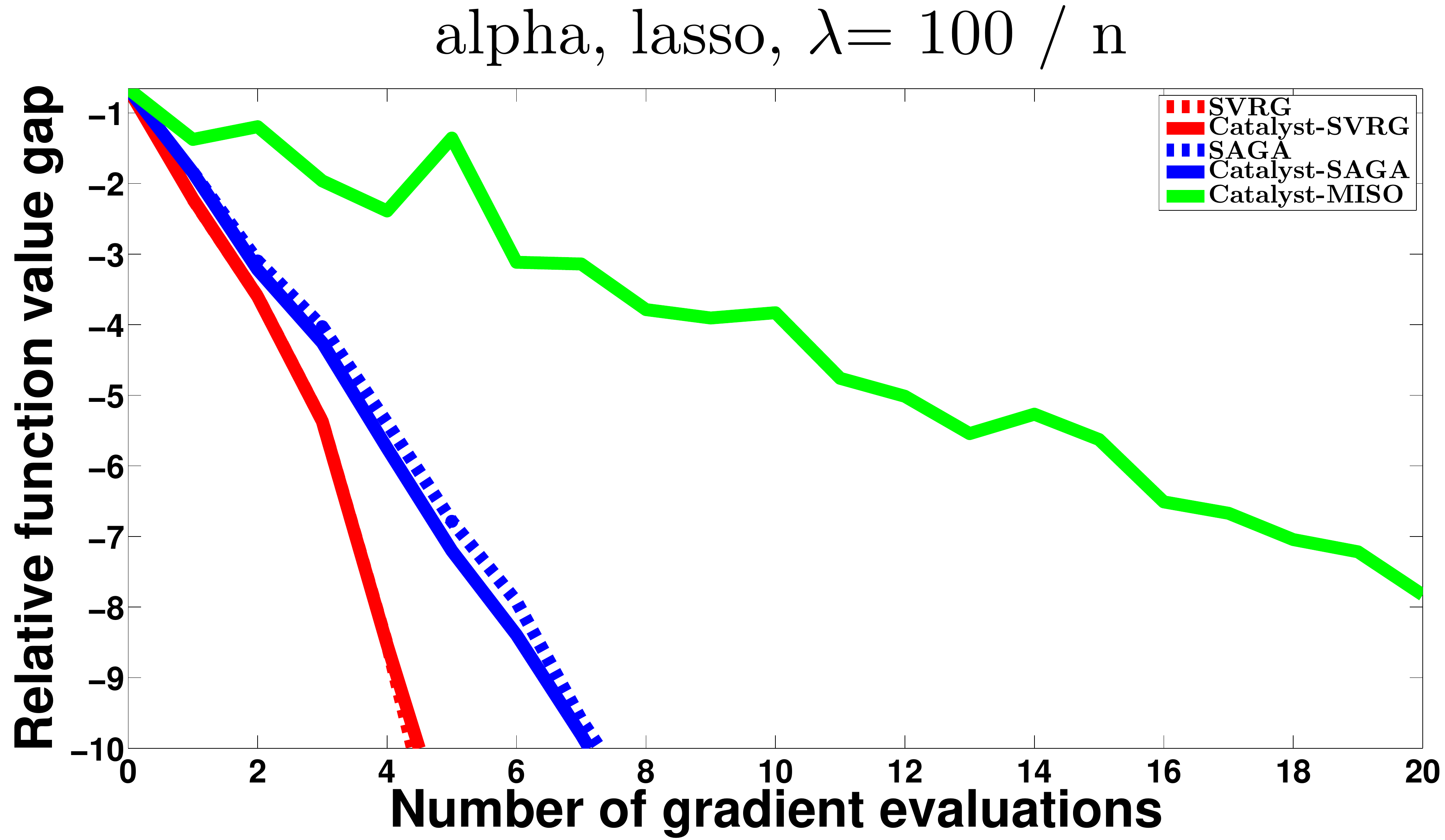}\\
   ~~\includegraphics[width=0.31\linewidth]{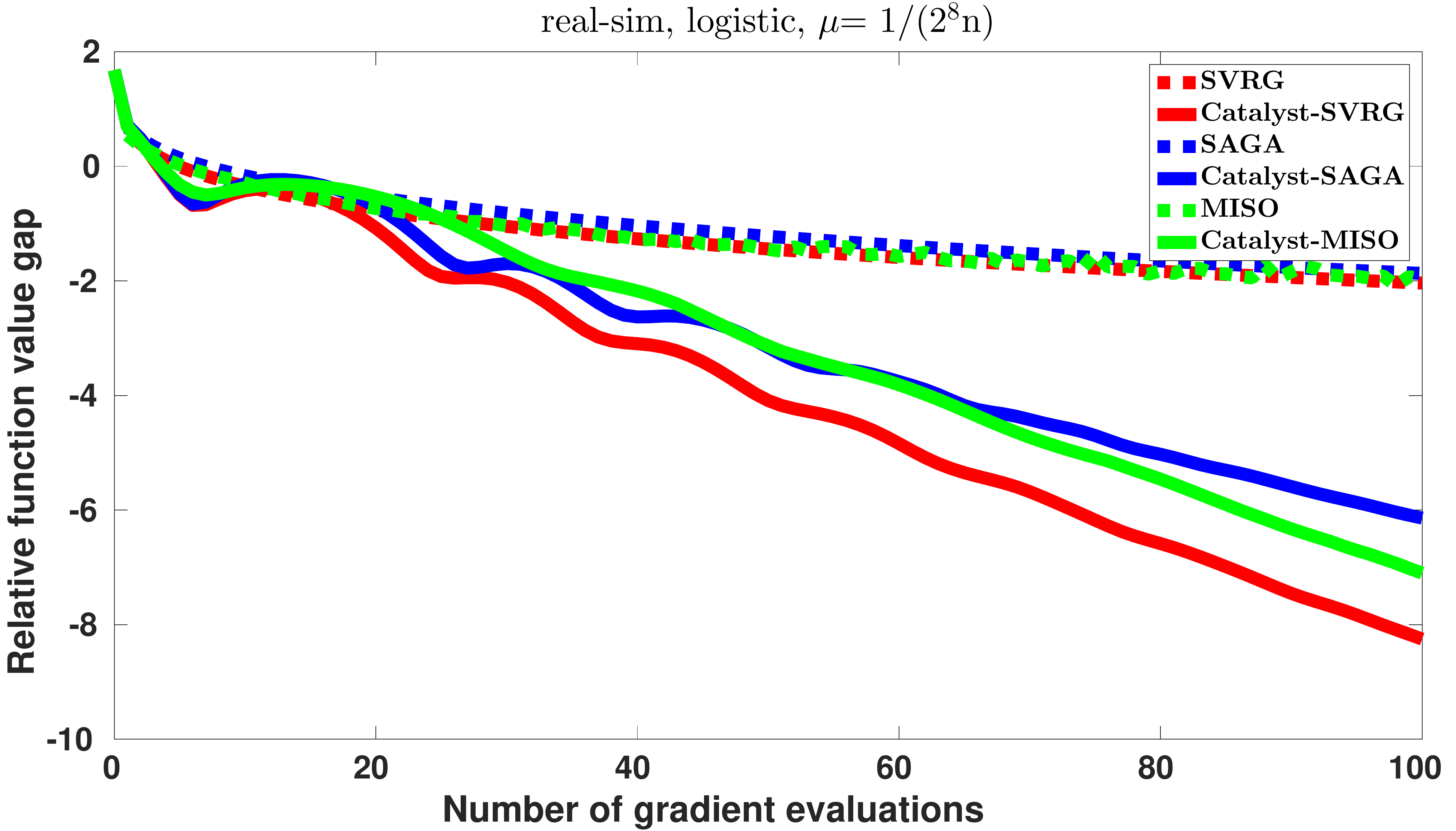}~ 
   ~~\includegraphics[width=0.31\linewidth]{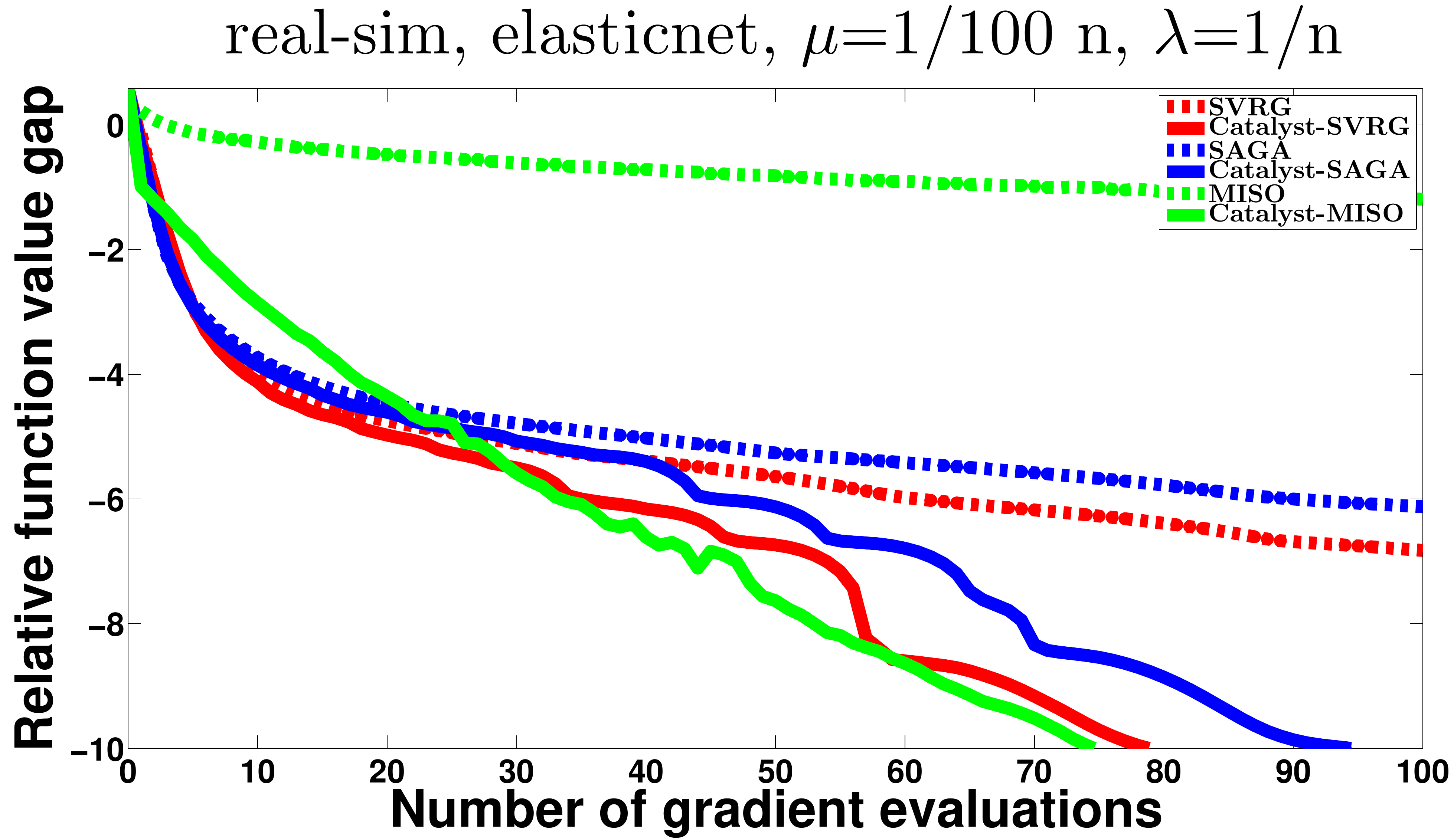}~ 
   ~~\includegraphics[width=0.31\linewidth]{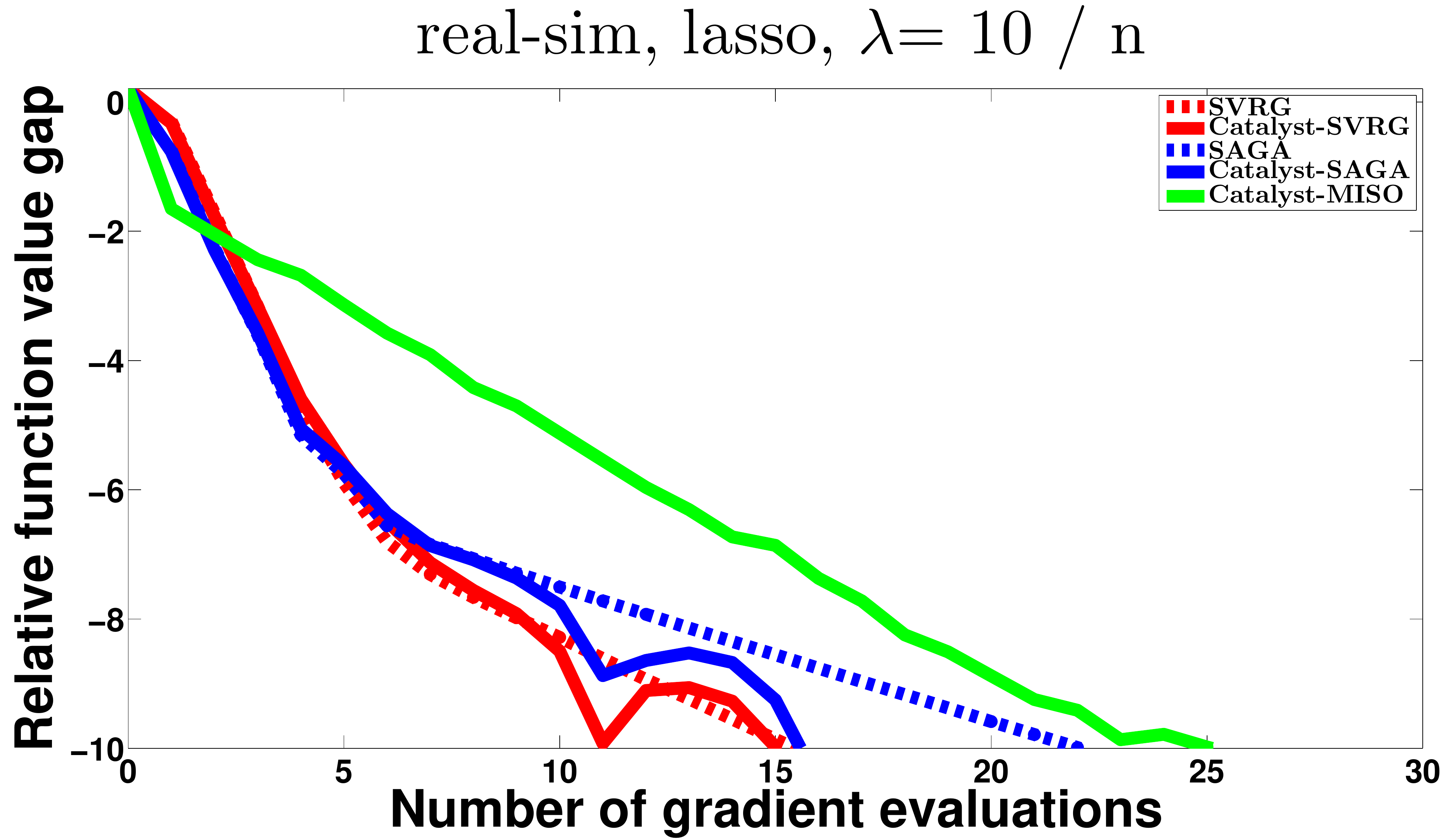}\\
   ~~\includegraphics[width=0.31\linewidth]{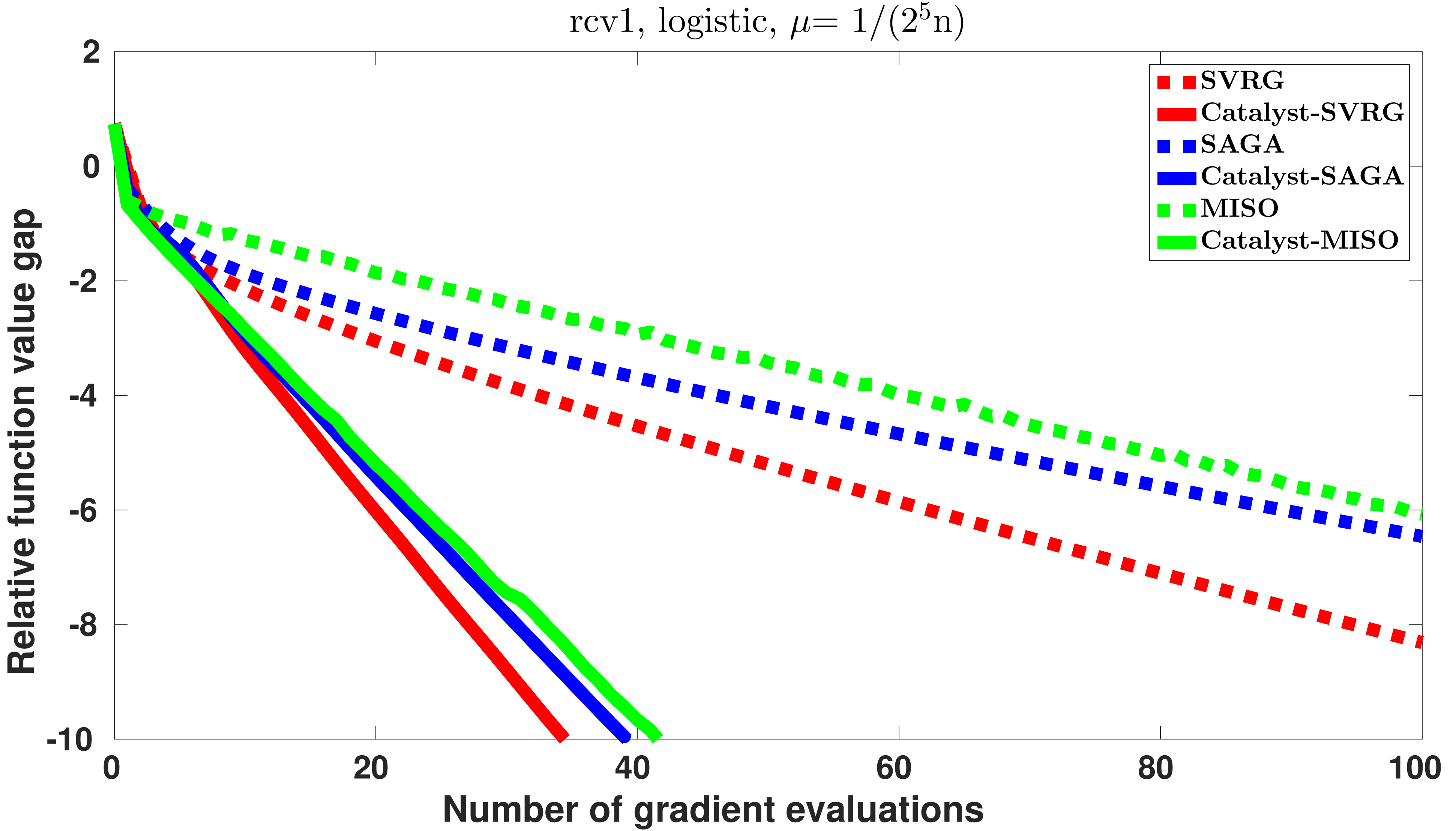}~ 
   ~~\includegraphics[width=0.31\linewidth]{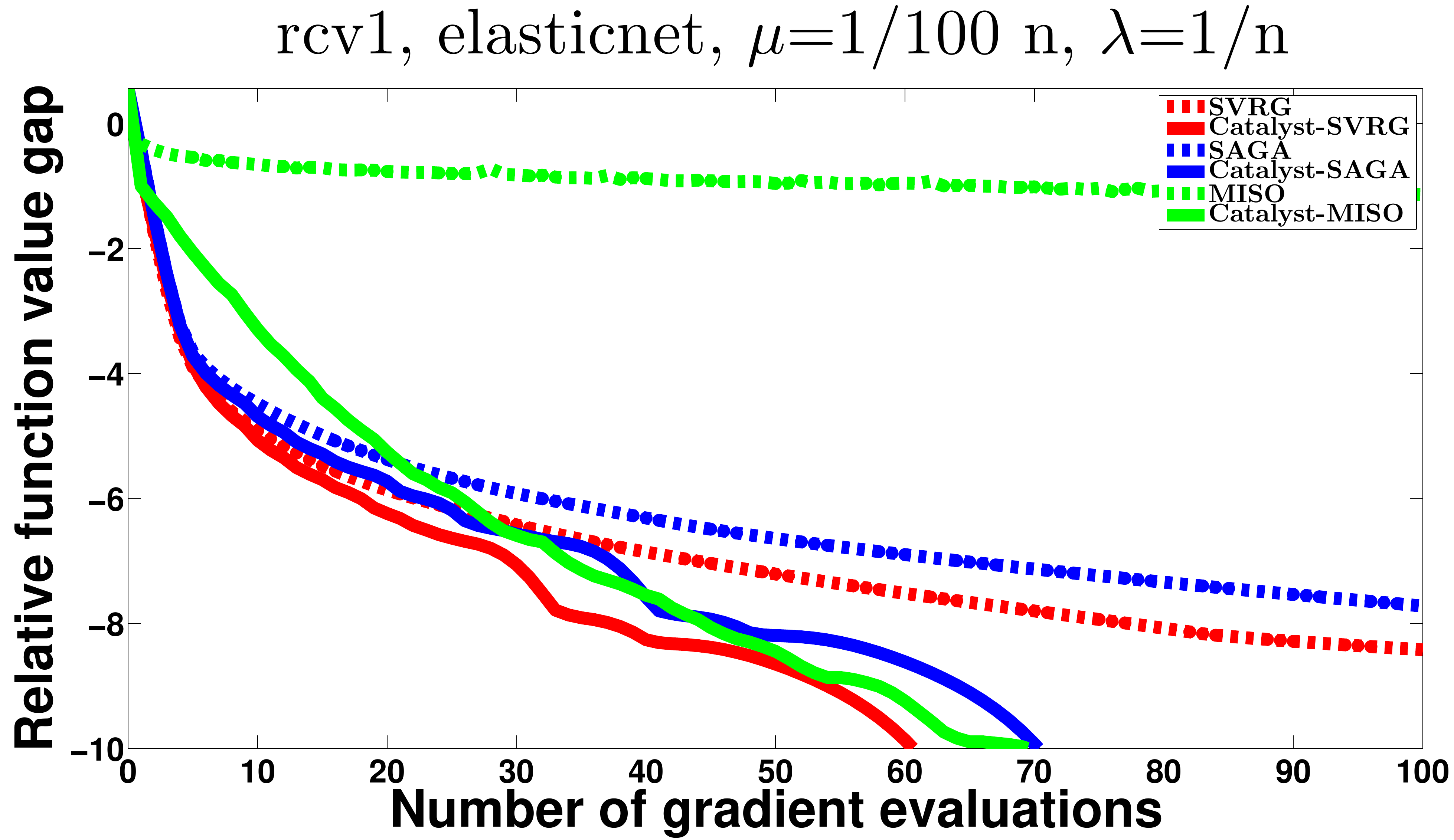}~ 
   ~~\includegraphics[width=0.31\linewidth]{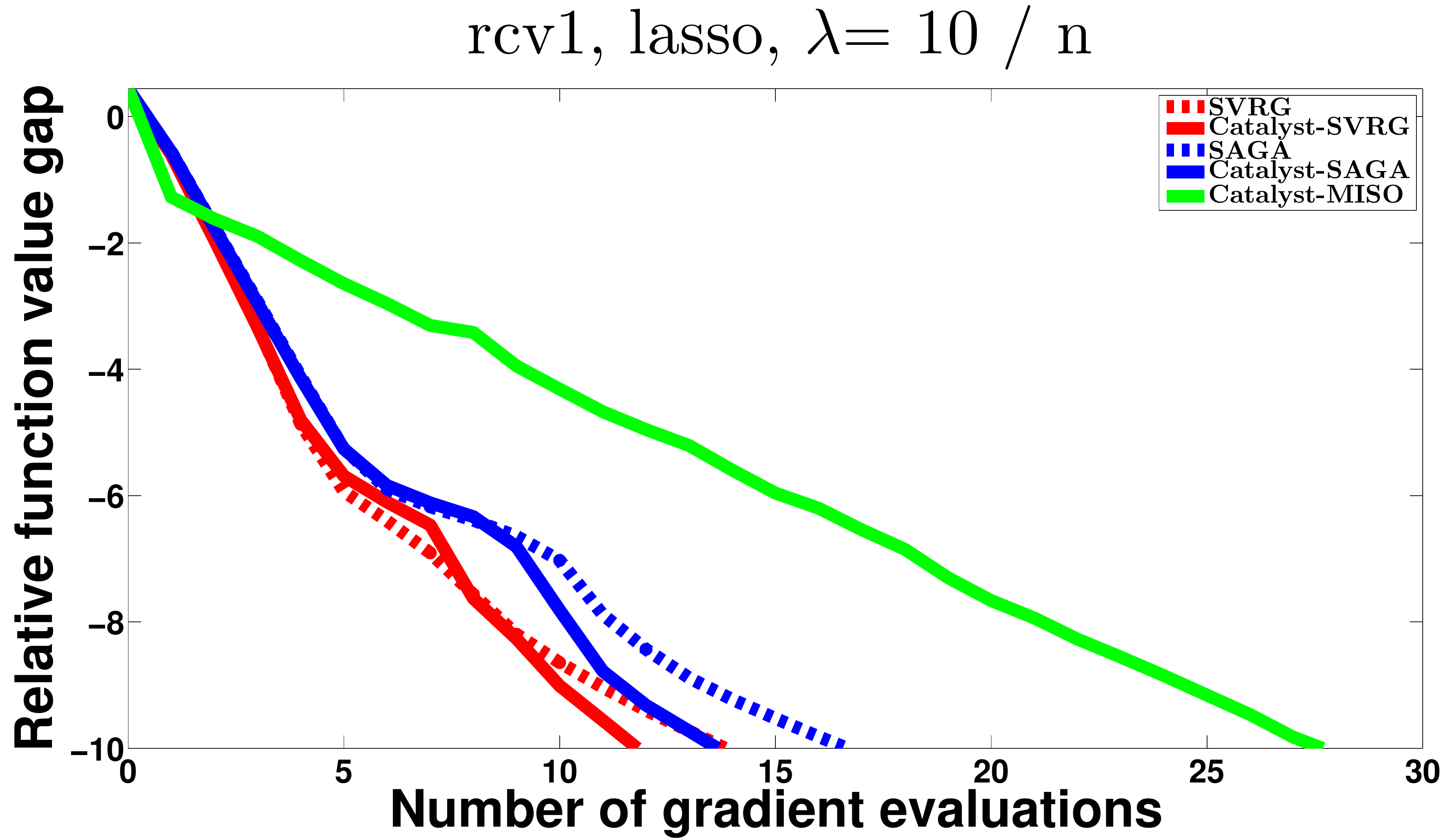}\\
   ~~\includegraphics[width=0.31\linewidth]{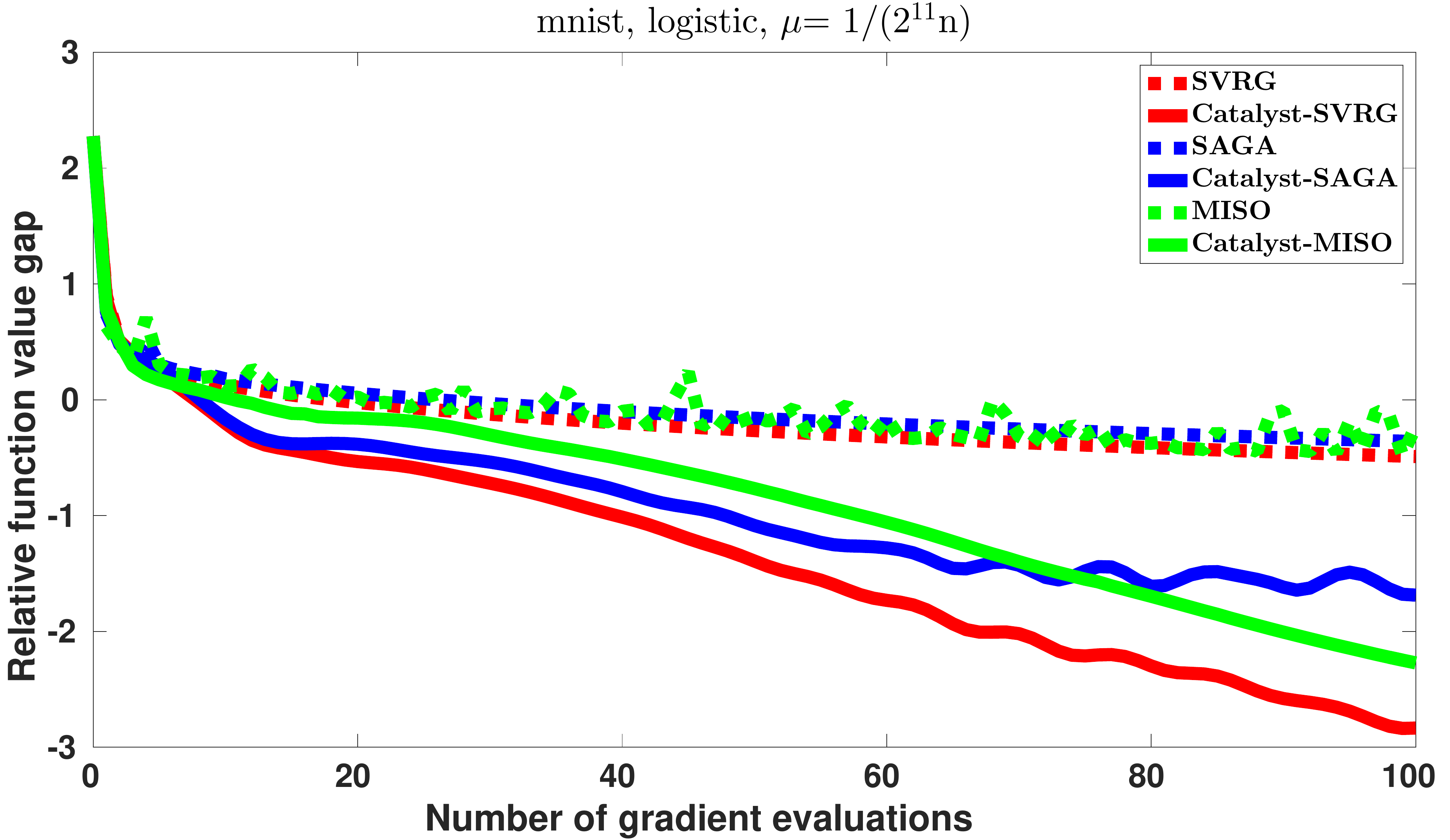}~ 
   ~~\includegraphics[width=0.31\linewidth]{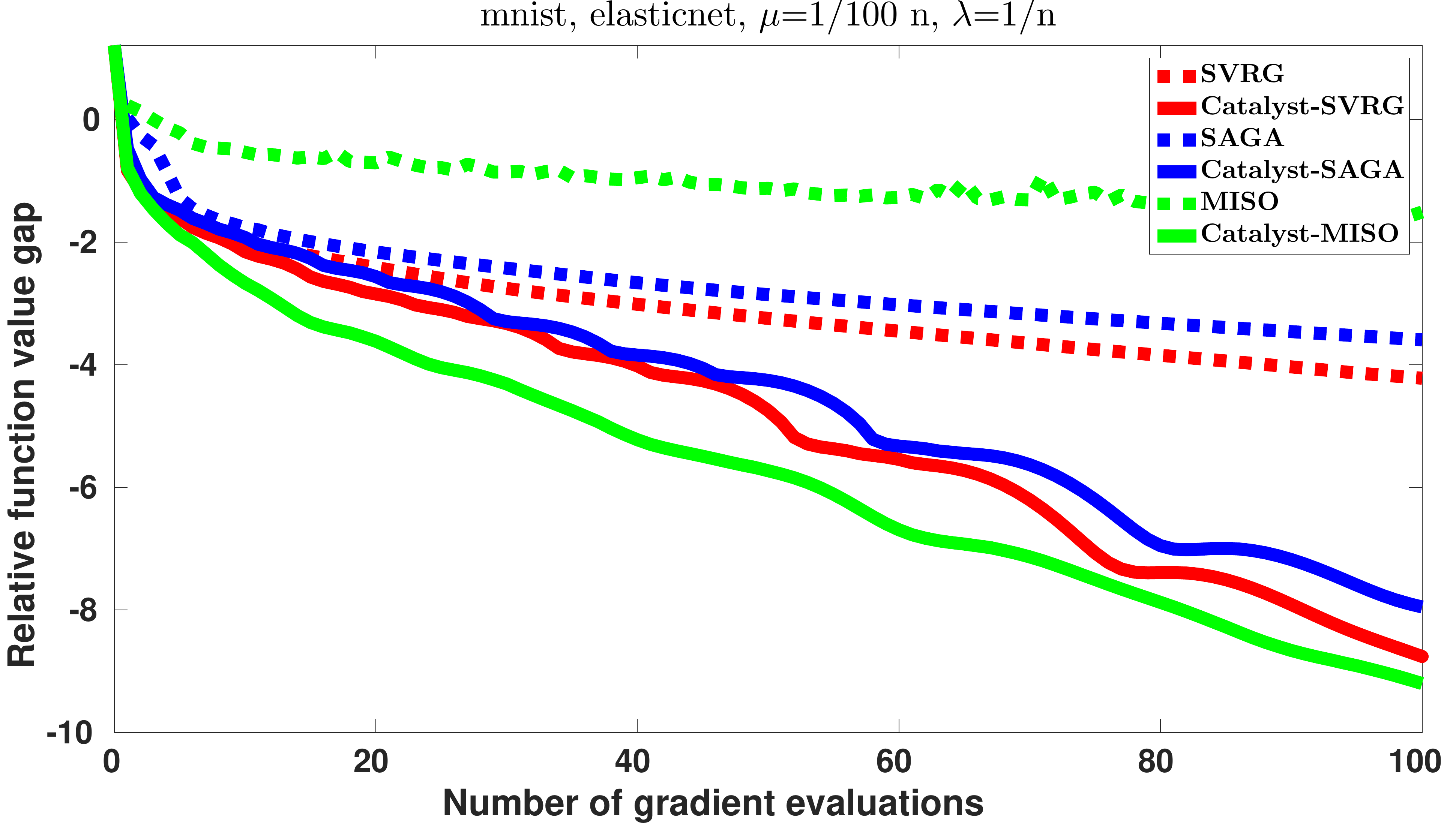}~ 
   ~~\includegraphics[width=0.31\linewidth]{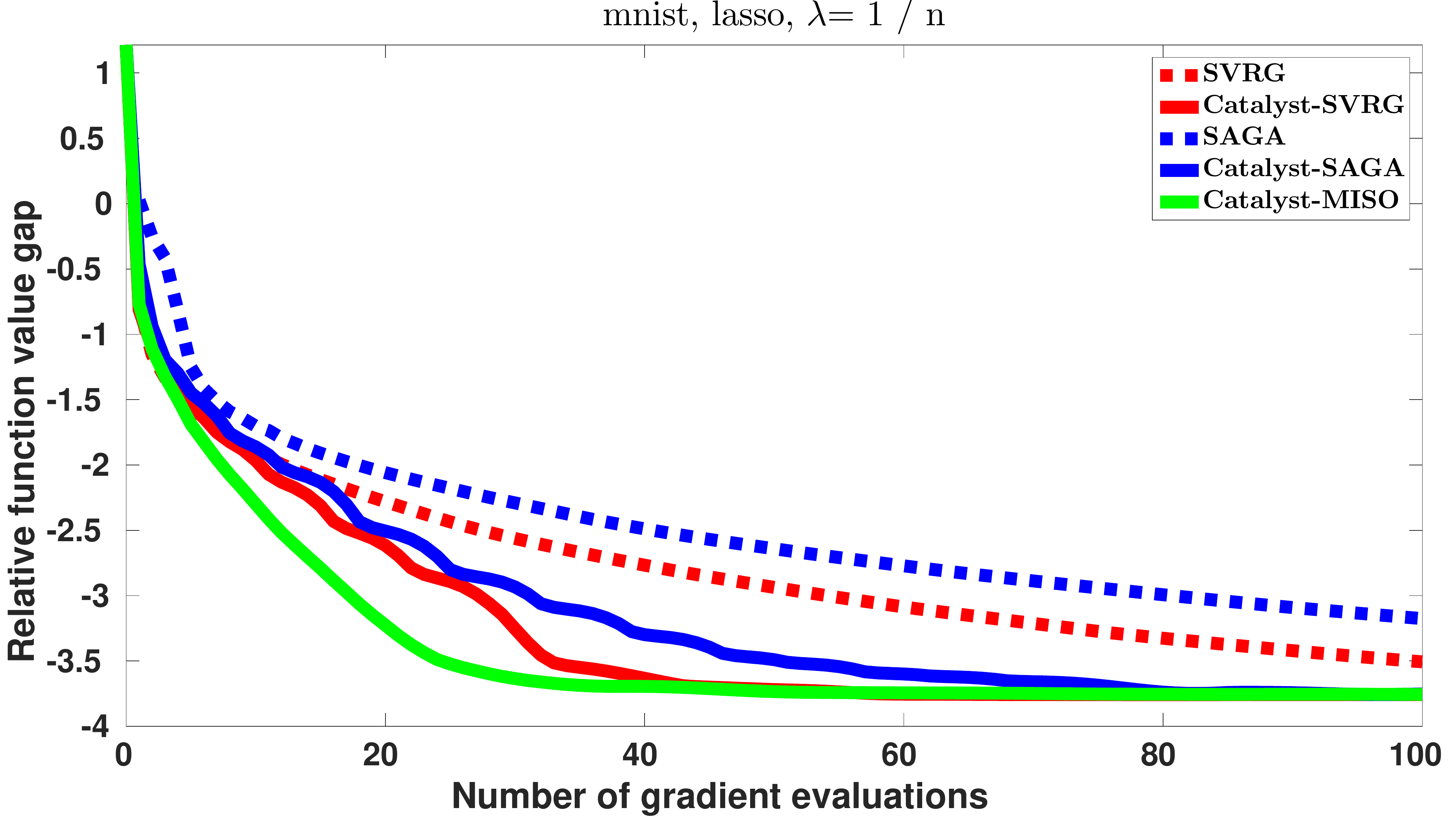}\\
   ~~\includegraphics[width=0.31\linewidth]{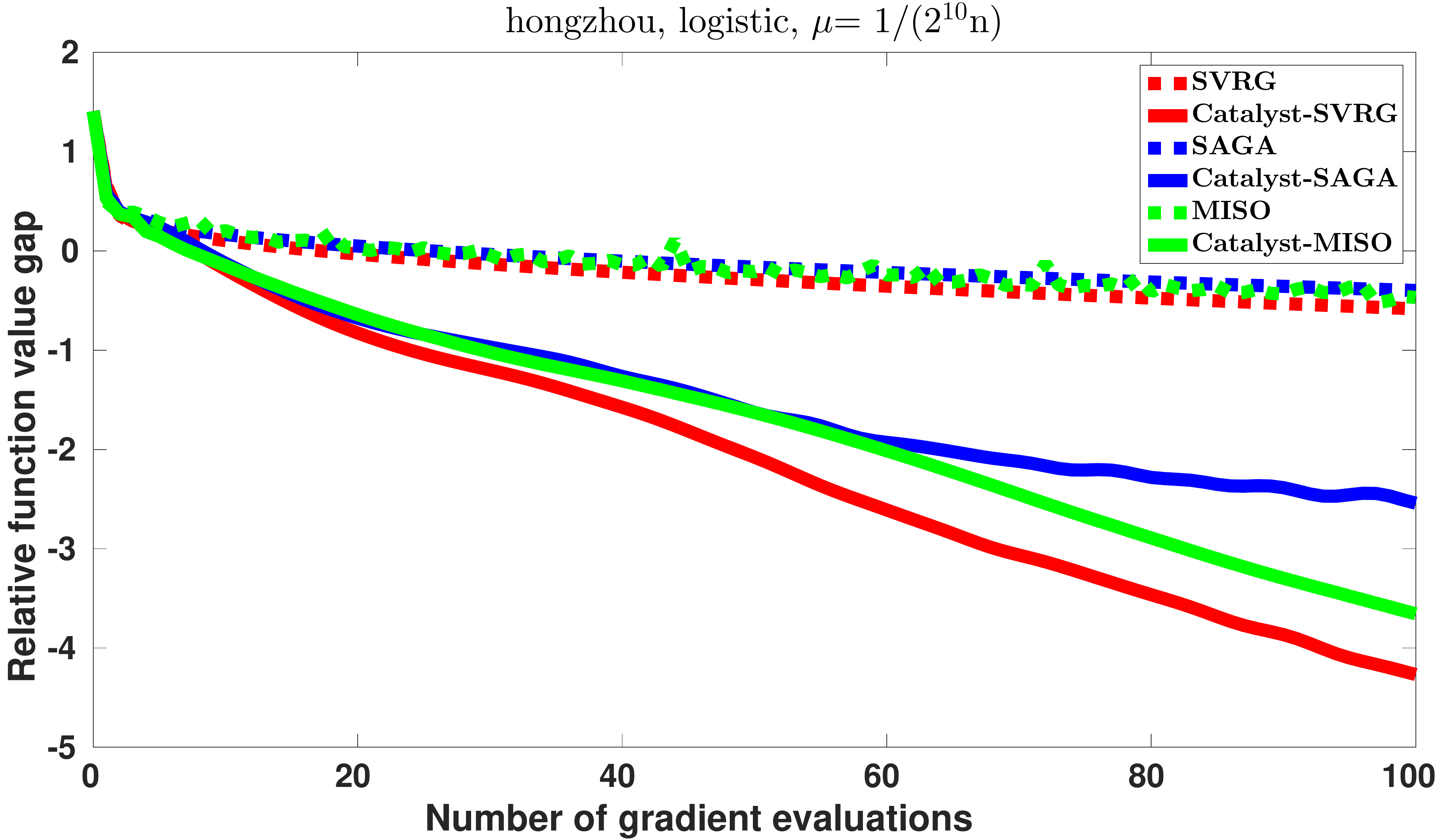}~ 
   ~~\includegraphics[width=0.31\linewidth]{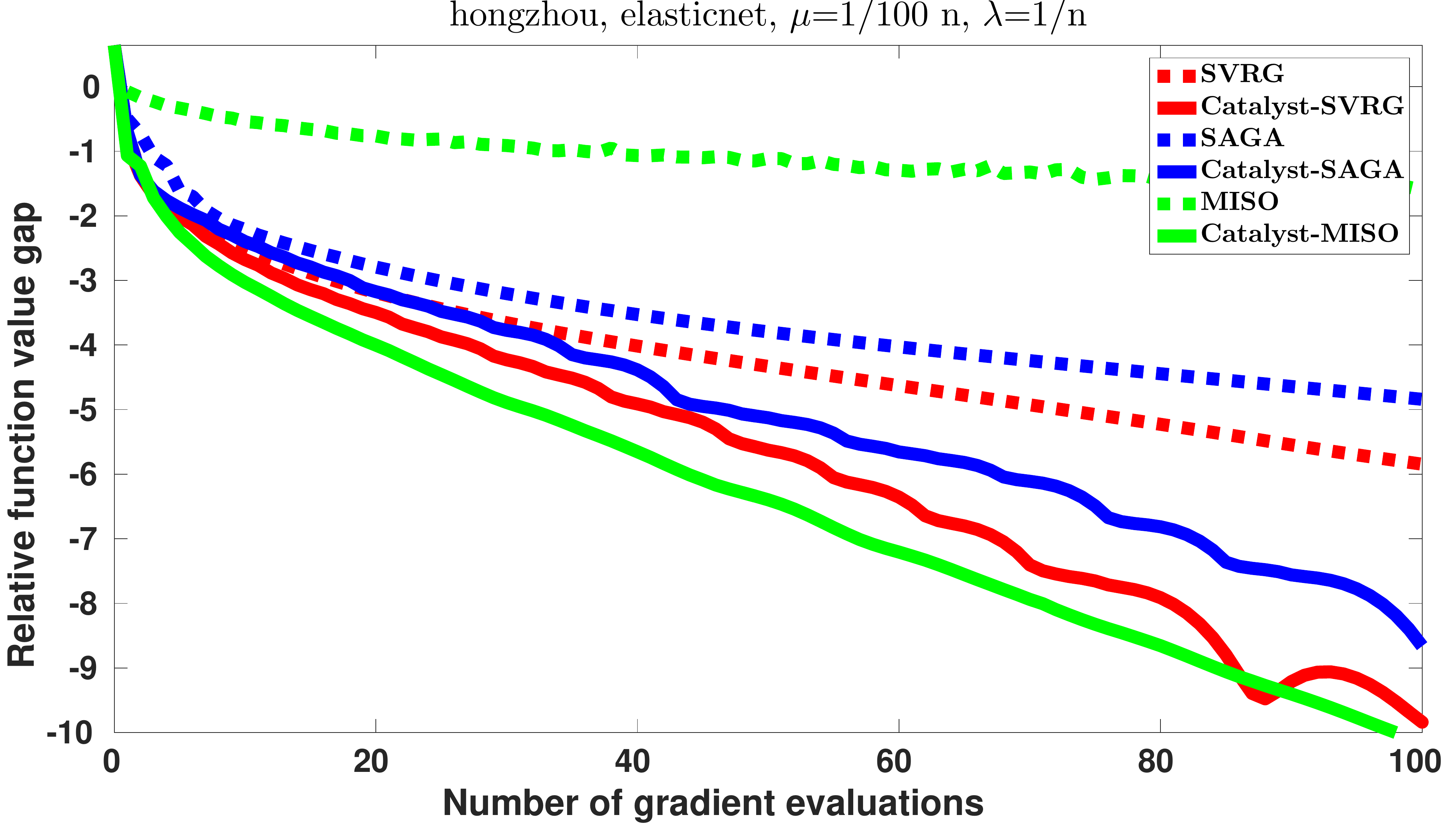}~ 
   ~~\includegraphics[width=0.31\linewidth]{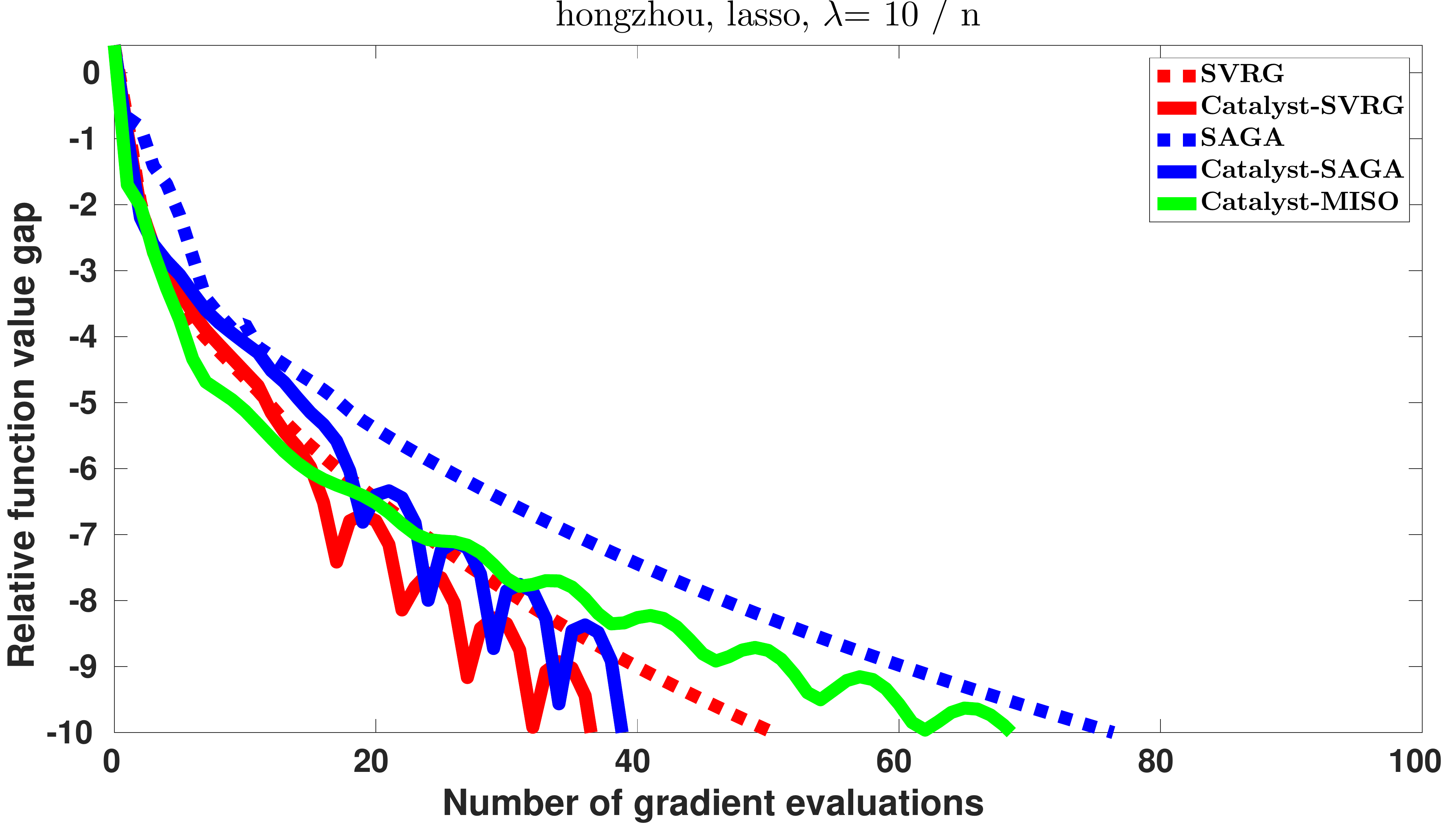}\\
   \caption{Experimental study of the performance of Catalyst applying to SVRG, SAGA and MISO. The dashed lines correspond to the original algorithms and the solid lines correspond to accelerated algorithms by applying Catalyst. We plot the relative function value gap~$(f(x_k)-f^\star)/f^\star$ in the number of gradient evaluations, on a logarithmic scale.}\label{catalyst:fig:diff_algo}
\end{figure}

\paragraph{Observations.} In Figure~\ref{catalyst:fig:diff_algo}, we observe that by applying Catalyst, we accelerate the original algorithms up to the limitations discussed above (comparing the dashed line and the solid line of the same color). In three data sets (covtype, real-sim and rcv1), significant improvements are achieved as expected by the theory for the ill-conditioned problems in logistic regression and Elastic-net. For data set alpha, we remark that an relative accuracy in the order $10^{-10}$ is attained in less than 10 iterations. This suggests that the problems is in fact well-conditioned and there is some hidden strong convexity for this data set. Thus, the incremental algorithms like SVRG or SAGA are already optimal under this situation and no further improvement can be obtained by applying Catalyst. 

\subsection{Empirical Effect on the Generalization Error} \label{subsec:comparison_test}
A natural question that applies to all incremental methods is whether or not the acceleration that we may see for minimizing an empirical risk on \emph{training} data affects the objective function and the test accuracy on new unseen \emph{test} data.
To answer this question, we consider the logistic regression formulation with the regularization parameter $\mu^\star$ obtained by cross-validation. 
Then, we cut each data set into $80\%$ of training data and set aside $20\%$ of the data point as test data.

\begin{figure}[hbtp]
   \centering
   
   ~~\includegraphics[width=0.31\linewidth]{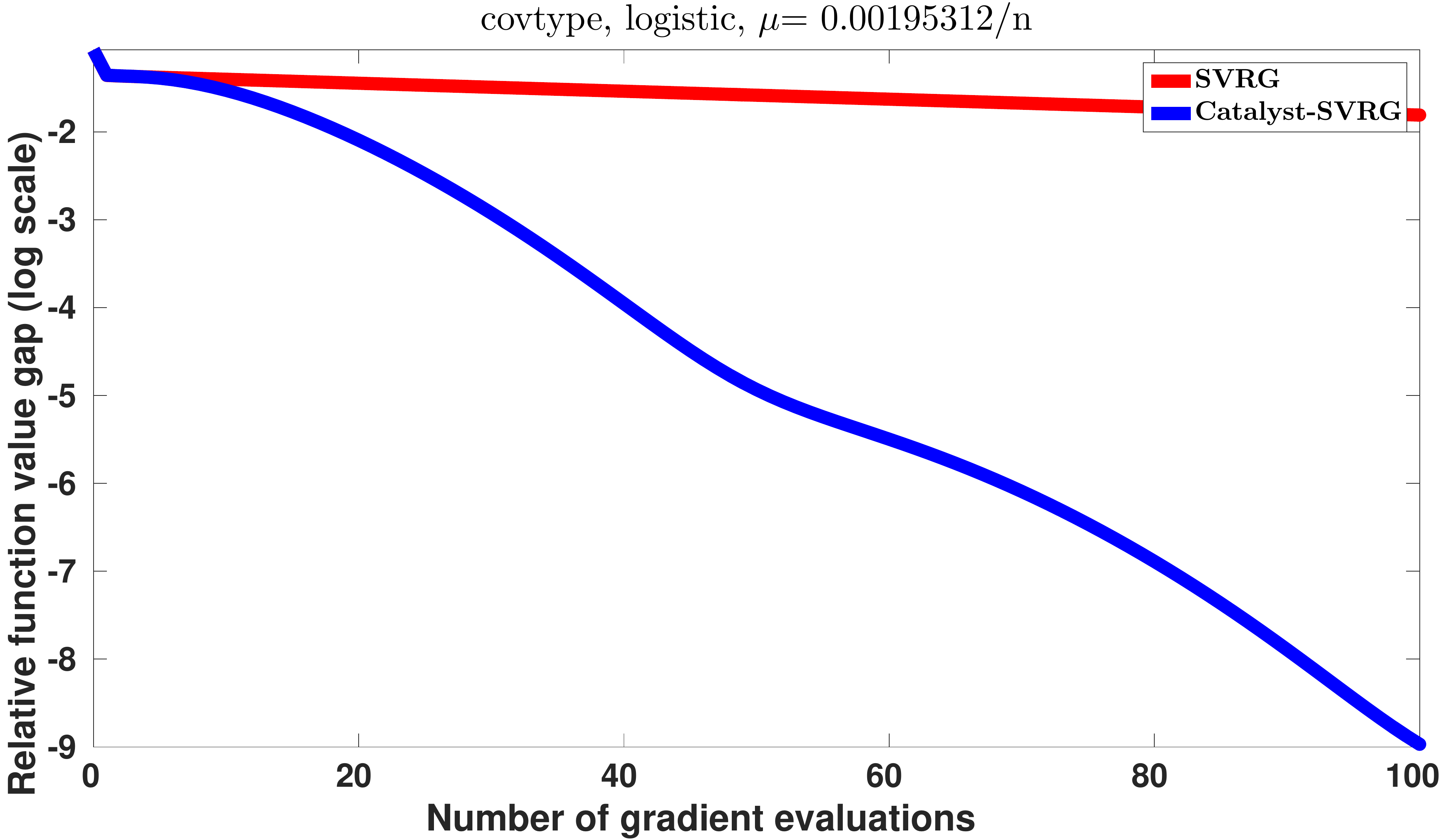}~ 
   ~~\includegraphics[width=0.31\linewidth]{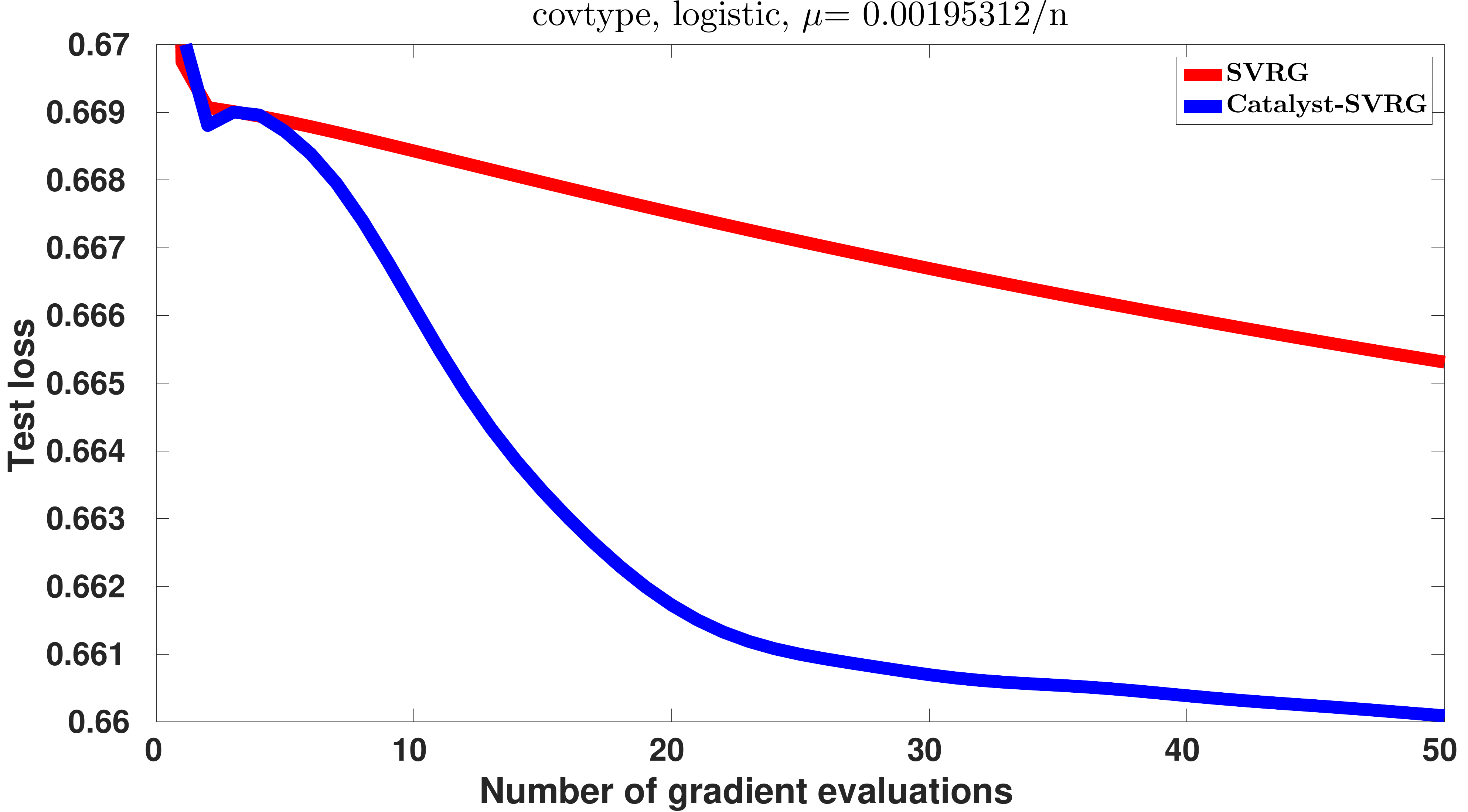}~
   ~~\includegraphics[width=0.31\linewidth]{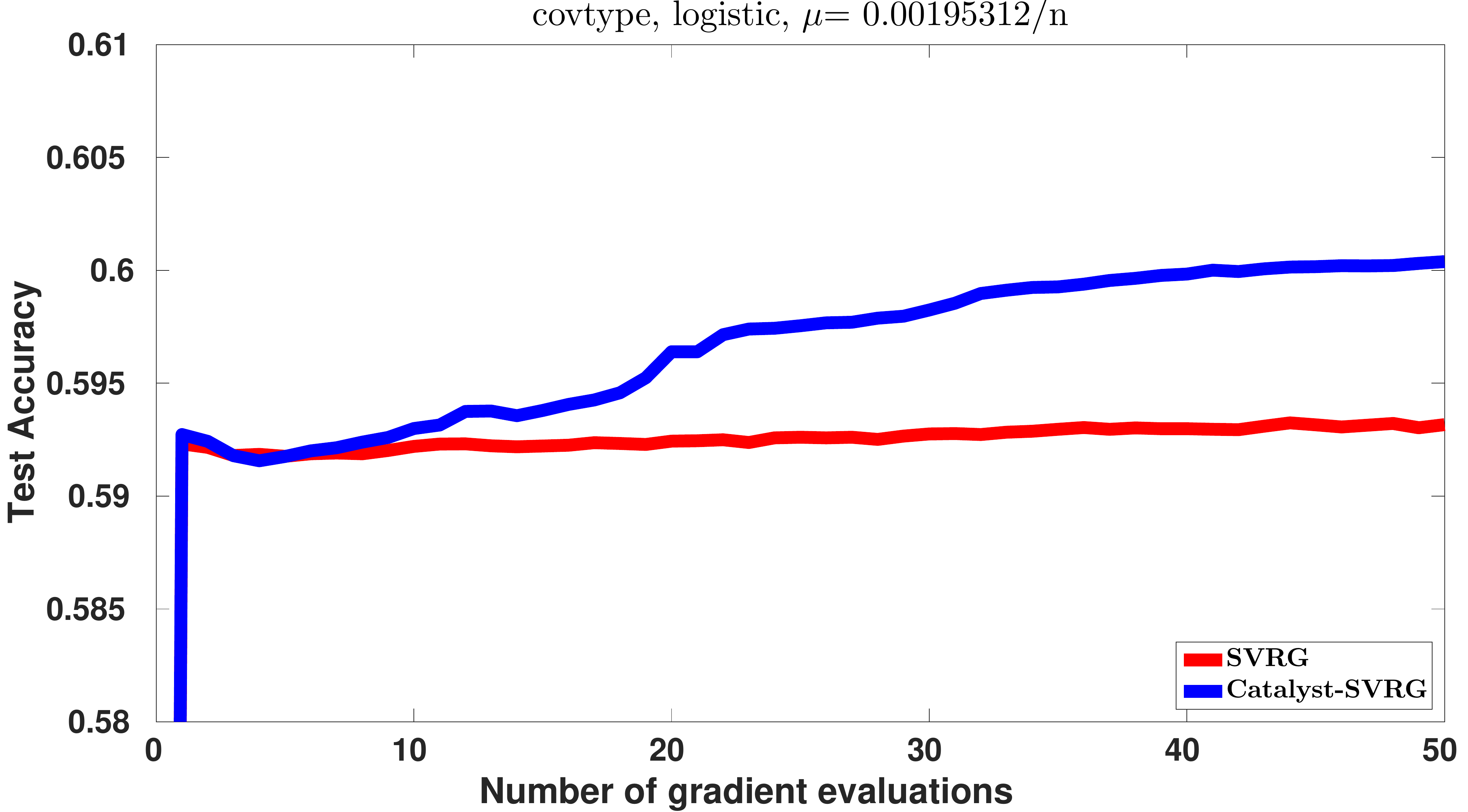}\\
   ~~\includegraphics[width=0.31\linewidth]{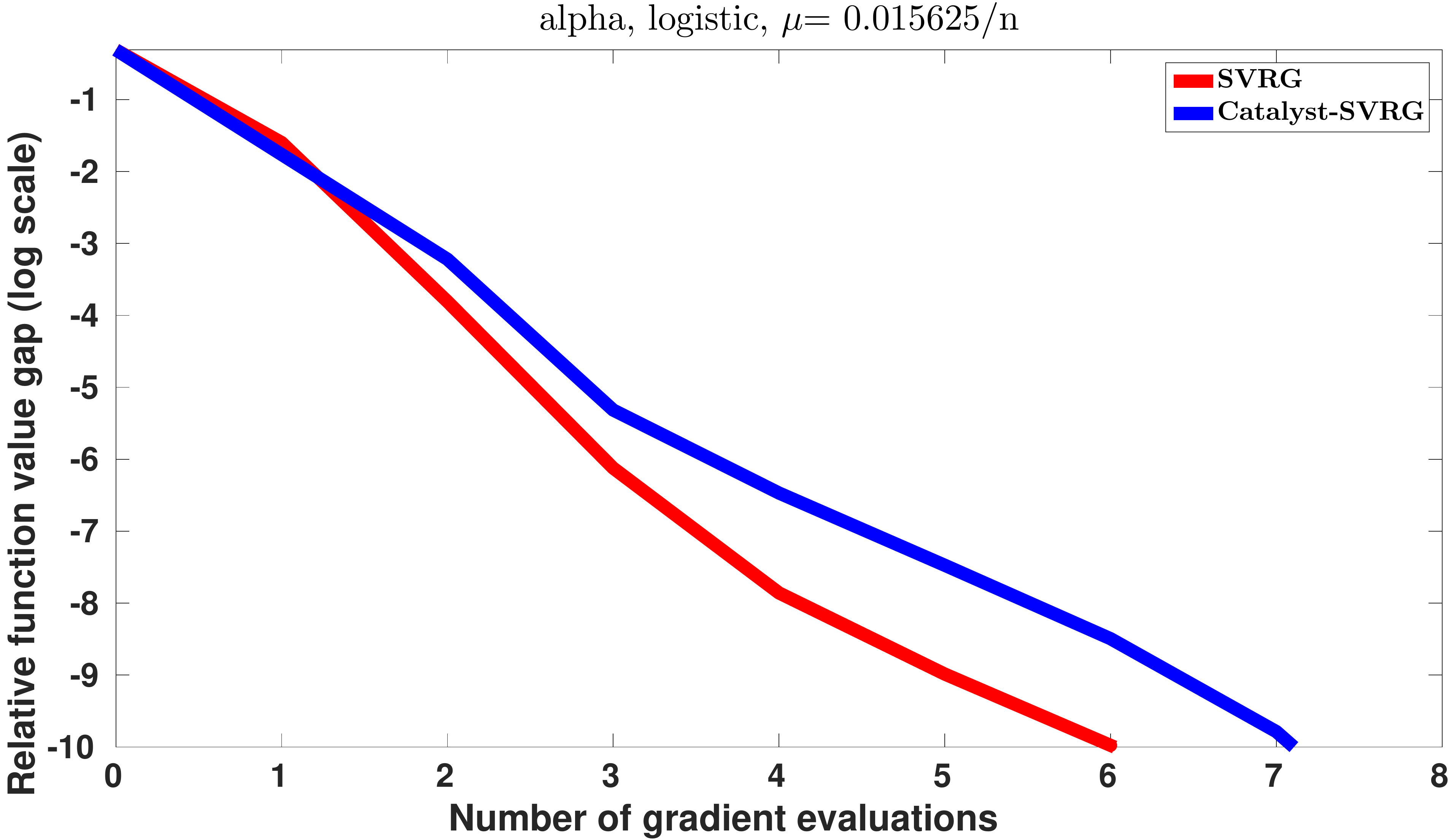}~ 
   ~~\includegraphics[width=0.31\linewidth]{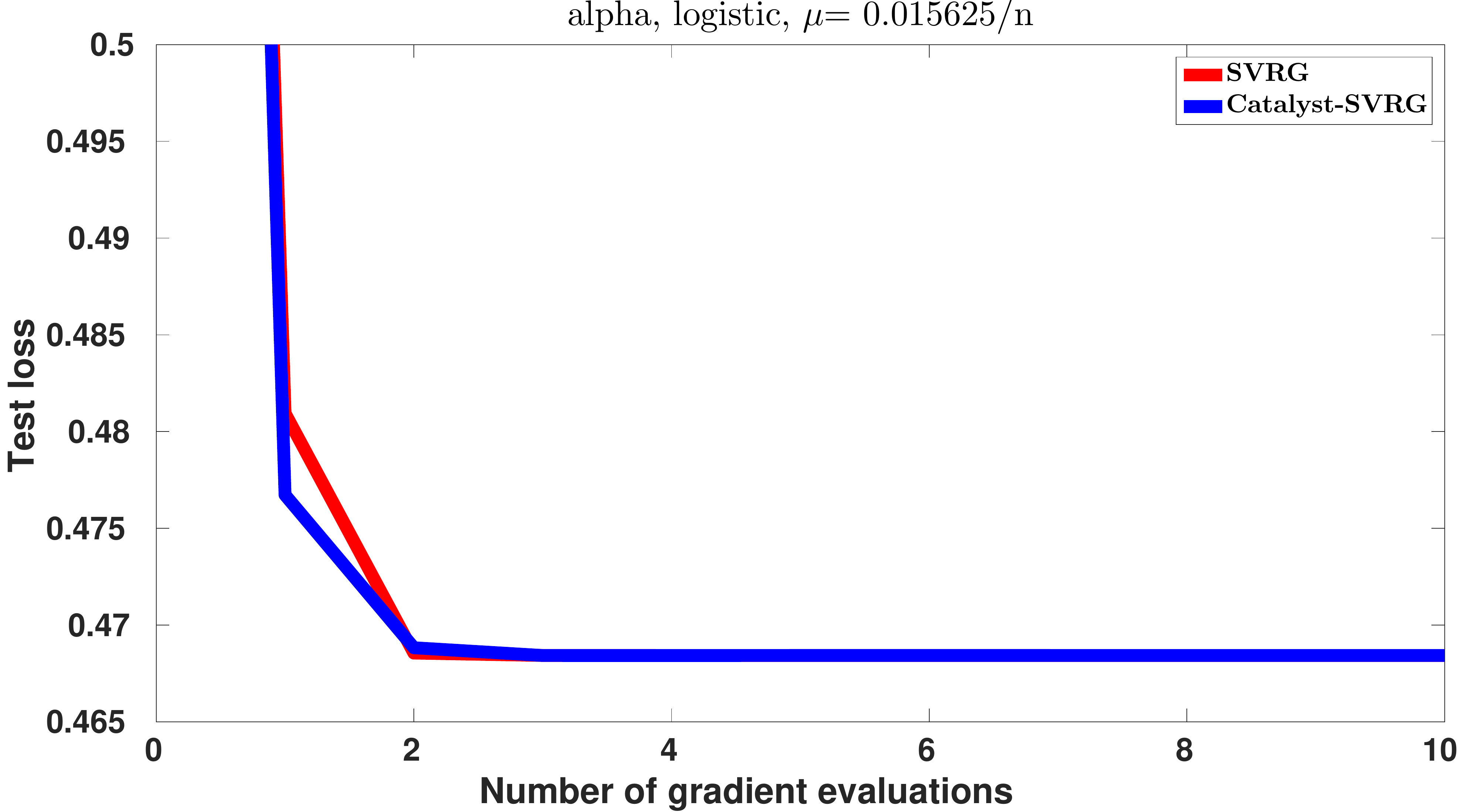}~
   ~~\includegraphics[width=0.31\linewidth]{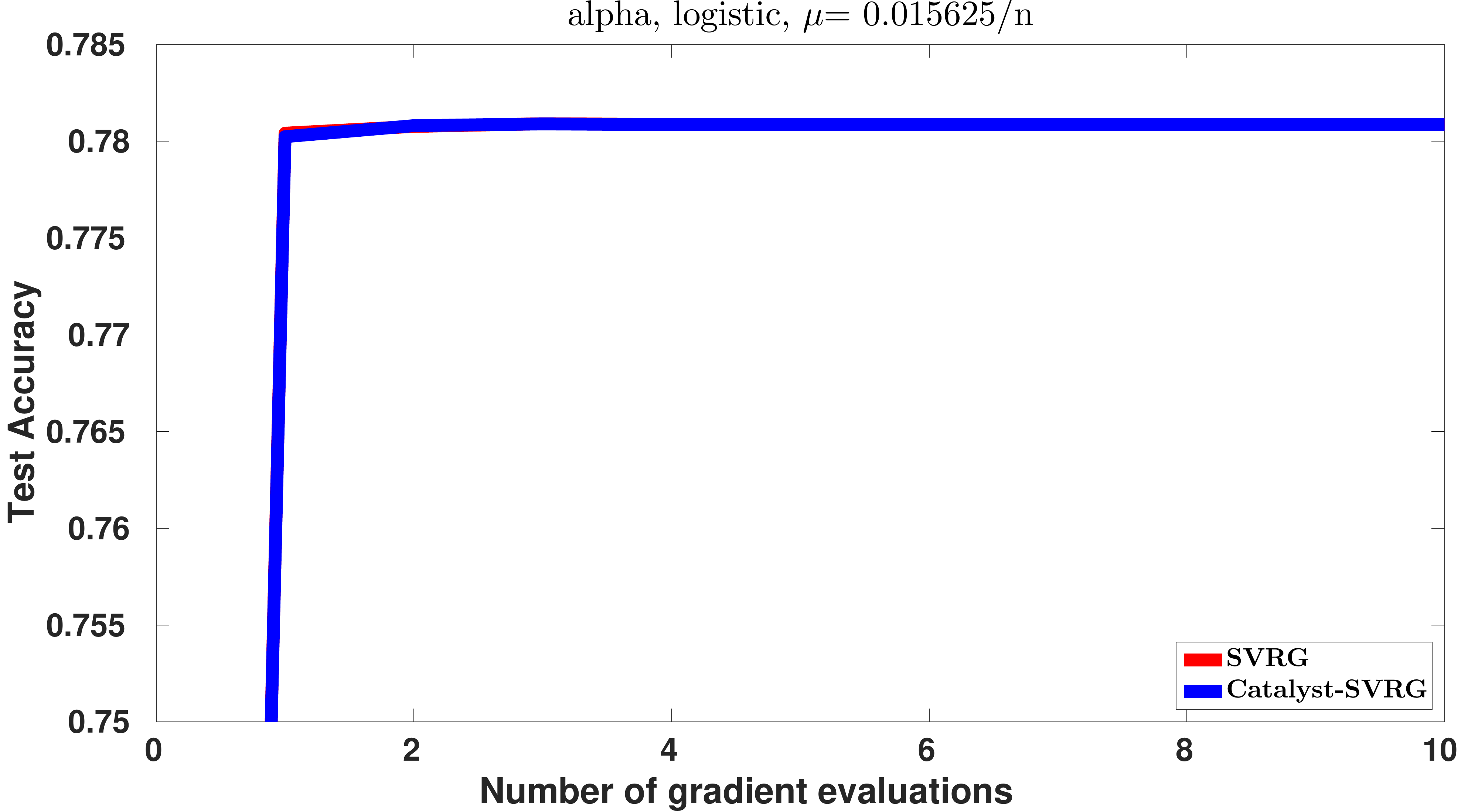}~\\
   ~~\includegraphics[width=0.31\linewidth]{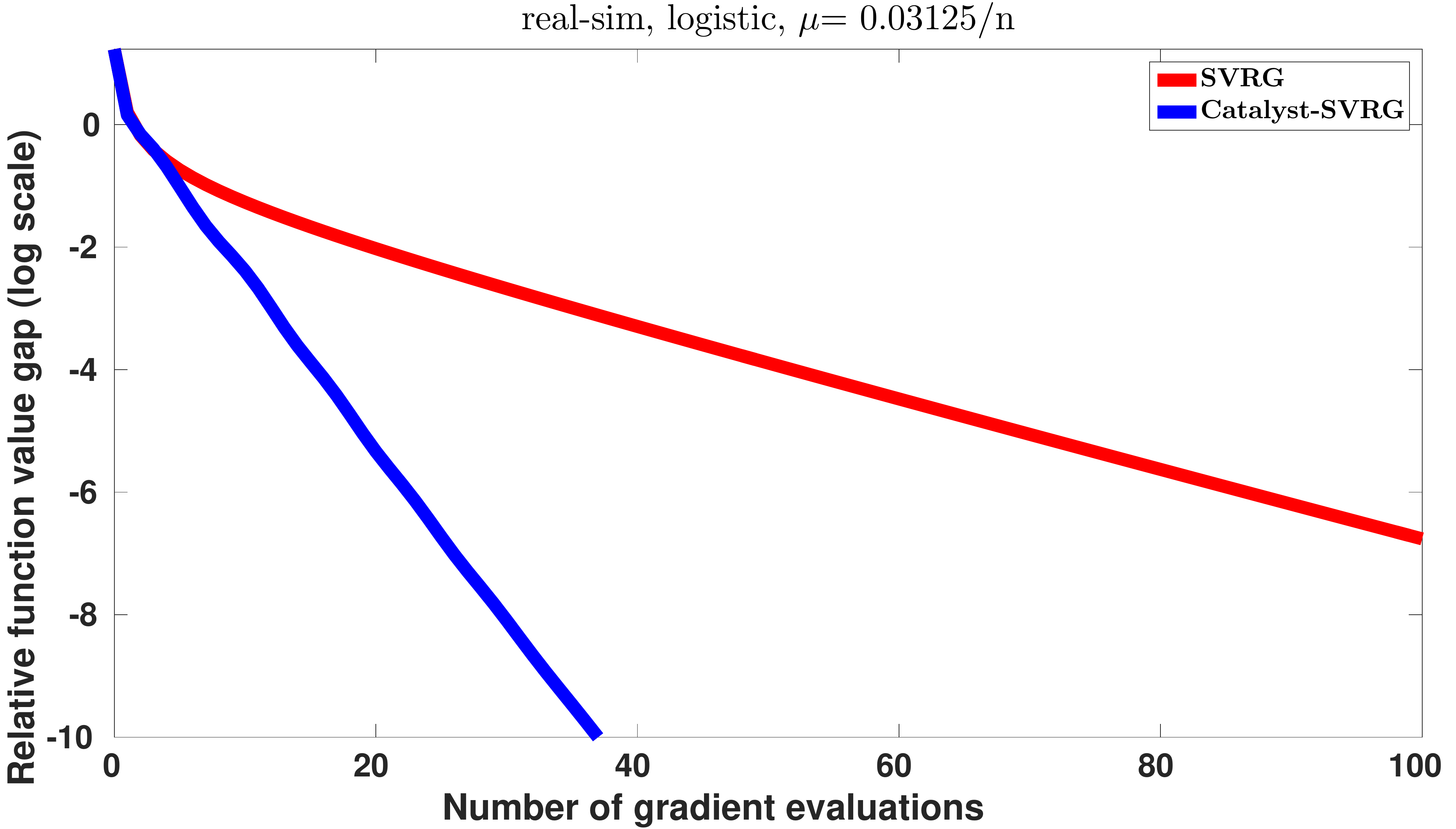}~ 
   ~~\includegraphics[width=0.31\linewidth]{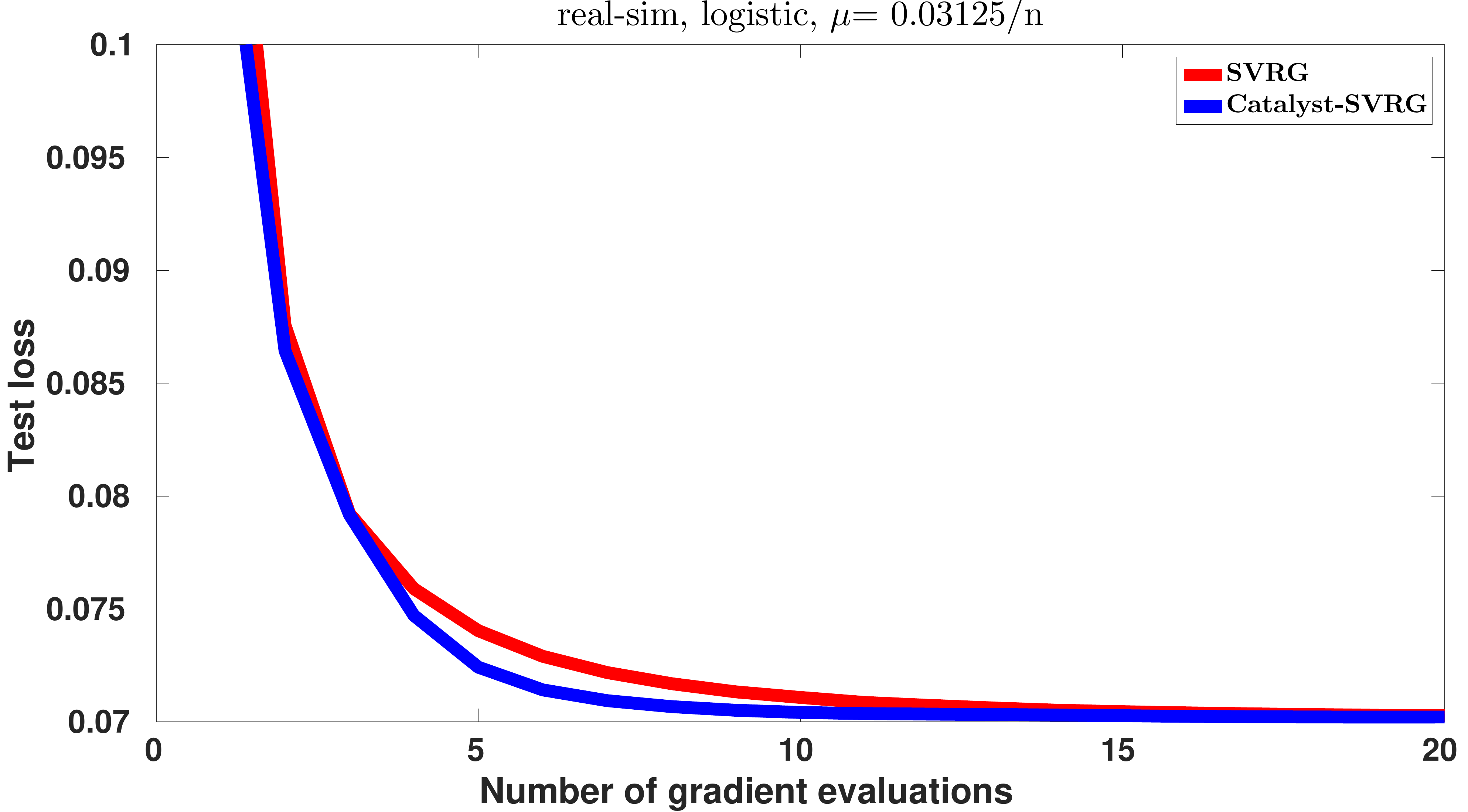}~
   ~~\includegraphics[width=0.31\linewidth]{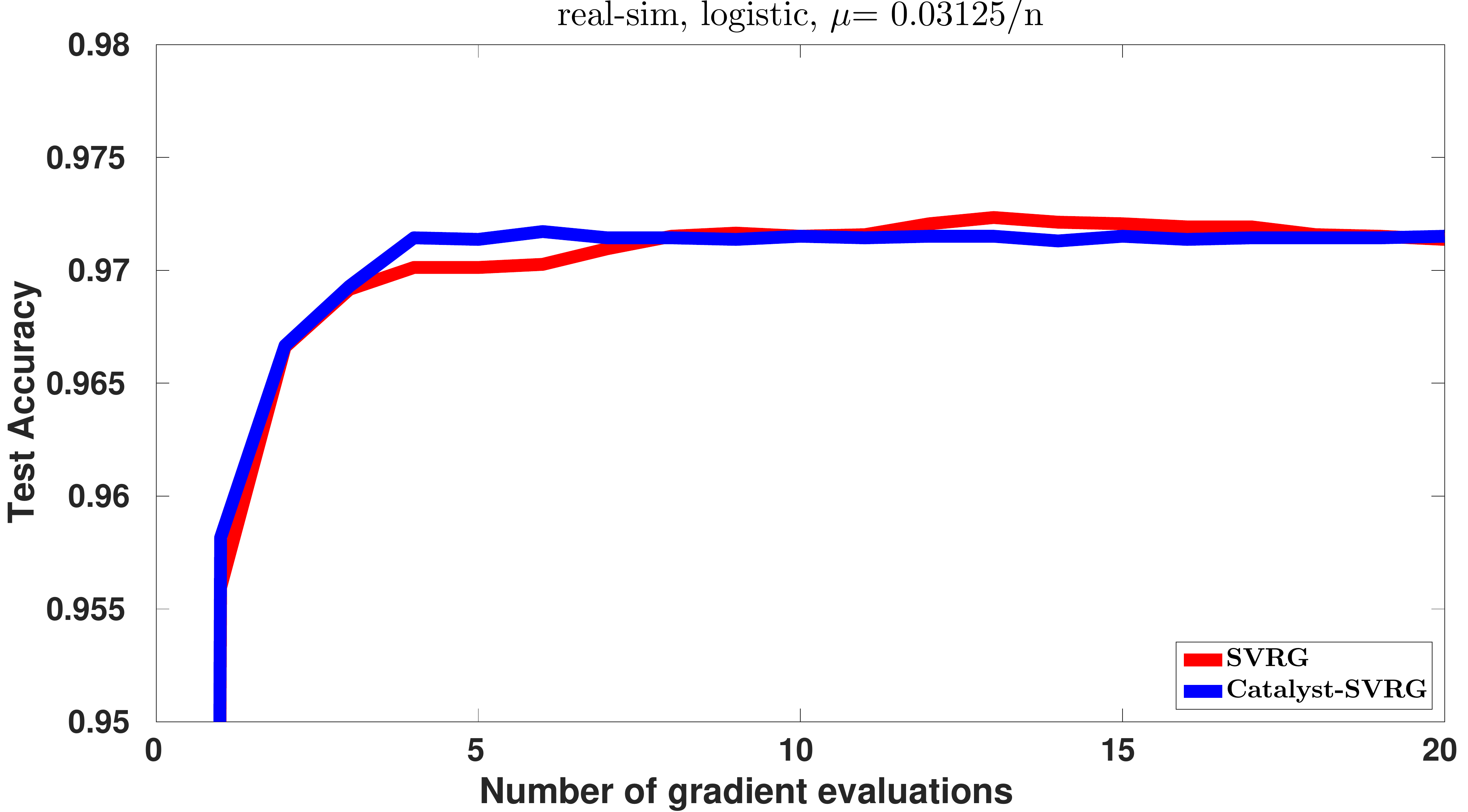}\\    
   ~~\includegraphics[width=0.31\linewidth]{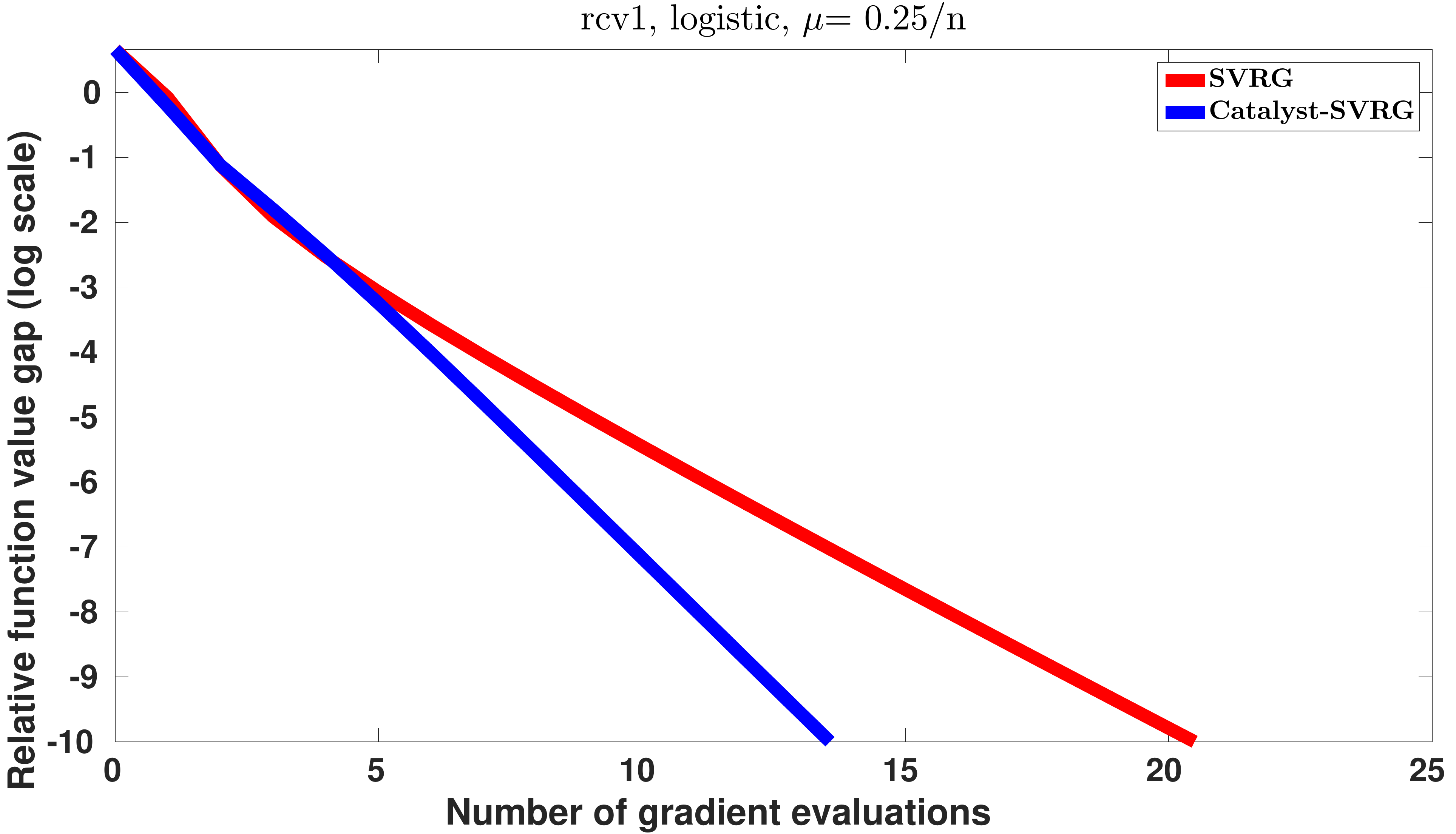}~ 
   ~~\includegraphics[width=0.31\linewidth]{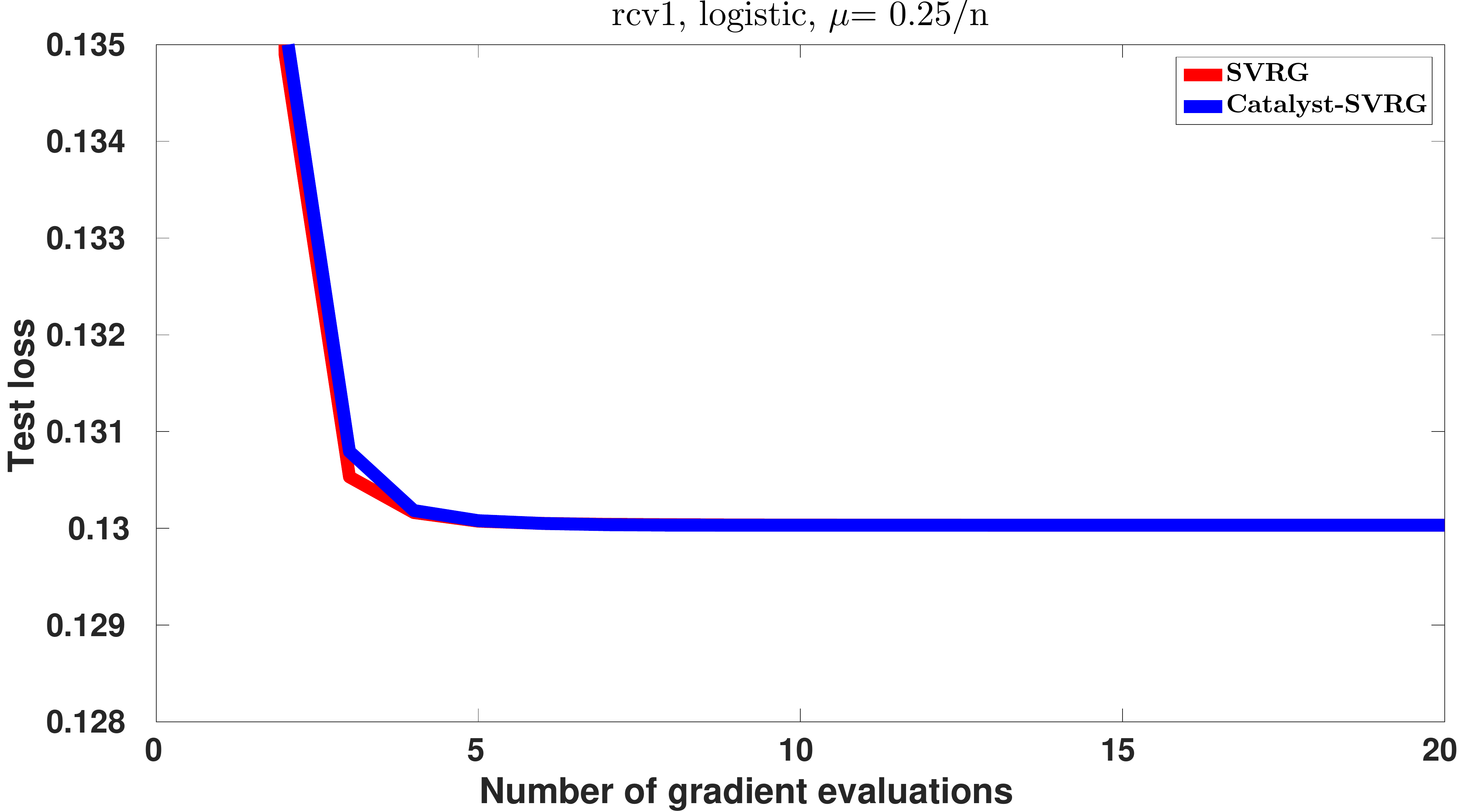}~
   ~~\includegraphics[width=0.31\linewidth]{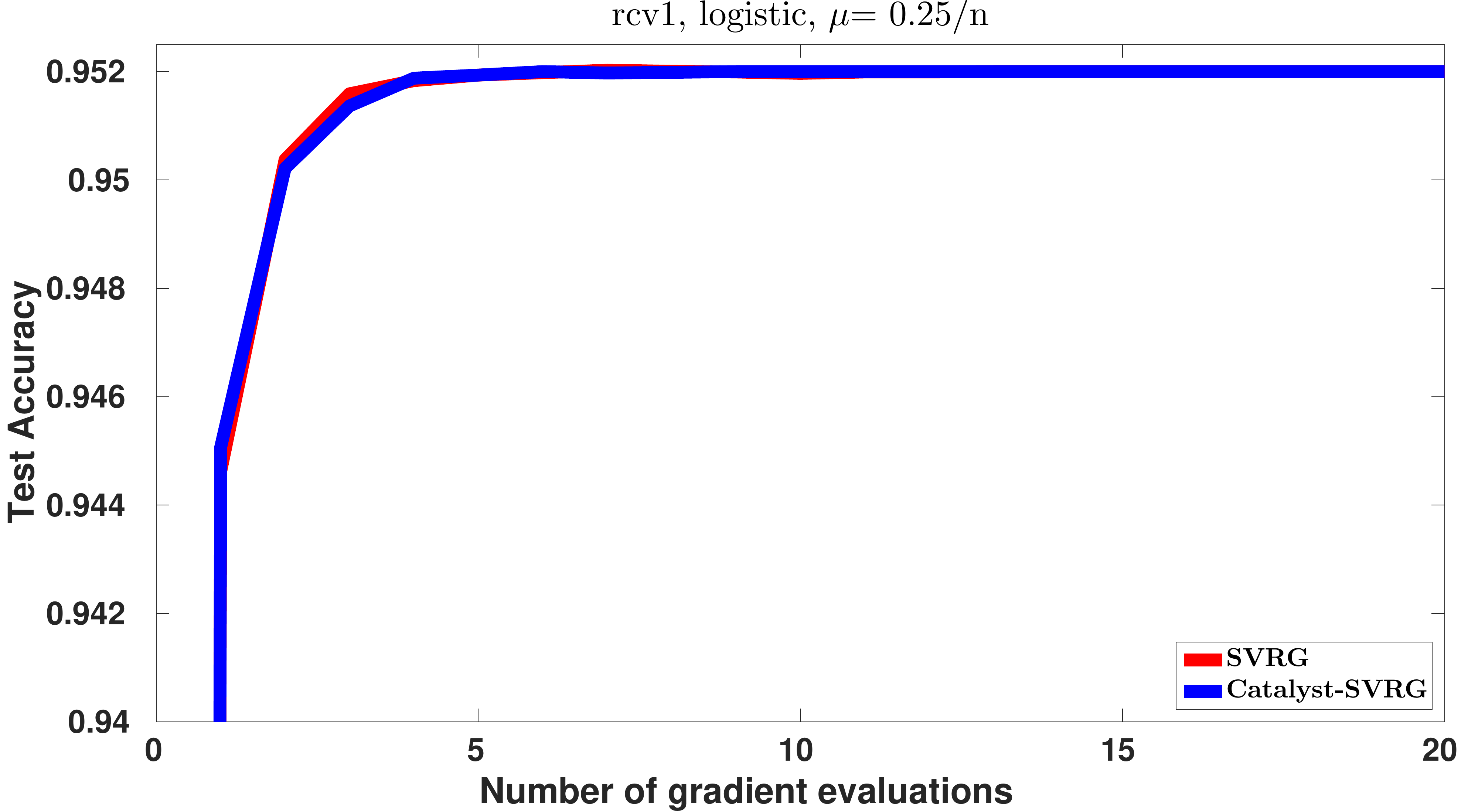}\\
   ~~\includegraphics[width=0.31\linewidth]{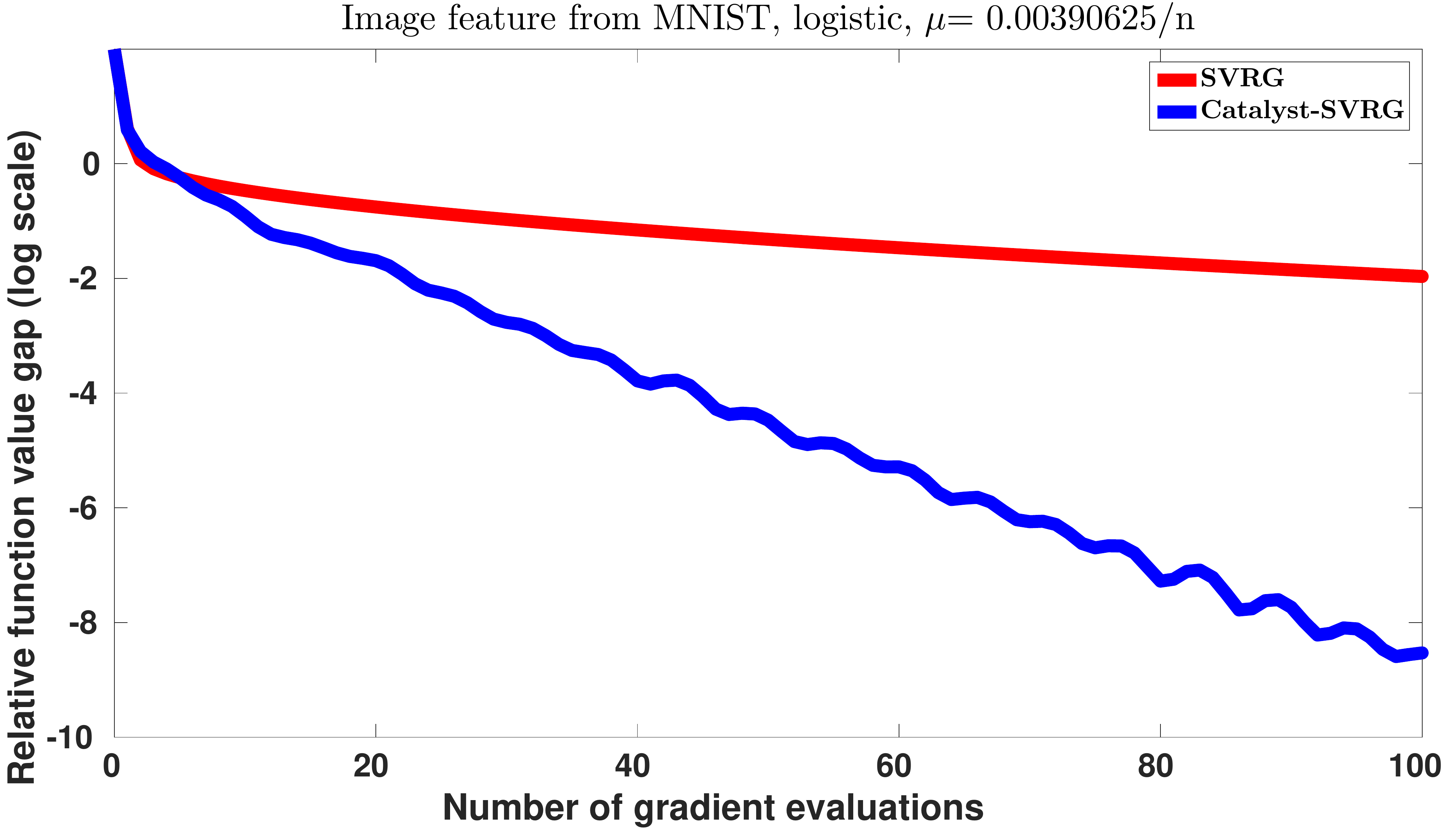}~ 
   ~~\includegraphics[width=0.31\linewidth]{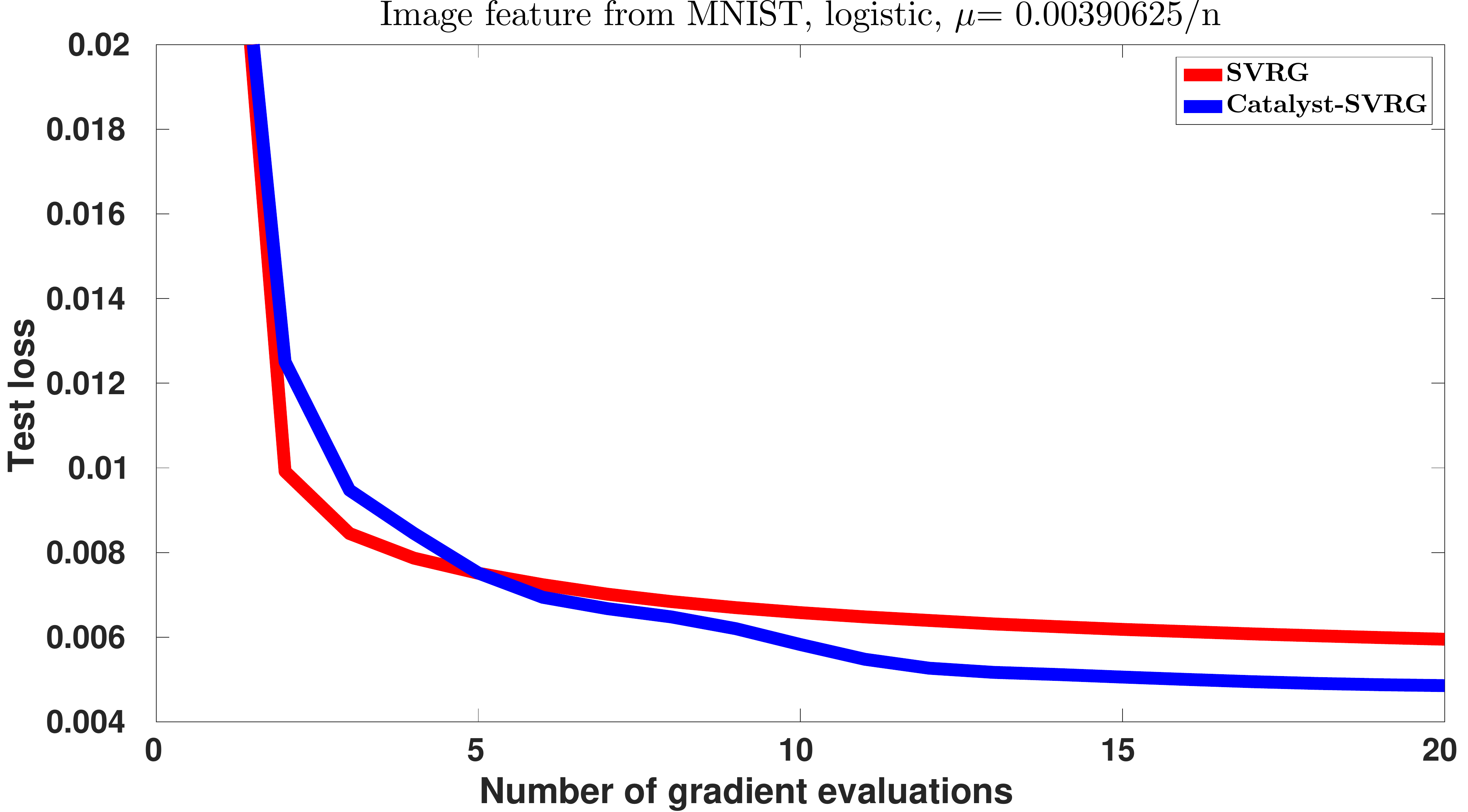}~
   ~~\includegraphics[width=0.31\linewidth]{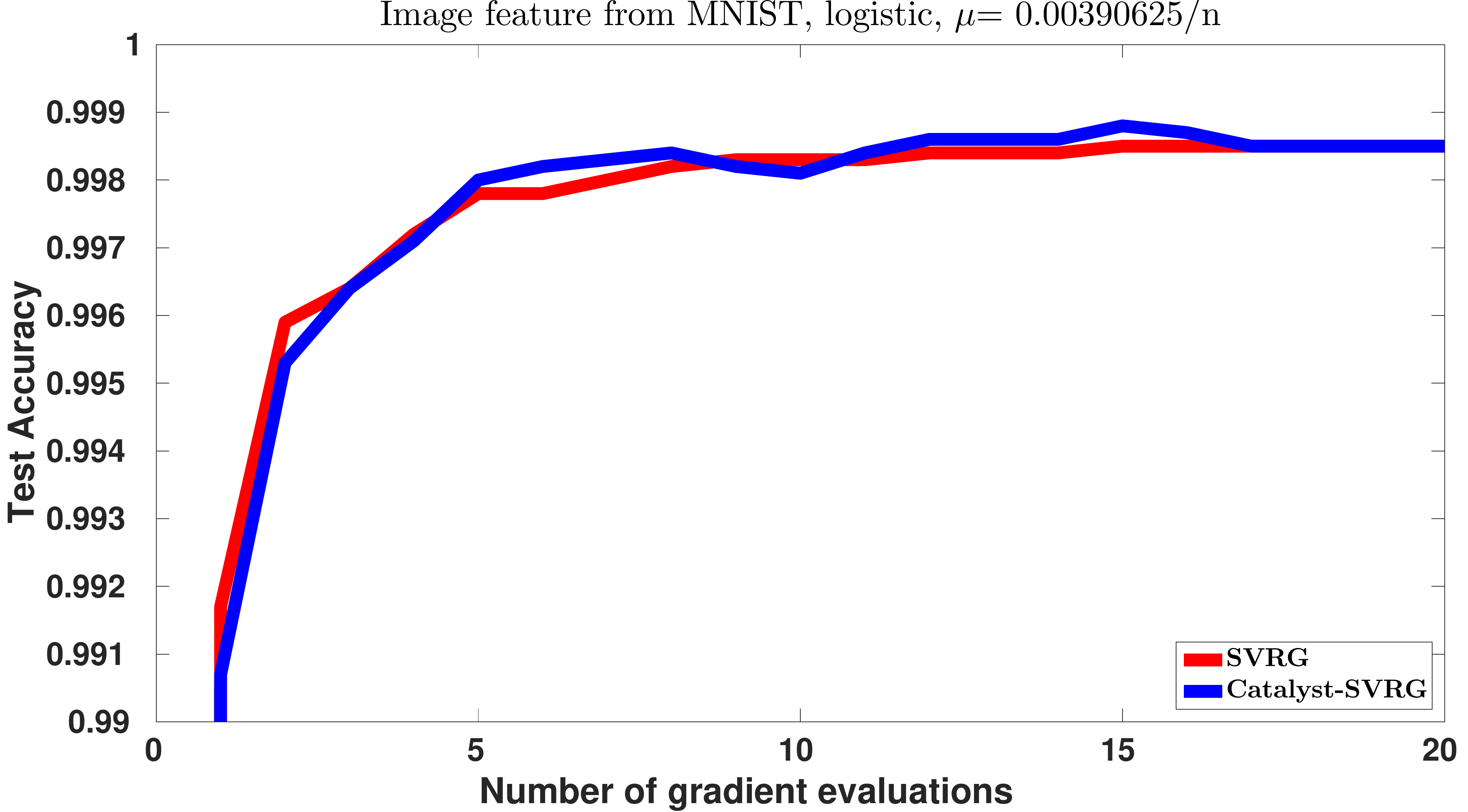}\\
   ~~\includegraphics[width=0.31\linewidth]{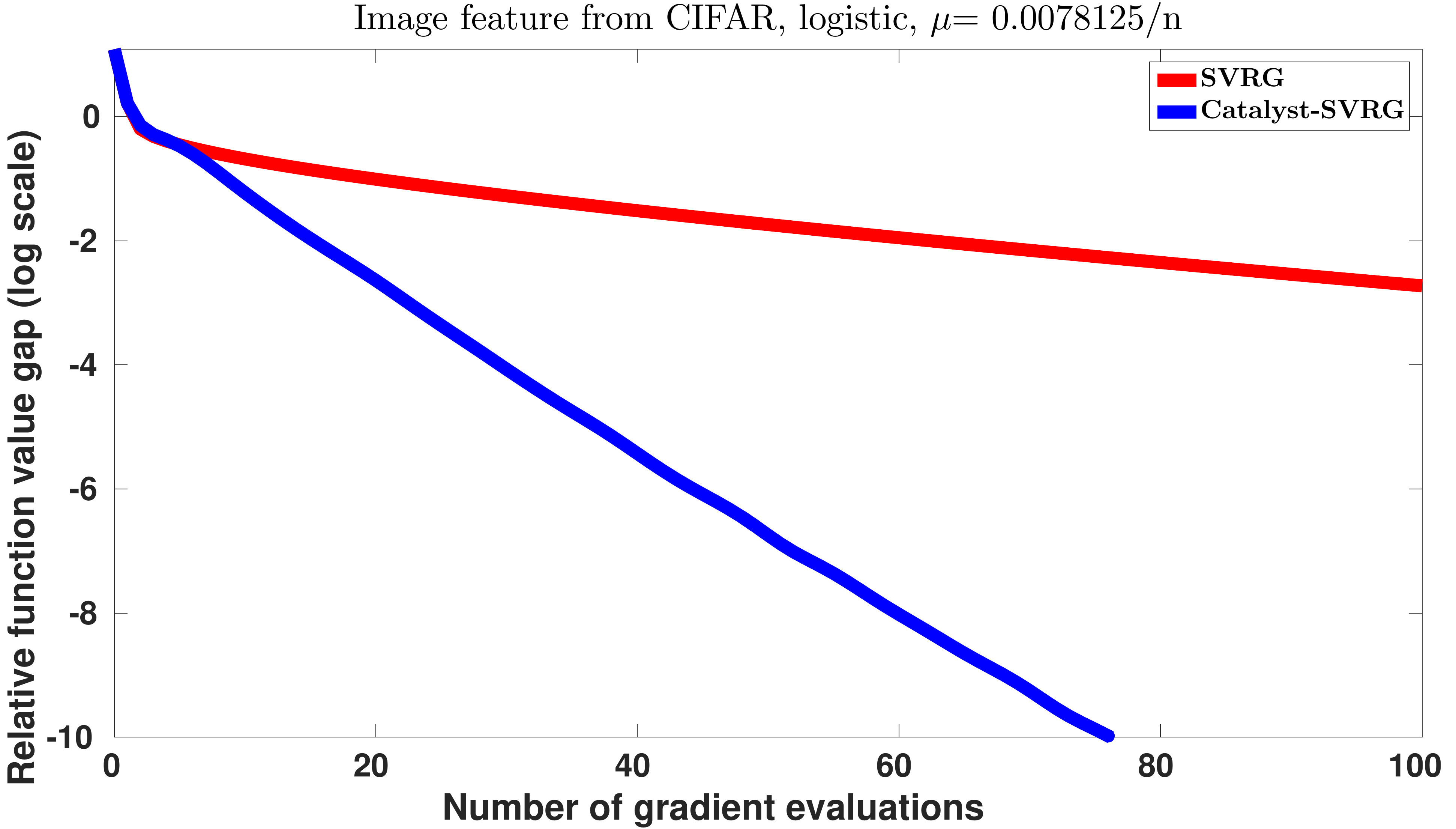}~ 
   ~~\includegraphics[width=0.31\linewidth]{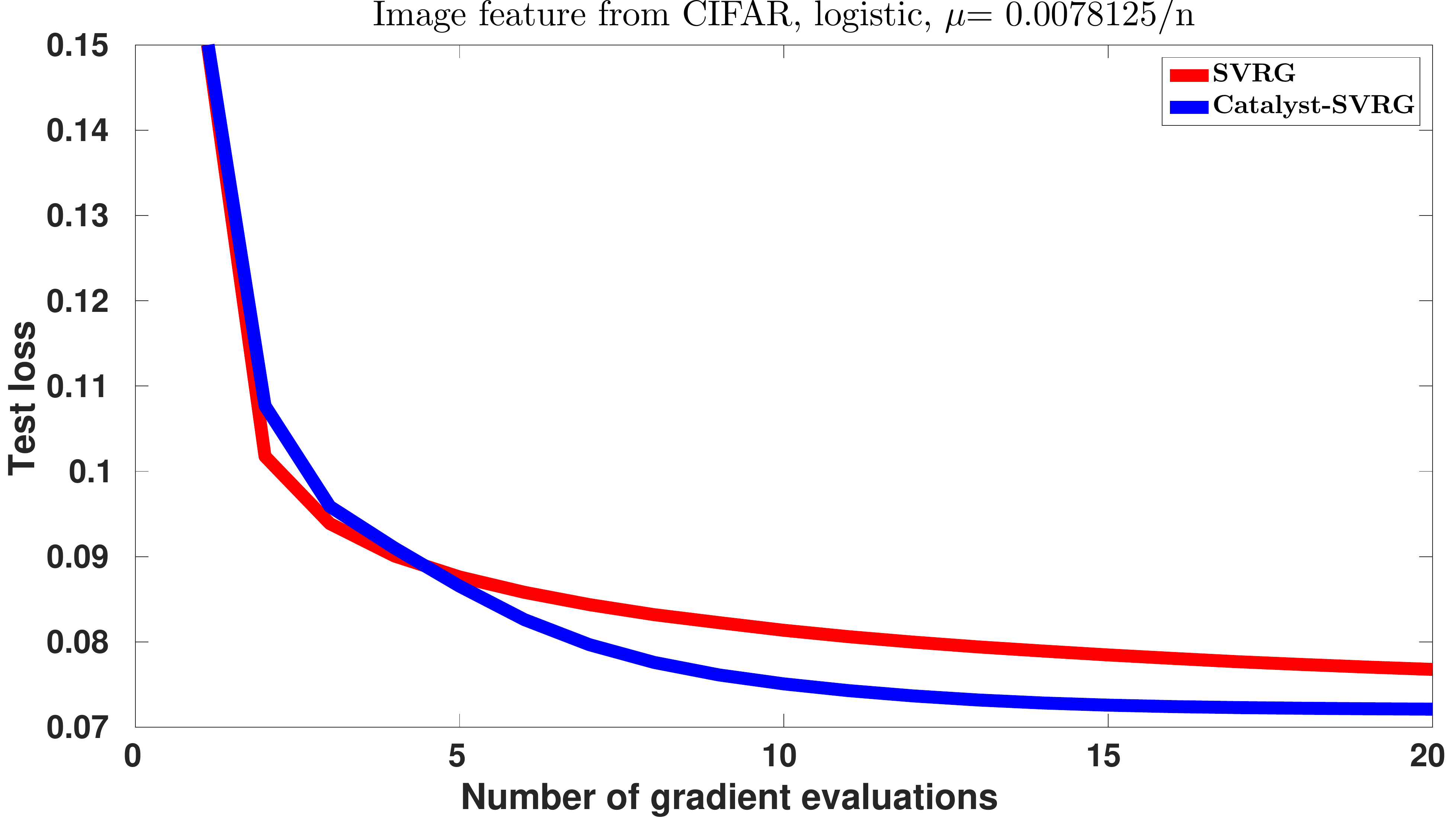}~
   ~~\includegraphics[width=0.31\linewidth]{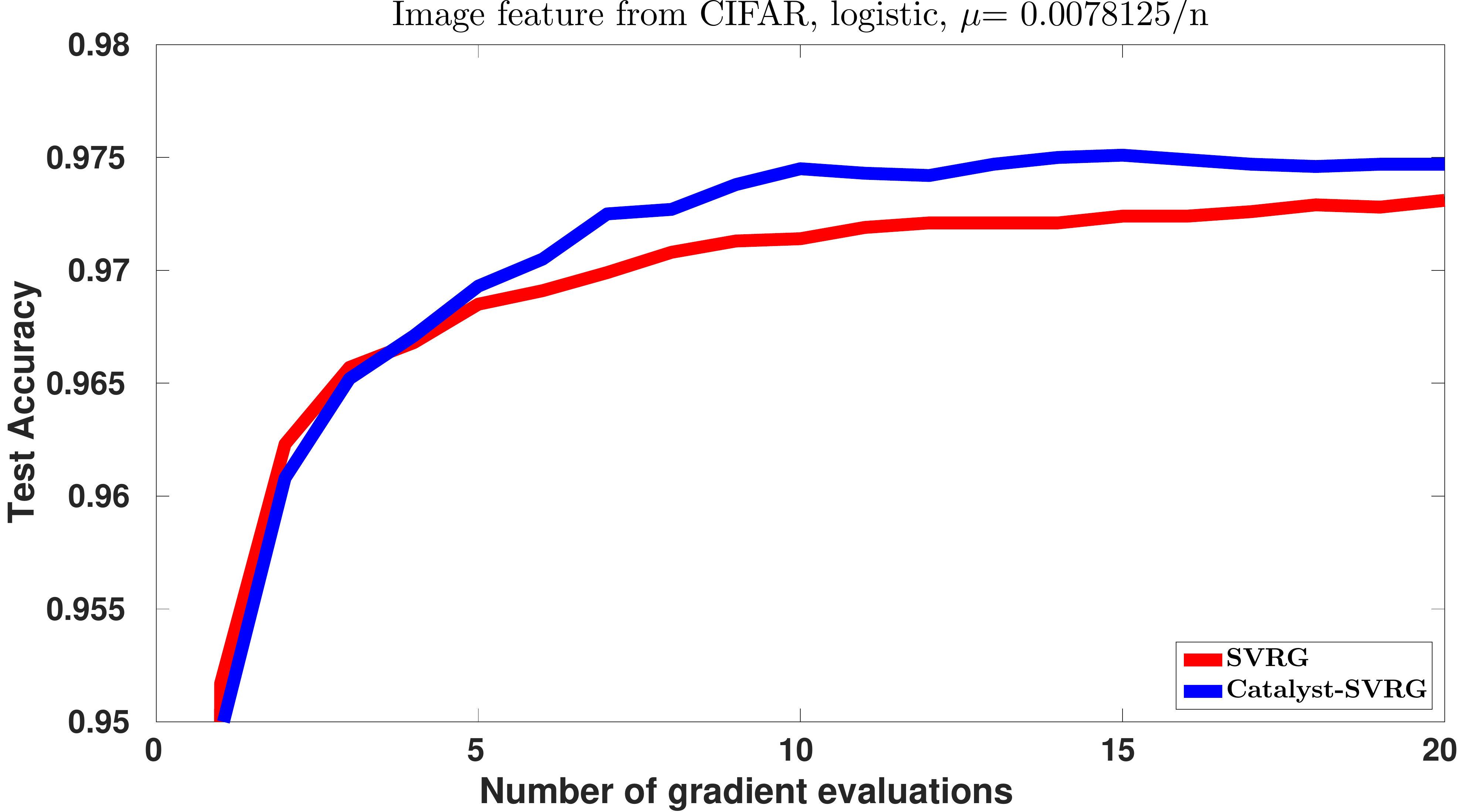}
   \caption{Empirical effect on the generalization error. For a logistic regression experiment, we report the value of the objective function evaluated on the training data on the left column, the value of the loss evaluated on a test set on the middle column, and the classification error evaluated on the test set on the right.}\label{fig:test}
\end{figure}

\paragraph{Observations on the test loss and test accuracy.} The left column of
Figure \ref{fig:test} shows the loss function on the training set, where
acceleration is significant in 5 cases out of 6. The middle column shows the loss
function evaluated on the test set, but on a non-logarithmic scale since the
optimal value of the test loss is unknown.  Acceleration appears in 4 cases out
of 6. Regarding the test accuracy, an important acceleration is obtained in 2 cases,
whereas it is less significant or negligible in the other cases.

\subsection{Study of the Parameter~$\kappa$}\label{subsec:comparison_kappa}
Finally, we evaluate the performance for different values of $\kappa$.
\begin{figure}[hbtp]
   \centering
   ~~\includegraphics[width=0.31\linewidth]{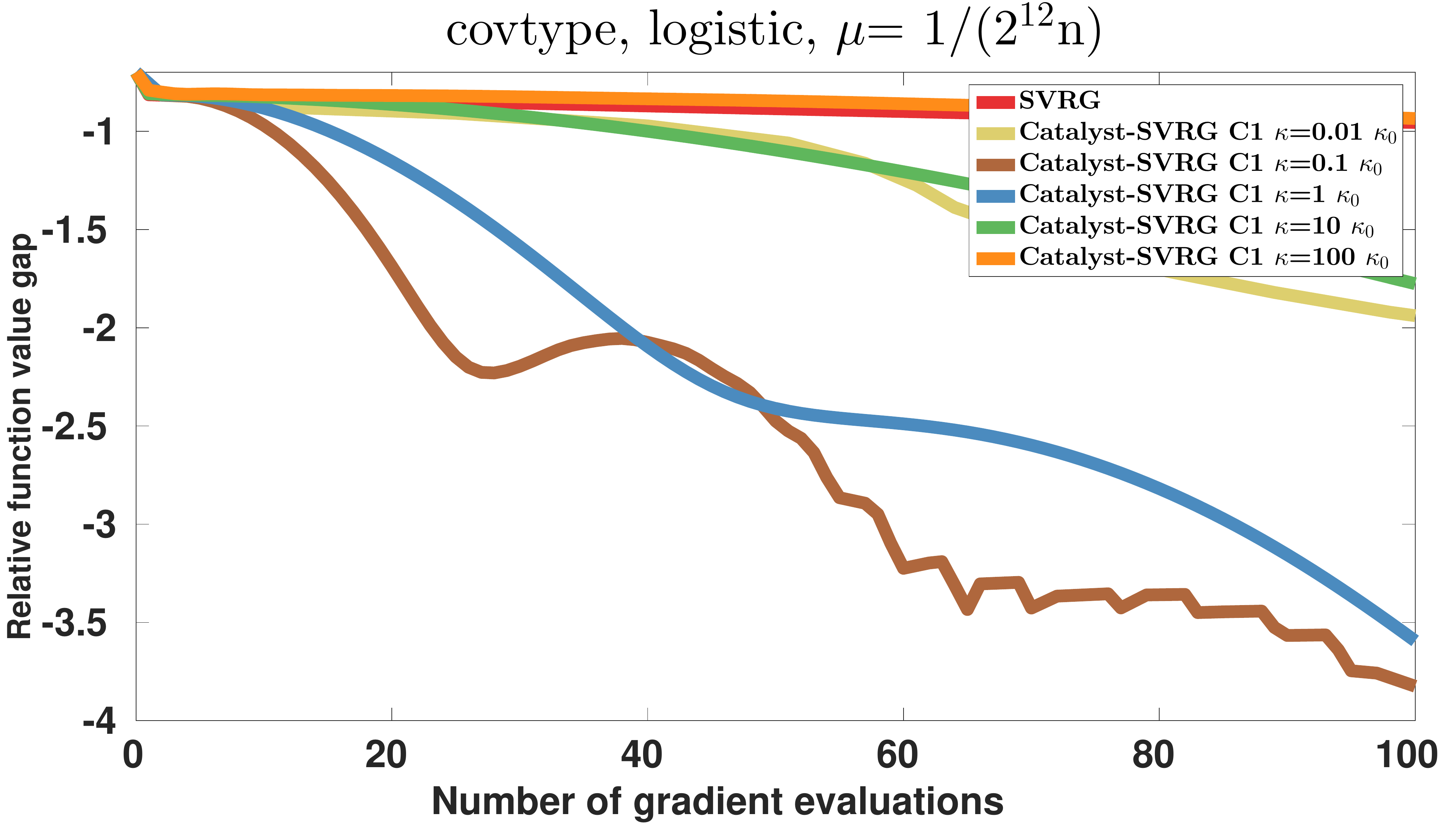}~ 
   ~~\includegraphics[width=0.31\linewidth]{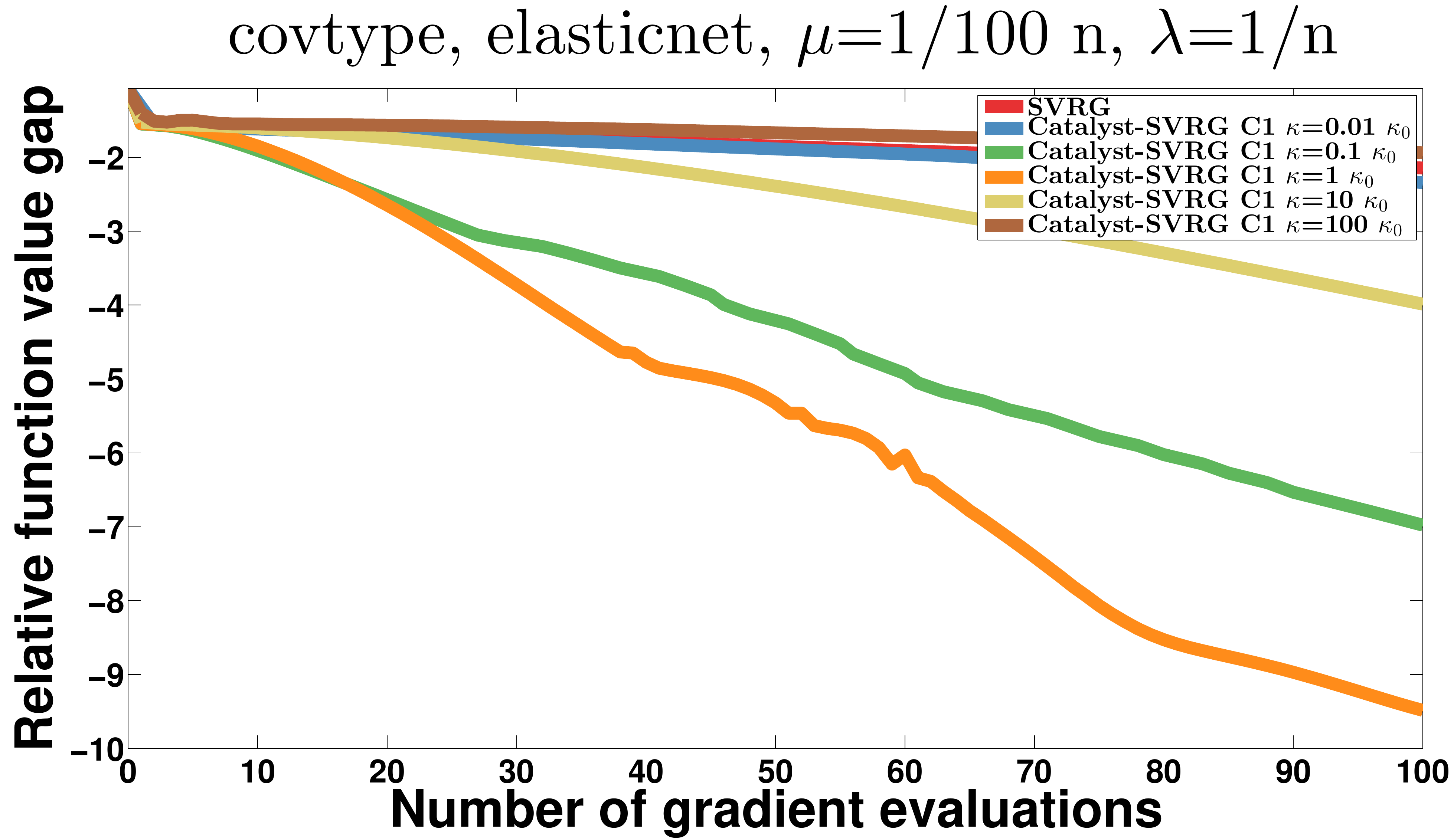}~ 
   ~~\includegraphics[width=0.31\linewidth]{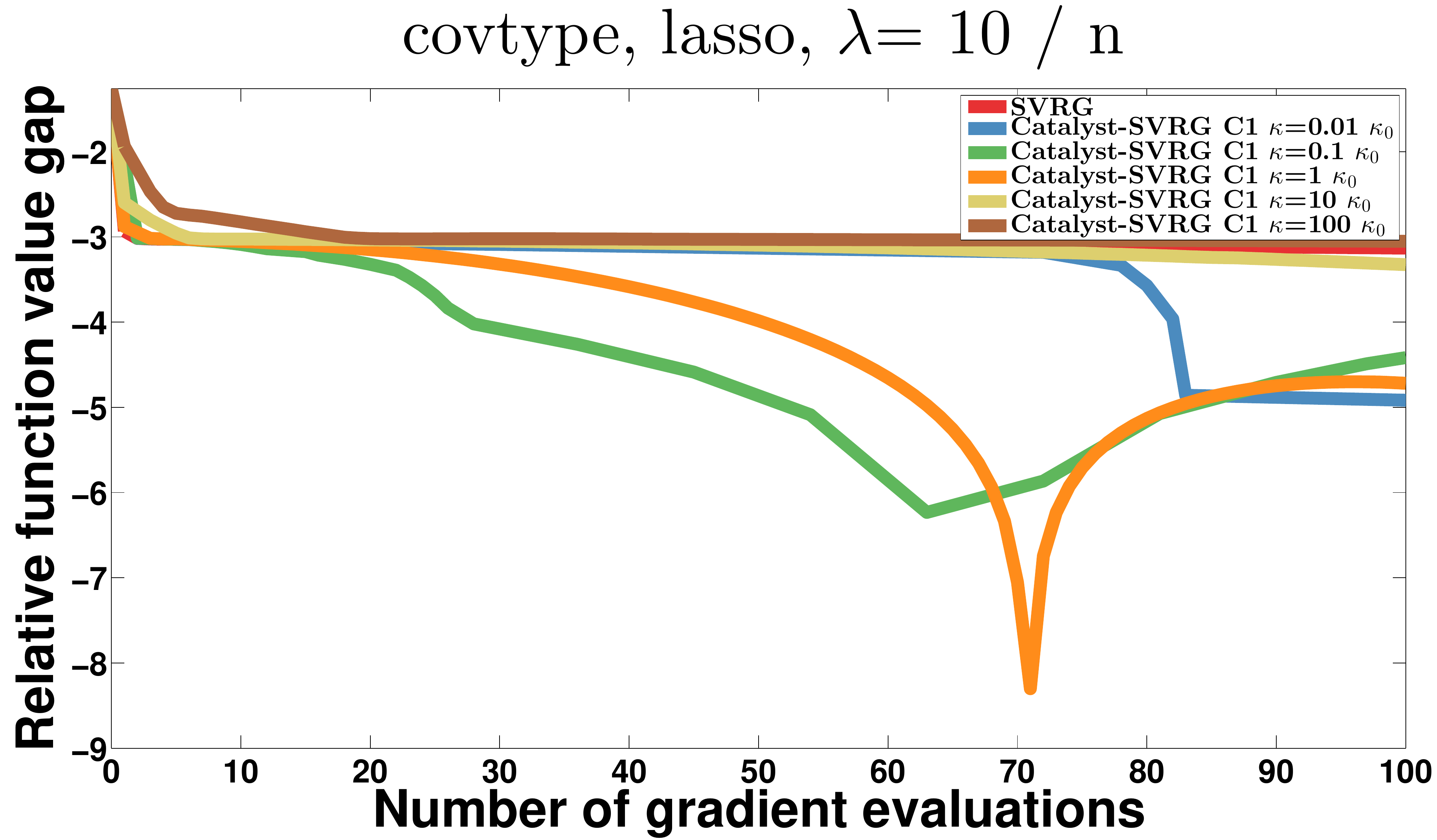}\\
    ~~\includegraphics[width=0.31\linewidth]{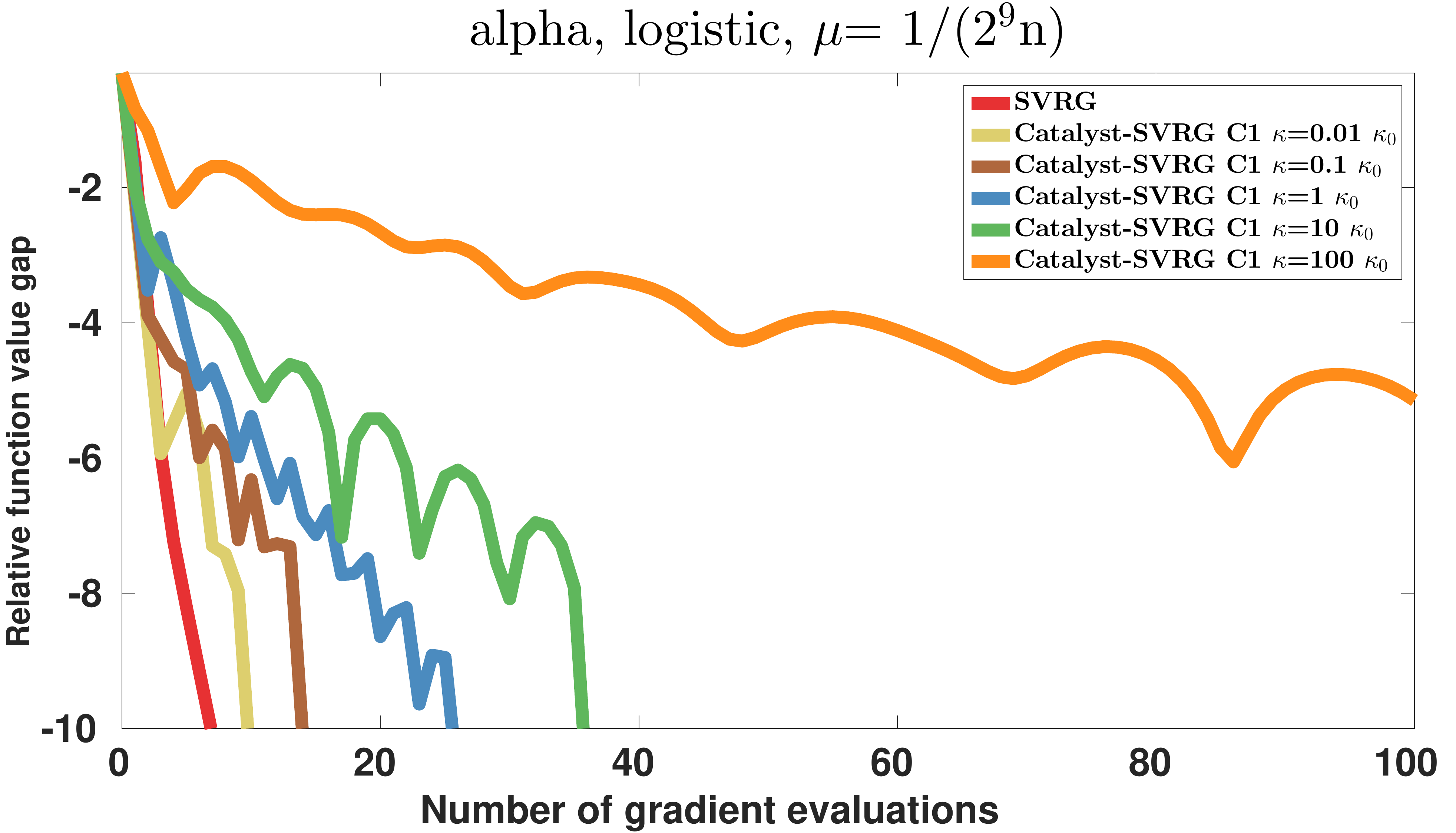}~ 
   ~~\includegraphics[width=0.31\linewidth]{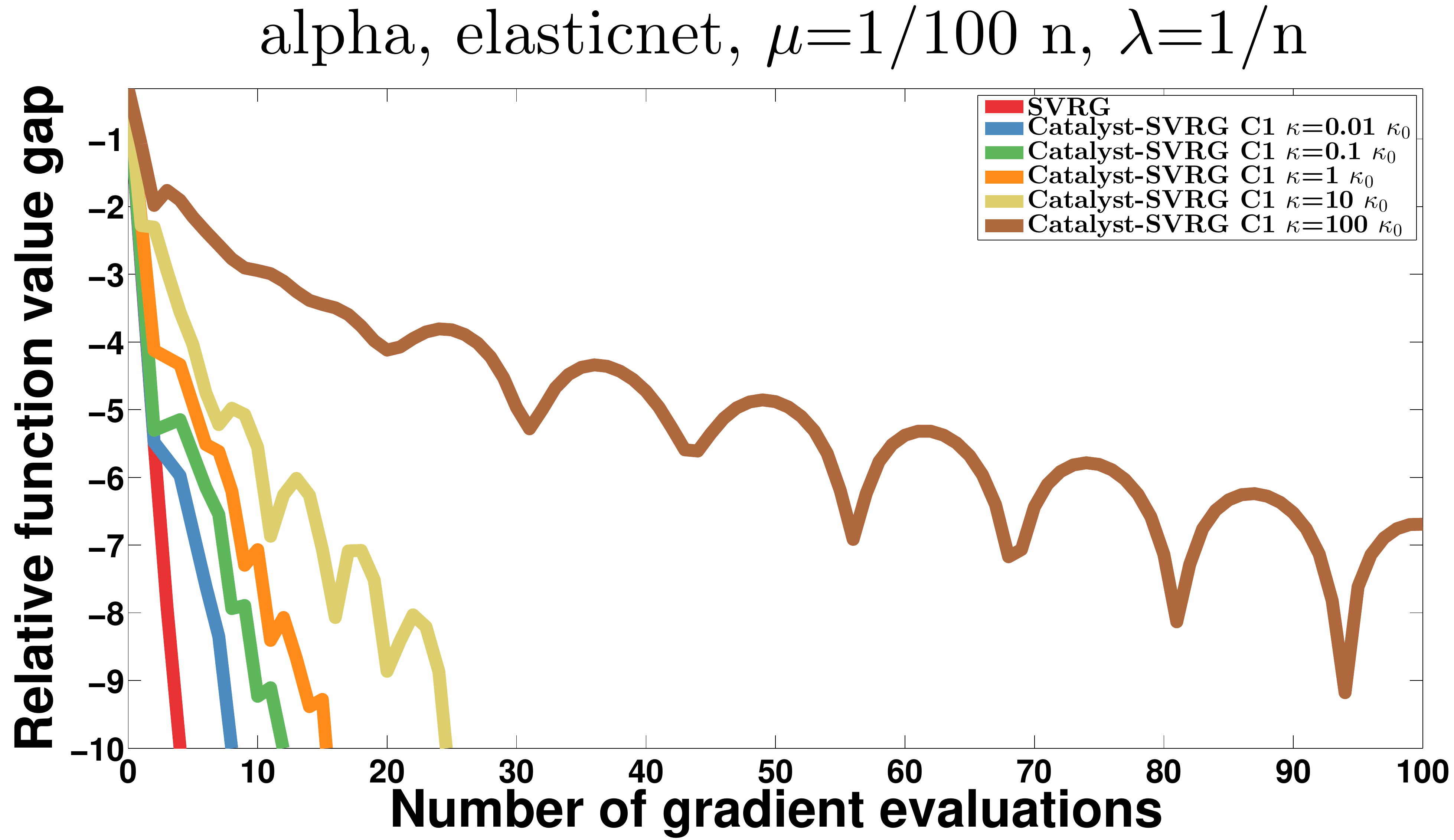}~ 
   ~~\includegraphics[width=0.31\linewidth]{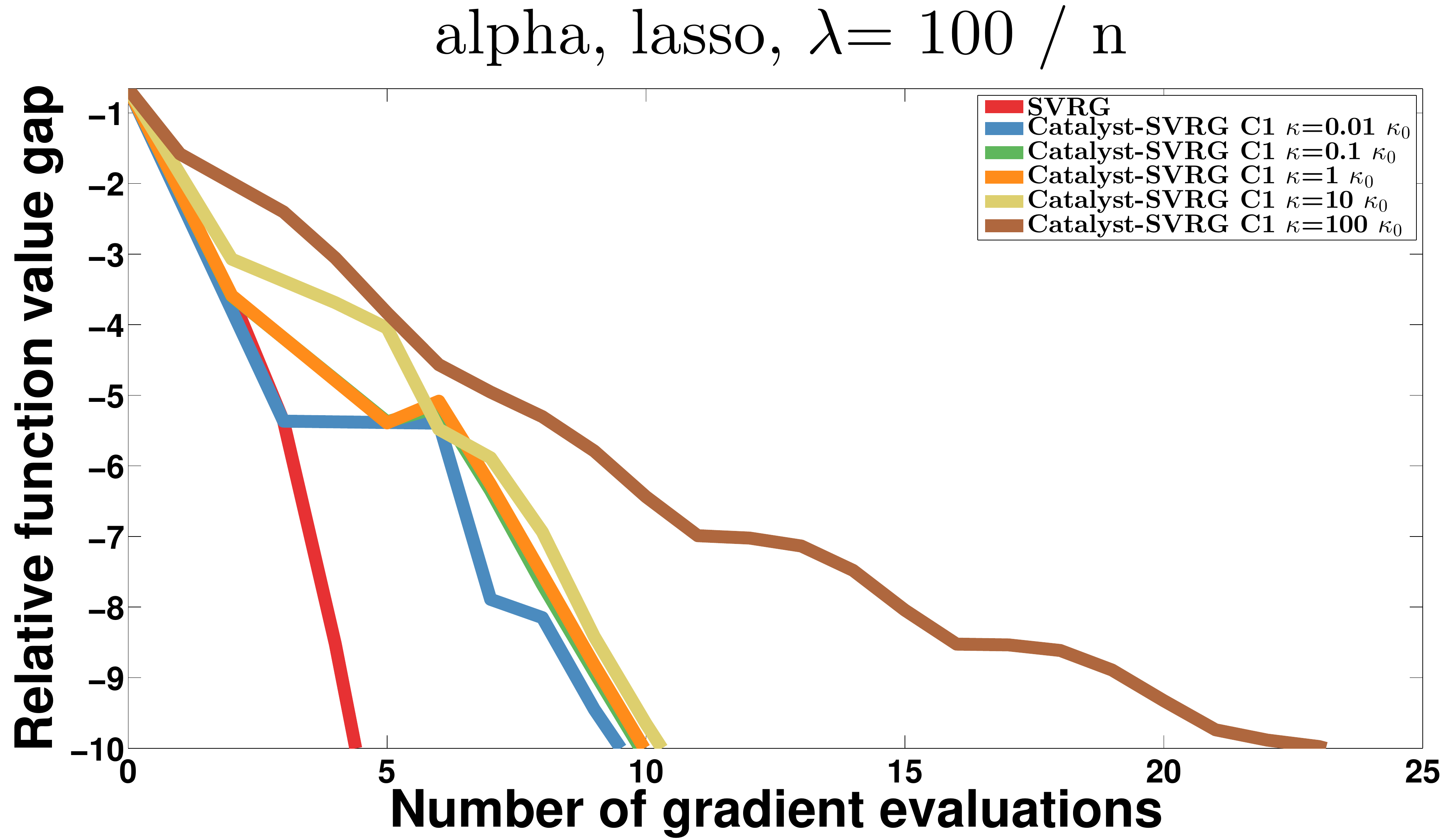}\\   
   ~~\includegraphics[width=0.31\linewidth]{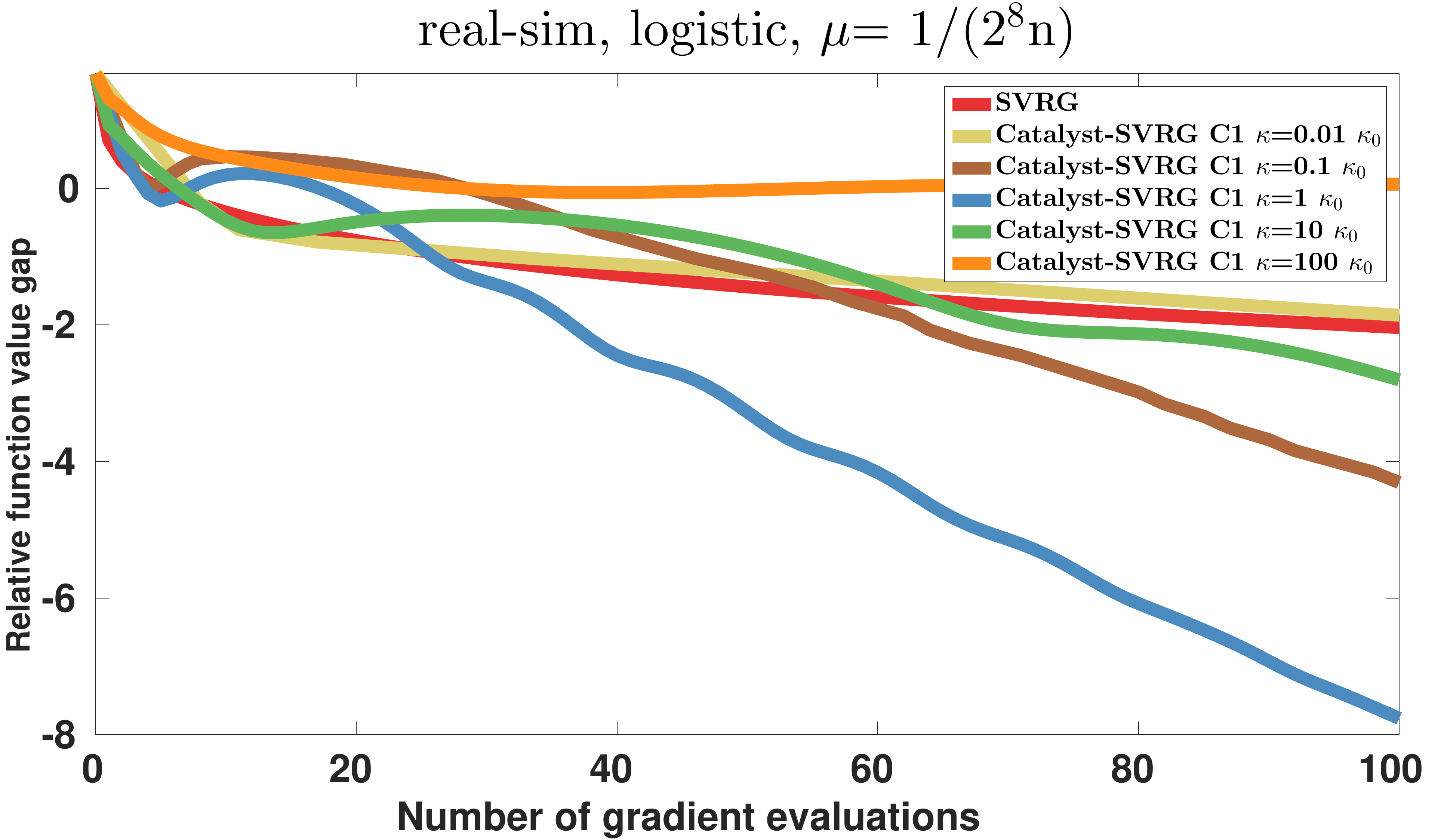}~ 
   ~~\includegraphics[width=0.31\linewidth]{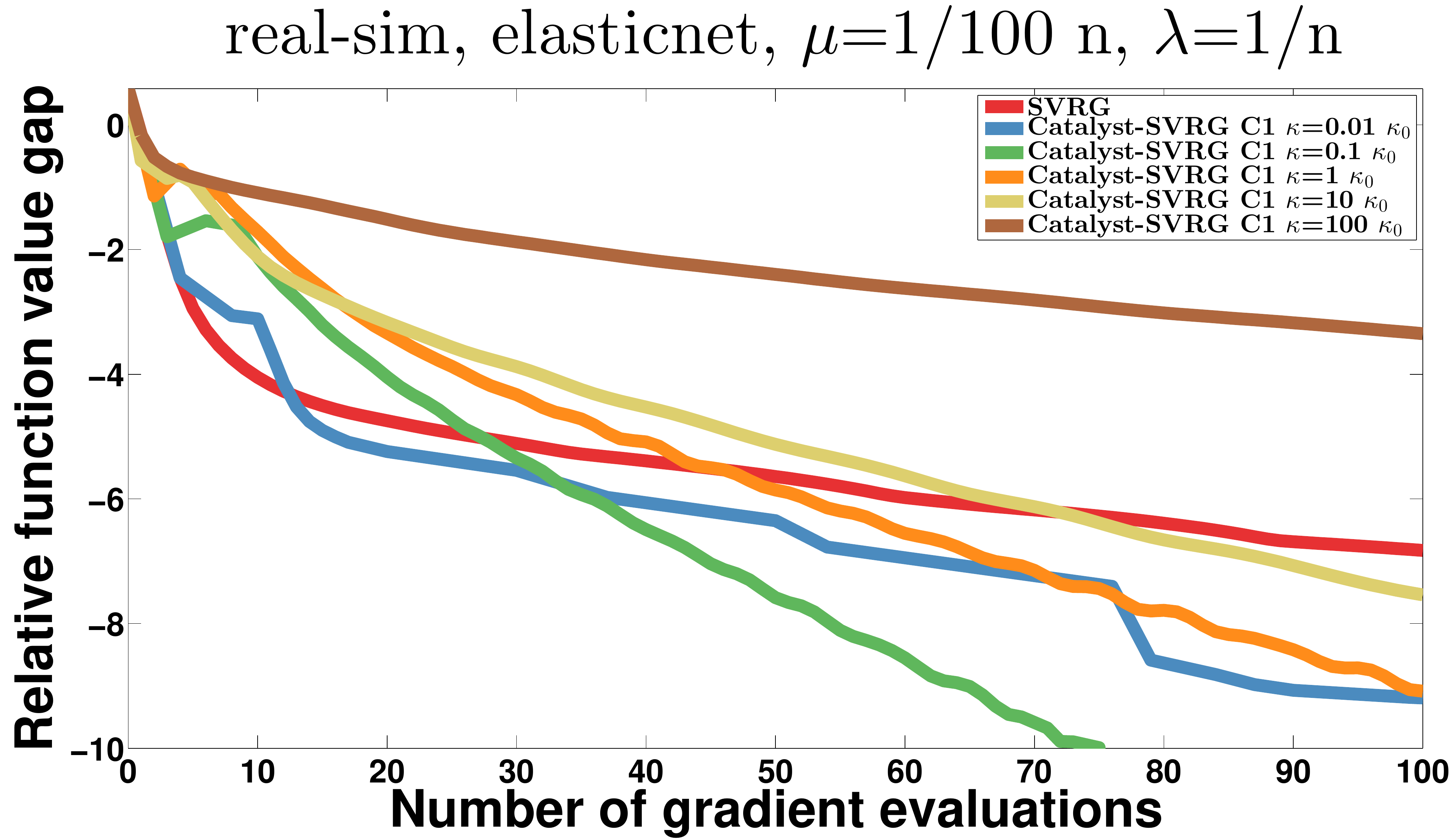}~ 
   ~~\includegraphics[width=0.31\linewidth]{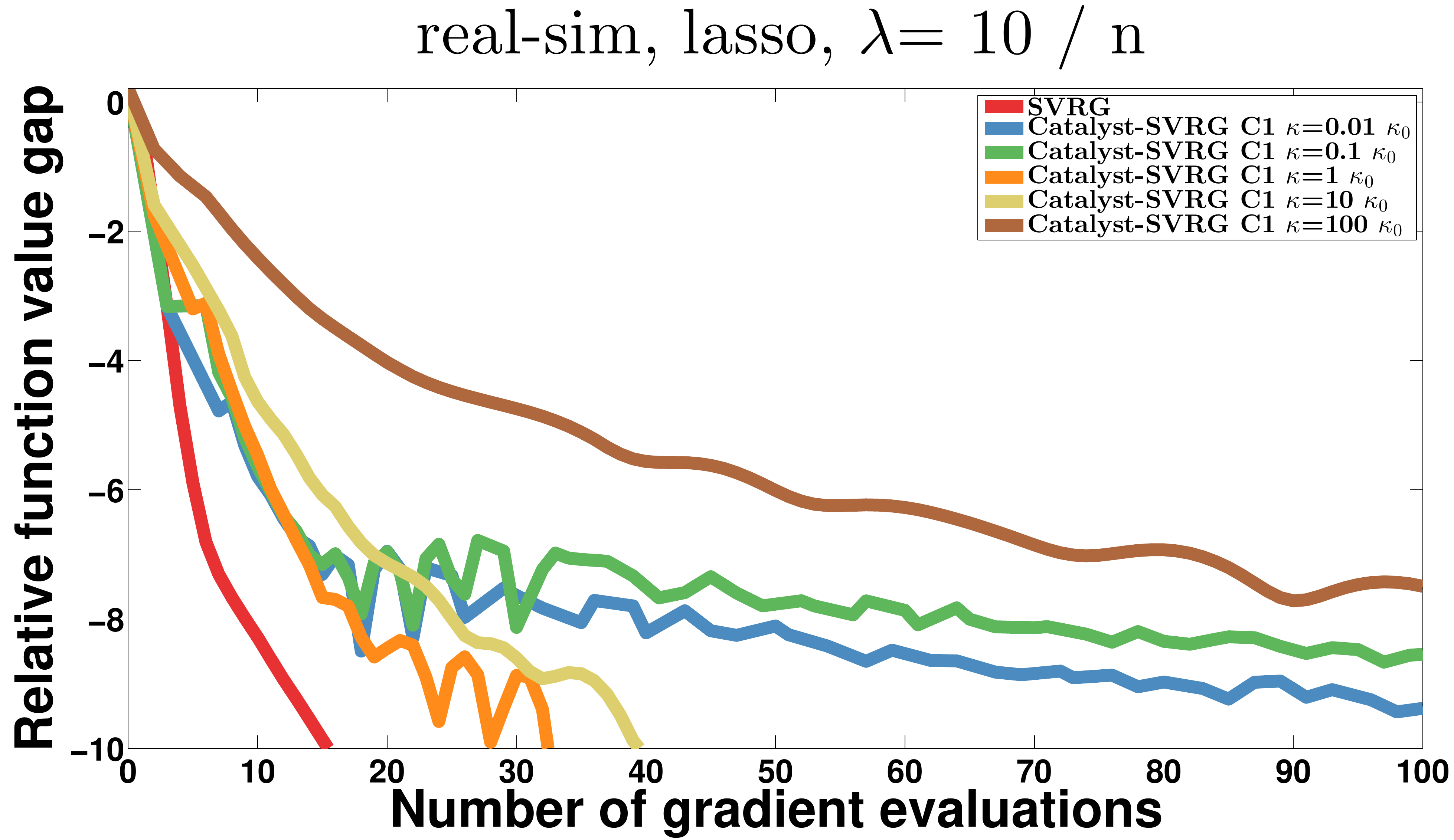}\\     
   ~~\includegraphics[width=0.31\linewidth]{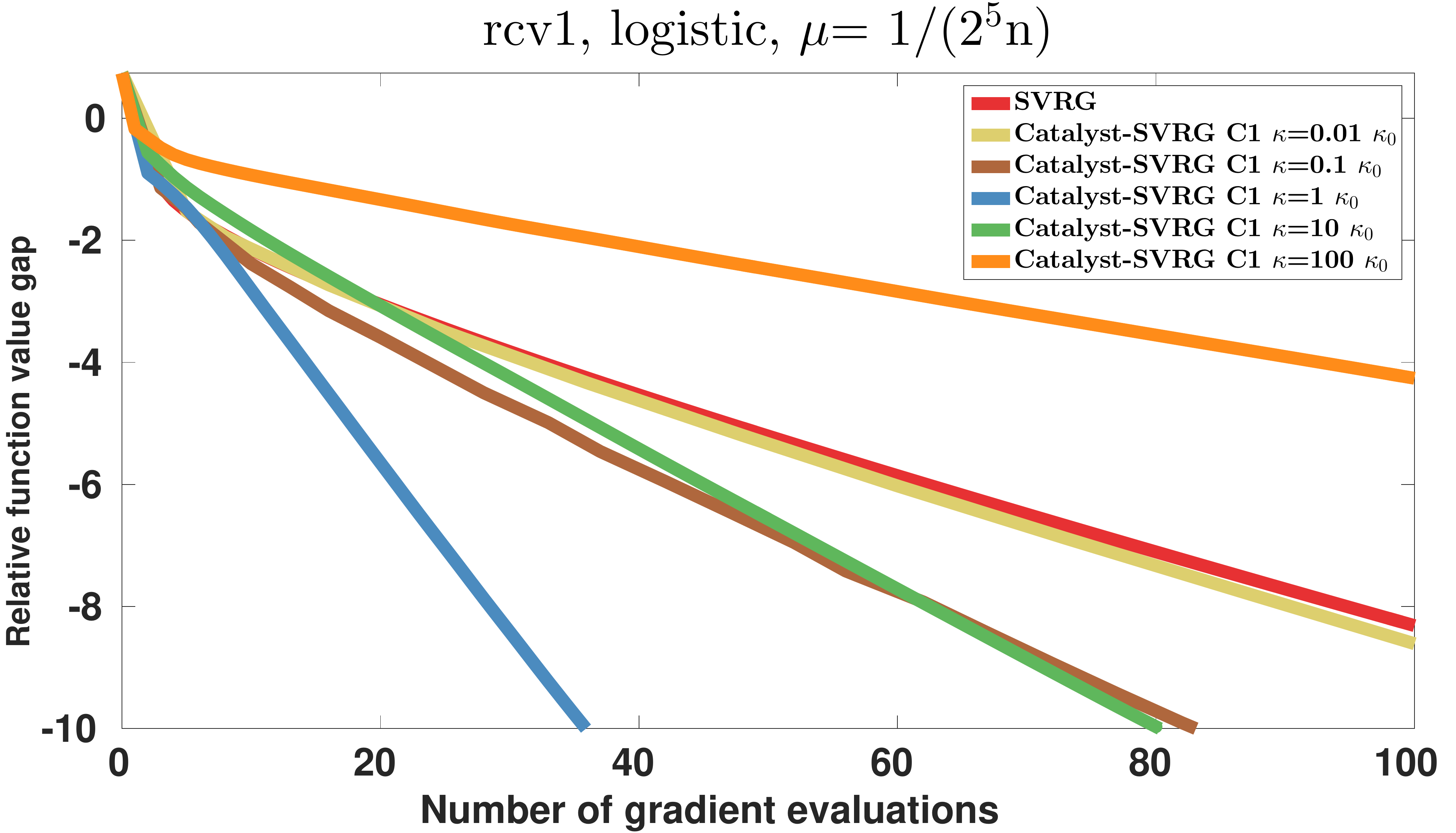}~ 
   ~~\includegraphics[width=0.31\linewidth]{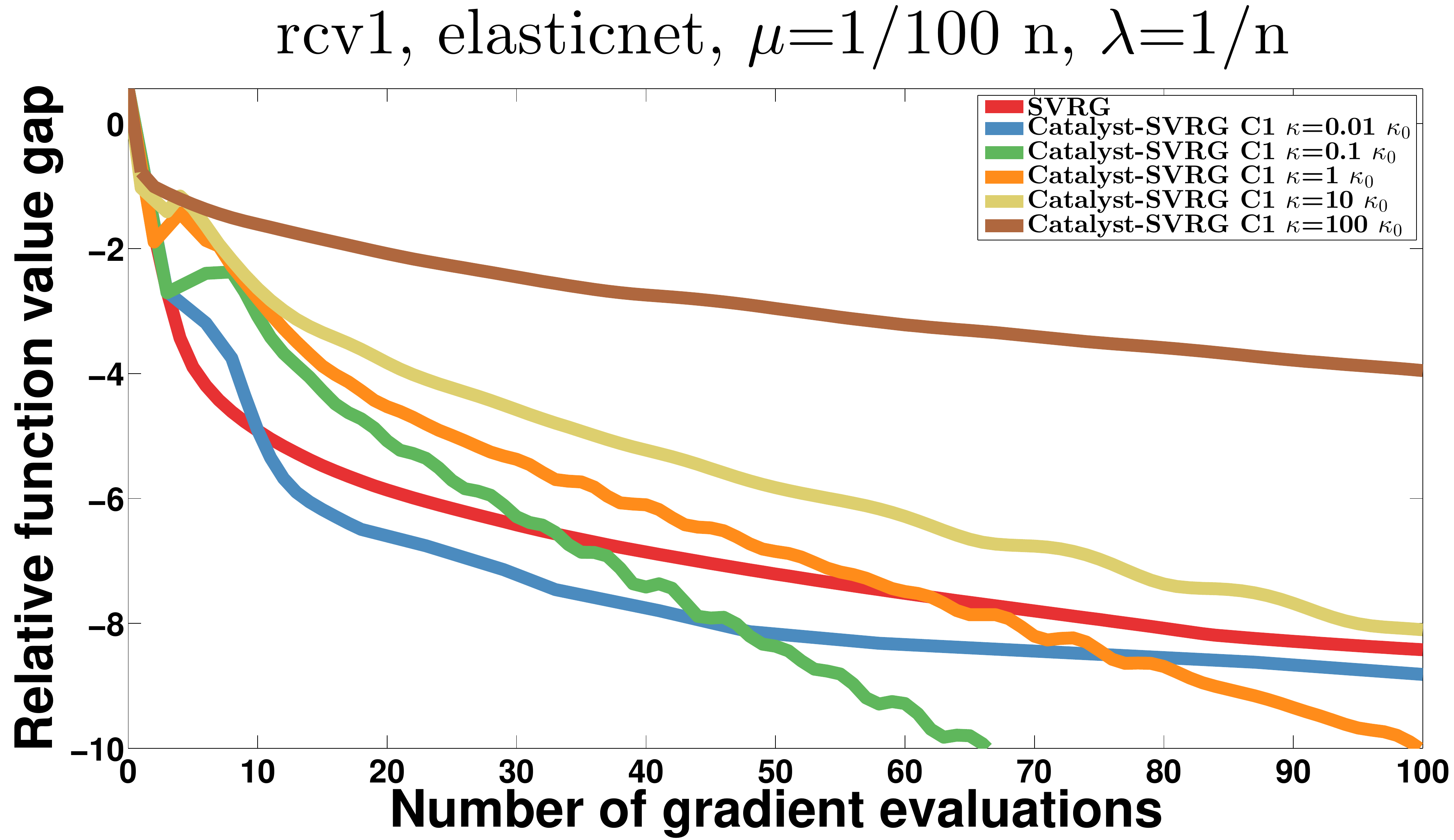}~ 
   ~~\includegraphics[width=0.31\linewidth]{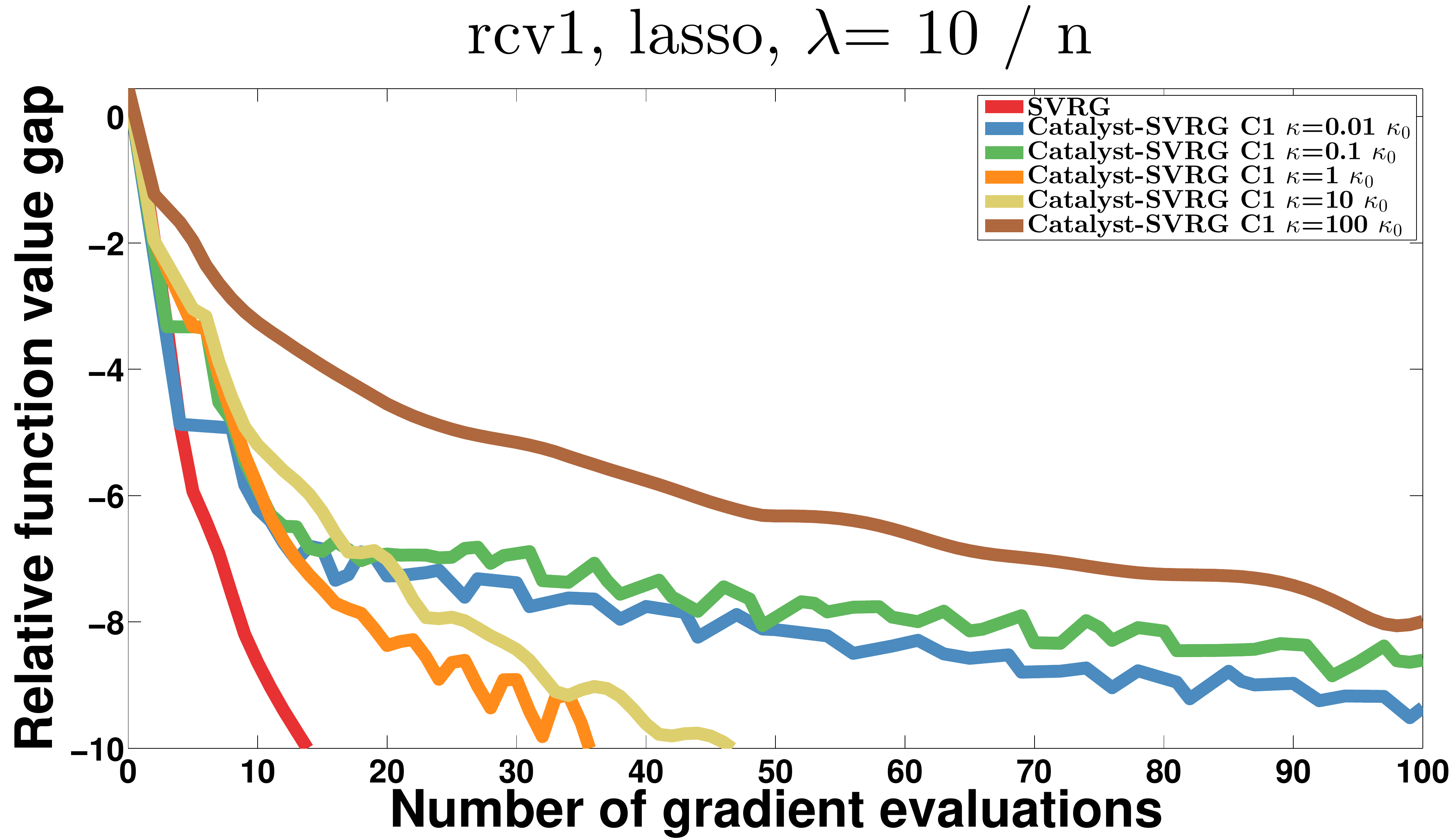}\\
   ~~\includegraphics[width=0.31\linewidth]{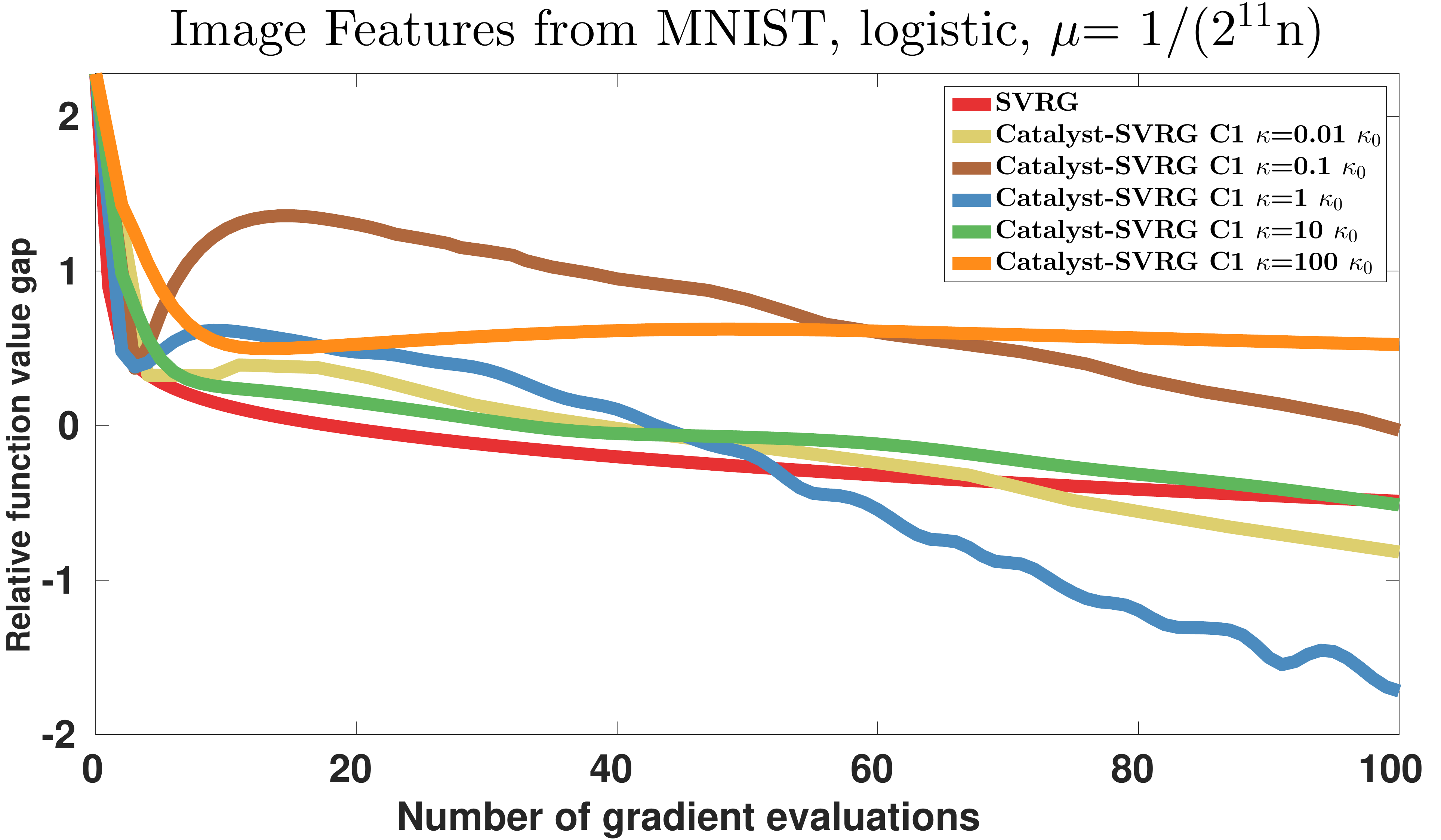}~ 
   ~~\includegraphics[width=0.31\linewidth]{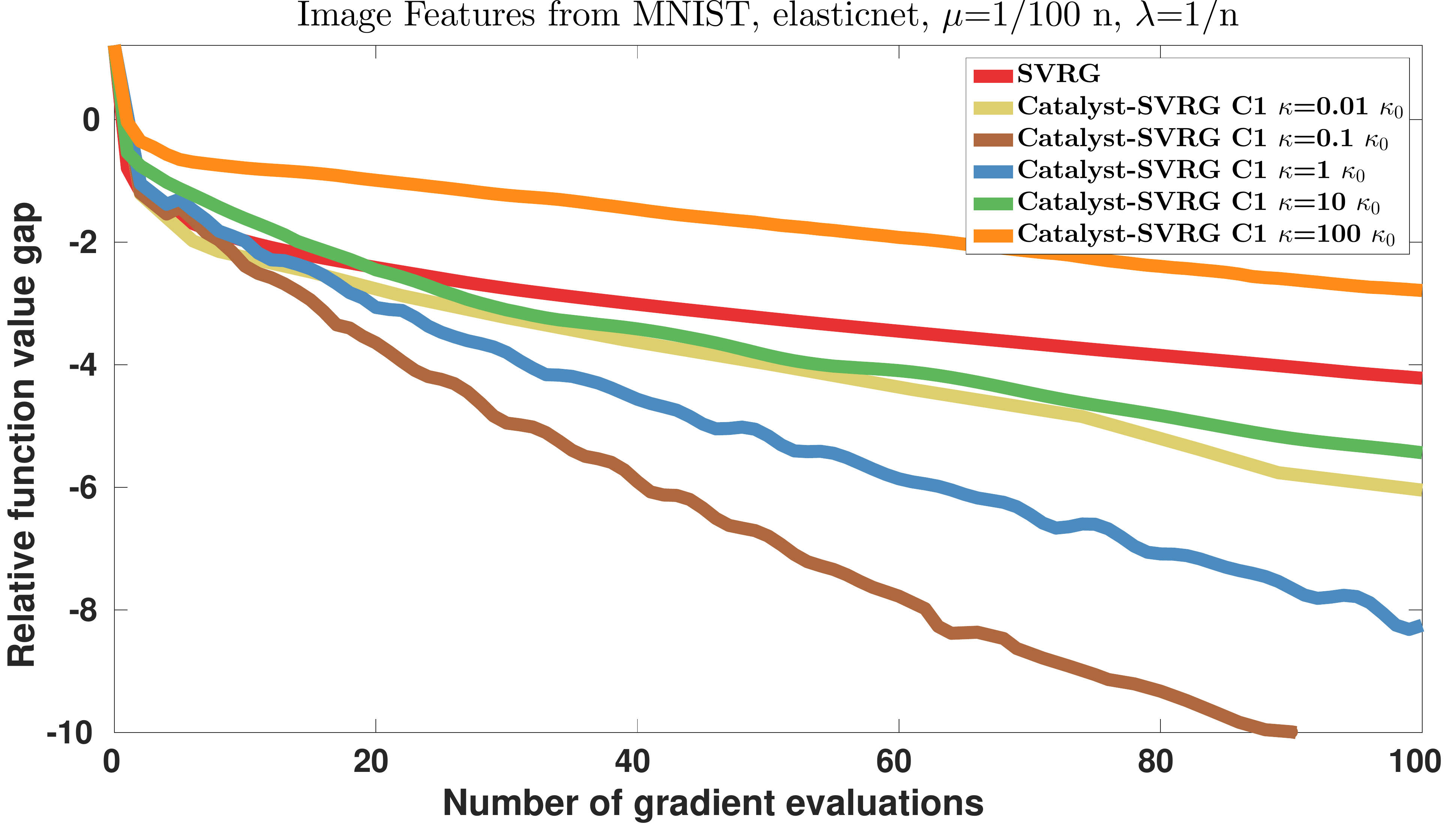}~ 
   ~~\includegraphics[width=0.31\linewidth]{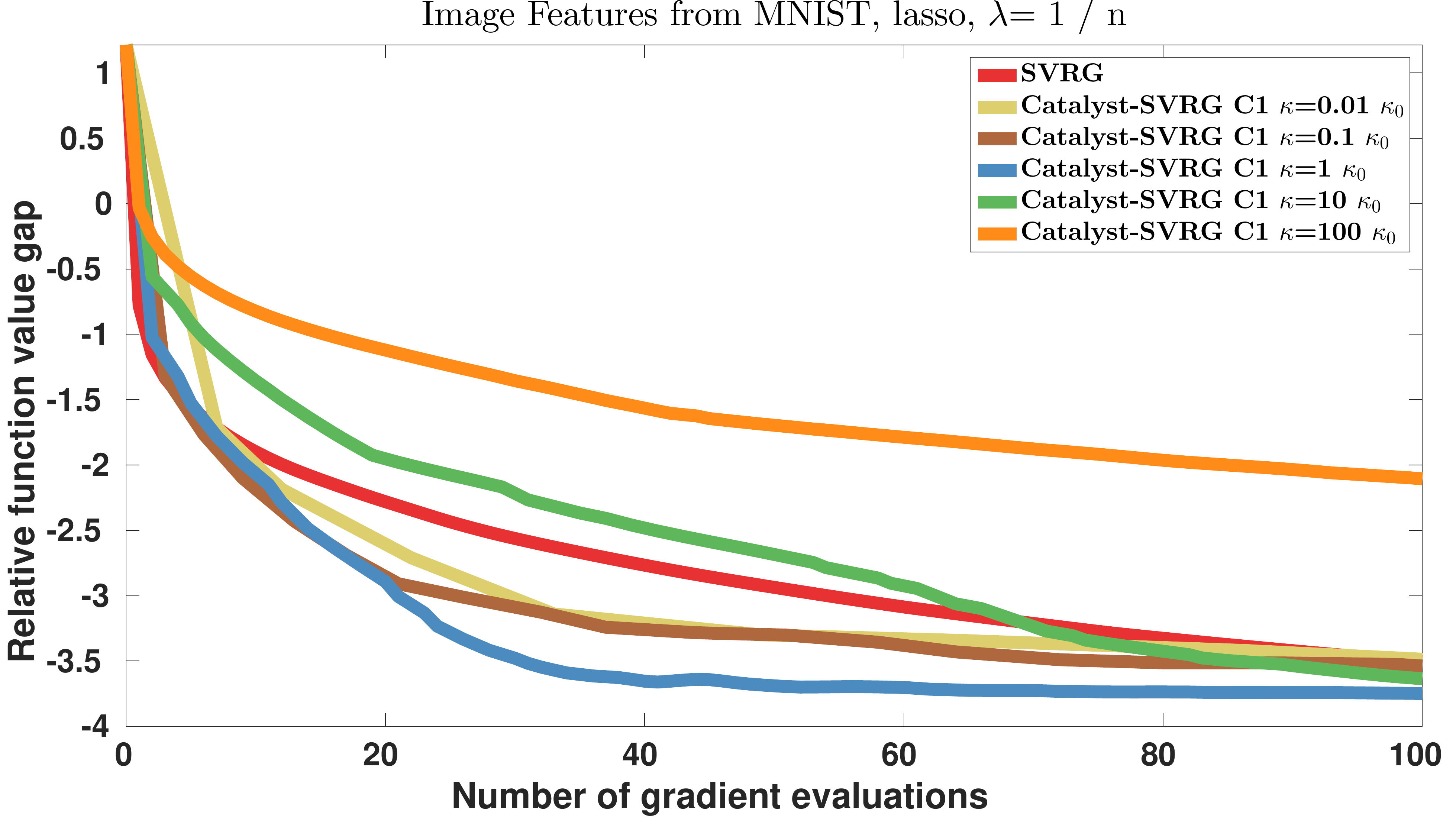}\\
      ~~\includegraphics[width=0.31\linewidth]{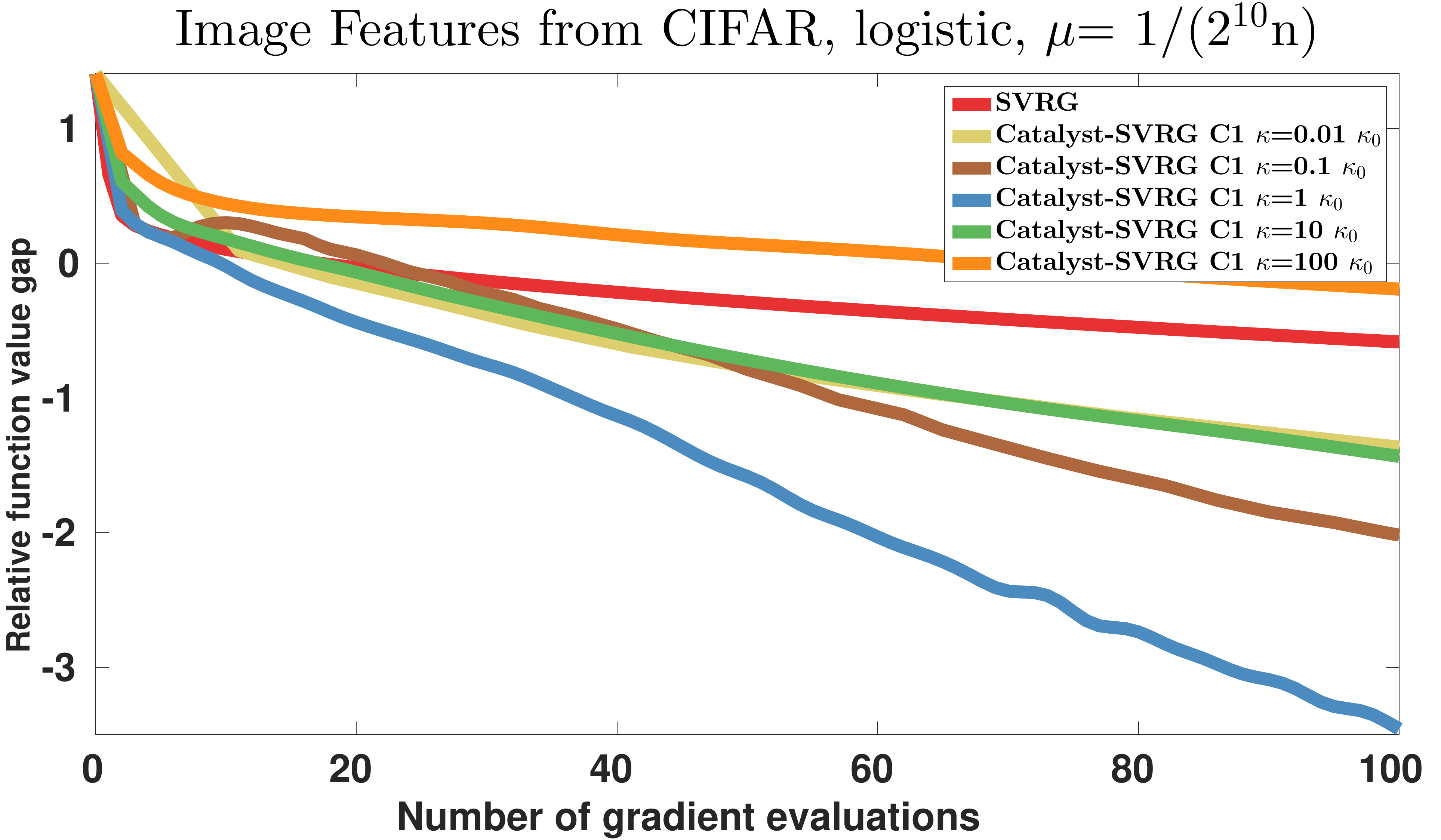}~ 
   ~~\includegraphics[width=0.31\linewidth]{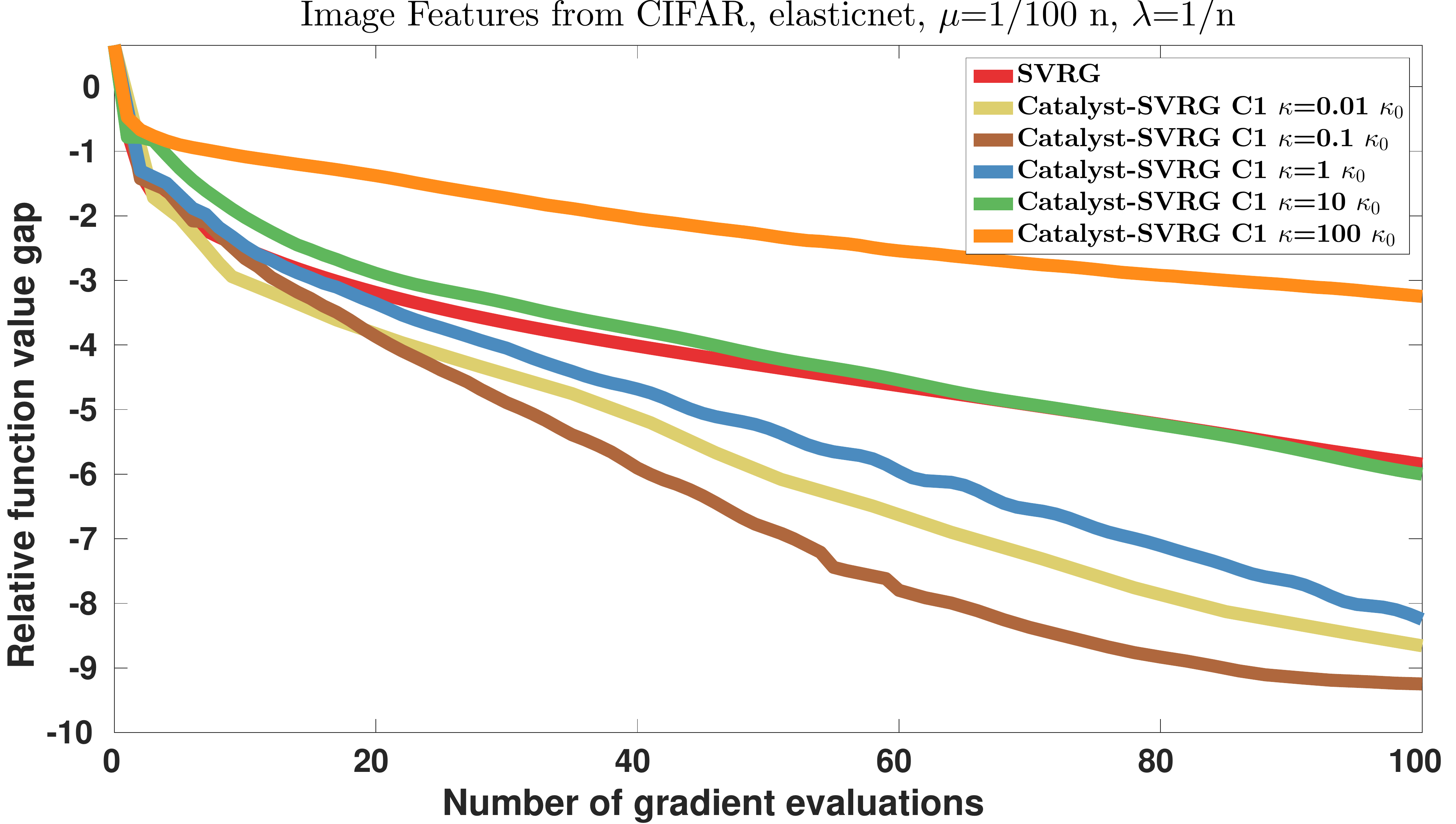}~ 
   ~~\includegraphics[width=0.31\linewidth]{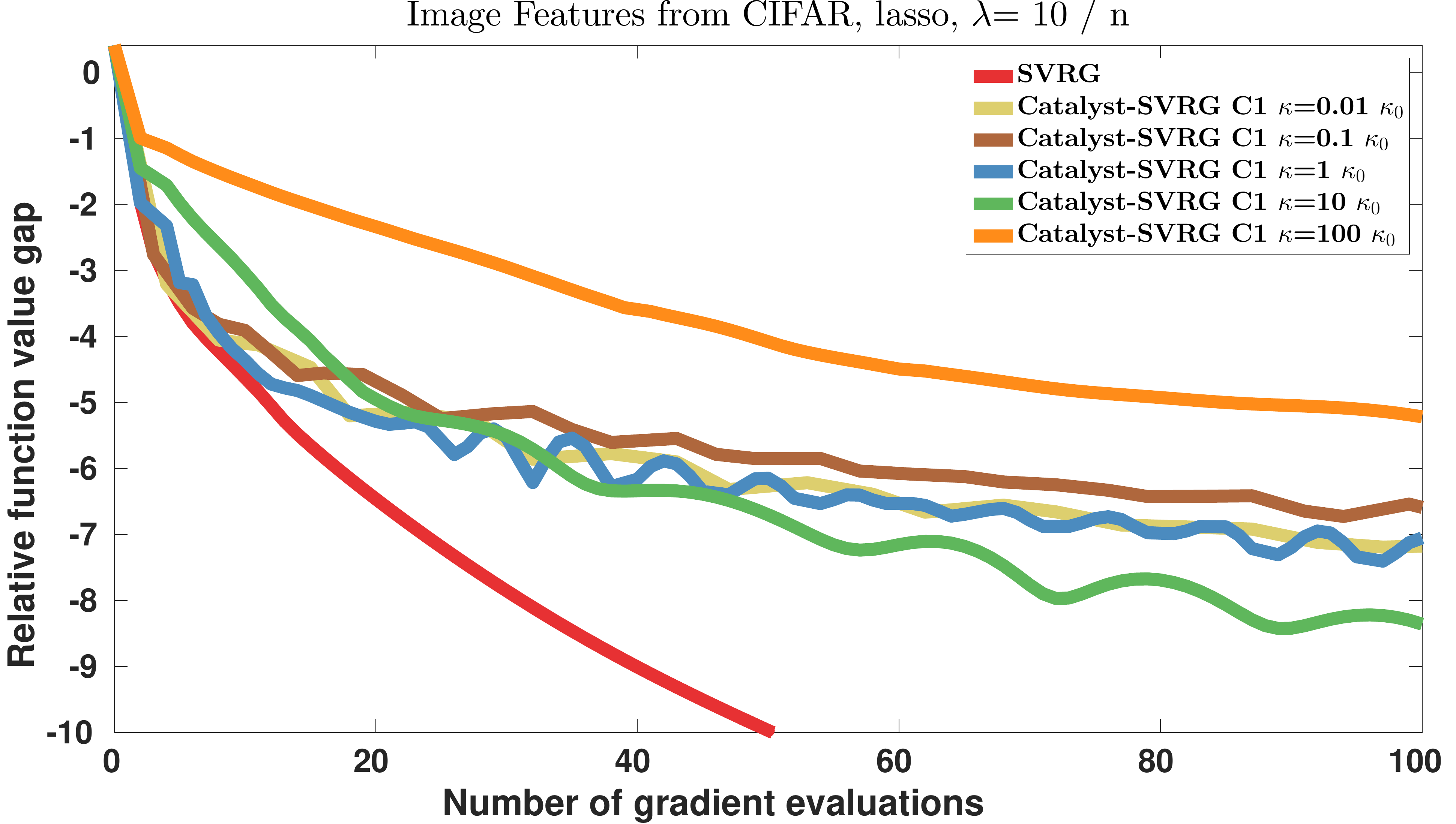}\\
   \caption{Evaluations of Catalyst-SVRG for different $\kappa$ using stopping criterion~\ref{C1}, where $\kappa_0$ is the theoretical choice given by the complexity analysis.} \label{catalyst:fig:compare_kappa_svrg}
\end{figure}

\paragraph{Observations for different choices of $\kappa$.} We consider a logarithmic grid $\kappa = 10^i \kappa_0$ with $i=-2,-1,\cdots,2$ and $\kappa_0$ is the optimal $\kappa$ given by the theory. We observe that for ill-conditioned problems, using optimal choice $\kappa_0$ provides much better performance than other choices, which confirms the theoretical result. For the data set of alpha or Lasso problems, we observe that the best choice is given by the smallest $\kappa=0.01\kappa_0$. This suggests that, as discussed before, there is a certain degree of strong convexity present in the objective even without any regularization.

\section{Conclusion}\label{sec:ccl}
We have introduced a generic algorithm called Catalyst that 
allows us to extend Nesterov's acceleration to a large class of first-order methods. 
We have shown that it can be effective in practice for ill-conditioned problems.
Besides acceleration, Catalyst also improves the numerical stability of a given
algorithm, by applying it to auxiliary problems that are better conditioned than
the original objective. For this reason, it also provides support to convex,
but not strongly convex objectives, to algorithms that originally require
strong convexity.
We have also studied experimentally many variants to identify the ones that are
the most effective and the simplest to use in practice. For incremental
methods, we showed that the ``almost-parameter-free'' variant,
consisting in performing a single pass over the data at every outer-loop
iteration, was the most effective one in practice.

Even though we have illustrated Catalyst in the context of finite-sum
optimization problems, the main feature of our approach is its versatility.
Catalyst could also be applied to other algorithms that have not been considered so far and give rise to new accelerated algorithms.

\acks{The authors would like to thank Dmitriy Drusvyatskiy, Anatoli Juditsky, Sham Kakade, Arkadi Nemirovski, Courtney Paquette, and Vincent Roulet for fruitful discussions. 
Hongzhou Lin and Julien Mairal were supported by the ERC grant SOLARIS (\#
714381) and a grant from ANR (MACARON project ANR-14-CE23-0003-01). 
Zaid Harchaoui was supported by NSF Award CCF-1740551 and
the program ``Learning in Machines and Brains'' of CIFAR. 
This work was performed while Hongzhou Lin was at Inria and Univ. Grenoble Alpes.}

\appendix
%%%%% Appendix \input{appendix.tex}
\section{Useful Lemmas}\label{appendix:proofs_aux}
\begin{lemma}[\bfseries Simple lemma on quadratic functions]\label{lemma:quadratic}
    For all vectors $x, y, z$ in~$\R^p$ and $\theta > 0$,
    \begin{displaymath}
       \|x-y\|^2 \geq (1-\theta) \|x-z\|^2 + \left(1-\frac{1}{\theta}\right) \| z-y\|^2.
    \end{displaymath}
\end{lemma}
\begin{proof}
   \begin{displaymath}
      \begin{split}
         \|x-y\|^2 & = \|x-z+z-y\|^2 \\
                   & = \|x-z\|^2 +\| z-y\|^2 + 2 \langle x-z , z-y \rangle \\
                   & = \|x-z\|^2 +\| z-y\|^2 + \left\|\sqrt{\theta}(x-z) + \frac{1}{\sqrt{\theta}}(z-y) \right\|^2 - \theta \| x-z\|^2 - \frac{1}{\theta}\|z-y\|^2 \\
                   & \geq (1-\theta) \|x-z\|^2 + \left(1-\frac{1}{\theta}\right) \| z-y\|^2. \\
      \end{split}
   \end{displaymath}
\end{proof}

\begin{lemma}[\bfseries Simple lemma on non-negative sequences]\label{lemma:sequences}
   Consider a increasing sequence  $(S_k)_{k \geq 0}$ and two non-negative sequences $(a_k)_{k \geq 0}$ and $(u_k)_{k \geq 0}$ such that for all $k$,
   \begin{equation}
      u_k^2 \leq S_k + \sum_{i=1}^k a_i u_i. \label{eq:relation}
   \end{equation}
   Then, 
   \begin{equation}
      S_k + \sum_{i=1}^k a_i u_i \leq \left(\sqrt{S_k} + \sum_{i=1}^k a_i \right)^2. \label{eq:induction}
   \end{equation}
\end{lemma}
\begin{proof}
   This lemma is identical to the Lemma~A.10 in the original Catalyst
   paper~\citep{catalyst}, inspired by a lemma of~\citet{proxinexact} for
   controlling errors of inexact proximal gradient methods.

   We give here an elementary proof for completeness based on induction. The
   relation~(\ref{eq:induction}) is obviously true for $k=0$. Then, we assume
   it is true for $k-1$ and prove the relation for~$k$.
   We remark that from~(\ref{eq:relation}),
   \begin{equation*}
      \left(u_k - \frac{a_k}{2}\right)^2  \leq S_k + \sum_{i=1}^{k-1} a_i u_i + \frac{a_k^2}{4},
   \end{equation*}
   and then
   \begin{equation*}
      u_k  \leq \sqrt{S_k + \sum_{i=1}^{k-1} a_i u_i + \frac{a_k^2}{4} }   + \frac{a_k}{2}.
   \end{equation*}
   We may now prove the relation~(\ref{eq:induction}) by induction,
   \begin{equation*}
      \begin{split}
         S_k + \sum_{i=1}^k a_i u_i & \leq  S_k + \sum_{i=1}^{k-1} a_i u_i + a_k\left(\frac{a_k}{2} + \sqrt{S_k + \sum_{i=1}^{k-1} a_i u_i + \frac{a_k^2}{4}}\right) \\
                                  & \leq  S_k + \sum_{i=1}^{k-1} a_i u_i + a_k\left({a_k} + \sqrt{S_k + \sum_{i=1}^{k-1} a_i u_i}\right) \\
                                  & \leq \left( \sqrt{S_k + \sum_{i=1}^{k-1} a_i u_i}  + {a_k} \right)^2 \\
                                  & = \left( \sqrt{(S_k-S_{k-1})+(S_{k-1}+ \sum_{i=1}^{k-1} a_i u_i)}  + {a_k} \right)^2 \\
                                  & \leq \left( \sqrt{(S_k-S_{k-1})+\left( \sqrt{S_{k-1}} + \sum_{i=1}^{k-1} a_i \right)^2}  + {a_k} \right)^2 ~~~\text{(by induction)}\\
                                  & \leq \left( \sqrt{S_k} + \sum_{i=1}^{k} a_i \right)^2 .
      \end{split}
   \end{equation*}
	The last inequality is obtained by developing the square $\left (\sqrt{S_{k-1} + \sum_{i=1}^{k-1} a_i } \right )^2$ and use the increasing assumption $S_{k-1} \leq S_k$.
\end{proof}

\begin{lemma}[\bfseries Growth of the sequence~$(A_k)_{k \geq 0}$]~\label{lemma:Ak}\newline
   Let $(A_k)_{k \geq 0}$ be the sequence defined in~(\ref{eq:Ak})
   where~$(\alpha_k)_{k \geq 0}$ is produced by (\ref{eq:alpha})
   with~$\alpha_0=1$ and~$\mu=0$.
   Then, we have the following bounds for all~$k \geq 0$,
   \begin{displaymath}
     \frac{2}{(k+2)^2} \leq A_k \leq \frac{4}{(k+2)^2} .
   \end{displaymath}
\end{lemma}
\begin{proof}
The righthand side is directly obtained from Lemma~\ref{lem:nesterov Ak} by noticing that $\gamma_0 = \kappa$ with the choice of $\alpha_0$. 
   Using the recurrence of~$\alpha_k$, we have for all~$k \geq 1$,
   \begin{displaymath}
      \alpha_k^2 = (1-\alpha_k) \alpha_{k-1}^2 = \prod_{i=1}^k (1-\alpha_i)\alpha_0^2 = A_k \leq \frac{4}{(k+2)^2}.
   \end{displaymath}
   Thus, $\alpha_k \leq \frac{2}{k+2}$ for all~$k \geq 1$ (it is also true for $k=0$). We now have all we need to conclude the lemma:
   \begin{displaymath}
      A_k = \prod_{i=1}^{k} (1-\alpha_i) \geq \prod_{i=1}^{k} \left(1- \frac{2}{i+2}\right) = \frac{2}{(k+2)(k+1)} \geq \frac{2}{(k+2)^2}.
   \end{displaymath}
\end{proof}

\section{Proofs of Auxiliary Results}\label{appendix:proofs}

\subsection{Proof of Lemma~\ref{prop:abs}}
\begin{proof}
   Let us introduce the notation $h'(z) \defin \frac{1}{\eta}(z- [z]_\eta)$  for the gradient mapping at $z$.
   The first order conditions of the convex problem defining $[z]_{\eta}$ give
 $$   h'(z)- \nabla h_0(z) \in \partial \psi([z]_\eta) .$$
Then, we may define 
\begin{align*}
 u & \defin \frac{1}{\eta}(z-[z]_\eta)- (\nabla h_0(z)-\nabla h_0([z]_\eta)) ,\\
  & =  h'(z) - \nabla h_0(z) + \nabla h_0([z]_\eta)  \in \partial h([z]_\eta). 
\end{align*}
Then, by strong convexity,
   \begin{displaymath}
      \begin{split}
         h^* & \geq h([z]_\eta) + u^\top(p(x)- [z]_{\eta}) + \frac{\kappa+\mu}{2}\|p(x)- [z]_{\eta}\|^2 \\
       		 & \geq h([z]_\eta) -\frac{1}{2(\kappa+\mu)} \Vert u \Vert^2.
               \end{split}
   \end{displaymath}
   Moreover, 
   \begin{displaymath}
      \begin{split}
         \|u\|^2 & = \left\Vert \frac{1}{\eta}(z-[z]_\eta) \right\Vert^2 -
         \frac{2}{\eta}\langle z-[z]_\eta, \nabla h_0(z)-\nabla h_0([z]_\eta)
         \rangle + \Vert \nabla h_0(z)-\nabla h_0([z]_\eta) \Vert^2 \\
         & \leq  \left\Vert h'(z) \right\Vert^2 -  \Vert \nabla h_0(z)-\nabla h_0([z]_\eta) \Vert^2 
         \\ & \leq  \left\Vert h'(z) \right\Vert^2,
      \end{split}
   \end{displaymath}
   where the first inequality comes from the relation
   \citep[][Theorem
   2.1.5]{nesterov} using the fact $h_0$ is $(1/\eta)$-smooth
   $$   \Vert \nabla h_0(z) -\nabla h_0([z]_\eta) \Vert^2 \leq \frac{1}{\eta} \langle z-[z]_\eta, \nabla h_0(z)-\nabla h_0([z]_\eta) \rangle.$$
   Thus,
   \begin{equation*}
      h([z]_\eta)-h^* \leq \frac{1}{2(\kappa+\mu)} \Vert u \Vert^2 \leq \frac{1}{2(\kappa+\mu)} \Vert h'(z) \Vert^2 .
   \end{equation*}
   As a result, 
   \begin{displaymath}
   	\Vert h'(z) \Vert \leq \sqrt{2\kappa \varepsilon} \quad \Rightarrow \quad h([z]_\eta)-h^* \leq \varepsilon. 
   \end{displaymath}
   \end{proof}

\subsection{Proof of Proposition~\ref{prop:convergence}}
\begin{proof}
   We simply use Theorem~\ref{thm:outer-loop-12} and specialize it to the choice of parameters.
   The initialization $\alpha_0=\sqrt{q}$ leads to a particularly simple form of the
   algorithm, where~$\alpha_k=\sqrt{q}$ for all~$k \geq 0$.
   Therefore, the sequence~$(A_k)_{k \geq 0}$ from Theorem~\ref{thm:outer-loop-12} is also
   simple since we indeed have $A_k  = (1-\sqrt{q})^{k}$. 
   Then, we remark that $\gamma_0=\mu (1-\sqrt{q})$ and thus, by strong convexity of~$f$,
$$ S_0 = (1-\sqrt{q})\left (f(x_0)-f^* + \frac{\mu}{2} \Vert x_0 - x^* \Vert^2 \right ) \leqslant 2 (1-\sqrt{q})(f(x_0) -f^*).$$
Therefore,
\begin{equation*}
   \begin{split}
      \sqrt{S_0}+3 \sum_{j=1}^{k} \sqrt {\frac{\epsilon_j }{A_{j-1}}}  
  & \leqslant  \sqrt{ 2(1-\sqrt{q})(f(x_0)   -f^*)}+ 3 \sum_{j=1}^{k} \sqrt {\frac{\epsilon_j }{A_{j-1}}} \\
 & =  \sqrt{ 2(1-\sqrt{q})(f(x_0)   -f^*)} \left[ 1+  \sum_{j=1}^{k}  {\underbrace{\left (\sqrt {\frac{1- \rho }{1- \sqrt{q}}}\right ) }_{\eta}}^{j}\right] \\
 & =  \sqrt{ 2(1-\sqrt{q})(f(x_0)   -f^*)} \,\, \frac{\eta^{k+1}-1}{\eta-1} \\
 & \leqslant  \sqrt{ 2(1-\sqrt{q})(f(x_0)   -f^*)} \,\, \frac{\eta^{k+1}}{\eta-1}.
   \end{split}
\end{equation*}

Therefore, Theorem~\ref{thm:outer-loop-12} combined with the previous inequality gives us
\begin{equation*}
   \begin{split}
      f(x_{k}) -f^*  & \leqslant  2A_{k-1}(1-\sqrt{q}) (f(x_0)   -f^*) \,\, \left (\frac{\eta^{k+1}}{\eta-1} \right )^2 \\
                     & =  2\left ( \frac{\eta}{\eta-1} \right )^2 (1-\rho)^{k}(f(x_0)  -f^*) \\
                     & =  2\left ( \frac{\sqrt{1-\rho}}{\sqrt{1-\rho} - \sqrt{1- \sqrt{q}}} \right )^2 (1-\rho)^{k}(f(x_0)  -f^*) \\
                     & =  2\left ( \frac{1}{\sqrt{1-\rho} - \sqrt{1- \sqrt{q}}} \right )^2 (1-\rho)^{k+1}(f(x_0)  -f^*).
   \end{split}
\end{equation*}
Since $\sqrt{1-x} + \frac{x}{2} $ is decreasing in $[0,1]$, we have $\sqrt{1-\rho}+\frac{\rho}{2} \geqslant \sqrt{1-\sqrt{q}} + \frac{\sqrt{q}}{2}$. Consequently,
$$ f(x_{k}) -f^* \leqslant \frac{8}{(\sqrt{q} -\rho)^2} (1-\rho)^{k+1} (f(x_0)  -f^*). $$
\end{proof}

\subsection{Proof of Proposition~\ref{prop:convergence cvx}}
\begin{proof}
 The initialization $\alpha_0= 1$ leads to $\gamma_0 =\kappa$ and $S_0 = \frac{\kappa}{2} \Vert x^* -x_0 \Vert^2$. Then,
\begin{equation*}
\begin{split}
         \sqrt{\frac{\gamma_0}{2} \Vert x_0-x^* \Vert^2 }+3 \sum_{j=1}^{k} \sqrt {\frac{\epsilon_j }{A_{j-1}}}  
 \leq  & \sqrt{\frac{\kappa}{2}\|x_0-x^\star\|^2}+ 3 \sum_{j=1}^{k} \sqrt {\frac{ (j+1)^2 \epsilon_j}{2}} ~~~\text{(from Lemma~\ref{lemma:Ak})}\\
        \leqslant &  \sqrt{\frac{\kappa}{2} \Vert x_0 - x^*\Vert^2}+ \sqrt{f(x_0)-f^*}\left (\sum_{j=1}^k \frac{1}{(j+1)^{1+\gamma/2}} \right),
        \end{split}
\end{equation*}
where the last inequality uses Lemma~\ref{lemma:Ak} to upper-bound the ratio $\varepsilon_j/A_j$.
Moreover,
$$  \sum_{j=1}^{k} \frac{1}{(j+1)^{1+\gamma/2}}  \leqslant \sum_{j=2}^{\infty} \frac{1}{j ^{1+\gamma/2}} \leqslant \int_1^{\infty} \frac{1}{x^{1+\gamma/2}} \, \text{d}x = \frac{2}{\gamma}.$$
Then applying Theorem~\ref{thm:outer-loop-12} yields
\begin{align*}
f(x_{k}) -f^* 
		    \leqslant \,\, & A_{k-1} \left ( \sqrt{\frac{\kappa}{2} \Vert x_0 - x^*\Vert^2}+ \frac{2}{\gamma} \sqrt{f(x_0)-f^*} \right )^2 \\
		    \leqslant \,\, & \frac{8}{(k+1)^2} \left ( \frac{\kappa}{2} \Vert x_0 - x^*\Vert^2 +\frac{4}{\gamma^2} (f(x_0) -f^*) \right ).  
\end{align*}
The last inequality uses $(a+b)^2 \leqslant 2(a^2+b^2)$.
\end{proof}

\subsection{Proof of Lemma~\ref{lemma:accuracy}} \label{appendix:accuracy}
\begin{proof}
   We abbreviate $\tau_\mtd$ by $\tau$ and $C = C_{\mtd}(h(z_0)-h^\star)$ to simplify the notation. Set 
$$T_0 =  \frac{1}{\tau}\log \left ( \frac{1}{1- e^{-\tau}} \frac{C}{\varepsilon} \right ). $$ 
For any $t \geq 0$, we have 
$$ \E[h(z_t) - h^*]  \leqslant C (1-\tau)^t  \leqslant C  \, e^{-t \tau}. $$
By Markov's inequality,
\begin{equation*}
   \p [h(z_t) - h^* > \epsilon] = \p [T(\epsilon) > t] \leqslant \frac{\E[h(z_t) - h^*]}{\epsilon} \leqslant \frac{C  \, e^{-t \tau}}{\epsilon}.
\end{equation*}
Together with the fact $\p \leqslant 1$ and $t \geq 0$. We have
$$ \p [T(\epsilon) \geqslant t+1] \leqslant \min \left \{ \frac{C}{\epsilon} e^{- t \tau } ,1 \right \}. $$
Therefore,
\begin{align*}
   \E[T(\epsilon)] & = \sum_{t=1}^{\infty} \p[T(\epsilon) \geqslant t] = \sum_{t=1}^{T_0} \p[T(\epsilon) \geqslant t] + \sum_{t=T_0+1}^{\infty} \p[T(\epsilon) \geqslant t] \\
                          & \leqslant T_0 + \sum_{t=T_0}^{\infty} \frac{C}{\epsilon} e^{-t \tau } = T_0  +  \frac{C}{\epsilon}  e^{- T_0 \tau} \sum_{t=0}^{\infty} e^{-t \tau} \\
                          &  = T_0  +  \frac{C}{\epsilon}  \frac{ e^{- \tau T_0}}{1- e^{-\tau}}  = T_0 +1.
\end{align*}
A simple calculation shows that for any $\tau \in (0,1)$, $\frac{\tau}{2} \leqslant 1-e^{-\tau}$ and then
$$ \E[T(\epsilon)]  \leqslant T_0 +1 =  \frac{1}{\tau}\log \left ( \frac{1}{1- e^{-\tau}} \frac{C}{\epsilon} \right )+1 \leqslant  \frac{1}{\tau}\log \left ( \frac{2 C}{ \tau \epsilon} \right )+1.$$
\end{proof}

\subsection{Proof of coerciveness property of the proximal operator} \label{appendix:nonexpansive}
\begin{lemma}
	Given a $\mu$-strongly convex function $f: \Real^p \rightarrow \Real$ and a positive parameter $\kappa>0$. For any $x$, $y \in \Real^p$, the following inequality holds,
	$$ \frac{\kappa}{\kappa+\mu} \langle y-x, p(y)-p(x) \rangle \geq \Vert p(y)-p(x) \Vert^2, $$
	where $\displaystyle p(x) = \argmin_{z \in \Real^p} \left \{ f(z) + \frac{\kappa}{2} \Vert z -x \Vert^2 \right \}$.
\end{lemma}
\begin{proof}
By the definition of $p(x)$, we have $0 \in \partial f(p(x)) +\kappa(p(x)-x)$, meaning that $\kappa(x-p(x)) \in \partial f(p(x))$. By strong convexity of $f$,
$$ \langle \kappa(y-p(y)) - \kappa(x-p(x)), p(y)-p(x)\rangle \geq  \mu \Vert p(y)-p(x) \Vert^2. $$
Rearranging the terms yields the desired inequality. 
\end{proof}
As a consequence, 
\begin{align*}
& \hspace*{-1cm} \left \Vert \frac{\kappa}{\kappa+\mu} (y_k-y_{k-1}) - (p(y_k)-p(y_{k-1})) \right \Vert^2 \\
~~~~~~~~~~& =  \left \Vert \frac{\kappa}{\kappa+\mu} (y_k-y_{k-1})\right \Vert^2  - 2 \frac{\kappa}{\kappa+\mu} \langle y_k -y_{k-1}, p(y_k)-p(y_{k-1}) \rangle  + \Vert p(y_k)-p(y_{k-1}) \Vert^2 \\
& \leq  \left \Vert \frac{\kappa}{\kappa+\mu} (y_k-y_{k-1})\right \Vert^2 \\
& \leq  \left \Vert y_k-y_{k-1}\right \Vert^2.
\end{align*}

\section{Catalyst for MISO/Finito/SDCA}\label{appendix:miso}
In this section, we present the
application of Catalyst to MISO/Finito~\citep{miso,finito}, which may be
seen as a variant of SDCA~\citep{accsdca}. The reason why these algorithms require a specific treatment is due to the fact
that their linear convergence rates are given in a different
form than~(\ref{eq:assumption}); specifically, Theorem 4.1 of~\citet{catalyst}
tells us that MISO produces a sequence of iterates~$(z_t)_{t \geq 0}$ for
minimizing the auxiliary objective~$h(z) = f(z)+\frac{\kappa}{2}\|z-y\|^2$ such that
\begin{displaymath}
   \E[h(z_t)] - h^\star \leq C_{\mtd}(1-\tau_{\mtd})^{t+1}(h^*-d_0(z_0)),
\end{displaymath}
where $d_0$ is a lower-bound of $h$ defined as the sum of a simple quadratic function and the composite regularization $\psi$.
More precisely, these algorithms produce a sequence~$(d_t)_{t \geq 0}$ of such lower-bounds, and the iterate $z_t$ is obtained
by minimizing $d_t$ in closed form. 
In particular, $z_t$ is obtained from taking a proximal step at a well chosen point~$w_t$, providing the following expression, 
$$ z_t = \prox_{\psi/(\kappa+\mu)}(w_t).$$
Then, linear convergence is achieved for the duality gap 
\begin{displaymath}
  \E[h(z_t) - h^\star] \leq  \E[h(z_t) - d_t(z_t)] \leq C_{\mtd}(1-\tau_{\mtd})^{t}(h^\star -d_0(z_0)).
\end{displaymath}
Indeed, the quantity $h(z_t) - d_t(z_t)$ is a natural upper-bound on $h(z_t) -
h^\star$,  which is simple to compute, and which can be naturally used for
checking the criterions (\ref{C1}) and~(\ref{C2}). Consequently, the expected complexity of solving a given problem is slightly different compared to Lemma~\ref{lemma:accuracy}.
\begin{lemma}[\bfseries Accuracy vs. complexity]\label{lemma:accuracy2}
   Let us consider a strongly convex objective~$h$ and denote $(z_t)_{t\geq 0}$ the sequence of iterates generated by MISO/Finito/SDCA. Consider the complexity $T(\varepsilon) = \inf \{t \geq 0, h(z_t)- d_t(z_t) \leq \varepsilon \}$, where $\varepsilon>0$ is the target accuracy and $h^\star$ is the minimum value of~$h$. Then,
$$  \mathbb{E}[T(\varepsilon)] \leq \frac{1}{\tau_{\mtd}} \log \left ( \frac{2C_{\mtd} (h^* -d_0(z_0))}{\tau_{\mtd}  \varepsilon}  \right ) +1,$$
where $d_0$ is a lower bound of $f$ built by the algorithm.
\end{lemma}

For the convergence analysis, the outer-loop complexity does not change as long as the algorithm finds approximate proximal points satisfying criterions (\ref{C1}) and~(\ref{C2}). It is then sufficient to control the inner loop complexity. As we can see, we now need to bound the dual gap $h^*-d_0(z_0)$ instead of the primal gap $h(z_0)-h^*$, leading to slightly different warm start strategies. Here, we show how to restart MISO/Finito.

\begin{proposition}[\bfseries Warm start for criterion~(\ref{C1})] \label{prop:inner_loop:c1 miso}
   Consider applying Catalyst with the same parameter choices as in Proposition~\ref{prop:inner_loop:c1} to MISO/Finito. At iteration $k+1$ of Algorithm~\ref{alg:catalyst1}, assume that we are given the previous iterate $x_k$ in~$\pepsk(y_{k-1})$, the corresponding dual function $d(x)$ and its prox-center $w_k$ satisfying 
$ x_k = \prox_{\psi/(\kappa+\mu)}(w_k).$ 
   Then, initialize the sequence $(z_t)_{t \geq 0}$ for minimizing $h_{k+1} = f + \frac{\kappa}{2} \Vert \cdot - y_k \Vert^2$ with, 
$$ z_0 = \prox_{\psi/(\kappa+\mu)} \left (w_k + \frac{\kappa}{\kappa+\mu}(y_k-y_{k-1}) \right ),$$
and initialize the dual function as
$$ d_0(x) = d(x) + \frac{\kappa}{2} \Vert x-y_k\Vert^2 - \frac{\kappa}{2} \Vert x - y_{k-1} \Vert^2.$$
Then, 
   \begin{enumerate}
      \item when $f$ is $\mu$-strongly convex, we have $h_{k+1}^\star -d_0(z_0)\leq C \varepsilon_{k+1}$ with the same constant as in (\ref{eq:warmstart}) and (\ref{eq:warmstart3}), where $d_0$ is the dual function corresponding to $z_0$;
      \item when $f$ is convex with bounded level sets, there exists a constant $B > 0$ identical to the one of (\ref{eq:warmstart2}) such that
         \begin{equation*}
            h_{k+1}^\star-d_0(z_0) \leq B.
         \end{equation*}
   \end{enumerate}
\end{proposition}
 \begin{proof}
 	The proof is given in Lemma D.5 of \cite{catalyst}, which gives 
 	$$ h_{k+1}^\star -d_0(z_0) \leq \epsilon_k + \frac{\kappa^2}{2(\kappa+\mu)} \Vert y_k -y_{k-1}\Vert^2 .$$ 
 	This term is smaller than the quantity derived from (\ref{eq:hk}), leading to the same upper bound.
 \end{proof}

\begin{proposition}[{\bf Warm start for criterion (\ref{C2})}]\label{lemma:warm start}
Consider applying Catalyst with the same parameter choices  as in Proposition~\ref{prop:inner_loop:c2} to MISO/Finito. At iteration $k+1$ of Algorithm~\ref{alg:catalyst1}, we assume that we are given the previous iterate $x_k$ in~$\gdeltak(y_{k-1})$ and the corresponding dual function $d(x)$.
Then, initialize the sequence $(z_t)_{t \geq 0}$ for minimizing $h_{k+1} = f + \frac{\kappa}{2} \Vert \cdot -y_k \Vert^2$ by
$$ z_0 = \prox_{\psi/(\kappa+\mu)} \left ( y_k - \frac{1}{\kappa+\mu} \nabla f_0(y_k) \right ),$$ 
where $f=f_0+\psi$ and $f_0$ is the smooth part of $f$,
and set the dual function $d_0$ by
$$ d_0(x) = f_0(y_k)+\langle \nabla f_0(y_k), x- y_k \rangle + \frac{\kappa+\mu}{2} \Vert x-y_k\Vert^2 +\psi(x).$$
Then, 
\begin{equation}\label{eq:appendix dual gap2}
h_{k+1}^* - d_0(z_0) \leq \frac{(L+\kappa)^2}{2(\mu+\kappa)} \Vert p(y_k) - y_k \Vert^2.
\end{equation} 
\end{proposition}
\begin{proof}
Since $p(y_k)$ is the minimum of $h_{k+1}$, the optimality condition provides 
$$ - \nabla f_0(p(y_k)) -  \kappa(p(y_k) - y_k) \in  \partial \psi(p(y_k)).$$ 
Thus, by convexity, 
\begin{align*}
& \psi(p(y_k)) + \langle - \nabla f_0(p(y_k)) -  \kappa(p(y_k) - y_k), z_0 - p(y_k)  \rangle \leq \psi(z_0), \\
& f_0(p(y_k)) + \frac{\kappa}{2} \Vert p(y_k) - y_k \Vert^2 + \langle  \nabla f_0(p(y_k))+  \kappa(p(y_k) - y_k), y_k - p(y_k)  \rangle \leq f_0(y_k).
\end{align*}
Summing up gives 
$$ h_{k+1}^* \leq f_0(y_k) +  \psi(z_0) + \langle  \nabla f_0(p(y_k))+  \kappa(p(y_k) - y_k), z_0 - y_k \rangle.  $$
As a result,
\begin{align*}
h_{k+1}^* - d_0(z_0) & \leq f_0(y_k) +  \psi(z_0) + \langle  \nabla f_0(p(y_k))+  \kappa(p(y_k) - y_k), z_0 - y_k \rangle -d_0(z_0) \\
				& =  \langle  \nabla f_0(p(y_k))+  \kappa(p(y_k) - y_k) - \nabla f_0(y_k), z_0 - y_k \rangle - \frac{\kappa+\mu}{2} \Vert z_0-y_k \Vert^2 \\
				& \leq \frac{1}{2(\kappa+\mu)} \Vert \nabla \underbrace{f_0(p(y_k)) - \nabla f_0(y_k)}_{\Vert \cdot \Vert \leq L \Vert p(y_k) -y_k \Vert}+  \kappa(p(y_k) - y_k)  \Vert^2 \\
				& \leq \frac{(L+\kappa)^2}{2(\mu+\kappa)} \Vert p(y_k) -y_k \Vert^2.
\end{align*}
\end{proof}
The bound obtained from (\ref{eq:appendix dual gap2}) is similar to the one form Proposition~\ref{prop:inner_loop:c2}, and differs only in the constant factor. Thus, the inner loop complexity in Section~\ref{subsection:warm start c2} still holds for MISO/Finito up to a constant factor. As a consequence, the global complexity of MISO/Finito applied to Catalyst is similar to one obtained by SVRG, yielding an acceleration for ill-conditioned problems.

\vskip 0.2in
\bibliography{bib}

\end{document}